\documentclass[11pt,leqno]{article}

\pdfoutput=1
\usepackage[a4paper,top=1.2in,bottom=1in,left=1.2in,right=1.2in]{geometry}

\usepackage{caption}
\usepackage{titlesec}
\renewcommand{\thesubsection}{\thesection.\arabic{subsection}}
\renewcommand{\thesubsubsection}{\thesubsection.\Alph{subsubsection}}
\titleformat{\section}[hang]{\normalfont\bfseries}{\thesection}{.5em}{}
\titlelabel{\subsubsection}{\empty}
\titleformat{\subsection}{\normalfont\bfseries}{\thesubsection.}{.5em}{}[]
\titleformat{\subsubsection}[runin]{\normalfont}{\thesubsubsection.}{0em}{}[]

\renewcommand{\abstract}[1]{{\gdef\thepoabstract{#1}}}

\usepackage{amsthm}
\usepackage{thmtools}
\declaretheoremstyle[spaceabove=1em,spacebelow=1em,headfont=\footnotesize\bfseries,bodyfont=\footnotesize]{problemstyle}

\makeatletter
\renewcommand\maketitle
{
\thispagestyle{empty}
\rmfamily\selectfont
\ifdefined\@title{\noindent\Large\bfseries\centering \@title \par}\else\fi
\vspace{2em}
\ifdefined\@author{\centering\normalfont \@author \par}
\vspace{.5em}
\ifdefined\thepoaffiliation{\noindent\small\thepoaffiliation\par}\fi\vspace{5em}\else\fi
\ifdefined\thepoabstract{\small\noindent{{\bfseries Abstract}.\;\;}\thepoabstract\par\vspace{3.5em}}\else\fi
\ifdefined\thepokeywords{{\noindent\bfseries Keywords\;\;}\thepokeywords\par\vspace{5em}}\else\fi
\ifdefined\theporuntitle{\fancyhead[L]{\footnotesize\theporuntitle}}\fi
}
\makeatother

\usepackage[usenames,dvipsnames]{color}
\definecolor{RefColor}{rgb}{0,0,.85}
\definecolor{UrlColor}{rgb}{.5,.5,.5}%
\RequirePackage[colorlinks,linkcolor=RefColor,citecolor=RefColor,urlcolor=UrlColor,linktoc=page]{hyperref}
\hypersetup{pdfinfo={Subject={ }}}
\RequirePackage[OT1]{fontenc}
\usepackage[numbers,sort]{natbib}

\usepackage[english]{babel}

\usepackage{enumitem}
\setlist[itemize]{leftmargin=1.5em}

\long\def\graybox#1{%
    \newbox\contentbox%
    \newbox\bkgdbox%
    \setbox\contentbox\hbox to \hsize{%
        \vtop{
            \kern\columnsep
            \hbox to \hsize{%
                \kern\columnsep%
                \advance\hsize by -2\columnsep%
                \setlength{\textwidth}{\hsize}%
                \vbox{
                    \parskip=\baselineskip
                    \parindent=0bp
                    #1
                }%
                \kern\columnsep%
            }%
            \kern\columnsep%
        }%
    }%
    \setbox\bkgdbox\vbox{
        \pdfliteral{0.95 0.95 0.95 rg}
        \hrule width  \wd\contentbox %
               height \ht\contentbox %
               depth  \dp\contentbox
        \pdfliteral{0 0 0 rg}
    }%
    \wd\bkgdbox=0bp%
    \vbox{\hbox to \hsize{\box\bkgdbox\box\contentbox}}%
    \vskip\baselineskip%
}

\definecolor{CommentColor}{rgb}{0,0,.50}
\newcounter{amargincounter}

\definecolor{CommentColor2}{rgb}{0.5,0,0}

\usepackage{amsmath,amssymb,amscd,amsfonts,amsthm,mathtools}
\usepackage[noabbrev,capitalize]{cleveref}
\usepackage{dsfont}
\usepackage{thmtools}
\usepackage{nccmath}
\usepackage{scalerel}
\theoremstyle{plain}
\declaretheoremstyle[postheadspace=.4em,headfont=\bfseries,bodyfont=\itshape,spaceabove=8pt,
spacebelow=10pt]{basic}
\theoremstyle{basic}
\declaretheorem[style=basic,name={Theorem}]{theorem}
\declaretheorem[style=basic,sibling=theorem,name={Fact}]{fact}
\declaretheorem[style=basic,sibling=theorem,name={Lemma}]{lemma}
\declaretheorem[style=basic,sibling=theorem,name={Proposition}]{proposition}
\declaretheorem[style=basic,sibling=theorem,name={Corollary}]{corollary}
\theoremstyle{definition}
\newtheorem{definition}[theorem]{Definition}
\newtheorem{example}[theorem]{Example}

\declaretheorem[style=definition,name={Remark},sibling=theorem]{remark}

\declaretheoremstyle[postheadspace=1em,
  mdframed={backgroundcolor=gray!10!white,
    hidealllines=true,
    innertopmargin=4pt,
    innerbottommargin=4pt,
    innerleftmargin=7pt,
    skipabove=8pt,
    skipbelow=10pt,
  nobreak=false}
]{grayboxed}
\declaretheorem[style=plain,qed=$\triangleleft$]{auxtheorem}
\declaretheorem[style=grayboxed,sibling=auxtheorem]{algorithm}

\DeclareMathOperator{\tsum}{{\textstyle\sum}}

\DeclareMathOperator{\msum}{\medmath\sum}

\newcommand{\mint}{\medint\int}
\newcommand{\argdot}{{\,\vcenter{\hbox{\tiny$\bullet$}}\,}}

\newcommand{\mean}{\mathbb{E}}
\newcommand{\equdist}{\stackrel{\text{\rm\tiny d}}{=}}

\newcommand{\braces}[1]{{\lbrace #1 \rbrace}}

\newcommand{\trans}{\sf{T}}
\DeclareRobustCommand{\svdots}{
  \vbox{%
    \baselineskip=0.33333\normalbaselineskip
    \lineskiplimit=0pt
    \hbox{.}\hbox{.}\hbox{.}%
    \kern-0.2\baselineskip
  }%
}

\usepackage{bibentry}

\usepackage{booktabs}
\usepackage{tikz}
\usepackage[low-sup]{subdepth}
\usetikzlibrary{calc}
\usetikzlibrary{decorations.markings,decorations.pathreplacing}
\tikzstyle{mybraces}=[mirrorbrace/.style={
          decoration={brace, mirror},
          decorate},brace/.style={
          decoration={brace},
          decorate}]
\usepackage{pst-node}
\usepackage{tikz-cd}

\usepackage{tocloft}

\cftsetpnumwidth{5cm}
\cftsetrmarg{6cm}

\usepackage{wrapfig}
\usepackage{subcaption}
\usepackage[nottoc]{tocbibind}
\settocbibname{References}

\DeclareMathOperator*{\medcup}{\raisebox{-.15em}{$\mathbin{\scalebox{1.5}{\ensuremath{\cup}}}$}}
\DeclareMathOperator*{\medcap}{\raisebox{-.15em}{$\mathbin{\scalebox{1.5}{\ensuremath{\cap}}}$}}

\newcommand{\step}{\refstepcounter{proofstep}$\theproofstep^\circ$\ }

\newcounter{proofstep}
\AtBeginEnvironment{proof}{\setcounter{proofstep}{0}}

\begin{document}

\newcommand{\kword}[1]{{\bf #1}\index{#1}}
\newcommand{\mkword}[2]{{\bf #1}\index{keywords}{#2}}
\newcommand{\xspace}{\mathbf{X}}
\newcommand{\yspace}{\mathbf{Y}}
\newcommand{\range}{\mathbf{T}}
\newcommand{\group}{\mathbb{G}}
\newcommand{\dTV}{d_{\text{\rm\tiny TV}}}
\renewcommand{\sp}[1]{\left<\mkern2mu\smash{#1}\mkern2mu\right>}
\newcommand{\todo}{[\textcolor{red}{todo}]}
\newcommand{\vol}{\text{\rm vol}}
\newcommand{\F}{\Pi}
\newcommand{\cF}{\F}
\newcommand{\intF}{\F^{\circ}}
\newcommand{\partF}{\tilde{\F}}
\newcommand{\face}{\textsc{Face}}
\newcommand{\n}{\mathbf{N}}
\newcommand{\Id}{\text{\rm id}}
\renewcommand{\Id}{\mathbf{1}}
\newcommand{\Fix}{\textsc{Fix}}
\newcommand{\C}{\mathbf{C}}
\renewcommand{\L}{\mathbf{L}}
\renewcommand{\H}{\mathbf{H}}
\newcommand{\gnorm}[1]{{\left\vert\kern-0.25ex\left\vert\kern-0.25ex\left\vert #1 \right\vert\kern-0.25ex\right\vert\kern-0.25ex\right\vert}}
\newcommand{\Dim}{\text{\rm Dim\,}}

\newcommand{\mysetminusD}{\hbox{\tikz{\draw[line width=0.6pt,line cap=round] (3pt,0) -- (0,6pt);}}}
\newcommand{\mysetminusT}{\mysetminusD}
\newcommand{\mysetminusS}{\hbox{\tikz{\draw[line width=0.45pt,line cap=round] (2pt,0) -- (0,4pt);}}}
\newcommand{\mysetminusSS}{\hbox{\tikz{\draw[line width=0.4pt,line cap=round] (1.5pt,0) -- (0,3pt);}}}
\newcommand{\mysetminus}{\mathbin{\mathchoice{\mysetminusD}{\mysetminusT}{\mysetminusS}{\mysetminusSS}}}
\renewcommand{\setminus}{\mysetminus}
\newcommand{\Stab}{\text{\rm Stab}}
\newcommand{\dpath}{d_{\text{\rm path}}}
\newcommand{\ef}{e}

\title{Representing and Learning \\ Functions Invariant Under Crystallographic Groups}
\author{Ryan P Adams and Peter Orbanz}

\begin{abstract}
  {
    Crystallographic groups describe the symmetries of
    crystals and other repetitive structures encountered in nature and the sciences.
    These groups include the wallpaper and space groups.
    We derive linear and nonlinear representations of functions that are (1) smooth and
    (2) invariant under such a group.
    The linear representation generalizes the Fourier basis to crystallographically invariant
    basis functions. We show that such a basis exists for each crystallographic group,
    that it is orthonormal in the relevant $\L_2$ space, and recover the standard Fourier
    basis as a special case for pure shift groups. The nonlinear representation
    embeds the orbit space of the group into a finite-dimensional Euclidean space.
    We show that such an embedding exists for every crystallographic group, and that it factors functions
    through a generalization of a manifold called an orbifold.
    We describe algorithms that, given a standardized description of the group,
    compute the Fourier basis and an embedding map. As examples,
    we construct crystallographically invariant neural networks, kernel machines,
    and Gaussian processes.
  }
\end{abstract}

\maketitle

\section{Introduction}

Among the many forms of symmetry observed in nature, those that arise from repetitive spatial patterns
are particularly important. These are described by sets of transformations
of Euclidean space called crystallographic groups \citep{Vinberg:Shvarsman:1993,Thurston:1997}.
For example, consider a problem in materials science, where atoms are arranged in a crystal lattice.
The symmetries of the lattice are then characterized by a crystallographic group $\group$.
Symmetry means that, if we apply one of the transformations in~$\group$
to move the lattice---say to rotate or shift it---the transformed lattice is indistinguishable from the
untransformed one.
In such a lattice, the Coulomb potential acting on any single electron due to a collection of fixed nuclei
does not change under any of the transformations in $\group$ \citep{bell1954group,lax2001symmetry,johnston1960group}.
If we think of the potential field as a function
on $\mathbb{R}^3$, this is an example of a $\group$-invariant function, i.e., a function whose values do
not change if its arguments are transformed by elements of the group.
When solving the resulting Schr\"odinger equation for single particle states, members of the group commute with the Hamiltonian, and quantum observables are again $\group$-invariant \citep{bell1954group,killingbeck1970group,johnston1960group,slater1965space,lax2001symmetry,dresselhaus2007group}.
A different example are ornamental tilings on the walls of the Alhambra, which, when regarded
as functions on $\mathbb{R}^2$, are invariant under two-dimensional crystallographic groups
\citep{Speiser:1956}.
The purpose of this work is to construct smooth invariant functions for any given crystallographic
group in any dimension.

For finite groups, invariant functions can be constructed easily by summing over all group elements;
for compact infinite groups, the sum can be replaced by an integral. This and related ideas have
received considerable attention in machine learning
\citep[e.g.,][]{kondor2008group,cohen2016group,Bloem-Reddy:Teh:2020}.
Such summations are not possible for crystallographic groups, which are neither finite nor compact,
but their specific algebraic and geometric properties
allow us to approach the problem in a different manner.
We postpone a detailed literature review to \cref{sec:related:work},
and use the remainder of this section to give a non-technical sketch of our results.

\subsection{A non-technical overview}

This section sketches our results in a purely heuristic way; proper
definitions follow in \cref{sec:definitions}.
\\[.5em]
{\noindent\bf Crystallographic symmetry}.
Crystallographic groups are groups that tile a Euclidean space $\mathbb{R}^n$ with a convex shape.
Suppose we place a convex polytope $\Pi$ in the space $\mathbb{R}^n$, say a square or a rectangle in the plane.
Now make a copy of $\Pi$, and use a transformation~${\phi:\mathbb{R}^2\rightarrow\mathbb{R}^2}$ to
move this copy to another location.
We require that $\phi$ is an isometry, which means it may shift, rotate or flip $\Pi$, but does not
change its shape or size.
Here are some examples, where the original shape $\Pi$ is marked in red:
\vspace{-.5em}
\begin{center}
  \begin{tikzpicture}
    \node at (0,0) {\includegraphics[width=2.5cm]{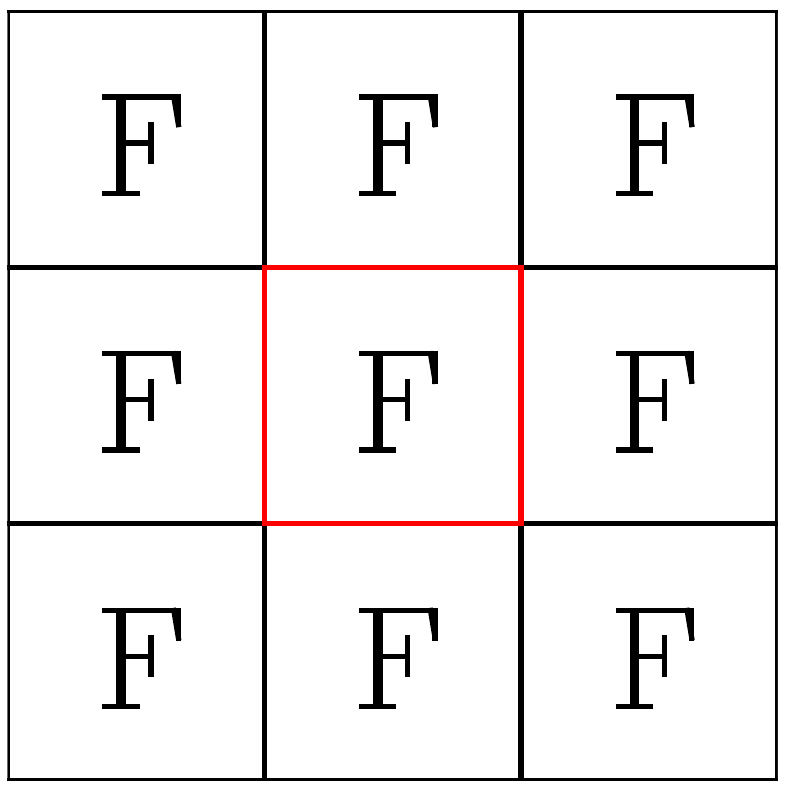}};
    \node at (4.5,0) {\includegraphics[width=2.5cm]{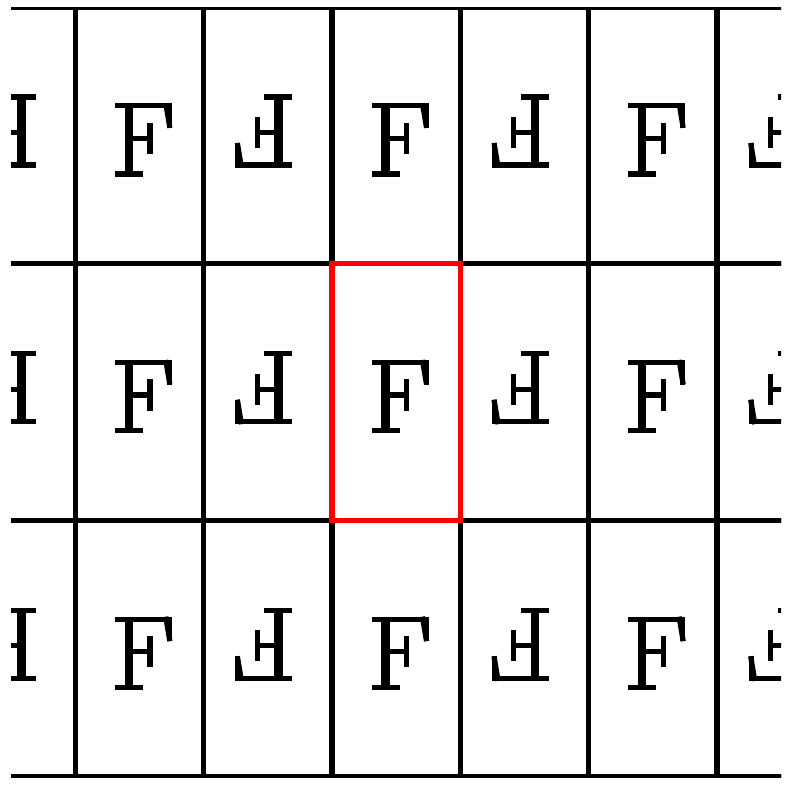}};
    \node at (9,0) {\includegraphics[width=2.5cm]{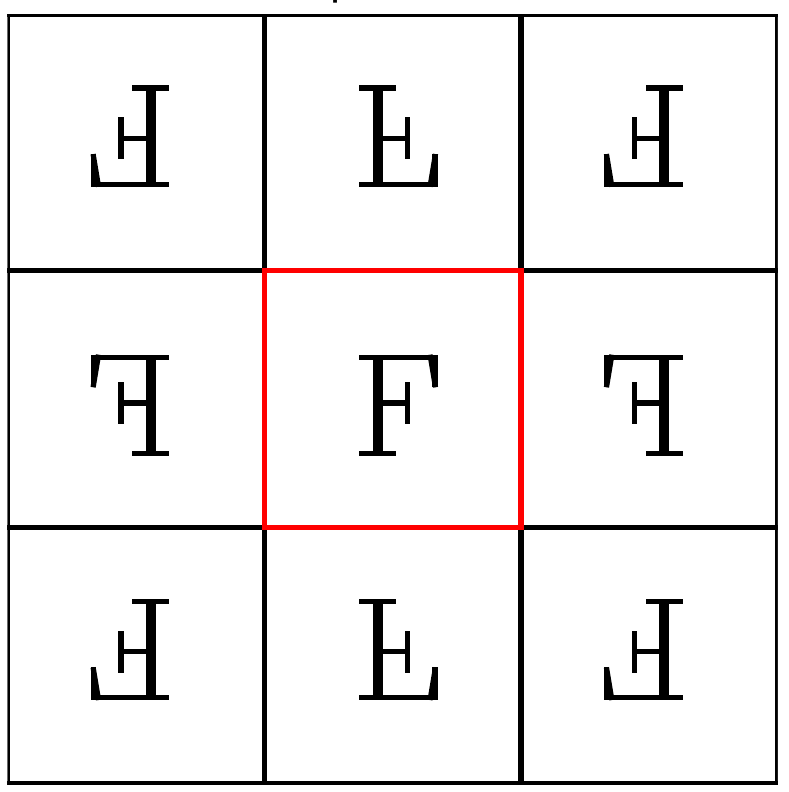}};
    \node at (0,-1.65) {\scriptsize \texttt{p1}};
    \node at (4.5,-1.65) {\scriptsize \texttt{p2}};
    \node at (9,-1.65) {\scriptsize \texttt{p2mm}};
  \end{tikzpicture}
\end{center}
\vspace{-1em}
The descriptors \texttt{p1}, \texttt{p2}, and \texttt{p2mm} follow the naming standard for groups
developed by crystallographers \citep{Hahn:2011}, and the symbol~``F'' is inscribed
to clarify which transformations are used.
The transformations in these examples are horizontal and vertical shifts (in~\texttt{p1}), rotations around the
corners of the rectangle in~(\texttt{p2}), and reflections about its edges~(\texttt{p2mm}).
Suppose we repeat one of these processes indefinitely so that the copies of $\Pi$ cover the entire plane
and overlap only on their boundaries. That requires a countably infinite number of
transformations, one per copy. Collect these into a set $\group$.
If this set forms a group, this group is called crystallographic.
Such groups describe all possible symmetries of crystals, and have been thoroughly studied in
crystallography. For each dimension $n$, there is---up to a natural notion of isomorphy that we explain in \cref{sec:definitions}---only
a finite number of crystallographic groups: Two on $\mathbb{R}$, 17 on $\mathbb{R}^2$,
230 on $\mathbb{R}^3$, and so forth.
Those on $\mathbb{R}^2$ are also known as wallpaper groups, and those on $\mathbb{R}^3$ as space groups.
\\[.5em]
{\noindent\bf The objects of interest}.
A function $f$ is invariant under  $\group$ if it satisfies
\begin{equation*}
  f(\phi x)\;=\;f(x)\qquad\text{ for all }\phi\in\group\text{ and all }x\in\mathbb{R}^n\;.
\end{equation*}
A simple way to construct such a function is to start with a tiling as above, define a function
on $\Pi$, and then replicate it on every copy of $\Pi$. Here are two examples on $\mathbb{R}^2$,
corresponding to (ii) and (iii) above, and an example on $\mathbb{R}^3$:
\vspace{-1em}
\begin{center}
  \begin{tikzpicture}
    \node at (0,0) {\includegraphics[width=3cm]{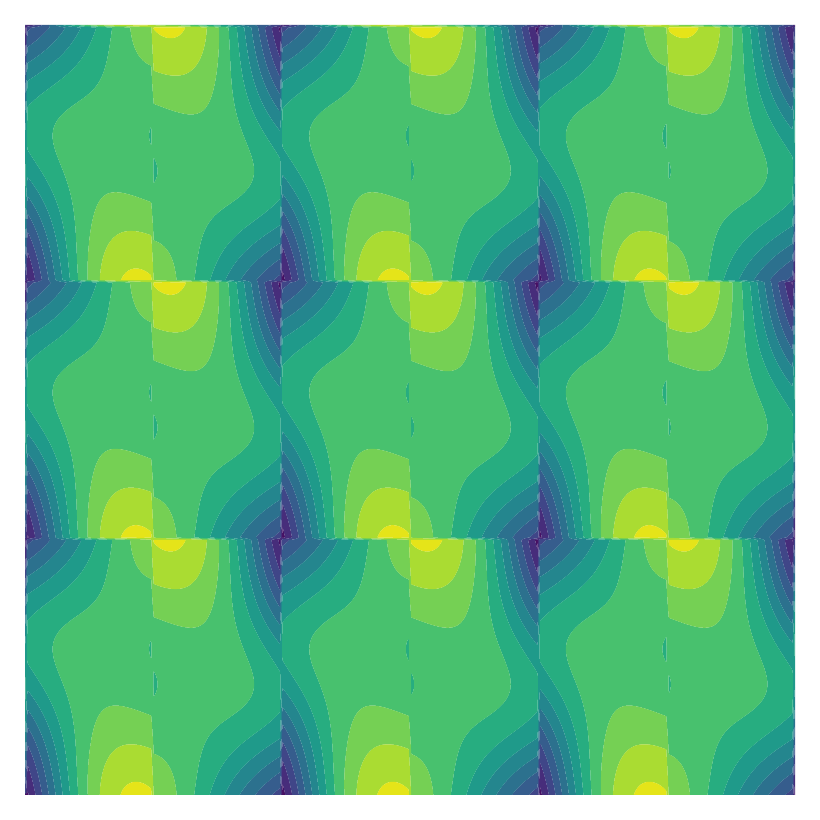}};
    \node at (5,0) {\includegraphics[width=3cm]{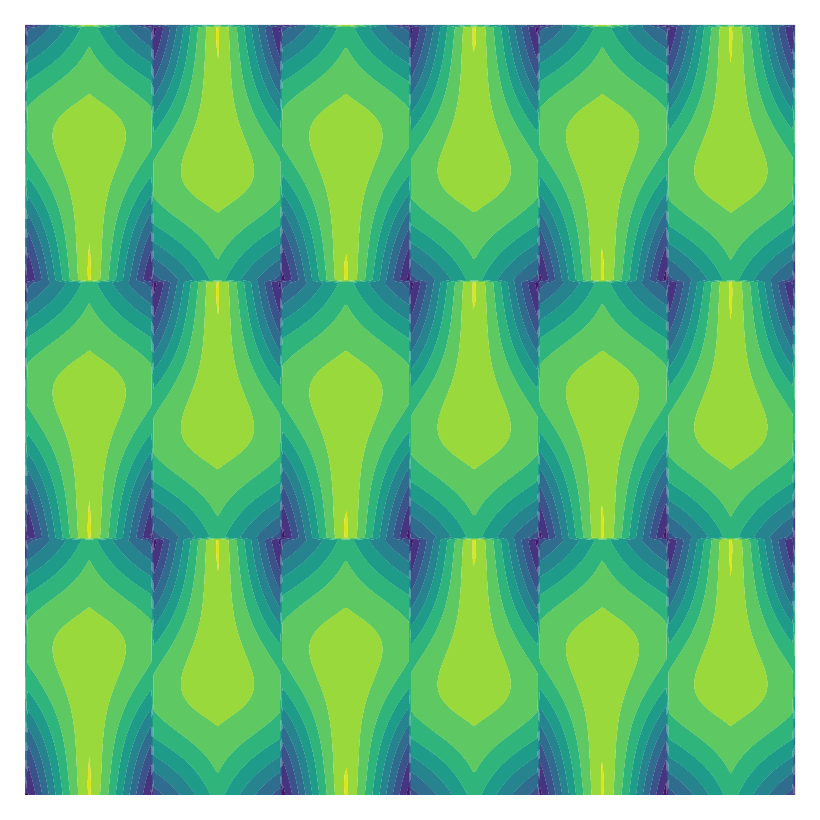}};
    \node at (10,0) {\includegraphics[width=3.5cm]{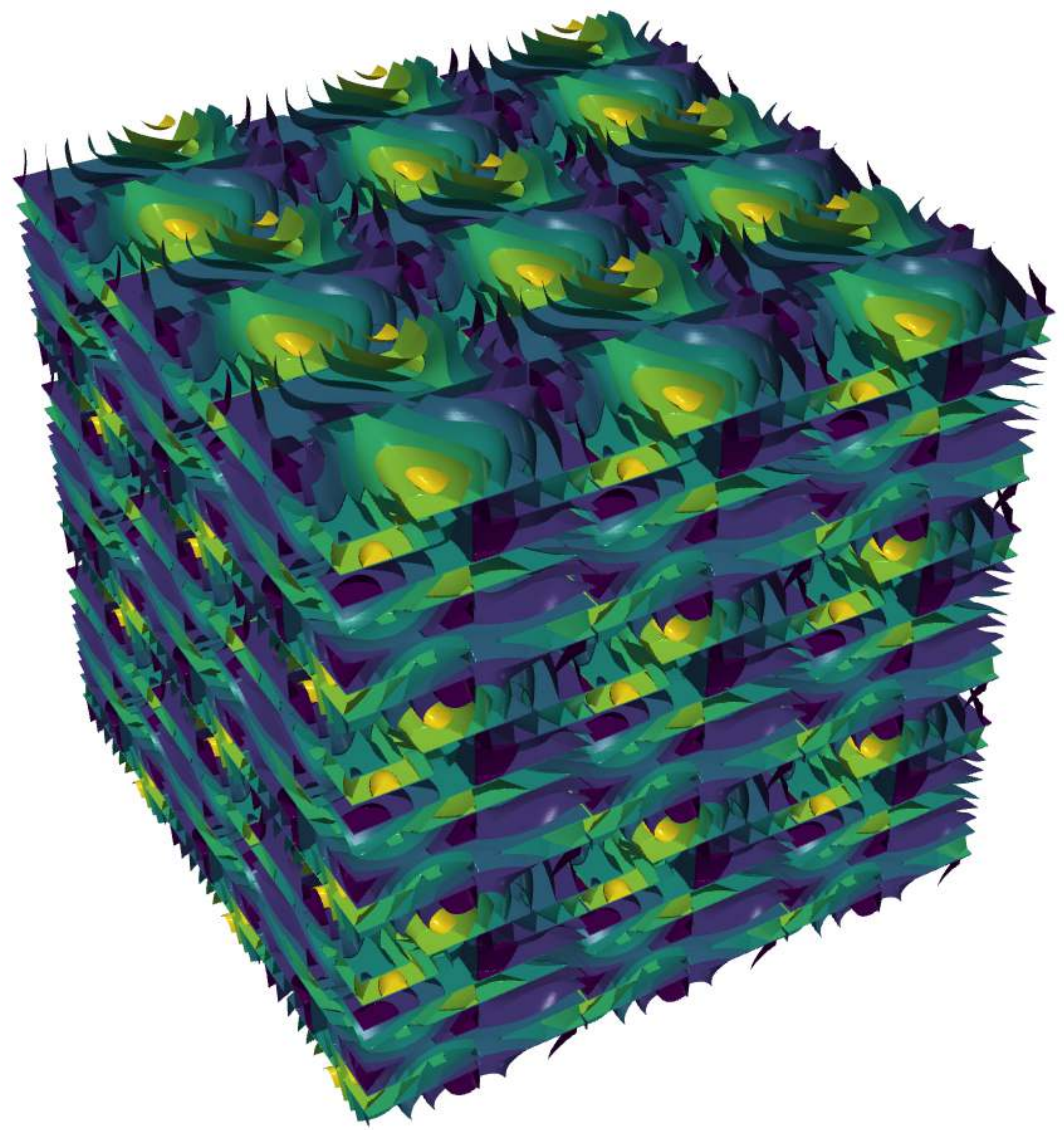}};
    \node at (0,-2) {\scriptsize \texttt{p2}};
    \node at (5,-2) {\scriptsize \texttt{p2mm}};
    \node at (11.5,-2) {\scriptsize \texttt{I4\textsubscript{1}22}};
  \end{tikzpicture}
\end{center}
\vspace{-1em}
However, as the examples illustrate,
functions obtained this way are typically not continuous.
Our goal is to construct smooth invariant functions, such as these:
\vspace{-1em}
\begin{center}
  \begin{tikzpicture}
    \node at (0,0) {\includegraphics[width=3cm]{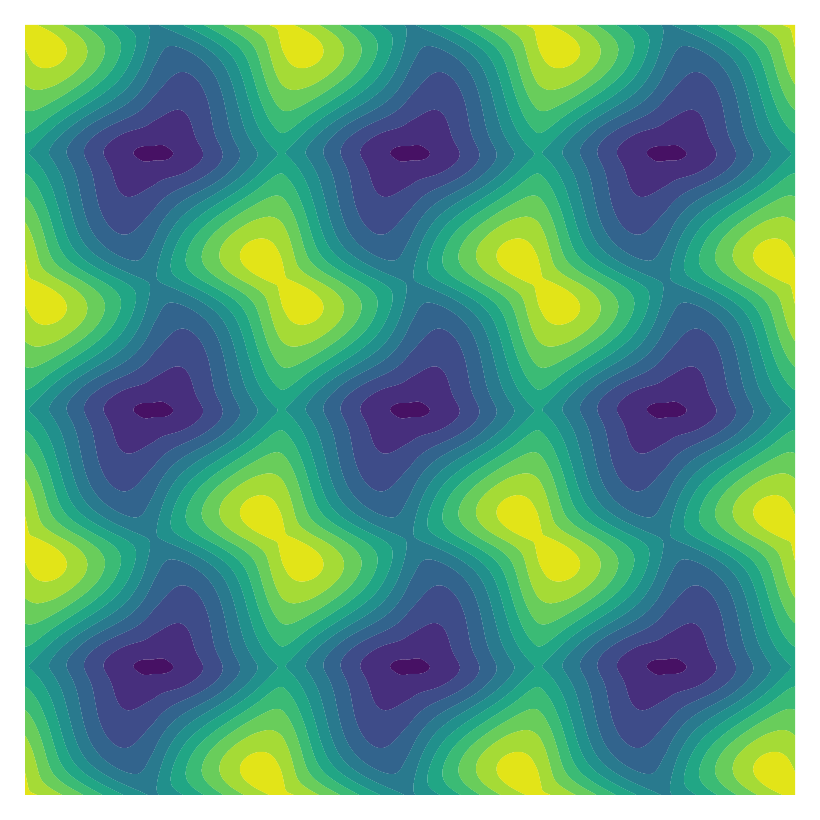}};
    \node at (5,0) {\includegraphics[width=3cm]{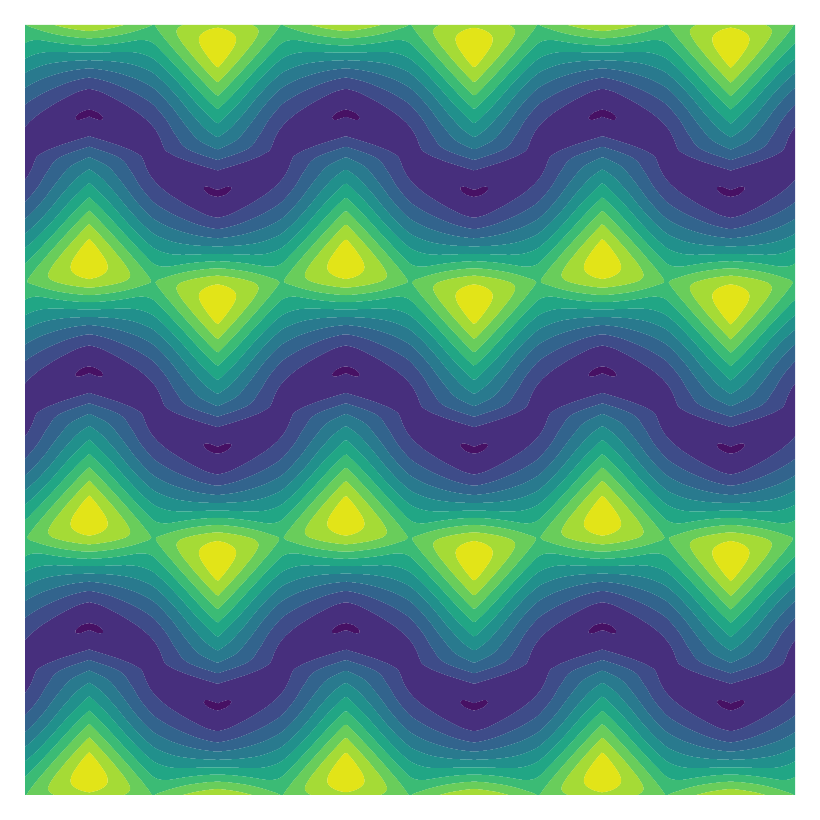}};
    \node at (10,0) {\includegraphics[width=3.5cm]{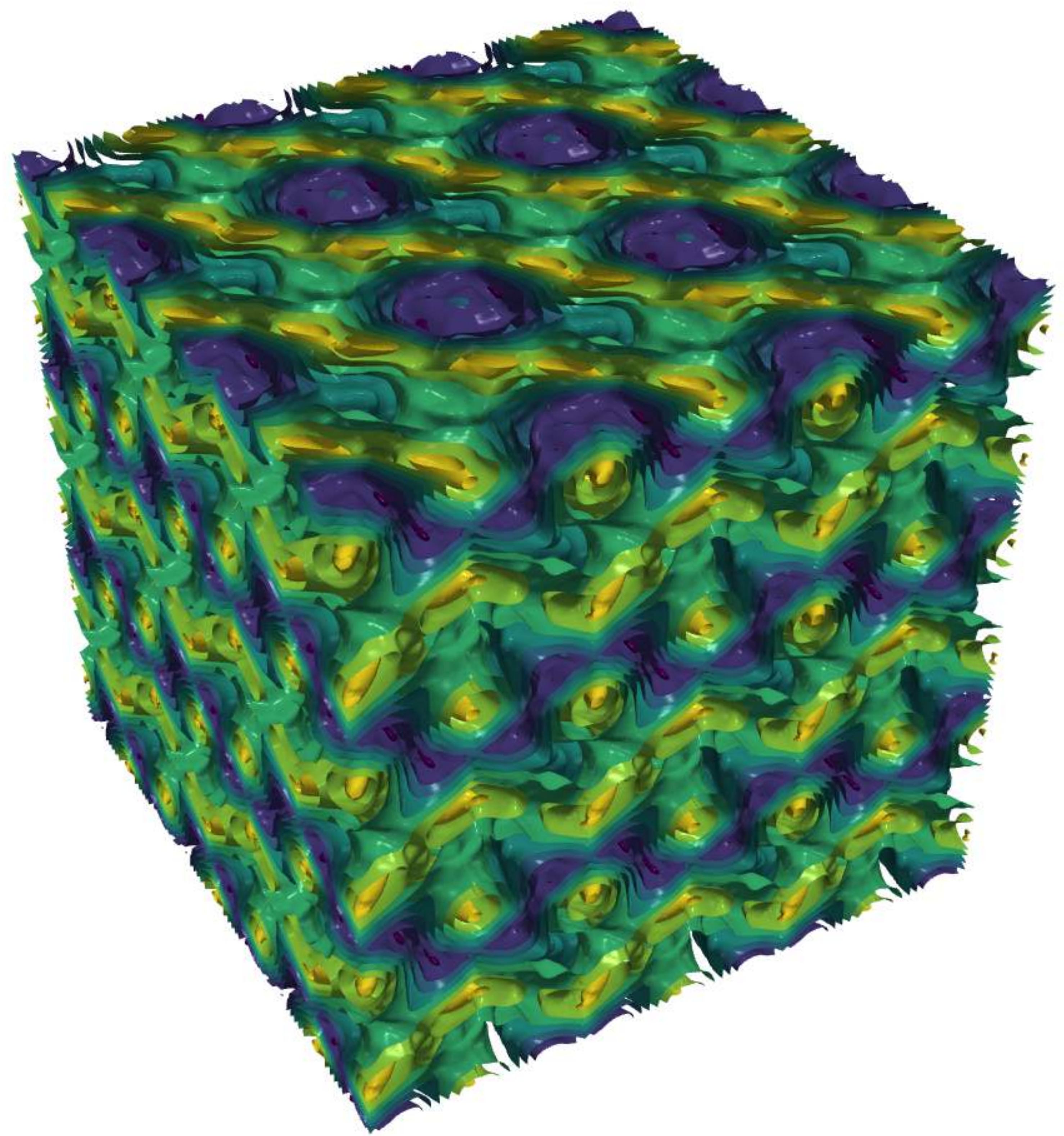}};
    \node at (0,-2) {\scriptsize \texttt{p2}};
    \node at (5,-2) {\scriptsize \texttt{p2mm}};
    \node at (11.5,-2) {\scriptsize \texttt{I4\textsubscript{1}22}};
  \end{tikzpicture}
\end{center}
\vspace{-1em}
We identify two representations of such functions, one linear and one nonlinear.
Working with either representation algorithmically requires a data structure representing the invariance
constraint. We construct such a structure, which we call an orbit graph,
in \cref{sec:orbit:graphs}. This graph is constructed from a description of the group
(which can be encoded as a finite set of matrices) and of $\Pi$ (a finite set of vectors).
\\[.5em]
{\noindent\bf Linear representations: Invariant Fourier transforms}.
We are primarily interested in two and three dimensions, but a one-dimensional example
is a good place to start: In one dimension, a convex polytope is always an interval, say ${\Pi=[0,1]}$. If we choose
$\group$ as the group $\mathbb{Z}$ of all shifts of integer length, it tiles the line $\mathbb{R}$ with $\Pi$.
In this case, an invariant function is simply a periodic function with period $1$. Smooth periodic
functions can be represented as a Fourier series,
\begin{equation*}
  f(x)\;=\;\msum_{i=0}^{\infty}a_i\cos\bigl(\tfrac{ix}{2\pi}\bigr)+b_i\sin\bigl(\tfrac{ix}{2\pi}\bigr)
\end{equation*}
for sequences of scalar coefficients $a_i$ and $b_i$.
Note each sine and cosine on the right is $\group$-invariant and infinitely often differentiable.
Now suppose we abstract from the specific form of these sines and cosines, and only
regard them as~$\group$-invariant functions that are very smooth.
The series representation above then has the general form
\begin{equation*}
  f(x)\;=\;\msum_{i=0}^{\infty}c_i\ef_i(x)\;,
\end{equation*}
where the $\ef_i$ are smooth, $\group$-invariant functions that depend only on $\group$ and $\Pi$,
and the~$c_i$ are scalar coefficients that depend on $f$. (In the Fourier series, $\ef_i$ is
a cosine for odd and a sine for even indices.) In \cref{sec:linear}, we obtain generalizations of
this representation to crystallographically invariant functions. To do so,
we observe that the Fourier basis can be derived as the set of eigenfunctions of the Laplace operator:
The sine and cosine functions above are precisely those functions ${\ef:\mathbb{R}\rightarrow\mathbb{R}}$ that solve
\begin{align*}
  -\Delta \ef&\;=\;\lambda\ef \\
  \text{ subject to }\quad\qquad\ef&\text{ is periodic with period }1
\end{align*}
for some ${\lambda\geq 0}$. (The negative sign is chosen to make the eigenvalues $\lambda$ non-negative.)
The periodicity constraint is equivalent to saying that $\ef$ is invariant
under the shift group~${\group=\mathbb{Z}}$. The corresponding problem for a general crystallographic group~$\group$ on~$\mathbb{R}^n$ is hence
\begin{equation}
  \label{eigenproblem:intro}
  \begin{split}
   -\Delta \ef\;&=\;\lambda\ef \\
  \text{ subject to }\quad\qquad\ef\;&=\;\ef\circ\phi \text{ for all }\phi\in\group\;.
  \end{split}
\end{equation}
\cref{theorem:spectral} shows that this problem has solutions for any dimension $n$, convex polytope~${\Pi\subset\mathbb{R}^n}$, and crystallographic
group $\group$ that tiles $\mathbb{R}^n$ with $\Pi$. As in the Fourier case, the solution functions ${\ef_1,\ef_2,\ldots}$ are very smooth.

If we choose ${\Pi\subset\mathbb{R}^2}$ as the square ${[0,1]^2}$ and~$\group$ as the group $\mathbb{Z}^2$ of discrete horizontal and vertical shifts---that is, the two-dimensional analogue of the example above---we recover the
two-dimensional Fourier transform. The function $\ef_0$ is constant; the functions ${\ef_1,\ldots,\ef_5}$ are shown in
\cref{fig:fourier:p1}.
If the group also contains other transformations, the basis looks less familiar.
These are the basis functions ${\ef_1,\ldots,\ef_5}$ for a group (\texttt{p3}) containing shifts and rotations of order three:
\begin{center}
  \includegraphics[width=.18\textwidth]{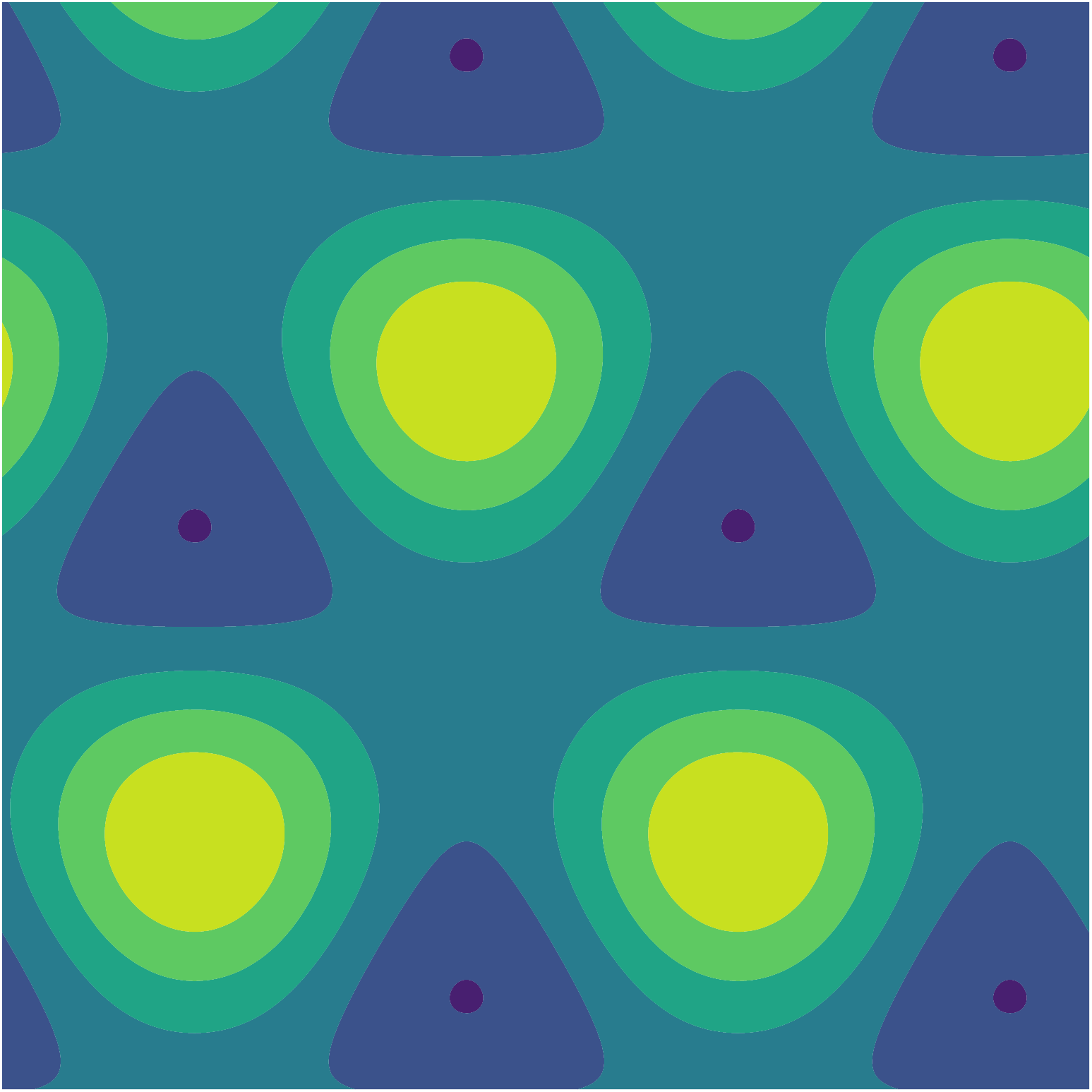}
  \hfill
  \includegraphics[width=.18\textwidth]{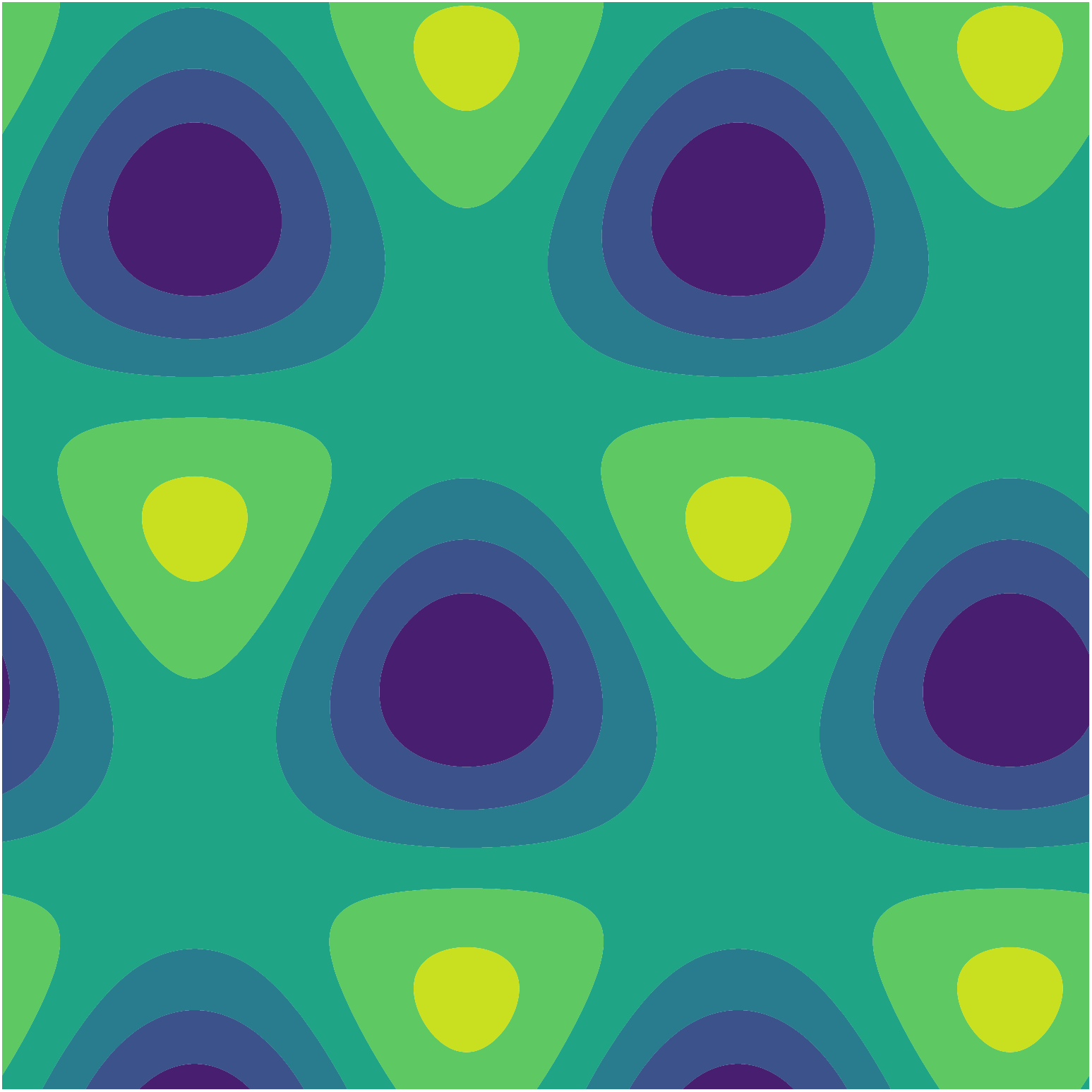}
  \hfill
  \includegraphics[width=.18\textwidth]{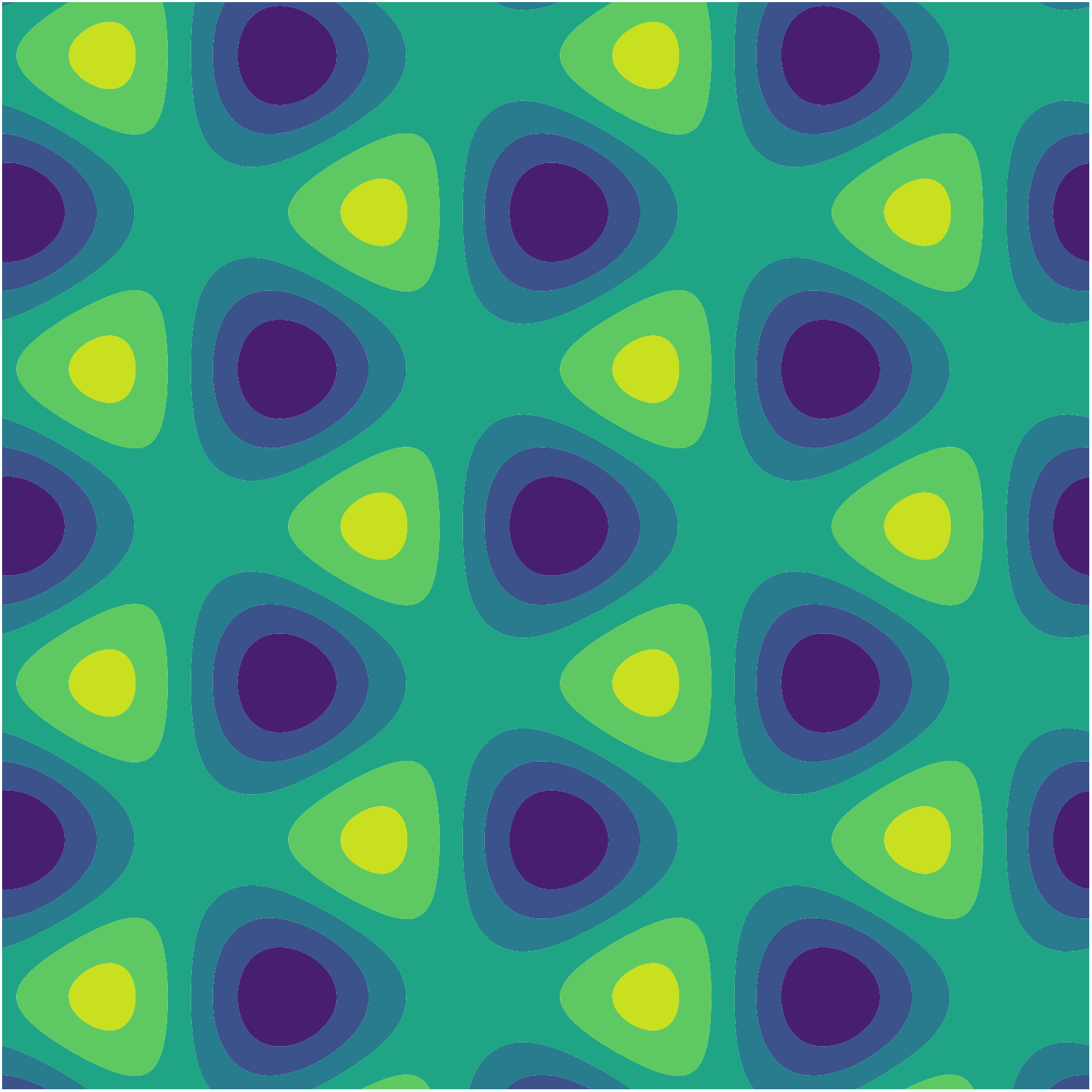}
  \hfill
  \includegraphics[width=.18\textwidth]{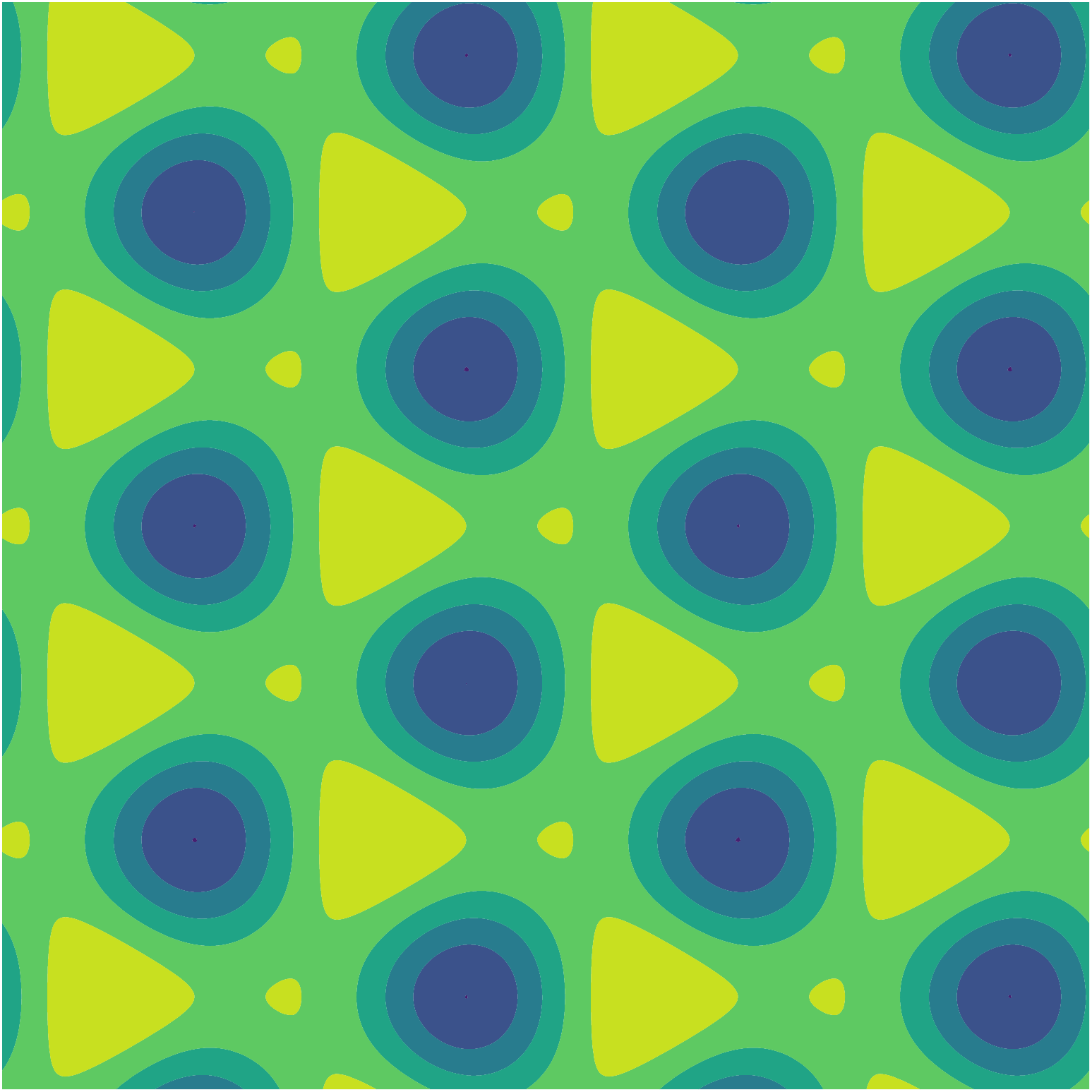}
  \hfill
  \includegraphics[width=.18\textwidth]{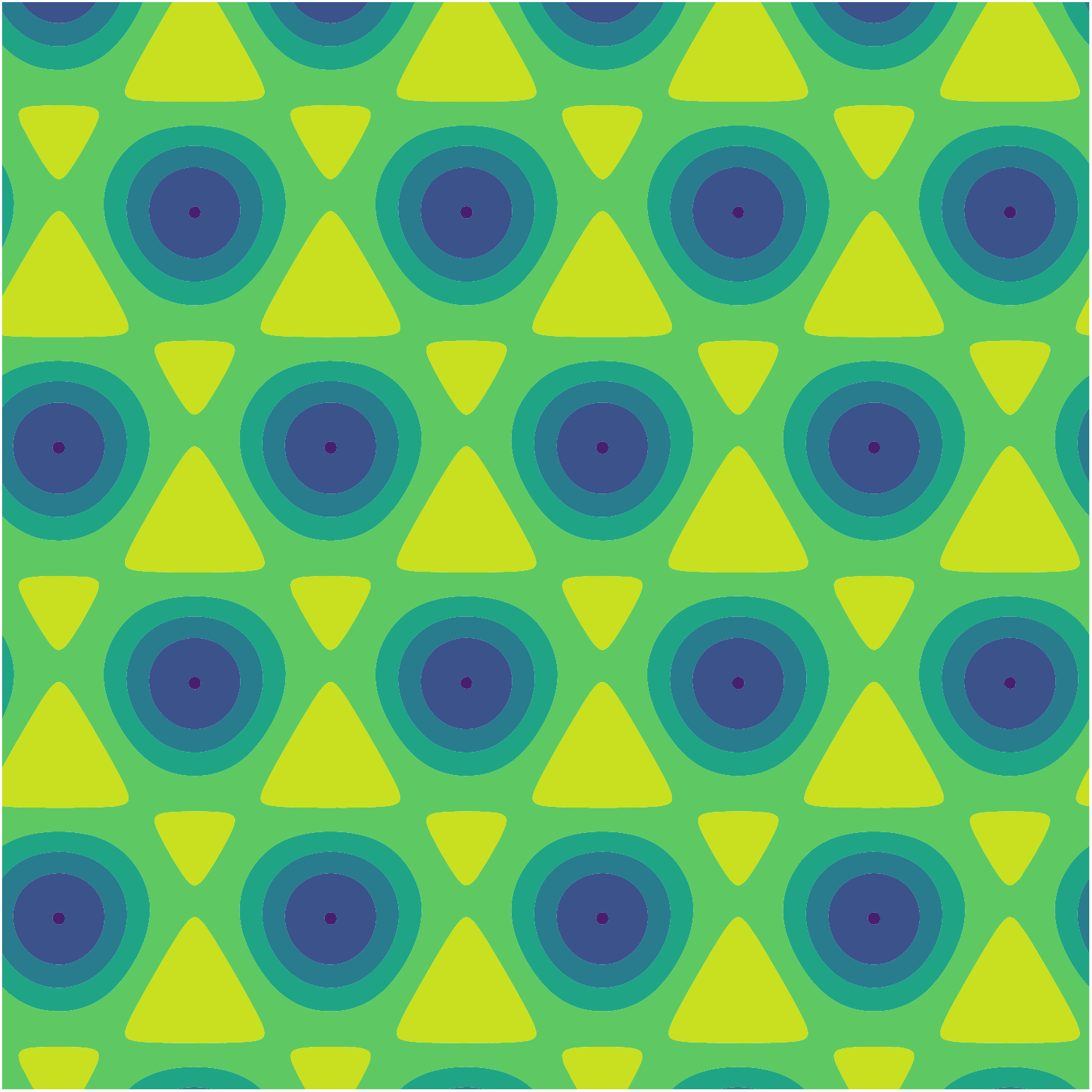}
\end{center}
The same idea applies in any finite dimension $n$. For ${n=3}$, the $\ef_i$ can be visualized
as contour plots. For instance, the first five non-constant basis elements for a specific three-dimensional
group, designated \texttt{I4\textsubscript{1}} by crystallographers, look like this:
\begin{center}
  \includegraphics[width=0.18\textwidth]{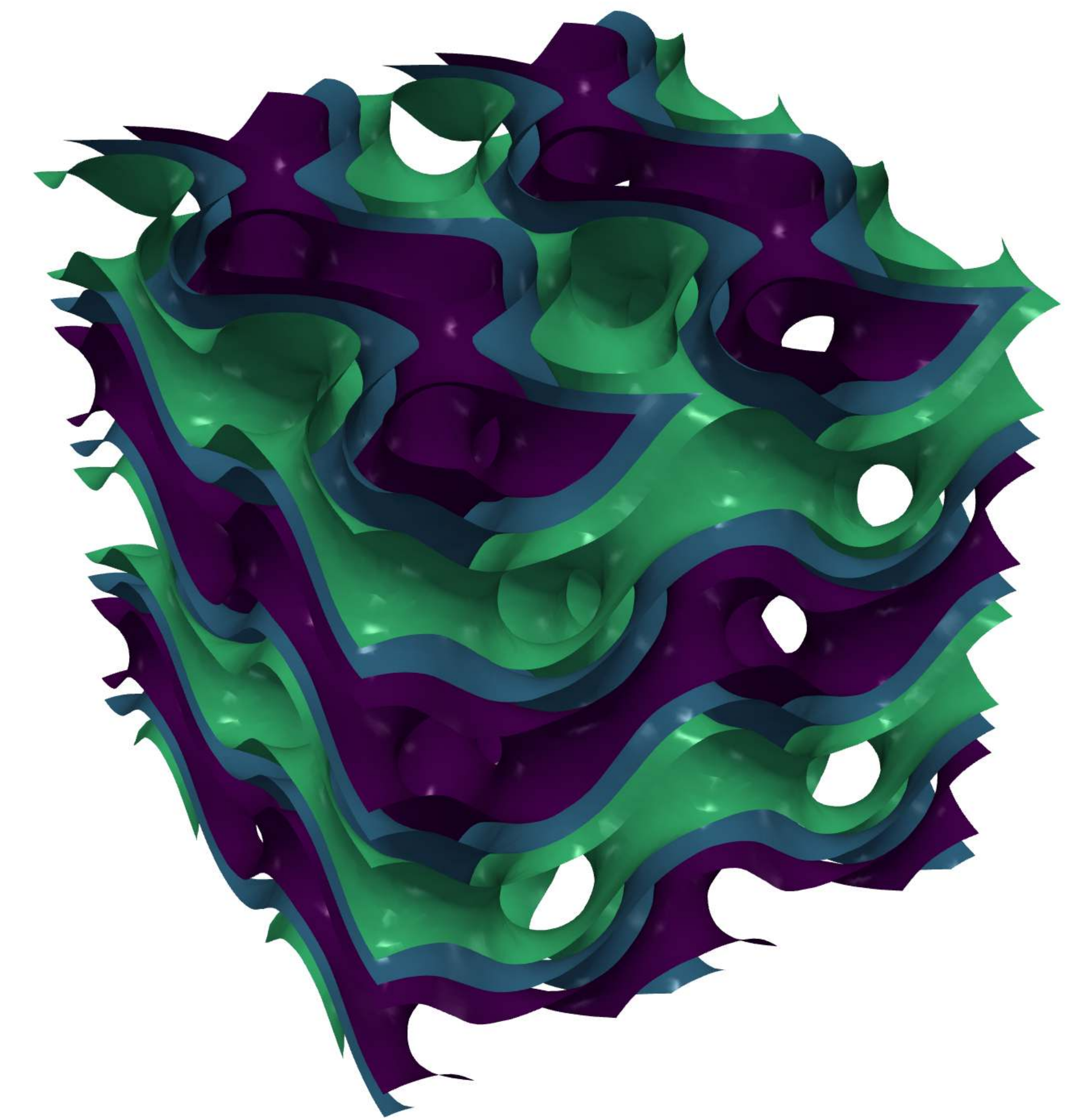}
  \hfill
  \includegraphics[width=0.18\textwidth]{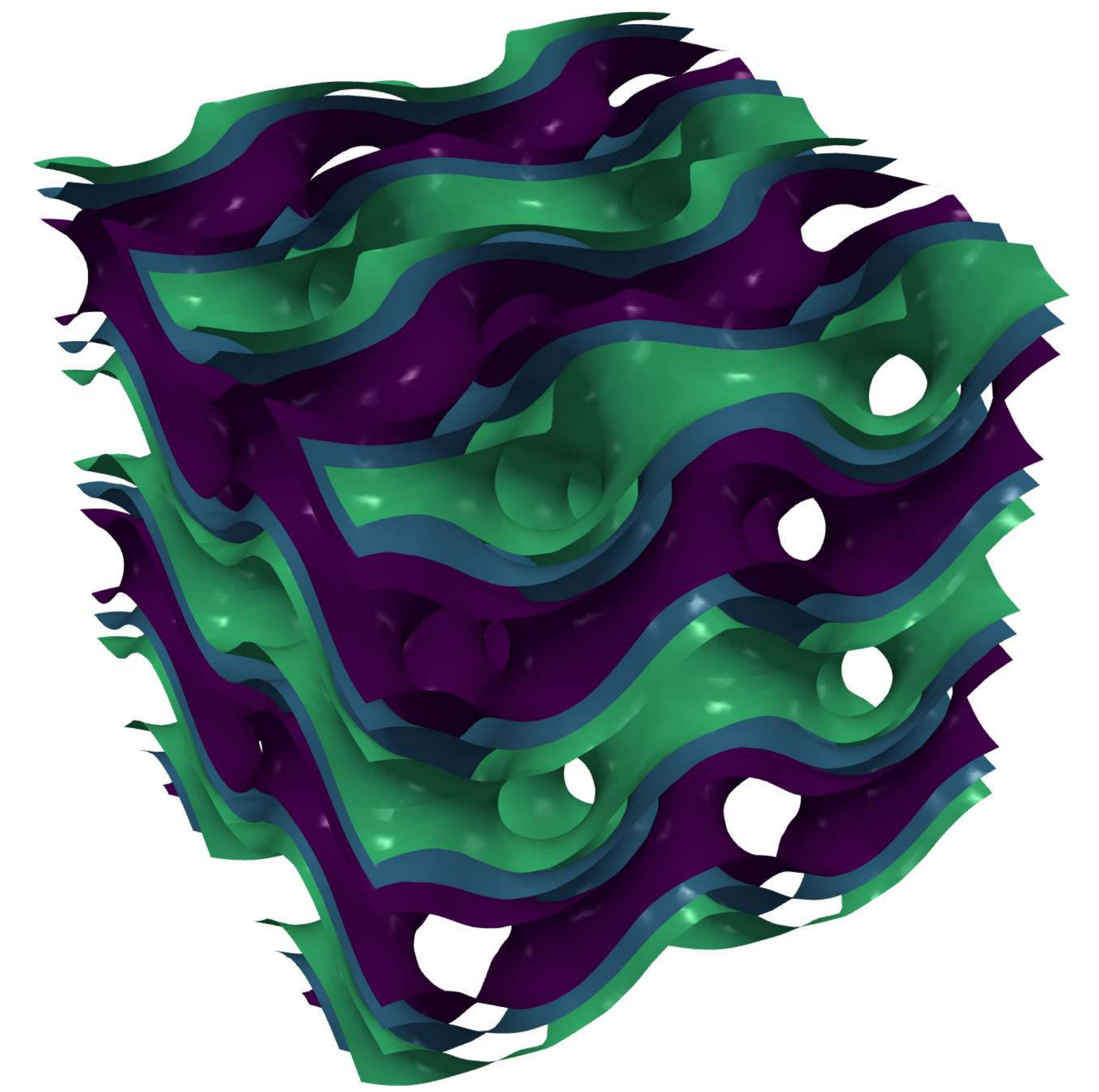}
  \hfill
  \includegraphics[width=0.18\textwidth]{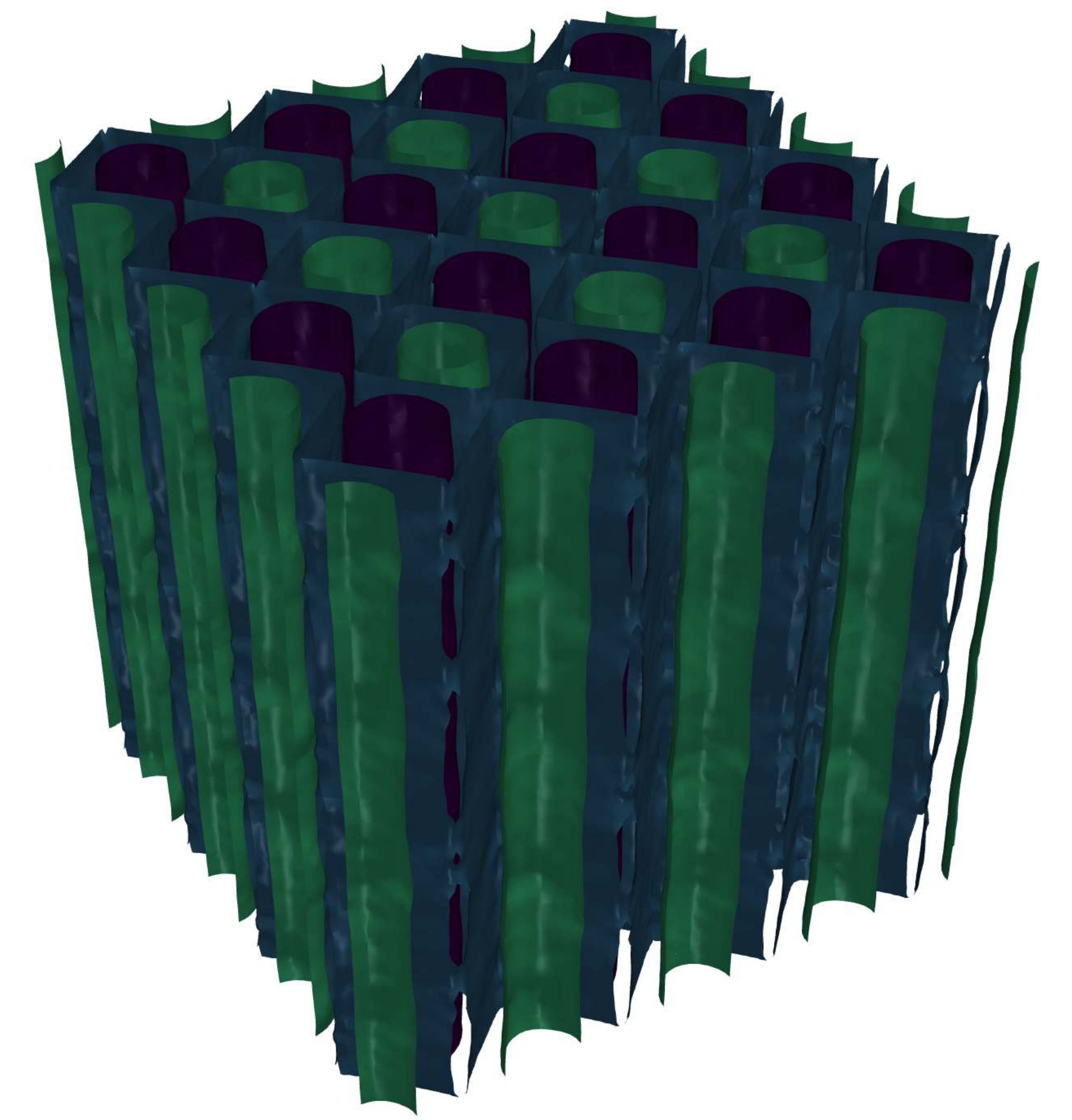}
  \hfill
  \includegraphics[width=0.18\textwidth]{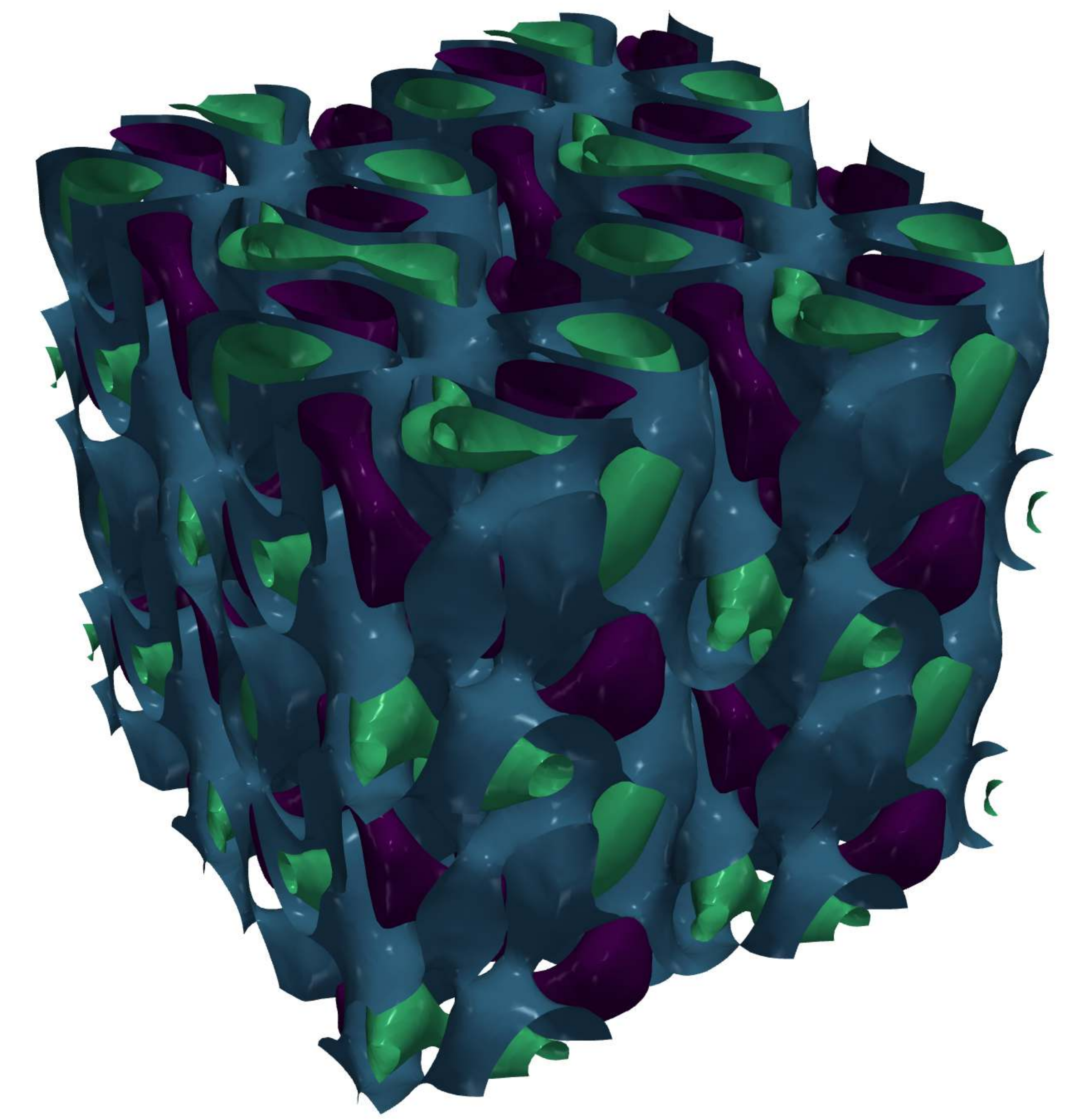}
  \hfill
  \includegraphics[width=0.18\textwidth]{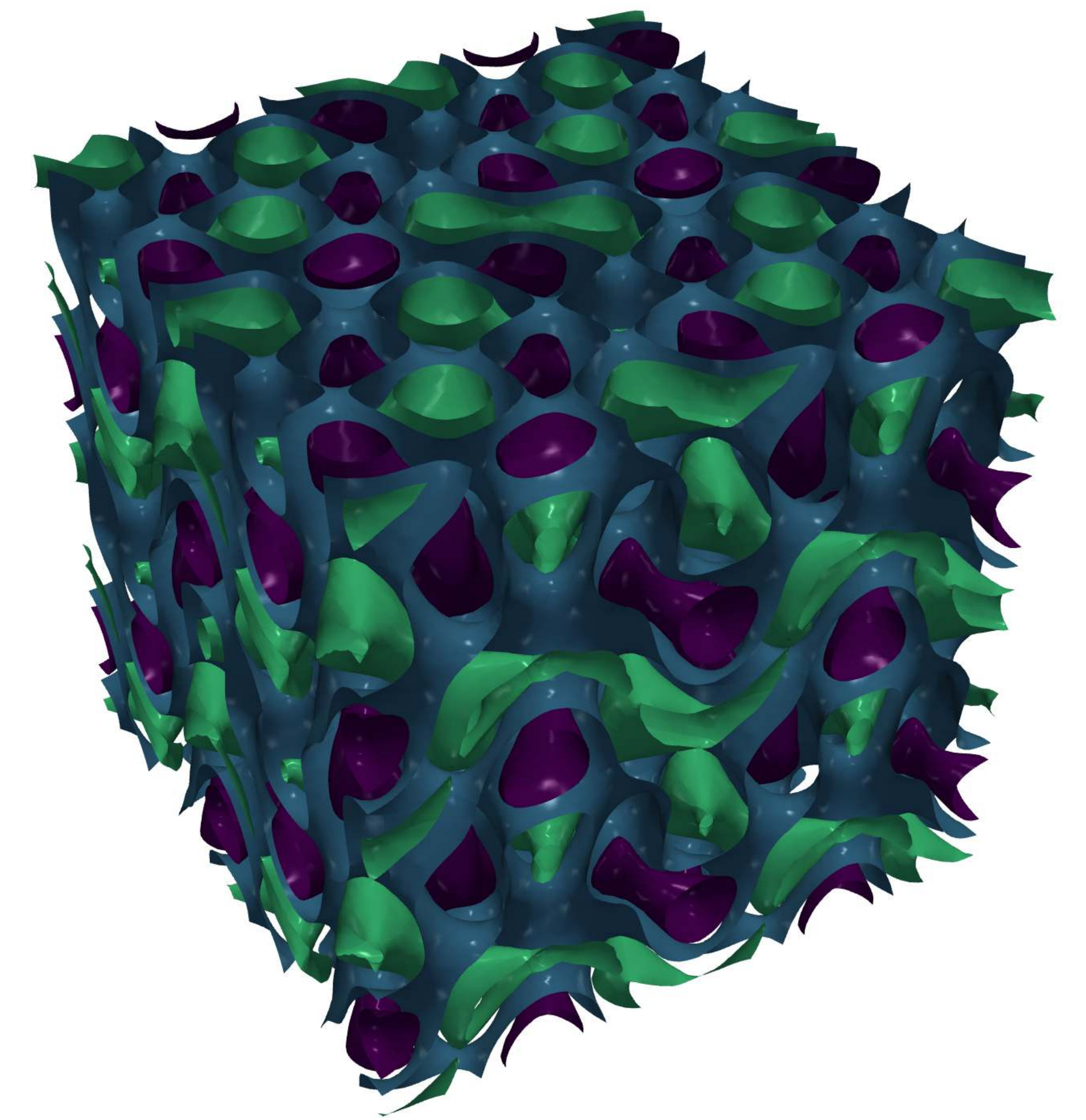}
\end{center}
Our results show that any continuous invariant function can be represented
by a series expansion in functions $\ef_i$. As for the Fourier transform, the functions form a orthonormal
basis of the relevant $\L_2$ space. The functions $\ef_i$ can hence be seen as a generalization of the
Fourier transform from pure shift groups to crystallographic groups. All of this is made precise in \cref{sec:linear}.
\\[.5em]
{\noindent\bf Nonlinear representations: Factoring through an orbifold}.
The second representation, in \cref{sec:nonlinear}, generalizes an idea of David MacKay \citep{mackay1998introduction}, who constructs periodic
functions on the line as follows: Start with a continuous function ${h:\mathbb{R}^2\rightarrow\mathbb{R}}$.
Choose a circle of circumference $1$ in $\mathbb{R}^2$, and restrict $h$ to the circle. The restriction is still
continuous. Now ``cut and unfold the circle with $h$ on it'' to obtain a function on the unit interval.
Since this function takes the same value at both interval boundaries, replicating it by shifts of integer length
defines a function on $\mathbb{R}$ that is periodic and continuous:
\begin{center}
  \begin{tikzpicture}
    \node at (0,0) {\includegraphics[height=3cm]{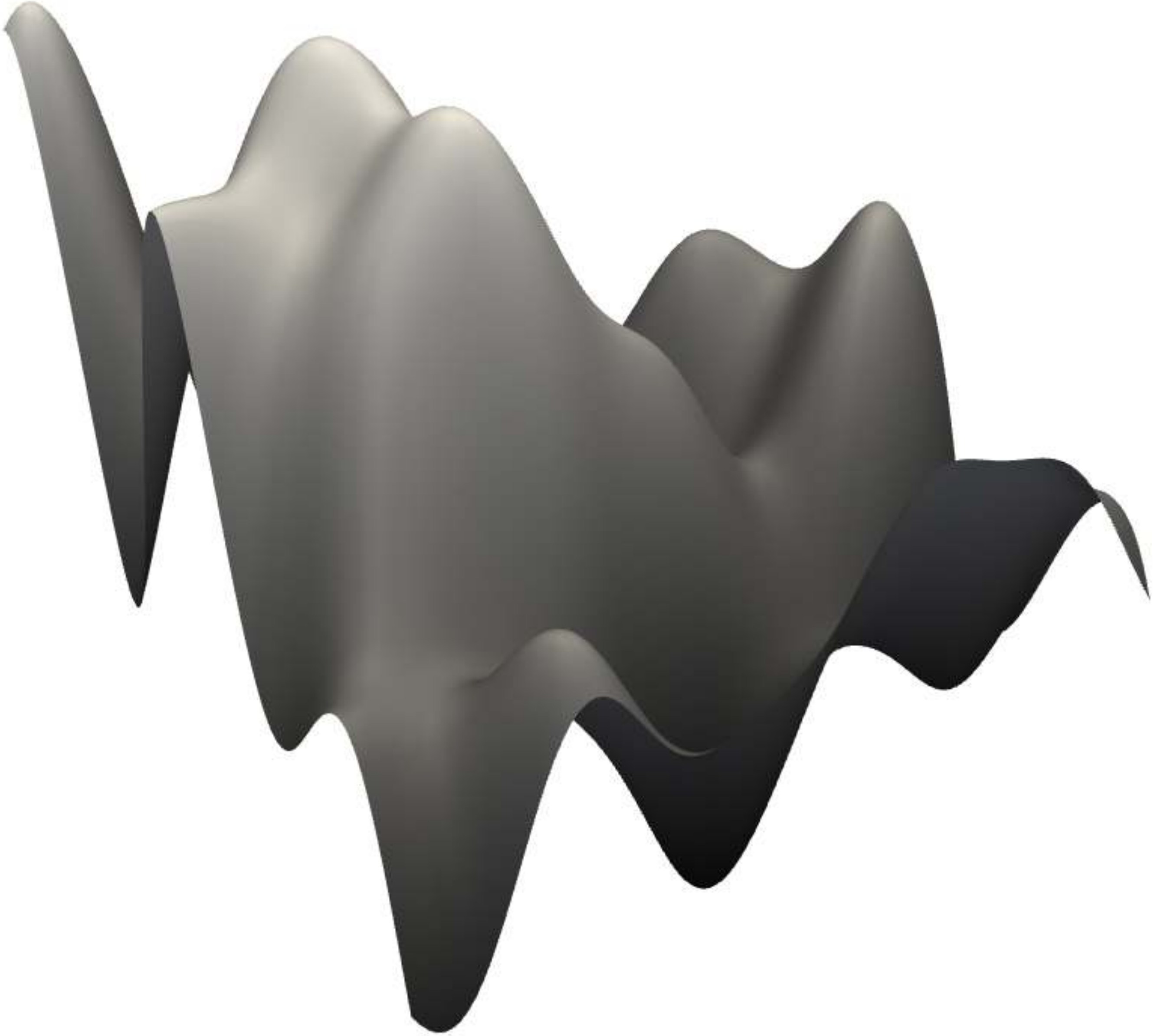}};
    \node at (3.5,0) {\includegraphics[height=3cm]{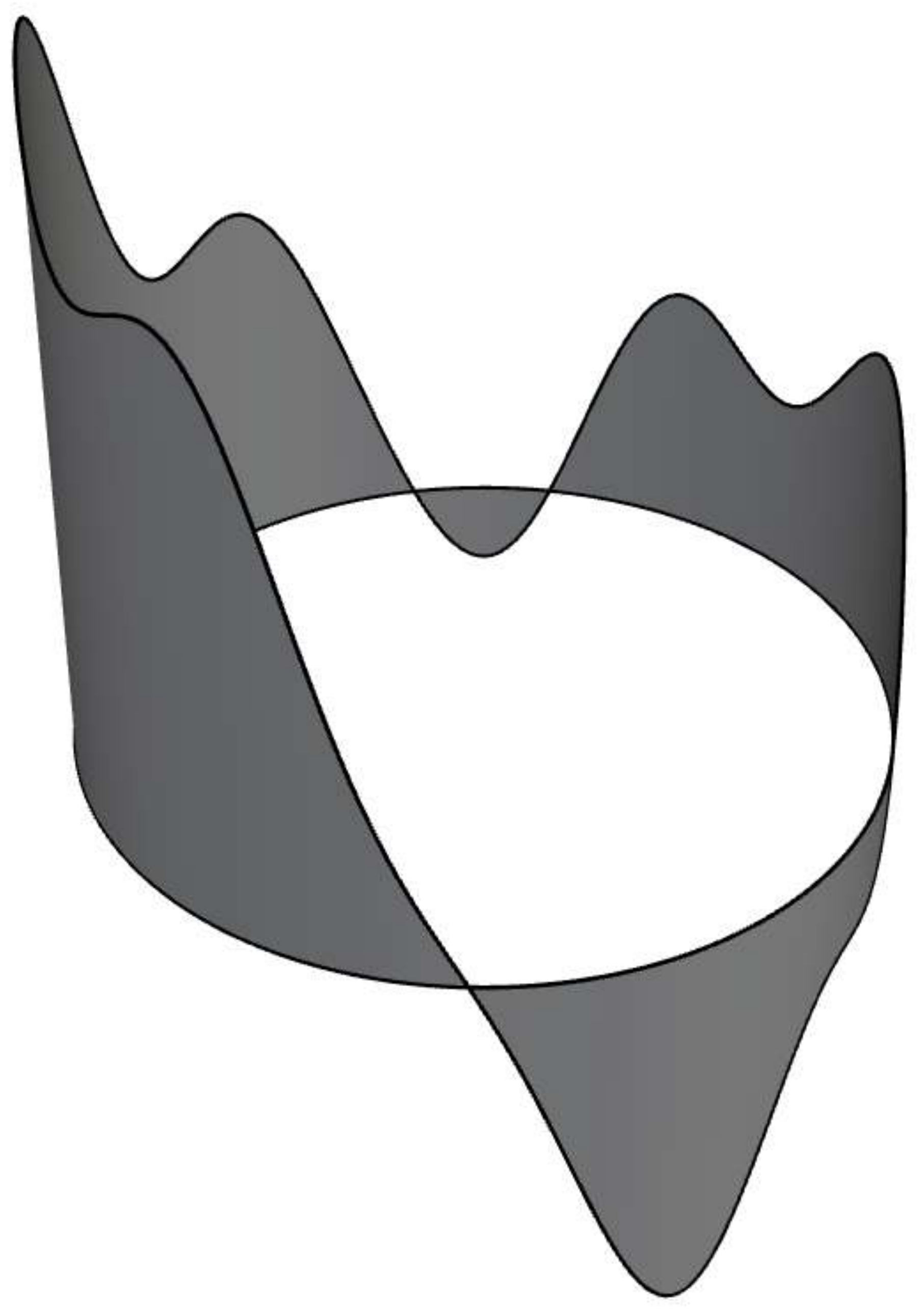}};
    \node at (9,0) {\includegraphics[height=2.8cm,width=6.5cm]{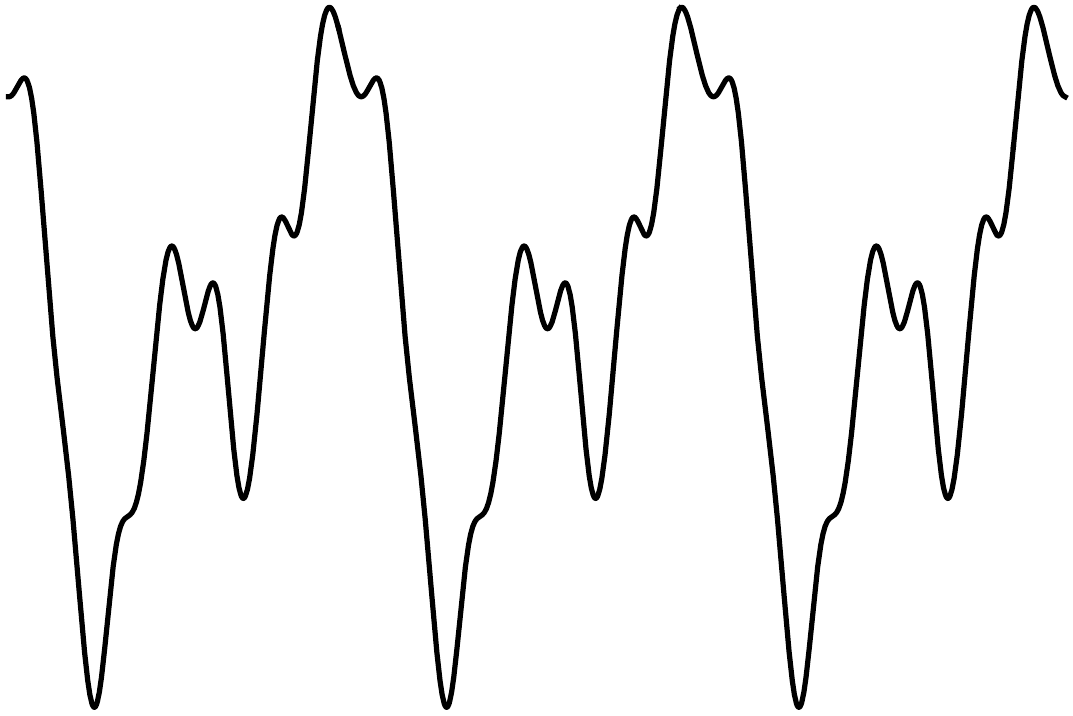}};
    \node at (0,-2) {\scriptsize function $h$ on $\mathbb{R}^2$};
    \node at (3.5,-2) {\scriptsize restrict $h$ to circle};
    \node at (9,-2) {\scriptsize unfold circle and replicate};
  \end{tikzpicture}
\end{center}
More formally, MacKay's approach constructs a function
${\rho:\mathbb{R}\rightarrow\text{circle}\subset\mathbb{R}^2}$ such that
\begin{equation*}
  f\text{ is continuous and periodic on $\mathbb{R}$ }
  \quad\Leftrightarrow\quad
  f\;=\;h\circ\rho\quad\text{ for some continuous }h:\mathbb{R}^2\rightarrow\mathbb{R}\;.
\end{equation*}
We show how to generalize this  construction to any finite dimension $n$, any crystallographic
group $\group$ on $\mathbb{R}^n$, and any convex polytope with which $\group$ tiles the space:
For each~$\group$ and~$\Pi$, there is a continuous, surjective map
\begin{equation}
  \label{intro:embedding:map}
  \rho:\mathbb{R}^n\rightarrow\Omega\qquad\text{ for some finite }N\geq n\text{ and a compact set }\Omega\subset\mathbb{R}^N
\end{equation}
such that
\begin{equation*}
  f\text{ is continuous and invariant }
  \quad\Leftrightarrow\quad
  f\;=\;h\circ\rho\quad\text{ for some continuous }h:\mathbb{R}^N\rightarrow\mathbb{R}\;.
\end{equation*}
This is \cref{theorem:embedding}.
\cref{sec:embedding:algorithm} shows how to compute a representation of $\rho$ using multidimensional scaling.

The set $\Omega$ can be thought of as an $n$-dimensional surface in a higher-dimensional space~$\mathbb{R}^N$.
If $\group$ contains only shifts, this surface is completely smooth, and hence a manifold. That is the
case in MacKay's construction, where $\Omega$ is the circle, and
the group~\texttt{p1} on $\mathbb{R}^2$, for which $\Omega$ is the torus shown on the left:
\vspace{-1em}
\begin{center}
  \begin{tikzpicture}
    \node at (0,0) {\includegraphics[height=2.6cm]{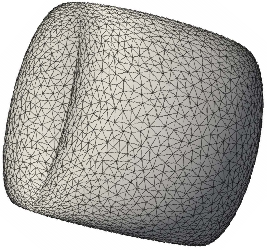}};
    \node at (6,0) {\includegraphics[height=3cm,angle=-20]{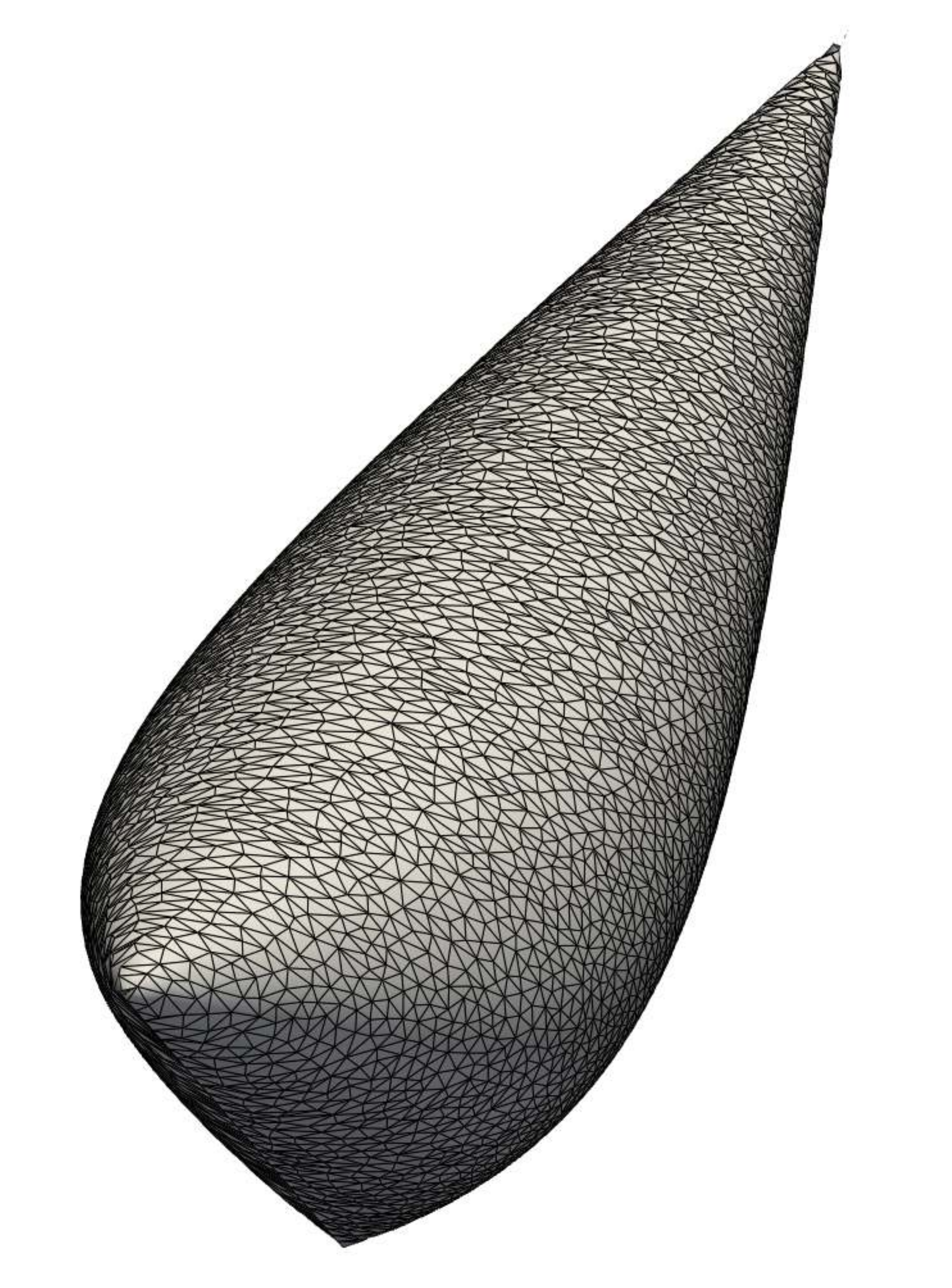}};
  \end{tikzpicture}
\end{center}
\vspace{-2em}
For most crystallographic groups, $\Omega$ is not a manifold, but
rather a more general object called an orbifold. The precise definition (see \cref{sec:orbifolds}) is somewhat technical, but
loosely speaking, an orbifold is a surface that resembles a manifold almost everywhere, except at a small number of
points at which it is not smooth.
That is illustrated by the orbifold on the right, which represents a group containing rotations, and has several ``sharp corners''.
\\[.5em]
{\noindent\bf Applications I: Neural networks}.
We can now define $\group$-invariant models by factoring through $\rho$.
To define an invariant
neural network, for example, start with a continuous neural network ${h_{\theta}:\mathbb{R}^N\rightarrow Y}$ with weight vector $\theta$ and some
output space $Y$. Then ${\rho\circ h_{\theta}}$ is a continuous and invariant neural network ${\mathbb{R}^n\rightarrow Y}$.
Here are examples for three groups (\texttt{cm}, \texttt{p4}, and \texttt{p4gm}) on $\mathbb{R}^2$, with three hidden layers and randomly generated weights:
\begin{center}
  \begin{tikzpicture}
    \node at (0,0) {\includegraphics[height=3cm]{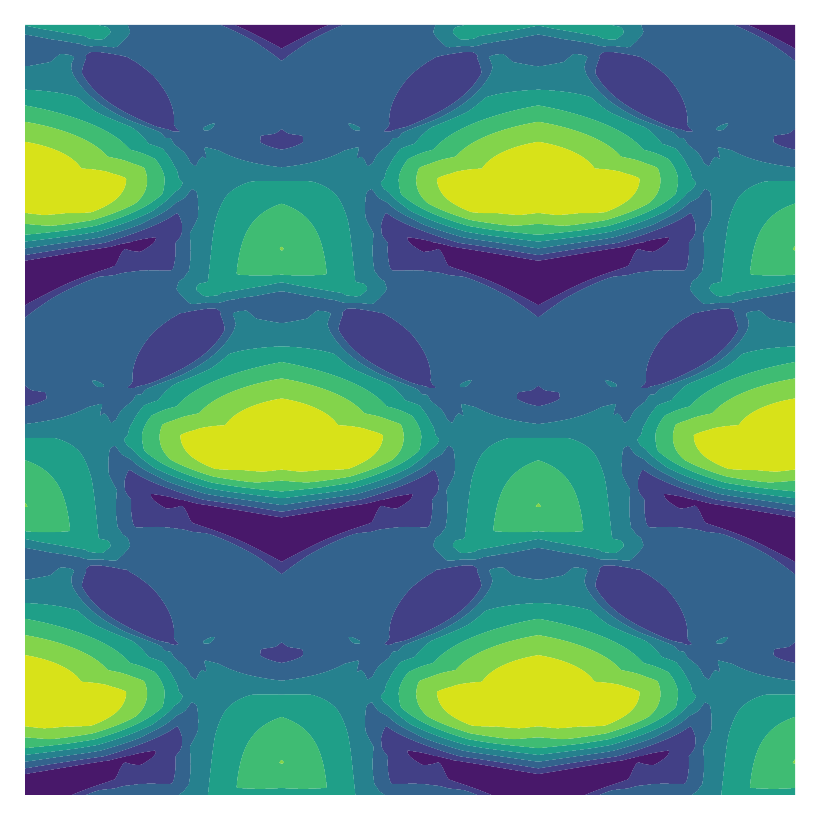}};
    \node at (5,0) {\includegraphics[height=3cm]{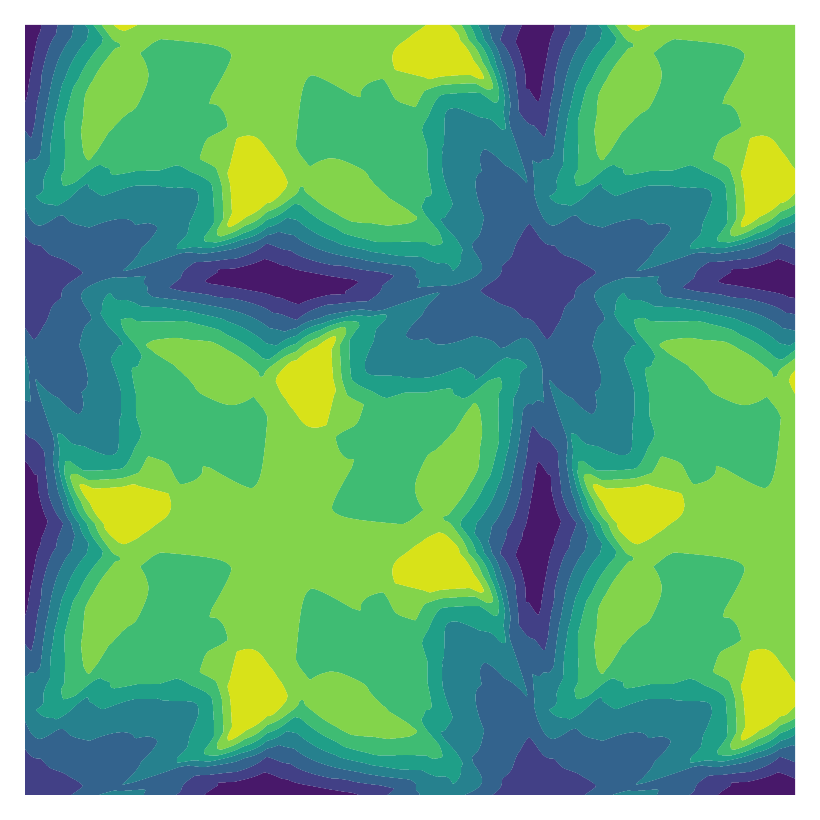}};
    \node at (10,0) {\includegraphics[height=3cm]{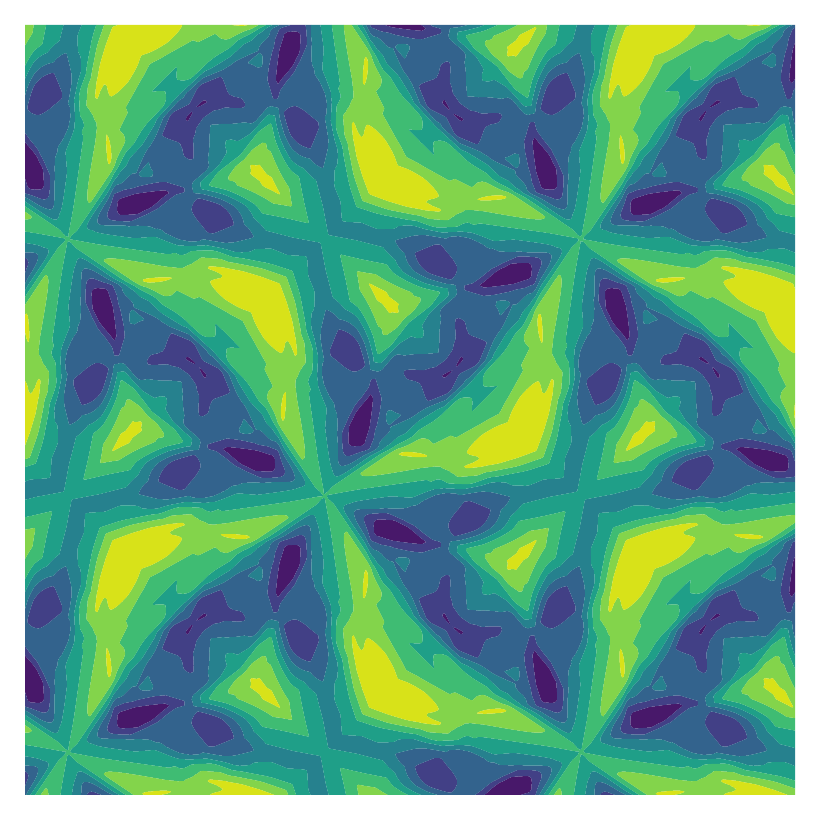}};
  \end{tikzpicture}
\end{center}
{\noindent\bf Applications II: Invariant kernels}.
We can similarly define $\group$-invariant reproducing kernels on $\mathbb{R}^n$, by starting with a
kernel $\hat{\kappa}$ on $\mathbb{R}^N$ and defining a function on $\mathbb{R}^n$ as
\begin{equation*}
    \kappa(x,y)\;=\;\hat{\kappa}\circ(\rho\otimes\rho)(x,y)\;=\;\hat{\kappa}(\rho(x),\rho(y))\;.
\end{equation*}
This function is again a kernel. In \cref{sec:kernels}, we
show that its reproducing kernel Hilbert space consists of continuous $\group$-invariant
functions on~$\mathbb{R}^n$.
We also show that, even though~$\mathbb{R}^n$ is not compact, $\kappa$ behaves essentially like a kernel on a
compact domain (\cref{result:kernel:compactness}). In particular,
it satisfies a Mercer representation and a compact embedding property, both of which usually require
compactness. This behavior is specific to kernels invariant under crystallographic groups, and
does not extend to more general groups of isometries on $\mathbb{R}^n$.
\\[.5em]
{\noindent\bf Applications III: Invariant Gaussian processes}.
There are two ways in which a Gaussian process (GP) can be invariant under a group: A GP is a distribution on
functions, and we can either ask for each function it generates to be invariant, or only require that its distribution is
invariant (see \cref{sec:gp} for definitions). The former implies the latter. Both types of processes can be constructed
by factoring through an orbifold:
Suppose we start with a kernel $\hat{\kappa}$ (a covariance function) and a real-valued function $\hat{\mu}$ (the mean function),
both defined on $\mathbb{R}^N$. If we then generate a random function $F$ on $\mathbb{R}^n$ as
\begin{equation*}
  F\;:=\;H\circ\rho\qquad\text{ where }\qquad H\,\sim\,\text{GP}(\hat{\mu},\hat{\kappa})\;,
\end{equation*}
the function $F$ is $\group$-invariant with probability 1.
The following are examples of such random functions, rendered as contour plots with non-smooth colormaps.
\begin{center}
  \begin{tikzpicture}
    \node at (0,0) {\includegraphics[width=3cm]{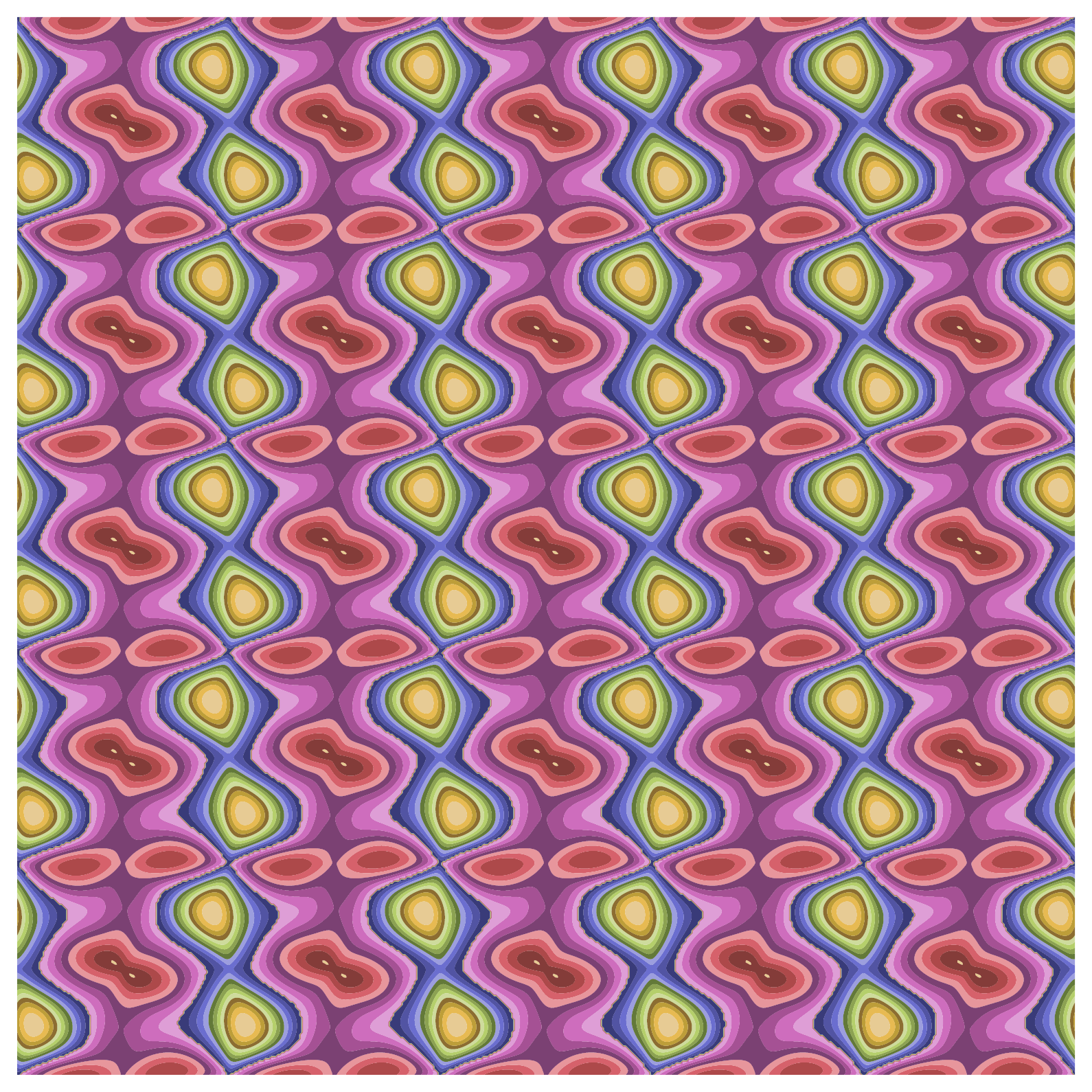}};
    \node at (3.5,0) {\includegraphics[width=3cm]{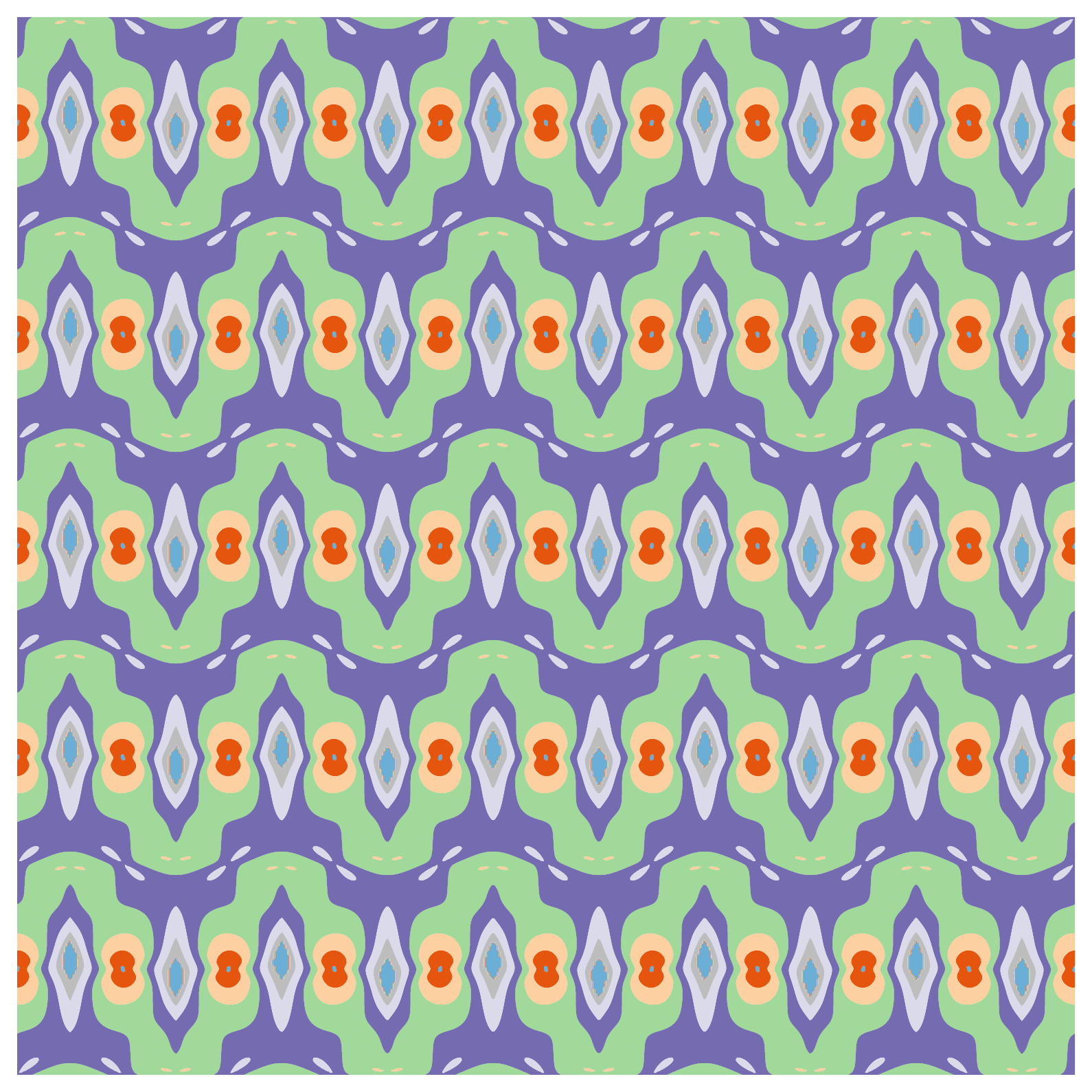}};
    \node at (7,0) {\includegraphics[width=3cm]{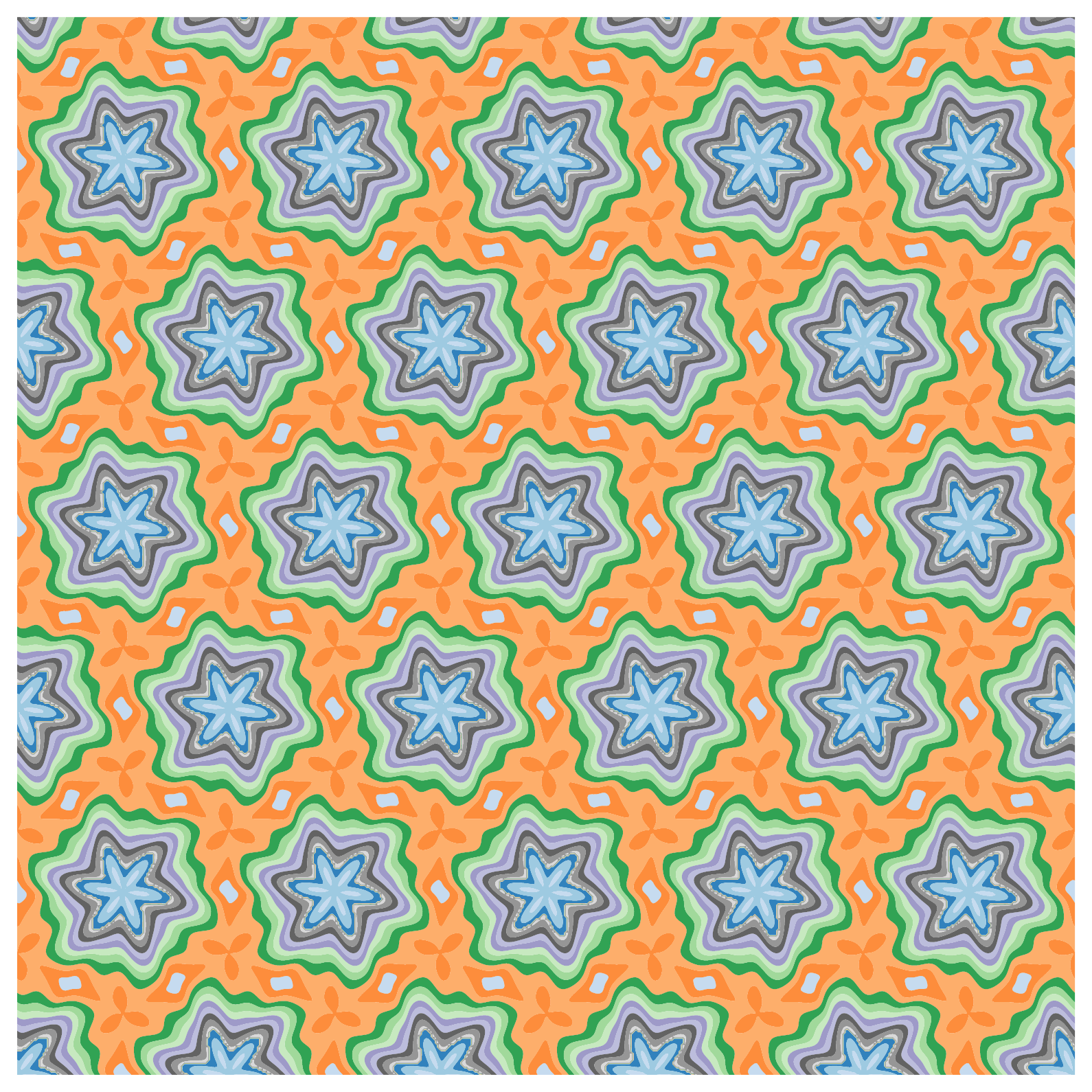}};
    \node at (10.5,0) {\includegraphics[width=3cm]{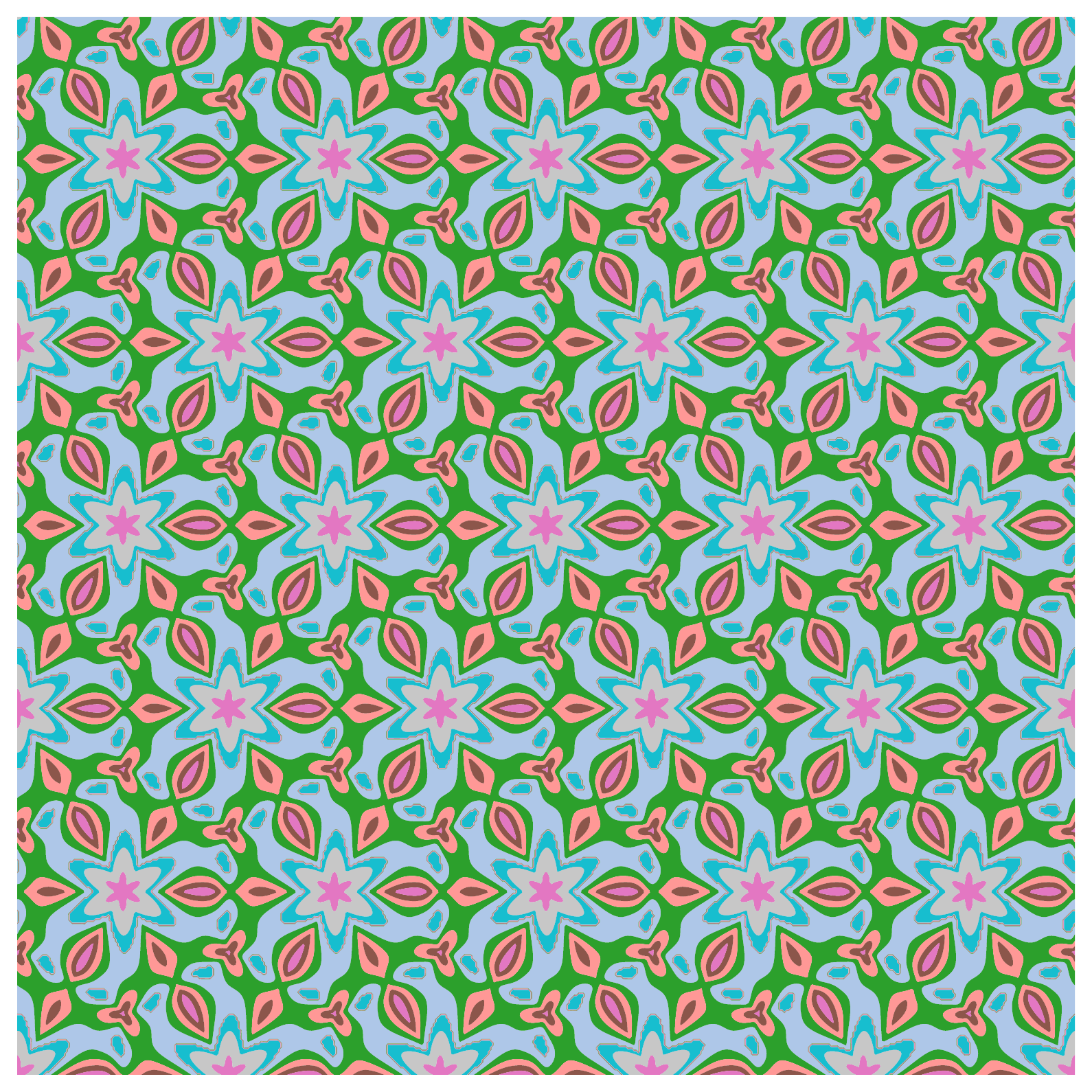}};
  \end{tikzpicture}
\end{center}
\noindent If we instead generate $F$ as
\begin{equation*}
  F\,\sim\,\text{GP}(\mu,\kappa)\qquad\text{ where }\qquad \mu\;:=\;\hat{\mu}\circ\rho\;\text{ and }\;\kappa\;:=\;\hat{\kappa}\circ(\rho\otimes\rho)\;,
\end{equation*}
the distribution of $F$ is $\group$-invariant. See \cref{sec:gp}.
\\[.5em]
{\noindent\bf Properties of the Laplace operator}.
\cref{sec:laplacian} studies differentials and Laplacians of crystallographically invariant functions
${f:\mathbb{R}^n\rightarrow\mathbb{R}}$.
The results are then used in the proof of the Fourier representation.
Consider a vector field $F$, i.e., a function ${F:\mathbb{R}^n\rightarrow\mathbb{R}^n}$. An example
of such a vector field is the gradient ${F=\nabla f}$. \cref{lemma:differentials} shows that
the gradient transforms under elements $\phi$ of $\mathbb{G}$ as
\begin{equation*}
  \nabla f(\phi x)\;=\;(\text{linear part of }\phi)\cdot\nabla f
  \quad\text{ or abstractly }\quad
  F\circ\phi\;=\;(\text{linear part of }\phi)\circ F\;.
\end{equation*}
\cref{lemma:gradient:field} shows that, for any vector field $F$ that transforms in this way,
the total flux through the boundary of the polytope $\Pi$ vanishes,
\begin{equation*}
  \mint_{\partial\Pi}F(x)^{\trans}(\text{normal vector of $\partial\Pi$ at $x$})dx\;=\;0\;.
\end{equation*}
We can combine this fact with a result from the theory of partial differential equations, the
so-called Green identity, which decomposes the Laplacian on functions on $\Pi$ as
\begin{equation*}
  -\Delta f\;=\;\text{self-adjoint component on interior of }\Pi\;-\;\text{correction term on }\partial\Pi\;.
\end{equation*}
\cref{green:identity} makes the statement precise. Using the fact that the flux vanishes,
we can show that the correction term on $\partial\Pi$ vanishes, and from that deduce that the
Laplace operator on invariant functions is self-adjoint (\cref{theorem:laplacian}).
That allows us to draw on
results from the spectral theory of self-adjoint operators to solve \eqref{eigenproblem:intro}.
\\[.5em]
{\noindent\bf Background and reference results}.
Since our methods draw on a number of different fields,
the appendix provides additional background on groups of isometries (App.~\ref{sec:isometries}),
functional analysis (App.~\ref{sec:function:spaces}), and orbifolds (App.~\ref{sec:orbifolds}),
and spectral theory (App.~\ref{sec:spectral}).

\newpage

\section{Preliminaries: Crystallographic groups}
\label{sec:definitions}

Throughout, we consider a Euclidean space $\mathbb{R}^n$, and write $d_n$ for Euclidean distance in~$n$ dimensions.
Euclidean volume (that is, Lebesgue measure on $\mathbb{R}^n$) is denoted $\vol_n$.
As we work with both sets and their boundaries, we must carefully distinguish
dimensions: The span of a set ${A\subset\mathbb{R}^n}$ is the smallest affine subspace that contains it.
We define the \kword{dimension} and \kword{relative interior} of $A$ as
\begin{equation*}
  \dim A\,:=\,\dim\text{span}\,A
  \quad\text{ and }\quad
  A^{\circ}\,:=\,\text{ largest subset of }A\text{ that is open in }\text{span}\,A\;.
\end{equation*}
The \kword{boundary} of $A$ is the set ${\partial A:=A\setminus A^{\circ}}$.
If $A$ has dimension ${k<n}$, then ${\vol_k(A)}$ denotes Euclidean volume in $\text{span}\,A$.
For example: If ${A\subset\mathbb{R}^3}$ is a closed line segment, then ${\dim A=1}$, and
${\vol_1(A)}$ is the length of the line segment, whereas ${\vol_3(A)=\vol_2(A)=0}$.
Taking the relative interior $A^\circ$ removes
the two endpoints, whereas interior of $A$ in $\mathbb{R}^3$ is the empty set.
(No such distinction is required for the closure $\bar{A}$,
since $A$ is closed in $\text{span}\,A$ if and only if it is closed in $\mathbb{R}^n$.)

\subsection{Defining crystallographic groups}
Consider a group $\group$
of isometries of $\mathbb{R}^n$. (See \cref{sec:isometries} for a
brief review of definitions.) Every isometry $\phi$ of $\mathbb{R}^n$ is of the form
\begin{equation}
  \label{eq:euclidean:isometry}
  \phi x\;=\;A_\phi x+b_\phi
  \qquad\text{ for some orthogonal $n\times n$ matrix }A_\phi\text{ and some }b_\phi\in\mathbb{R}^n\;.
\end{equation}
Let ${M\subset\mathbb{R}^n}$ be a set. We say that $\group$
\kword{tiles} the space $\mathbb{R}^n$ with $M$
if the image sets ${\phi M}$ completely cover the space so that only their boundaries overlap:
\begin{equation*}
  \medcup\nolimits_{\phi\in\group}\phi M\;=\;\mathbb{R}^n
  \qquad\text{ and }\qquad
  \phi M\cap\psi M\;\subset\;\partial(\phi M)
  \quad\text{ whenever }\phi\neq\psi\;.
\end{equation*}
Each set $\phi M$ is a \kword{tile},
and the collection ${\group M:=\braces{\phi M|\phi\in\group}}$ is a \kword{tiling} of $\mathbb{R}^n$.

By a \kword{convex polytope}, we mean the convex hull of a finite set of points
\citep{Ziegler:1995}. Let~${\cF\subset\mathbb{R}^n}$ be an $n$-dimensional convex polytope.
The boundary $\partial\cF$
consists of a finite number of ${(n-1)}$-dimensional convex polytopes, called the \kword{facets}
of $\cF$. Thus, if $\group$ tiles $\mathbb{R}^n$ with $\Pi$, only points on
facets are contained in more than one tile.
\begin{definition}
A \kword{crystallographic group} is a group of isometries that
tiles $\mathbb{R}^n$ with an $n$-dimensional convex polytope $\cF$. 
\end{definition}
\noindent
The polytope $\cF$ is then also called a \kword{fundamental region} (in geometry) or an
\kword{asymmetric unit} (in materials science) for $\group$.
This definition of crystallographic groups
differs from those given in the literature, but we clarify in
\cref{remark:definition} that it is equivalent.

\begin{figure}[t]
\centering
\begin{subfigure}[b]{0.16\textwidth}
\includegraphics[width=\textwidth]{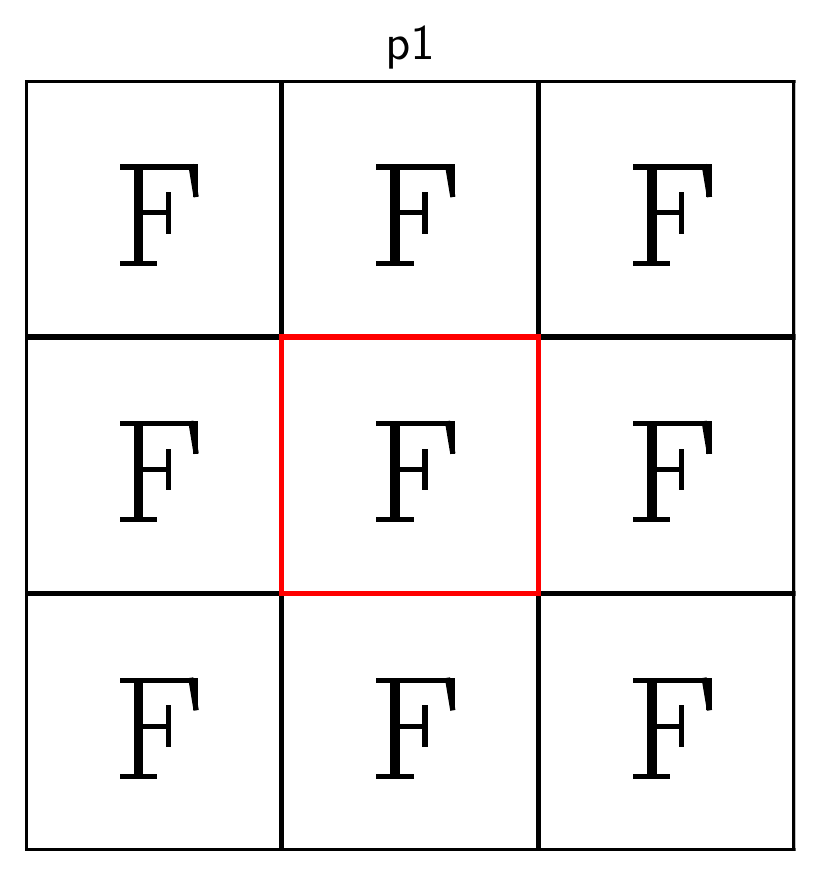}
\end{subfigure}\hfill
\begin{subfigure}[b]{0.16\textwidth}
\includegraphics[width=\textwidth]{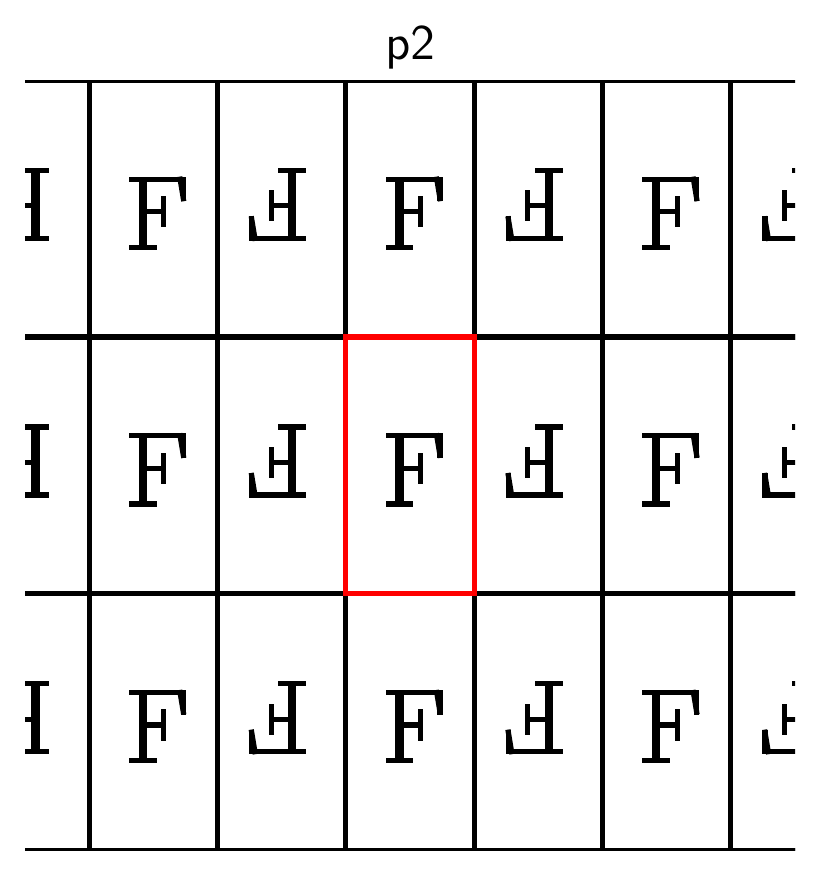}
\end{subfigure}\hfill
\begin{subfigure}[b]{0.16\textwidth}
\includegraphics[width=\textwidth]{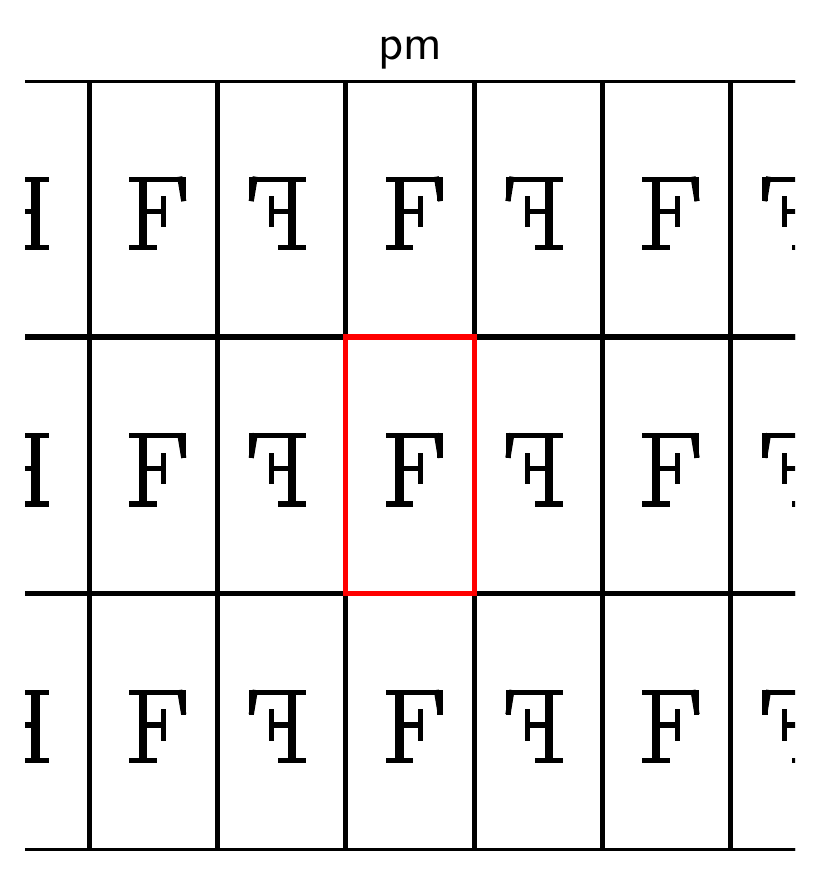}
\end{subfigure}\hfill
\begin{subfigure}[b]{0.16\textwidth}
\includegraphics[width=\textwidth]{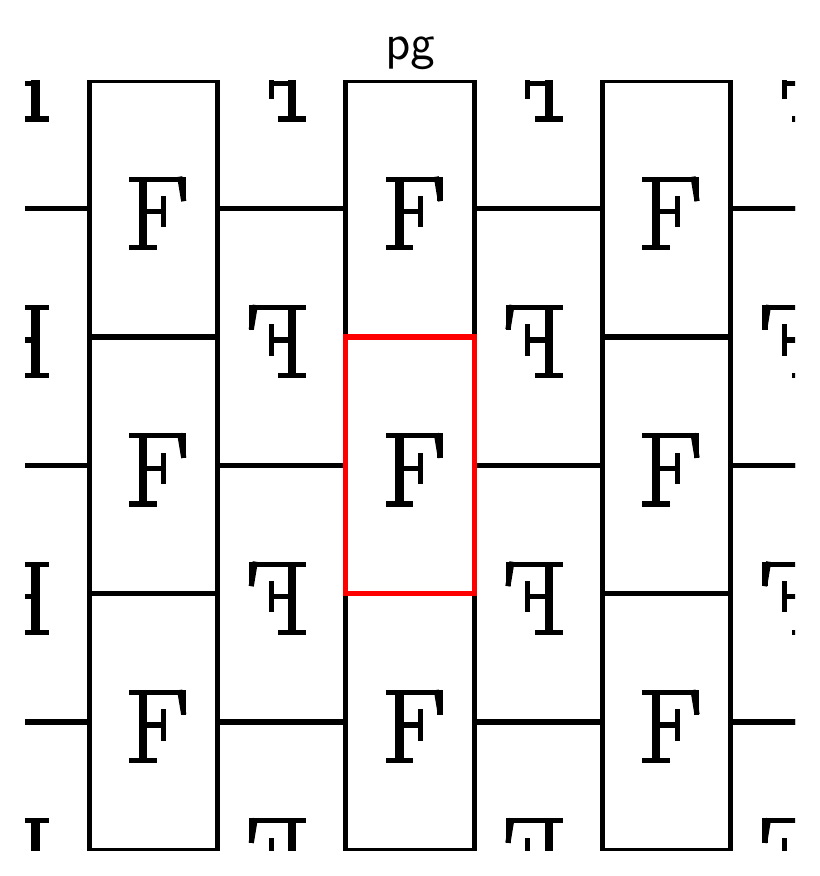}
\end{subfigure}\hfill
\begin{subfigure}[b]{0.16\textwidth}
\includegraphics[width=\textwidth]{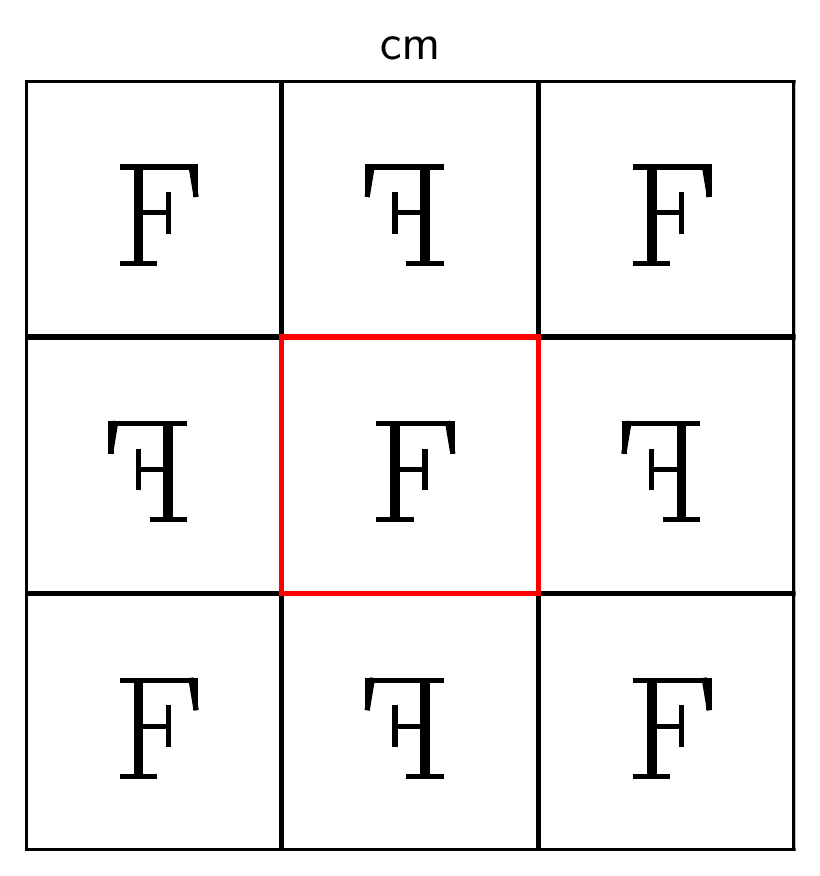}
\end{subfigure}\hfill
\begin{subfigure}[b]{0.16\textwidth}
\includegraphics[width=\textwidth]{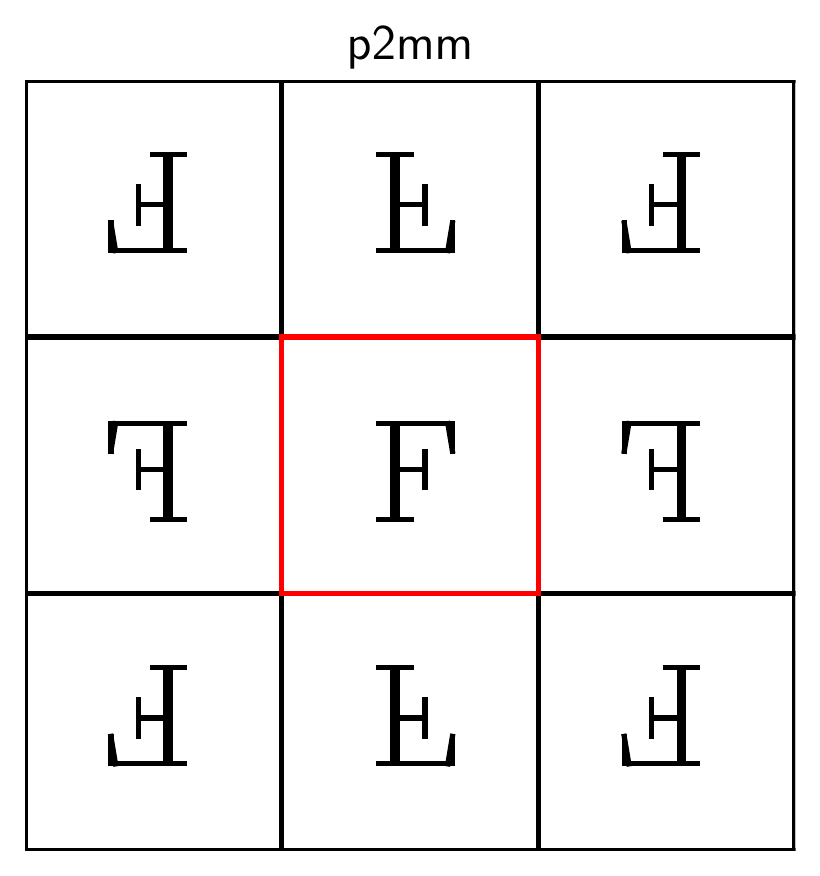}
\end{subfigure}
\\
\begin{subfigure}[b]{0.16\textwidth}
\includegraphics[width=\textwidth]{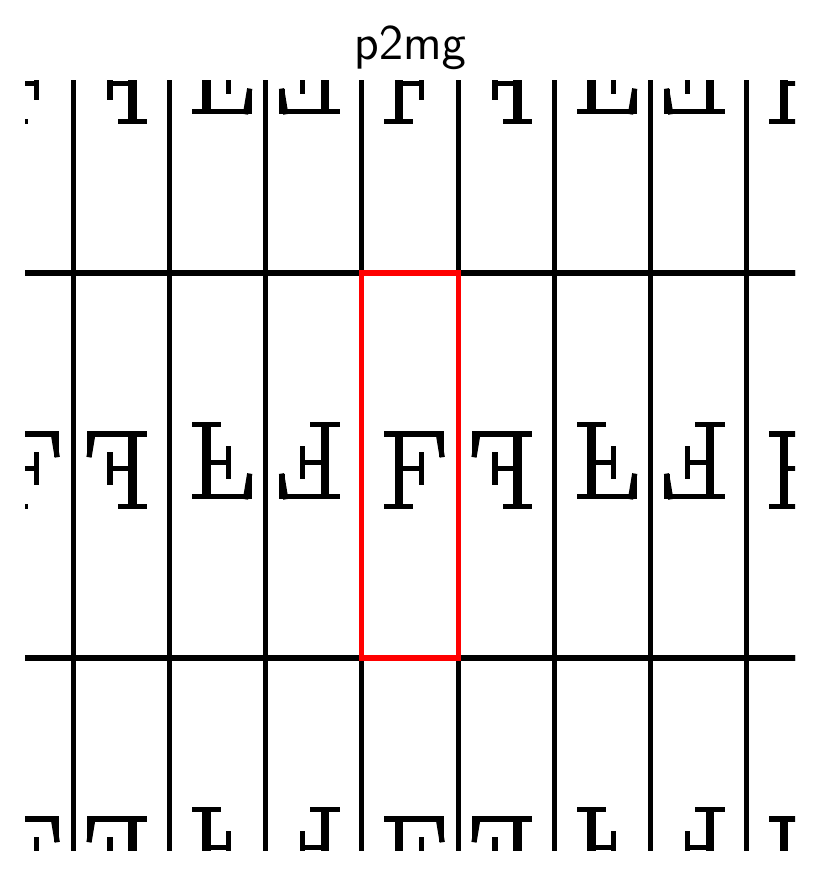}
\end{subfigure}\hfill
\begin{subfigure}[b]{0.16\textwidth}
\includegraphics[width=\textwidth]{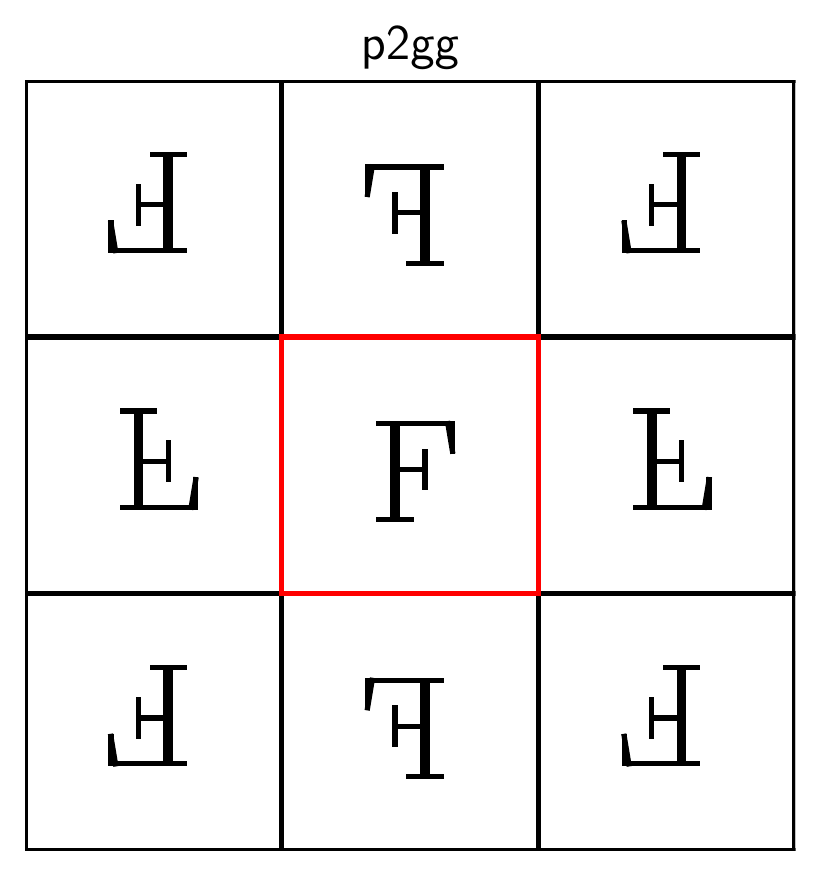}
\end{subfigure}\hfill
\begin{subfigure}[b]{0.16\textwidth}
\includegraphics[width=\textwidth]{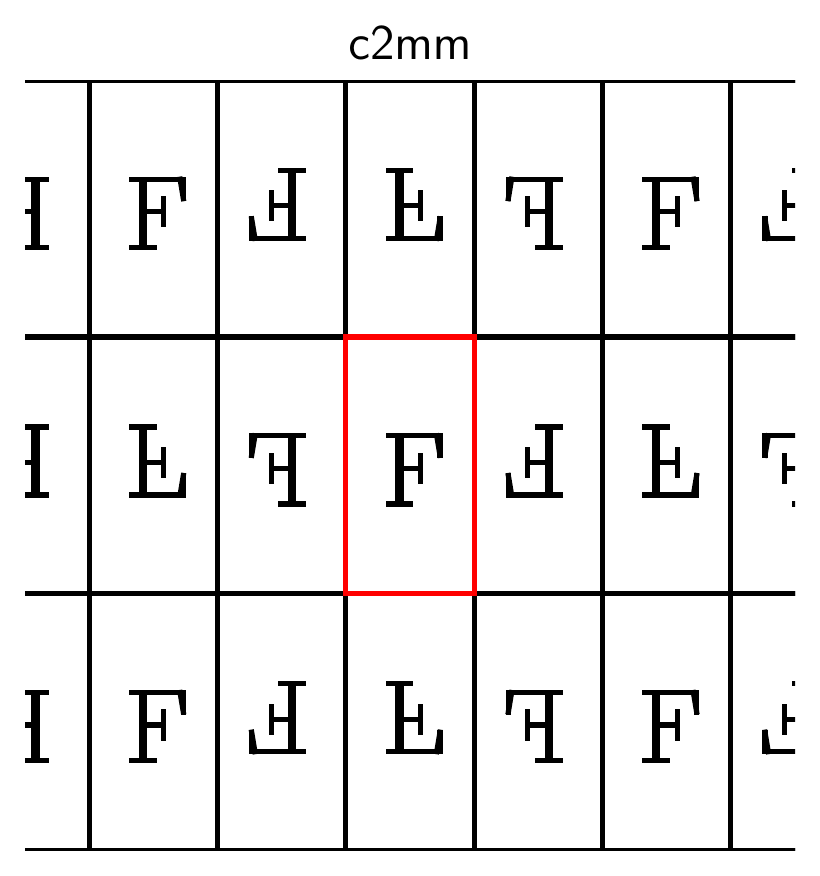}
\end{subfigure}\hfill
\begin{subfigure}[b]{0.16\textwidth}
\includegraphics[width=\textwidth]{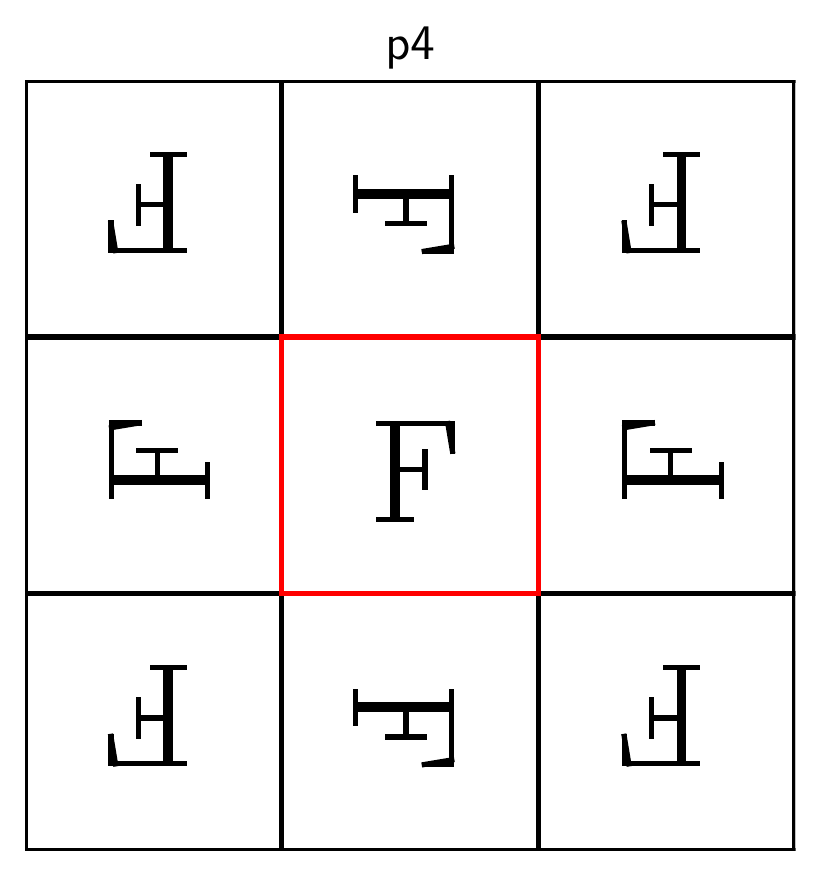}
\end{subfigure}\hfill
\begin{subfigure}[b]{0.16\textwidth}
\includegraphics[width=\textwidth]{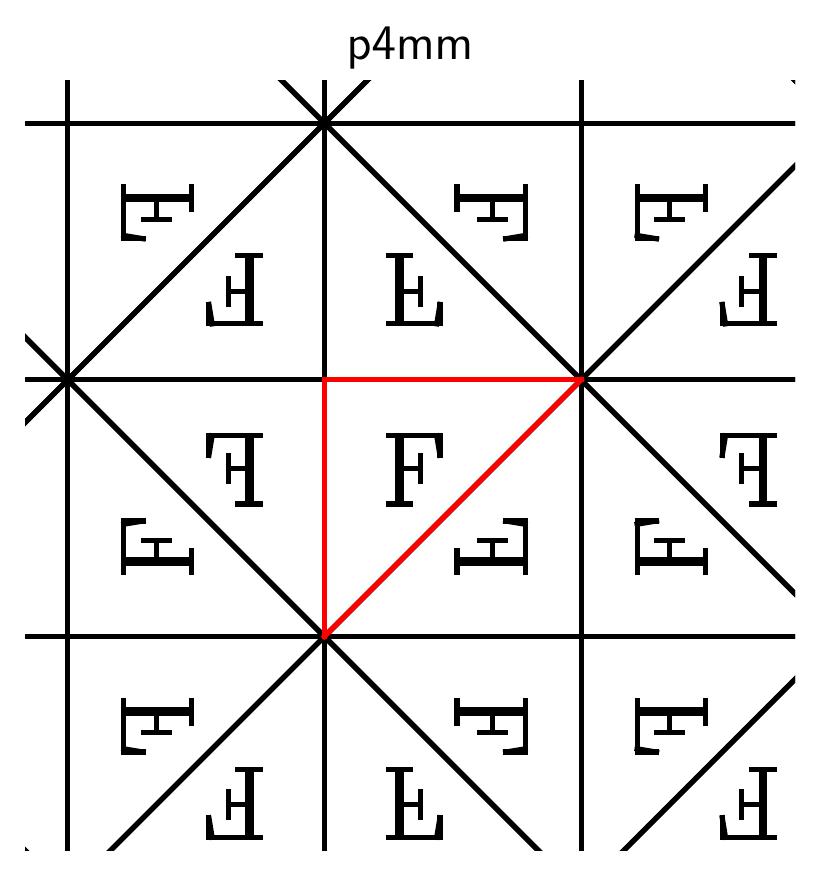}
\end{subfigure}\hfill
\begin{subfigure}[b]{0.16\textwidth}
\includegraphics[width=\textwidth]{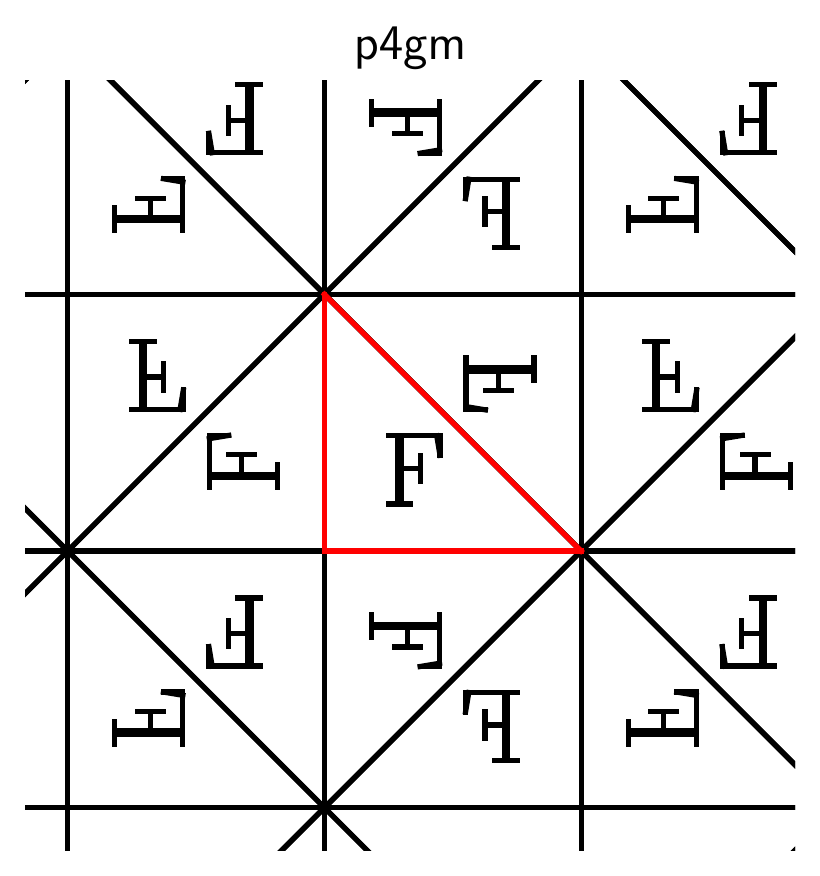}
\end{subfigure}
\\
\begin{subfigure}[b]{0.16\textwidth}
\includegraphics[width=\textwidth]{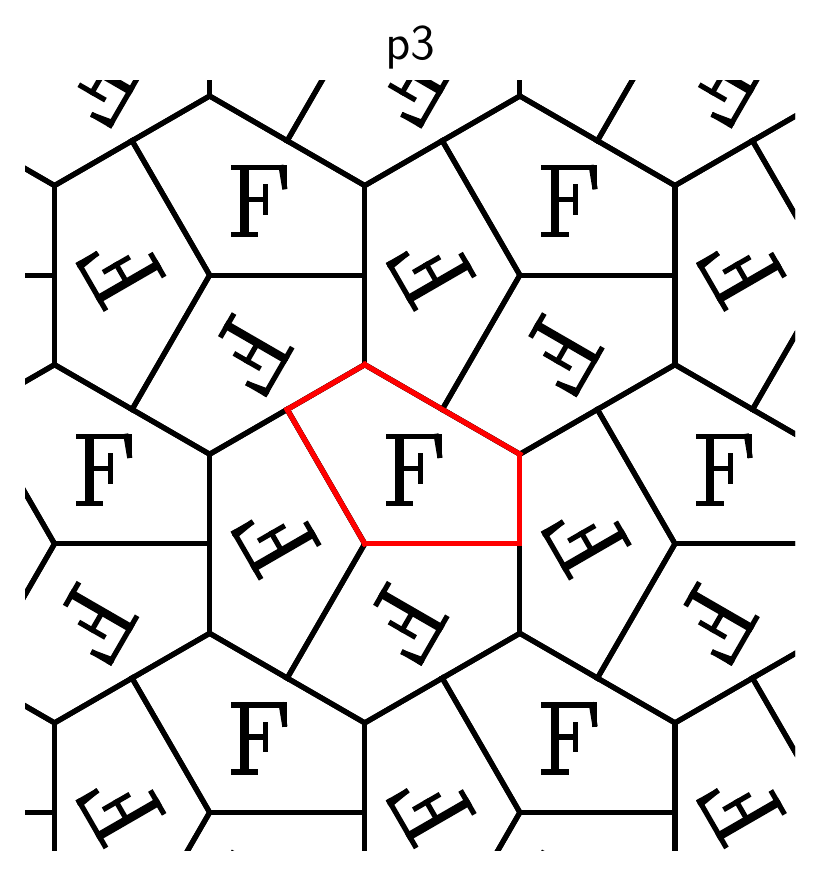}
\end{subfigure}\hfill
\begin{subfigure}[b]{0.16\textwidth}
\includegraphics[width=\textwidth]{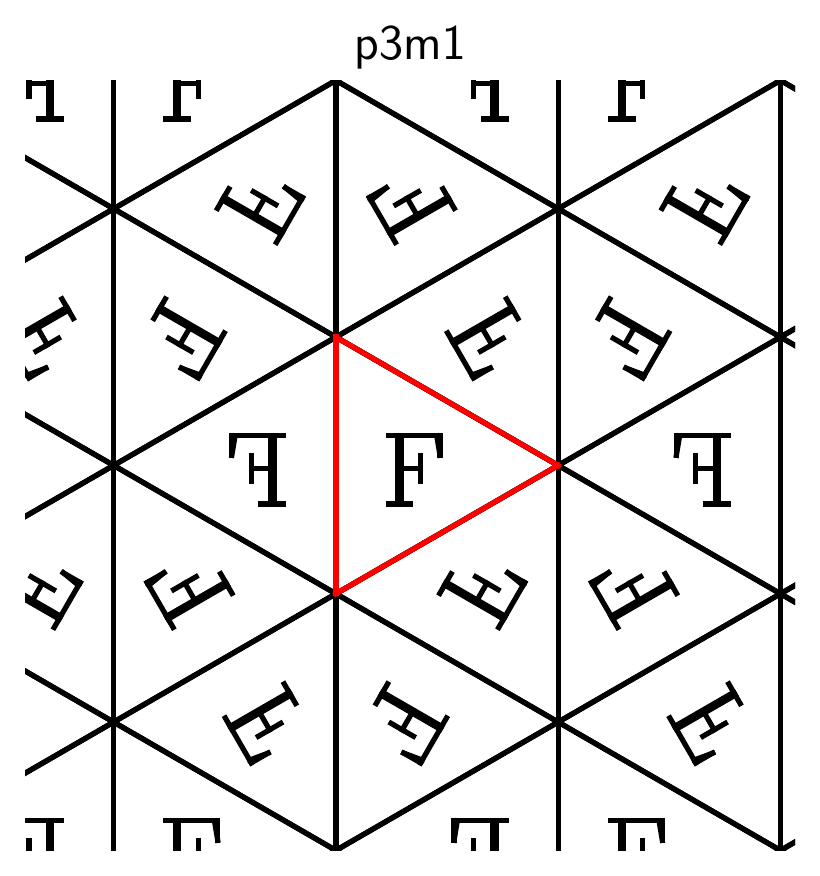}
\end{subfigure}\hfill
\begin{subfigure}[b]{0.16\textwidth}
\includegraphics[width=\textwidth]{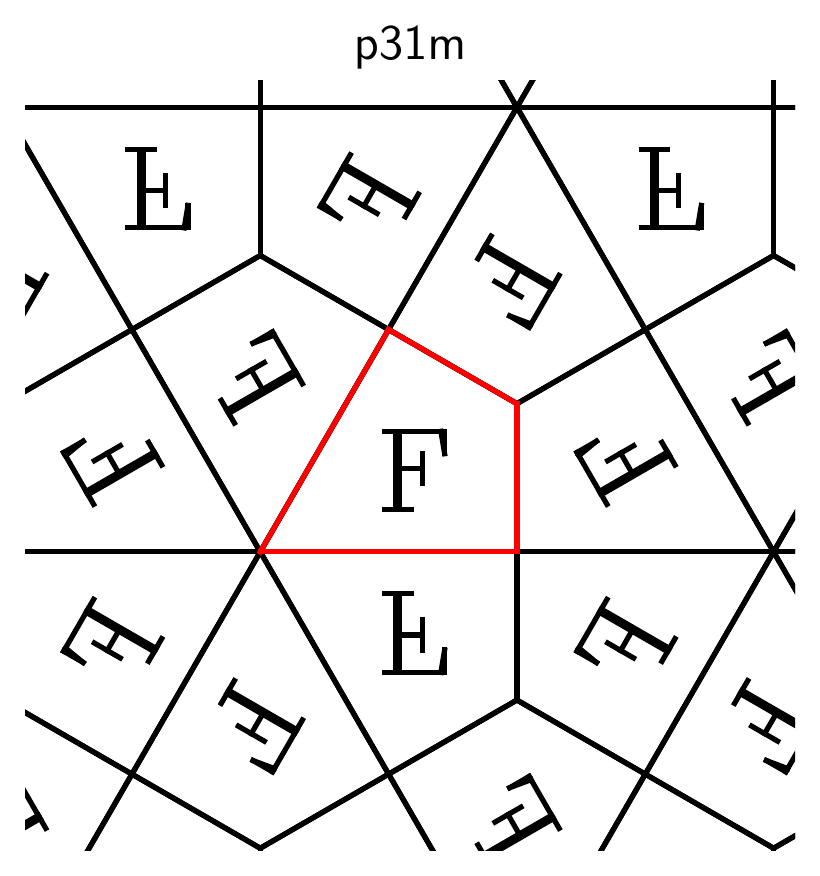}
\end{subfigure}\hfill
\begin{subfigure}[b]{0.16\textwidth}
\includegraphics[width=\textwidth]{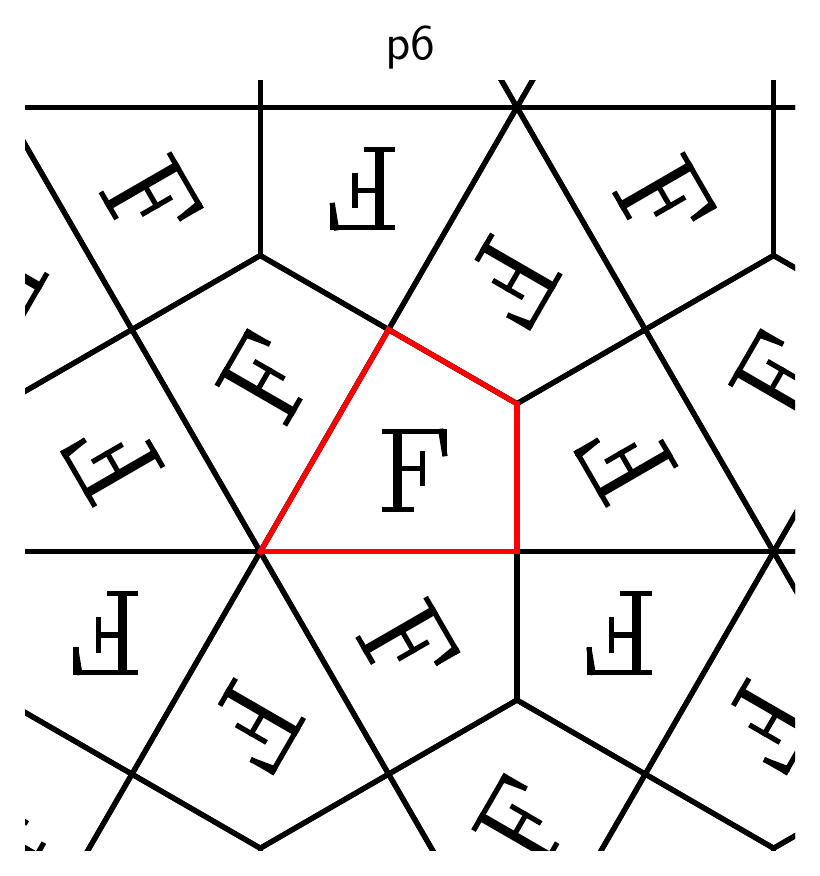}
\end{subfigure}\hfill
\begin{subfigure}[b]{0.16\textwidth}
\includegraphics[width=\textwidth]{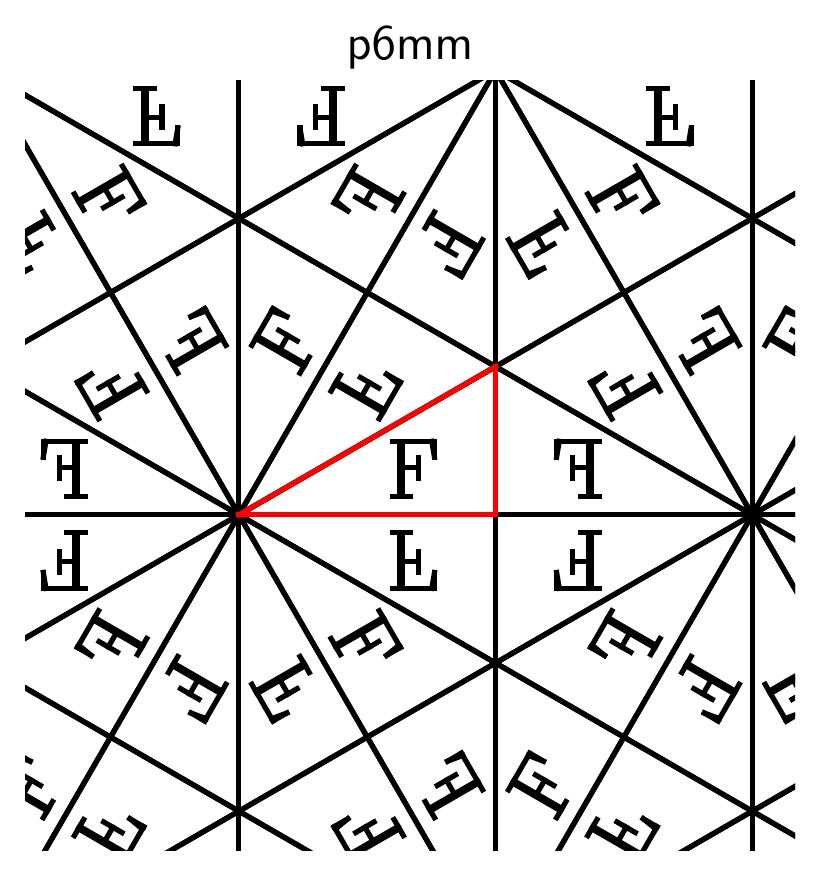}
\end{subfigure}
\caption{Tiling behavior of the 17 crystallographic groups on $\mathbb{R}^n$ (the wallpaper groups).}
\label{fig:wallpaper}
\end{figure}

\subsection{Basic properties}

Some properties of $\group$ can be read right off the definition:
Since $\group$ tiles the entire space with a set $\cF$ of finite diameter, we must have ${|\group|=\infty}$.
Since $\Pi$ is $n$-dimensional and convex, it contains an open metric ball of positive radius.
Each tile contains a copy of this ball, and these copies do not overlap. It follows that
\begin{equation}
  \label{discrete}
  d(\phi(x),\psi(x))\;>\;\varepsilon\qquad\text{ for all distinct }\phi,\psi\in\group\text{ and all }x\in\Pi^{\circ}\;.
\end{equation}
A group of isometries that satisfies \eqref{discrete}
for some ${\varepsilon>0}$ is called \kword{discrete}, in contrast to groups which contain, e.g., continuous rotations.
Discreteness implies $\group$ is countable, but not all countable groups of isometries are discrete
(the group $\mathbb{Q}^n$ of rational-valued shifts is a non-example).
In summary, every crystallographic group is
an infinite, discrete (and hence countable) subgroup of the Euclidean group on $\mathbb{R}^n$.

Suppose we choose one of the tilings in \cref{fig:wallpaper},
and rotate or shift the entire plane with the tiling on it.
Informally speaking, that changes the tiling, but not the tiling mechanism, and it is natural to consider
the two tilings isomorphic. More formally,
two crystallographic groups $\group$ and $\group'$ are \kword{isomorphic} if there is
an orientation-preserving, invertible, and affine (but not necessarily isometric) map
${\gamma:\mathbb{R}^n\rightarrow\mathbb{R}^n}$ such that
${\group'=\gamma\group}$, where ${\gamma\group:=\braces{\gamma\phi\gamma^{-1}\,|\,\phi\in\group}}$.
\begin{fact}[{\citep[][4.2.2]{Thurston:1997}}]
  Up to isomorphy, there are only finitely many crystallographic groups on $\mathbb{R}^n$ for each ${n\in\mathbb{N}}$.
  Specifically, there 17 such groups for ${n=2}$, and 230 for ${n=3}$.
\end{fact}

\section{Preliminaries: Invariant functions}
\label{sec:invariant:functions}

\newcommand{\FctG}{\mathbf{F}_{\group}}
\newcommand{\CG}{\mathbf{C}_{\group}}

A function ${f:\mathbb{R}^n\rightarrow\xspace}$, with values in some set $\xspace$, is
\kword{$\phi$-invariant} if it satisfies
\begin{equation*}
  f(\phi x)=f(x)\quad\text{ for all }x\in\mathbb{R}^n\qquad\text{ or in short }\quad f\circ\phi=f\;.
\end{equation*}
It is \kword{$\group$-invariant} if it is $\phi$-invariant
for all ${\phi\in\group}$.
We are specifically interested in $\group$-invariant functions that
are continuous, and write
\begin{equation*}
  \C(M):=\braces{f:\mathbb{R}^n\rightarrow\mathbb{R}\,|\,\text{$f$ continuous}}
  \quad\text{ and }\quad
  \CG:=\braces{f\in\C(\mathbb{R}^n)\,|\,
    f\text{ is $\group$-invariant}}\;.
\end{equation*}
More generally, a function
${f:(\mathbb{R}^n)^k\rightarrow\xspace}$ is \kword{$\group$-invariant in each argument} if
\begin{equation}
  \label{eq:invariant:k:arguments}
  f(\phi_1x_1,\ldots,\phi_kx_k)\;=\;f(x_1,\ldots,x_k)\quad\text{ for all }\phi_1,\ldots\phi_k\in\group\text{ and }x_1,\ldots,x_k\in\mathbb{R}^n\;.
\end{equation}

\subsection{Tiling with functions}

To construct a $\group$-invariant function, we may start with a function $h$ on $\Pi$
and ``replicate it by tiling''. For that to be possible,
$h$ must in turn be the restriction of a $\group$-invariant function to
$\Pi$. It must then satisfy ${h(\phi x)=h(x)}$ if both~$\phi x$ and~$x$ are in $\Pi$.
We hence define the relation
\begin{equation*}
  x\;\sim\;y
  \quad:\Longleftrightarrow\quad
  x,y\in\cF\text{ and }y=\phi(x)\text{ for some }\phi\in\group\setminus\braces{\Id}\;.
\end{equation*}
We note immediately that ${x\sim y}$ implies each point is also contained in
an adjacent tile, so both must be on the boundary $\partial\Pi$ of $\Pi$.
The requirement
\begin{equation}
  \label{pbc}
  h(x)\;=\;h(y) \qquad\text{ whenever }x\sim y
\end{equation}
is therefore a \kword{periodic boundary condition}.
If it holds, the function
\begin{equation}
  \label{eq:tiled:function}
  f(x)\;:=\;h(\phi^{-1} x)\qquad\text{ for }x\in\phi(\cF)\text{ and each }\phi\in\group
\end{equation}
is well-defined on $\mathbb{R}^n$, and is $\group$-invariant.
Conversely, every $\group$-invariant function $f$ can be obtained this way (by choosing $h$ as the
restriction $f|_{\Pi}$).
Informally, \eqref{eq:tiled:function} says that we stitch together function segments on tiles that are all copies of
$h$, and these segments overlap on the tile boundaries. The boundary condition ensures
that wherever such overlaps occur, the segments have the same value, so that
\eqref{eq:tiled:function} produces no ambiguities.
The special case of \eqref{pbc} for pure shift groups---where
$A_\phi$ is the identity matrix for all ${\phi\in\group}$---is known as a
\kword{Born-von Karman boundary condition} (e.g., \citet{Ashcroft:Mermin:1976}).

\subsection{Orbits and quotients}

An alternative way to express invariance is as follows: A function is $\group$-invariant
if and only if it is constant on each set of the form
\begin{equation*}
  \group(x)\;:=\;\braces{\phi x \,|\,\phi\in\group}
  \qquad\text{ for each }x\in\mathbb{R}^n\;.
\end{equation*}
The set $\group(x)$ is called the \kword{orbit} of $x$.
We see immediately that each orbit of a crystallographic group is countably infinite, but
locally finite: The definition of discreteness in~\eqref{discrete}
implies that every bounded subset of $\mathbb{R}^n$
contains only finitely many points of each orbit.
We also see that each point ${x\in\mathbb{R}^n}$ is in one and only one orbit, which means
the orbits form a partition of $\mathbb{R}^n$. The assignment ${x\mapsto\group(x)}$ is
hence a well-defined map
\begin{equation*}
  q(x)\;:=\;\group(x)
  \quad\text{ with image }\quad
  \mathbb{R}^n/\group\;:=\;q(\mathbb{R}^n)\;=\;\braces{\group(x)\,|\,x\in\mathbb{R}^n}\;.
\end{equation*}
The orbit set $\mathbb{R}^n/\group$ is also called the \kword{quotient set} or just the \kword{quotient} of
$\group$, and $q$ is called the \kword{quotient map} (e.g., \citet{Bonahon:2009}).
Since the orbits are mutually disjoint, we can informally think of $q$ as collapsing each
orbit into a single point, and $\mathbb{R}^n/\group$ is the set of such points.

\begin{figure}[b]
  \caption{\emph{Left}: A point $x$ and points on its orbit $\group(x)$ in a shift tiling of $\mathbb{R}^2$.
    \emph{Right}: The infimum in the definition of $d_\group(\group(x),\group(y))$ may be attained by the
    Euclidean distance $d_n(x,y)$ between a point $y$
    in $\Pi$ and a point $\phi x$ in a different tile~$\phi\Pi$.}
  \vspace{-.7cm}
  \begin{center}
    \begin{tikzpicture}
      \node at (0,0) {\includegraphics{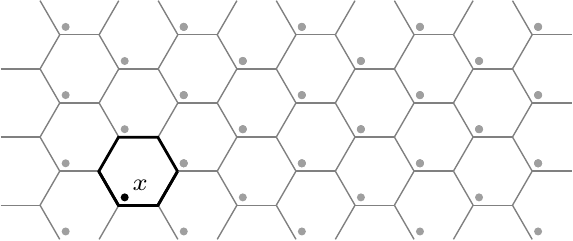}};
      \node at (7,0) {\includegraphics{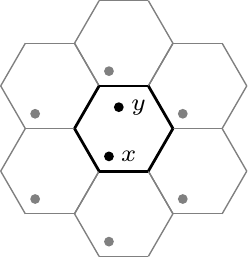}};
    \end{tikzpicture}
  \end{center}
  \label{fig:orbits}
\end{figure}

Quotient spaces are abstract but useful tools for expressing invariance properties:
For any function ${f:\mathbb{R}^n\rightarrow\mathbb{R}}$, we have
\begin{equation}
  \label{eq:bijection:f:quotient}
  f\text{ is $\group$-invariant }
  \quad\Longleftrightarrow\quad
  f=\hat{f}\circ q
  \quad\text{ for some function }\hat{f}:\mathbb{R}^n/\group\rightarrow\mathbb{R}\;,
\end{equation}
since each point of ${\mathbb{R}^n/\group}$ represents an orbit and $f$ is invariant iff it is constant on orbits.
We can also use the quotient to express continuity, by equipping it with a topology that satisfies
\begin{equation}
  \label{eq:bijection:f:quotient:continuous}
  f\in\C_{\group}
  \quad\Longleftrightarrow\quad
  f=\hat{h}\circ q
  \quad\text{ for some continuous }\hat{h}:\mathbb{R}^n/\group\rightarrow\mathbb{R}\;.
\end{equation}
There is exactly one such topology, called the \kword{quotient topology} in the literature. Its definition
can be made more concrete by metrizing it:
\begin{fact}[see \citet{Bonahon:2009}, Theorem 7.7]
  \label{fact:metric}
  If $\group$ is crystallographic, the function
  \begin{equation*}
    d_\group(\omega_1,\omega_2)\,:=\,\inf\braces{d(x,y)\,|\,{x\in \omega_1,y\in \omega_2}}
    \quad\text{ for }\omega_1,\omega_2\in\mathbb{R}^n/\group
  \end{equation*}
  is a valid metric on ${\mathbb{R}^n/\group}$, and it metrizes the quotient
  topology. A subset ${U\subset\mathbb{R}^n/\group}$ is open if and only
  if its preimage ${q^{-1}U}$ is open in $\mathbb{R}^n$.
\end{fact}
Since $\group$ is discrete, the infimum in $d_\group$ is a minimum.
The distance of two orbits (considered as points in $\mathbb{R}^n/\group$)
is hence the shortest Euclidean distance between points in these orbits (considered as sets
in $\mathbb{R}^n$), see \cref{fig:orbits} (right).
If $x$ and $y$ are points in the polytope $\Pi$, we have
\begin{equation*}
  x\sim y\quad\Longleftrightarrow\quad d_{\group}(\group(x),\group(y))\;=\;0\;.
\end{equation*}
Informally speaking, $d_\group$ implements the periodic boundary condition \eqref{pbc}.
The metric space $(\mathbb{R}^n/\group,d_\group)$ is also called the \kword{quotient space} or
\kword{orbit space} of $\group$. A very important property of crystallographic groups is that they
have compact quotient spaces:\nolinebreak
\begin{fact}[{\citep[][Proposition 1.6]{Vinberg:Shvarsman:1993}}]
  \label{fact:compact:quotient}
  If a discrete group $\group$ of isometries tiles $\mathbb{R}^n$ with a set~$M$,
  the quotient space $(\mathbb{R}^n/\group,d_\group)$ is homeomorphic to the quotient space~$M/\group$.
  If $\group$ is crystallographic and tiles
  with a convex polytope, then $(\mathbb{R}^n/\group,d_\group)$ is compact.
\end{fact}

\subsection{Transversals and projections}

Since orbit spaces are abstract objects, we can only work with them implicitly.
One way to do so is by representing each orbit by one of its points in $\mathbb{R}^n$.
A subset of $\mathbb{R}^n$ that contains exactly one point of each orbit is called a
\kword{transversal}. In general, transversals can be exceedingly complex sets \citep{Becker:Kechris:1996},
but crystallographic
groups always have simple transversals. \cref{algorithm:transversal} in the next section
constructs a transversal explicitly.
In the following, we will always write $\tilde{\Pi}$ to mean
\begin{equation*}
  \tilde{\Pi}\;:=\;
  \text{ a transversal contained in $\Pi$ computed by \cref{algorithm:transversal}.}
\end{equation*}
Given such a transversal, we can define the \kword{projector} ${p:\mathbb{R}^n\rightarrow\Pi}$ as
\begin{equation}
  \label{eq:projector}
  p(x)\;:=\;\text{ the unique element of }\group(x)\cap\widetilde{\Pi}\;.
\end{equation}
If we think of each point in $\tilde{\Pi}$ as a concrete representative of an element of ${\mathbb{R}^n/\group}$,
then~$p$ is similarly a concrete representation of the quotient map $q$, and we can translate
the identities above accordingly:
The projector is by definition $\group$-invariant,
since we can write~$f$ in \eqref{eq:tiled:function} as ${f=h\circ p}$. That shows
\begin{equation}
  \label{eq:bijection:Rn:Pi}
  f:\mathbb{R}^n\rightarrow\mathbb{R}\text{ is $\group$-invariant }
  \;\;\Longleftrightarrow\;\;
  f\;=\;h\circ p \text{ for some }h:\cF\rightarrow\mathbb{R}\text{ satisfying }\eqref{pbc}\;.
\end{equation}
Although $p$ is not continuous as a function ${\mathbb{R}^n\rightarrow\Pi}$, continuity only fails at the boundary, and
$p$ behaves like a continuous function when composed with $h$:
\begin{lemma}
  \label{lemma:projector}
  Let ${h:\Pi\rightarrow\xspace}$ be a continuous function with values in a topological space $\xspace$.
  If $h$ satisfies \eqref{pbc}, then ${h\circ p}$ is a continuous $\group$-invariant function
  ${\mathbb{R}^n\rightarrow\xspace}$. It follows that
  \begin{equation}
  \label{eq:projector:composition}
  f\in\C_{\group}
  \;\;\Longleftrightarrow\;\;
  f\;=\;h\circ p \quad\text{ for some continuous }h:\cF\rightarrow\mathbb{R}\text{ satisfying }\eqref{pbc}\;.
\end{equation}
\end{lemma}
\noindent Since $p$ exists for any choice of $\group$ and $\Pi$, and since it can be evaluated algorithmically,
we have hence reduced the problem of constructing continuous invariant functions to the problem
of finding functions that satisfy the periodic boundary condition \eqref{pbc}.

\section{Taking quotients algorithmically: Orbit graphs}
\label{sec:orbit:graphs}

\begin{figure}[t]
  \begin{center}
    \begin{tikzpicture}
      \begin{scope}
      \node at (0,0) {\includegraphics[width=2.8cm]{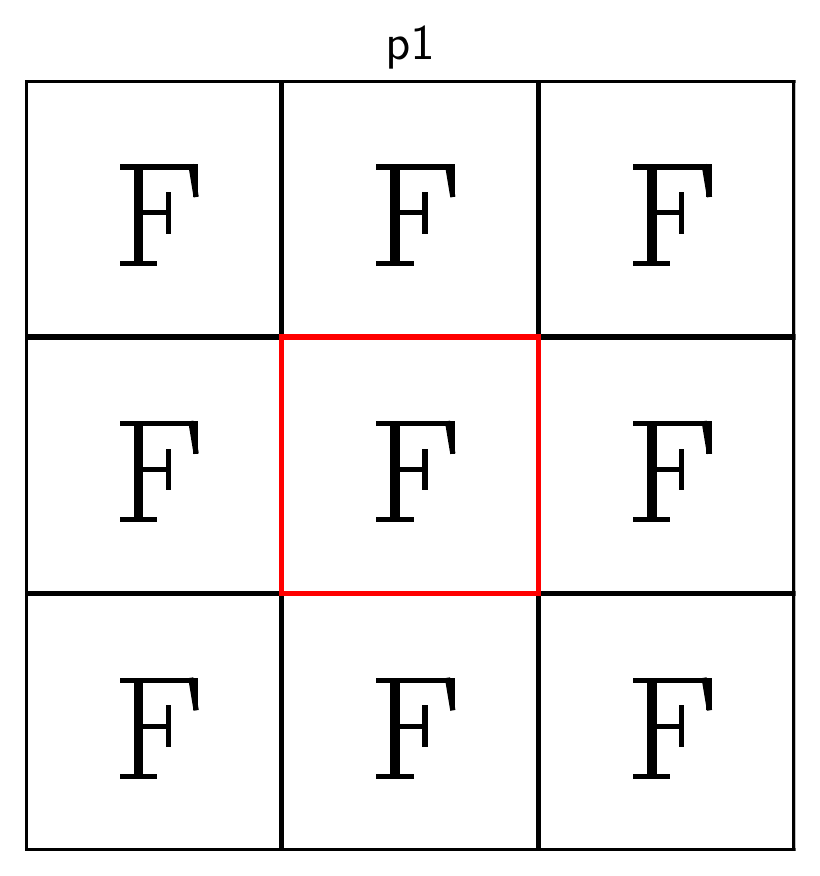}};
      \node at (4,0) {\includegraphics[width=2.8cm]{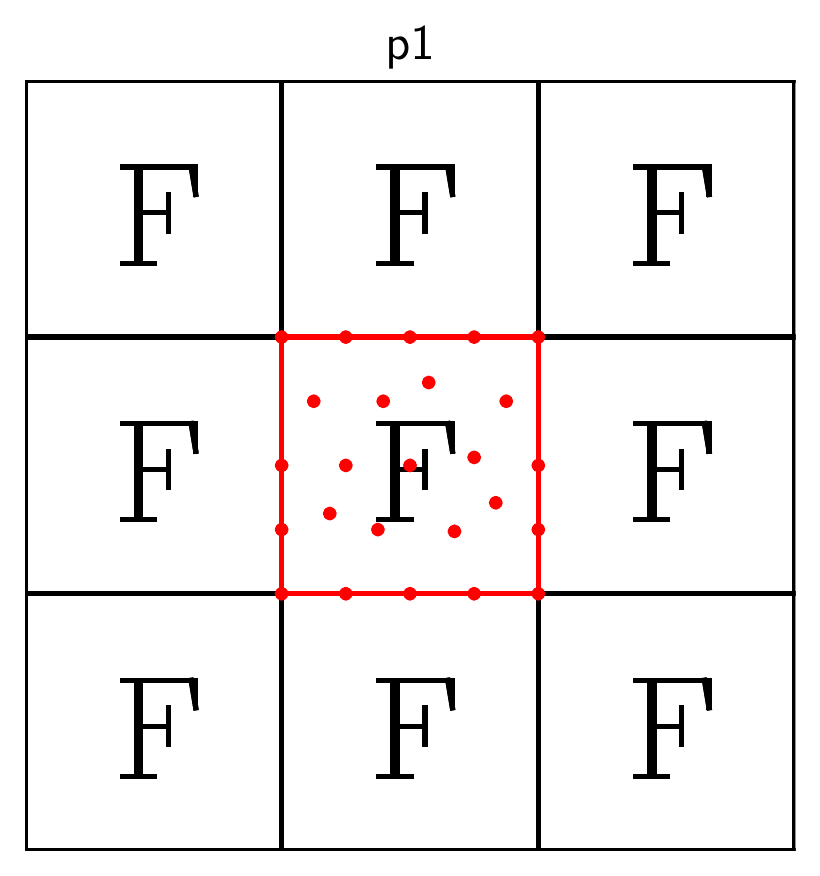}};
      \node at (8,0) {\includegraphics[width=2.8cm]{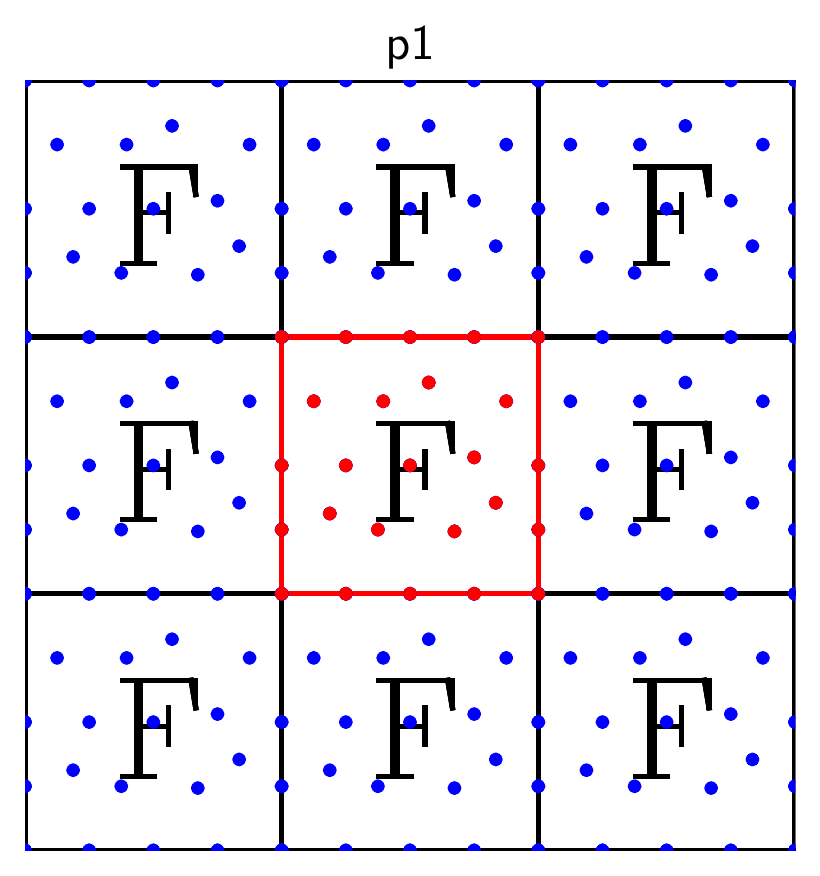}};
      \node at (12,0) {\includegraphics[width=2.8cm]{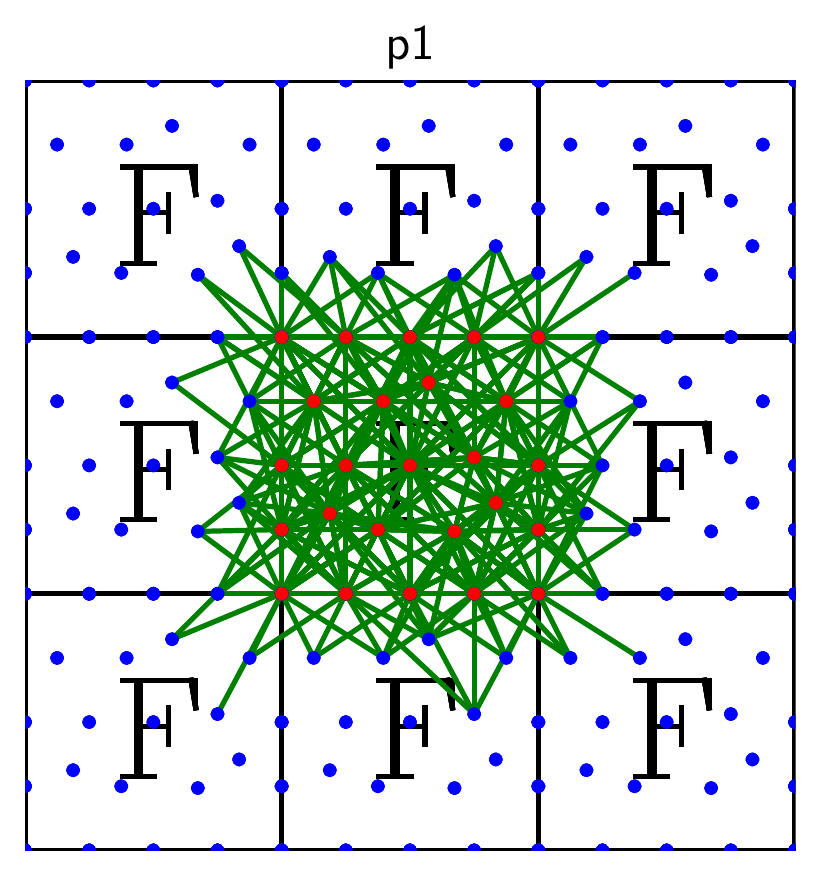}};
      \end{scope}
      \begin{scope}[yshift=-1.8cm]
        \node at (0,0) {\scriptsize $\Pi$ within the tiling};
        \node at (4,0) {\scriptsize construct net};
        \node at (8,0) {\scriptsize apply elements of $\mathcal{A}_\Pi$};
        \node at (12,0) {\scriptsize orbit graph};
      \end{scope}
      \begin{scope}[yshift=-3.6cm]
      \node at (0,0) {\includegraphics[width=2.8cm]{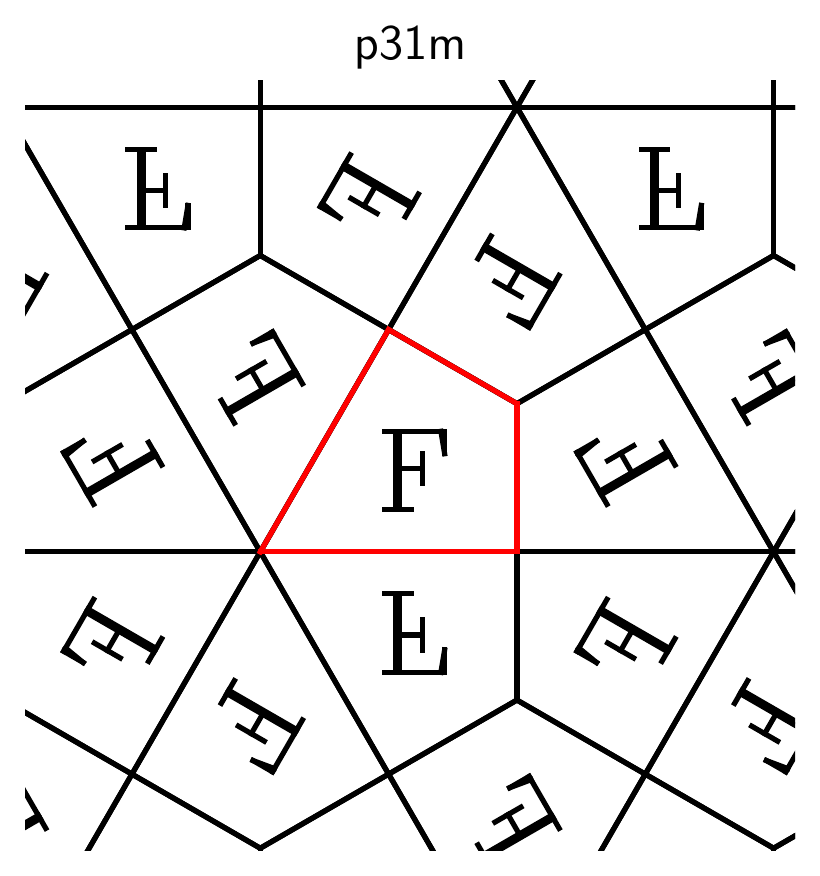}};
      \node at (4,0) {\includegraphics[width=2.8cm]{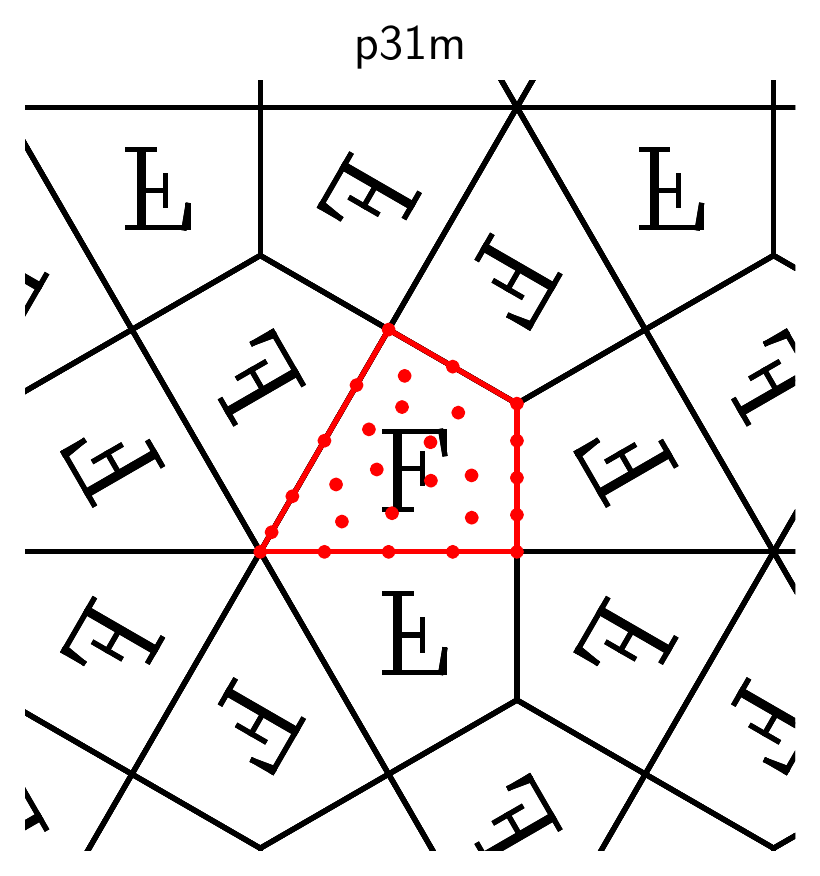}};
      \node at (8,0) {\includegraphics[width=2.8cm]{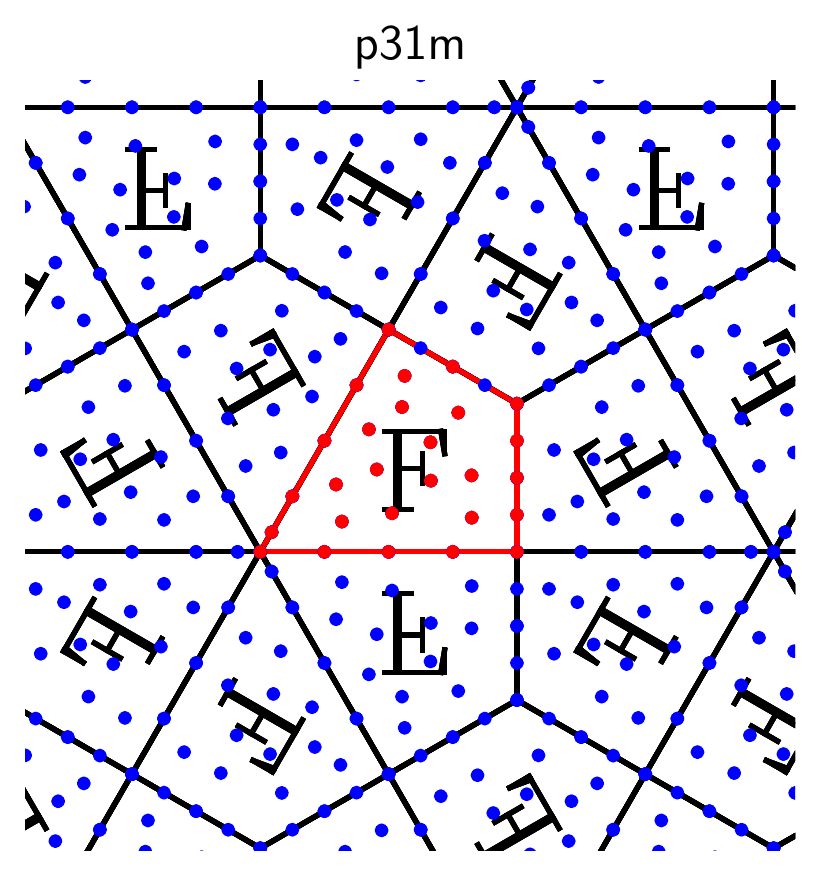}};
      \node at (12,0) {\includegraphics[width=2.8cm]{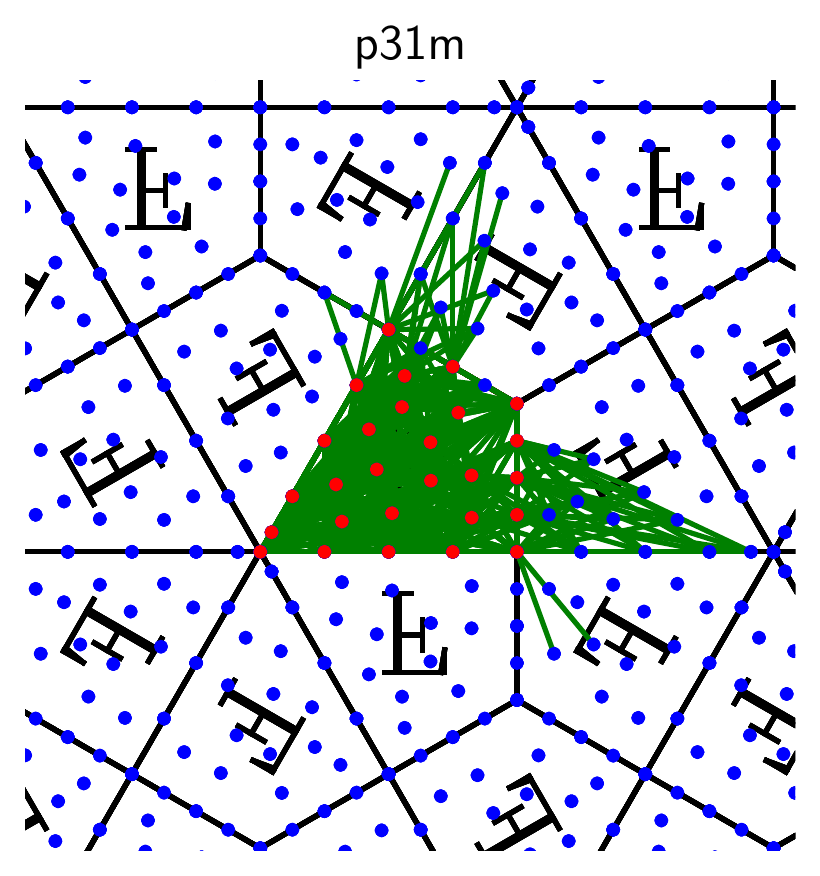}};
      \end{scope}
    \end{tikzpicture}
  \end{center}
  \vspace{-.5cm}
  \caption{Orbit graph construction for two plane groups: \texttt{p1} and \texttt{p31m}. Starting with the fundamental region (left), place points to construct an~$\epsilon$-net (middle left),
    apply local group operations to these points (middle right), then add edges which may include vertices outside the fundamental region (right). }
  \label{fig:orbit:graphs}
\end{figure}
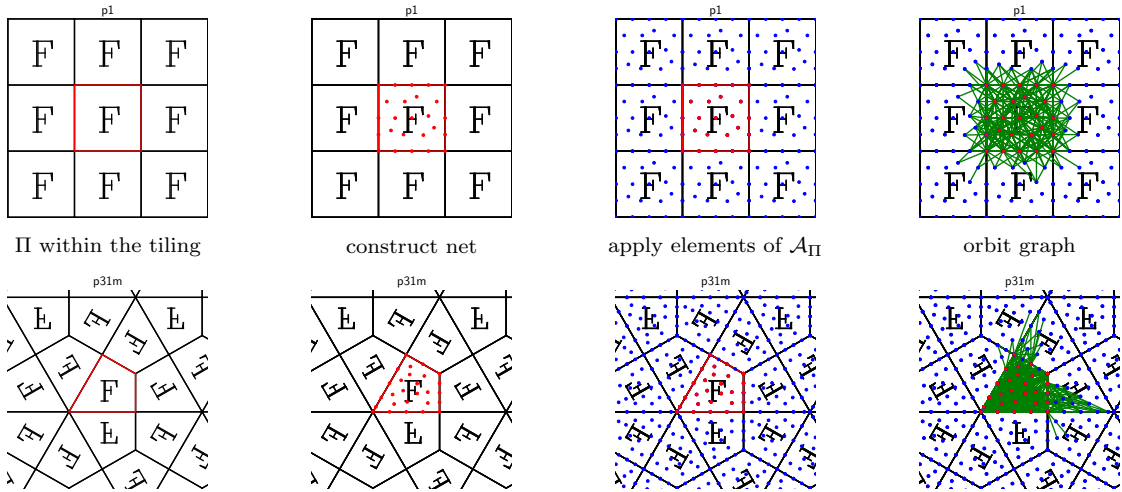
To work with invariant functions computationally, we
must approximate the quotient metric. We do so using a data structure that we call an orbit graph,
in which two points are connected if their orbits are close to each other. More formally, any undirected graph
is a metric space when equipped with path length as distance. The metric space defined by the graph
$\mathcal{G}$ below discretizes the metric space ${(\mathbb{R}^n/\group,d_\group)}$.
To define $\mathcal{G}$, fix constants~${\varepsilon,\delta>0}$.
A finite set $\Gamma$ is an \kword{$\varepsilon$-net} in $\Pi$ if each point lies within distance~$\varepsilon$ of~$\Gamma$,
\begin{equation*}
  \Gamma\text{ is an $\varepsilon$-net }
  \quad:\Leftrightarrow\quad
  \text{ for each }x\in\Pi\text{ there exists }z\in\Gamma\text{ such that }d_n(x,z)\,<\,\varepsilon\;,
\end{equation*}
see e.g., \citet{Cooper:Hodgson:Kerckhoff:2000}.
If $\Gamma$ is an $\varepsilon$-net in $\Pi$, we call the graph
\begin{equation*}
  \mathcal{G}=\mathcal{G}(\varepsilon,\delta)=(\Gamma,E)
  \quad\text{ where }\quad E:=\braces{(x,y)\,|\,d_\group(\group(x),\group(y)<\delta}
\end{equation*}
an \kword{orbit graph} for $\group$ and $\Pi$.

\subsection{Computing orbit graphs}

Algorithmically, an orbit graph can be constructed as follows: Constructing an $\varepsilon$-net is a standard problem in
computational geometry and can be solved efficiently (e.g., \citet{Haussler:Welzl:1987}).
Having done so, the problem we have to solve is:
\begin{equation*}
  \text{Given points }x,y\in\Pi\text{, determine }d_\group(\group(x),\group(y))\;.
\end{equation*}
Since $\Pi$ is a polytope, its diameter
\begin{equation*}
  \text{diam}(\Pi)\;:=\;\max\braces{d_n(x,y)\,|\,x,y\in\Pi}\;<\;\infty
\end{equation*}
can also be evaluated computationally. By definition of $d_\group$, we have
\begin{equation*}
  d_\group(\group(x),\group(y))
  \;=\;
  \min_{\phi\in\group}d_n(x,\phi y)
  \;\leq\;
  d_n(x,y)
  \;\leq\;
  \text{diam}(\Pi)\;.
\end{equation*}
That shows the minimum is always attained for a point $\phi y$ on a tile $\phi\Pi$ that lies within
distance $\text{diam}(\Pi)$ of $x$. The set of transformations that specify these tiles is
\begin{equation*}
  \mathcal{A}_\Pi
  \;=\;
  \braces{\phi\in\group \,|\,d_n(x,\phi z)\leq\text{diam}(\Pi)\text{ for some }z\in\Pi}\;.
\end{equation*}
This set is always finite, since $\group$ is discrete and the ball of radius $\text{diam}(\Pi)$ is compact.
We can hence evaluate the quotient metric as
\begin{equation*}
  d_\group(\group(x),\group(y))
  \;=\;
  \min\braces{d_n(x,\phi y)\,|\,\phi\in\mathcal{A}_\Pi}\,,
\end{equation*}
which reduces the construction of $E$ to a finite search problem. In summary:
\begin{algorithm}[Constructing the orbit graph]{\ }\\[.5em]
  \begin{tabular}{cl}
  1.) & Construct the~$\varepsilon$-net~$\Gamma$.\\
  2.) & Find local group elements~$\mathcal{A}_\Pi$.\\
  3.) & For each pair~$x,y\in\Gamma$, find $d_\group(\group(x),\group(y))=\min\braces{d_n(x,\phi y)\,|\,\phi\in\mathcal{A}_\Pi}$.\\
  4.) & Add an edge between~$x$ and~$y$ if~$d_\group(\group(x),\group(y)) < \delta$.
  \end{tabular}
\end{algorithm}
The construction is illustrated in \cref{fig:orbit:graphs}.

\subsection{Computing a transversal}
\label{sec:proof:transversal}

Recall that the \kword{faces} of a polytope are its vertices, edges, and so forth; the
facets are the ${(n-1)}$-dimensional faces. The polytope itself is also a face, of dimension $n$.
See \citep{Ziegler:1995} for a precise definition. Given $\Pi$ and $\group$, we will call two
faces $S$ and $S'$ $\group$-equivalent if ${S'=\phi S}$ for some ${\phi\in\group}$.
Thus, if ${S=\Pi}$, its equivalence class is $\braces{\Pi}$.
If $S$ is a facet, it is equivalent to at most one distinct facet, so its
equivalence class has one or two elements. The equivalence classes
of lower-dimensional faces may be larger---if $\group$ is \texttt{p1} and $\Pi$ a square,
for example, all four vertices of $\Pi$ are $\group$-equivalent.
\begin{algorithm}[Constructing a transversal]{\ }\\[.5em]
  \label{algorithm:transversal}
  \begin{tabular}{rl}
    1) & Start with an exact tiling. Enumerate all faces of $\Pi$.\\
    2) & Sort faces into $\group$-equivalence classes.\\
    3) & Select one face from each class and take its relative interior.\\
    4) & Output the union $\tilde{\Pi}$ of these relative interiors.
  \end{tabular}
\end{algorithm}
\vspace{.5em}
\begin{lemma}
  The set $\tilde{\Pi}$ constructed by \cref{algorithm:transversal} is a transversal.
\end{lemma}

  \begin{proof}
    The relative interiors of the faces of a convex polytope are mutually disjoint and
    their union is $\Pi$, so each point ${x\in\Pi}$ is on exactly one such relative interior.
    Let $S$ be the face with ${x\in S^{\circ}}$, and consider any ${\phi\in\group}$.
    Since the tiling is exact, ${\phi S}$ is either a face of $\Pi$ or ${\phi S\cap\Pi=\varnothing}$.
    If ${\phi x\in\Pi}$, the intersection cannot be empty, so $\phi S$ is a face and hence
    $\group$-equivalent to $S$.
    It follows that the interior of a face of $\Pi$ intersect the orbit $\group x$ if and only
    if it is in the equivalence class of $S$.
    Since we select exactly one element of this class, exactly one point of $\group x$
    is contained in $\tilde{\Pi}$.
  \end{proof}

\subsection{Computing the projector}

Since $\group$ is crystallographic, it contains shifts in $n$ linearly independent directions, and these shifts hence specify
a coordinate system of $\mathbb{R}^n$. More precisely: There are $n$ elements~${\phi_1,\ldots,\phi_n}$ of
$\group$ that (1) are pure shifts (satisfy ${A_{\phi_i}=\mathbb{I}}$), (2) are linearly independent, and (3) are the shortest
such elements (in terms of the Euclidean norm of~$b_{\phi_i}$). Up to a sign, each of these elements is uniquely determined.
We refer to the vectors~${\phi_1,\ldots,\phi_n}$ as the \kword{shift coordinate} system of~$\group$.
\begin{algorithm}[Computing the projector]{\ }\\[.5em]
  \label{algorithm:projector}
  \begin{tabular}{cl}
  1.) & Perform a basis change from the shift coordinates to the standard basis of $\mathbb{R}^n$.\\
  2.) & Set ${\tilde{x}=(x_1\mod 1,\ldots,x_n\mod 1)}$.\\
  3.) & Find ${\phi\in\group}$ such that ${\phi\tilde{x}\in\tilde{\Pi}}$.\\
  4.) & Apply the reverse change of basis from standard to shift coordinates.
  \end{tabular}
\end{algorithm}

\section{Linear representation: Invariant Fourier transforms}
\label{sec:linear}

In this section, we obtain a basis representation for invariant functions: given a
crystallographic group~$\group$, we construct a sequence of $\group$-invariant functions
${\ef_1,\ef_2,\ldots}$ on $\mathbb{R}^n$ such that any $\group$-invariant continuous function
can be represented as a (possibly infinite) linear combination ${\sum_{i\in\mathbb{N}}c_i\ef_i}$.
If $\group$ is generated by $n$ orthogonal shifts, the functions $e_i$
are an $n$-dimensional Fourier basis.
\cref{theorem:spectral} below obtains an analogous basis for each crystallographic group $\group$.

\subsection{Representation theorem}

For any open set ${M\subseteq\mathbb{R}^n}$, we define the Laplace operator on twice differentiable
functions~${h:M\rightarrow\mathbb{R}}$ as
\begin{equation*}
  \Delta h
  \;:=\;
  \mfrac{\partial^2 h}{\partial x_1^2}+\ldots+\mfrac{\partial^2 h}{\partial x_n^2}
  \;=\;
  \nabla^{\trans}(\nabla h)\;.
\end{equation*}
Now consider specifically functions ${e:\mathbb{R}^n\rightarrow\mathbb{R}}$.
Fix some ${\lambda\in\mathbb{R}}$, and consider the constrained partial differential equation
\begin{equation}
  \label{sturm:liouville}
  \begin{aligned}
    -\Delta e&\;=\;\lambda e\quad&&\text{ on }\mathbb{R}^n\\
    \text{subject to }\qquad e&\;=\;e\circ\phi\quad&&\text{ for }\phi\in\group\;.
  \end{aligned}
\end{equation}
Clearly, there is always a trivial solution, namely the constant function ${e=0}$.
If \eqref{sturm:liouville} has a non-trivial solution $e$, we call this $e$ a \kword{$\group$-eigenfunction} and
$\lambda$ a \kword{$\group$-eigenvalue}
of the linear operator $-\Delta$. Denote the set of solutions by
\begin{equation*}
  \mathcal{V}(\lambda)
  \;:=\;
  \braces{e:\mathbb{R}^n\rightarrow\mathbb{R}\,|\,e\text{ satisfies }\eqref{sturm:liouville}}\;.
\end{equation*}
Since $0$ is a solution, and any linear combination of solutions is again a solution,
${\mathcal{V}(\lambda)}$ is a vector space,
called the \kword{eigenspace} of $\lambda$. Its dimension
\begin{equation*}
  k(\lambda)\;:=\;\dim\mathcal{V}(\lambda)
\end{equation*}
is the \kword{multiplicity} of $\lambda$.
\begin{theorem}[Crystallographically invariant Fourier basis]
  \label{theorem:spectral}
  Let $\group$ be a crystallographic group that tiles $\mathbb{R}^n$ with
  a convex polytope $\Pi$.
  Then the constrained problem \eqref{sturm:liouville} has solutions for countably many
  distinct values ${\lambda_1,\lambda_2,\ldots}$ of $\lambda$, and these values satisfy
  \begin{equation*}
    0\,=\,\lambda_1\,<\,\lambda_2\,<\,\lambda_3<\ldots
    \qquad\text{ and }\qquad
    \lambda_i\;\xrightarrow{i\rightarrow\infty}\;\infty\;.
  \end{equation*}
  Every solution function $e$ is infinitely often differentiable.
  There is a sequence ${e_1,e_2,\ldots}$ of solutions whose restrictions
  ${e_1|_\Pi,e_2|_\Pi,\ldots}$ to $\Pi$ form an orthonormal basis of the the space $\L_2(\Pi)$,
  and satisfy
  \begin{equation*}
    \big|\braces{\,j\in\mathbb{N}\,|\,e_j\in\mathcal{V}(\lambda_i)}\big|\;=\;k(\lambda_i)
    \qquad\text{ for each }i\in\mathbb{N}\;.
  \end{equation*}
  A function ${f:\mathbb{R}^n\rightarrow\mathbb{R}}$ is $\group$-invariant and continuous if and only if
  \begin{equation*}
    f\;=\;\msum_{i\in\mathbb{N}}c_ie_i
    \qquad\text{ for some sequence }c_1,c_2,\ldots\in\mathbb{R}\;,
  \end{equation*}
  where the series converges in the supremum norm.
\end{theorem}
\begin{proof}
  See \cref{appendix:proofs:spectral}.
\end{proof}
\begin{remark}
  The space ${\L_2(\mathbb{R}^n)}$ contains no non-trivial $\group$-invariant functions, since
  for every ${f\in\C_\group}$
  \begin{equation*}
    \|f\|_{\L_2(\mathbb{R}^n)}\;=\;\tsum_{\phi\in\group}\|\,f|_{\phi\Pi}\,\|_{\L_2(\phi\Pi)}\;=\;
    \begin{cases}
      0 & f=0\text{ almost everywhere}\\
      \infty &\text{otherwise}
    \end{cases}\;.
  \end{equation*}
  On the other hand, the restriction $f|_\Pi$ is in $\L_2(\Pi)$, and completely determines $f$.
  That makes $\L_2(\Pi)$ the natural $\L_2$-space in the context of crystallographic invariance,
  and is the reason why the restrictions $e_i|_\Pi$ are used in the theorem.
  Since~$\L_2(\phi\Pi)$ is isometric to $\L_2(\Pi)$ for all ${\phi\in\group}$, it does not matter which tile we restrict to.
\end{remark}
\begin{figure}
  \begin{tikzpicture}
    \begin{scope}
      \node at (0,0) {\includegraphics[height=2.7cm]{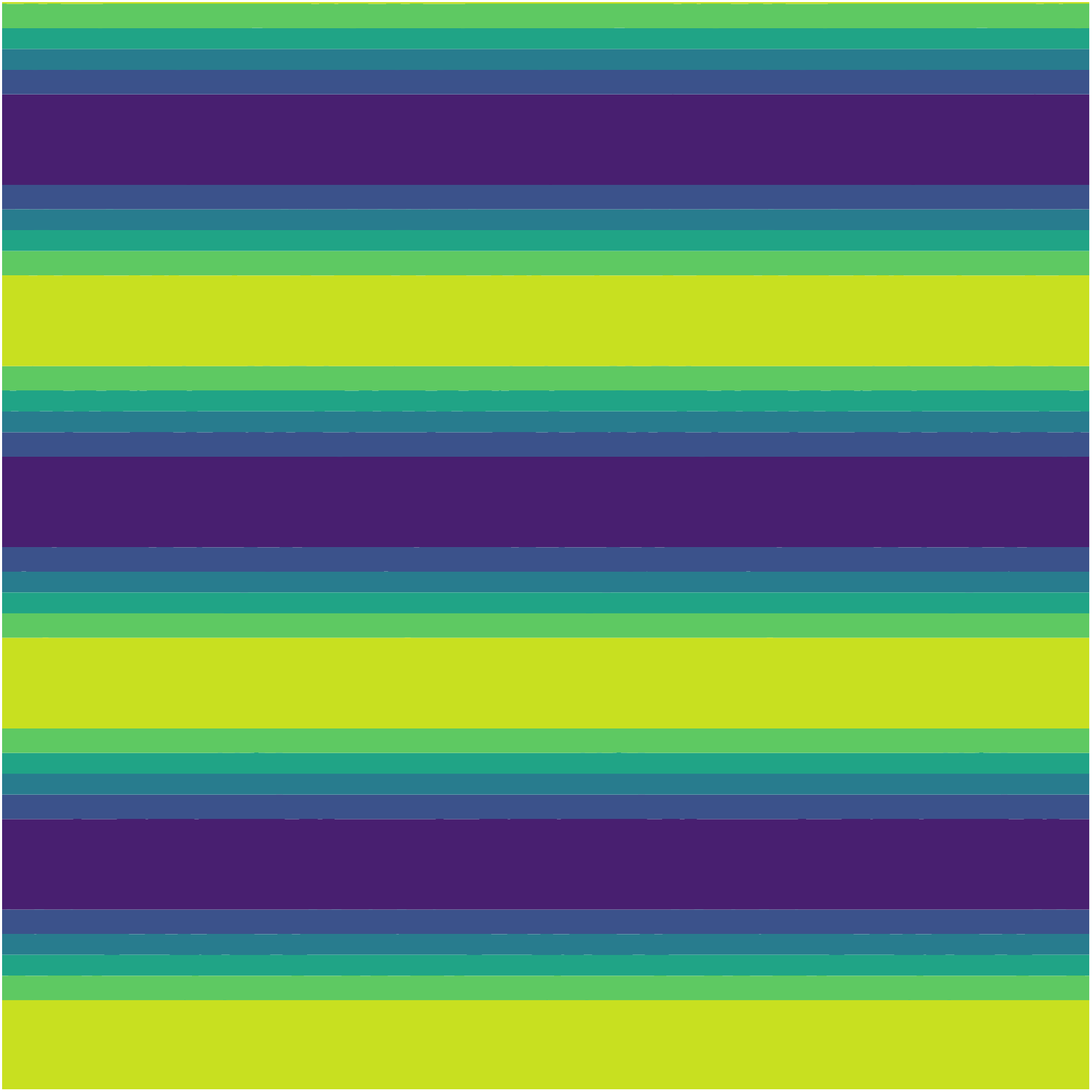}};
      \node at (3,0) {\includegraphics[height=2.7cm]{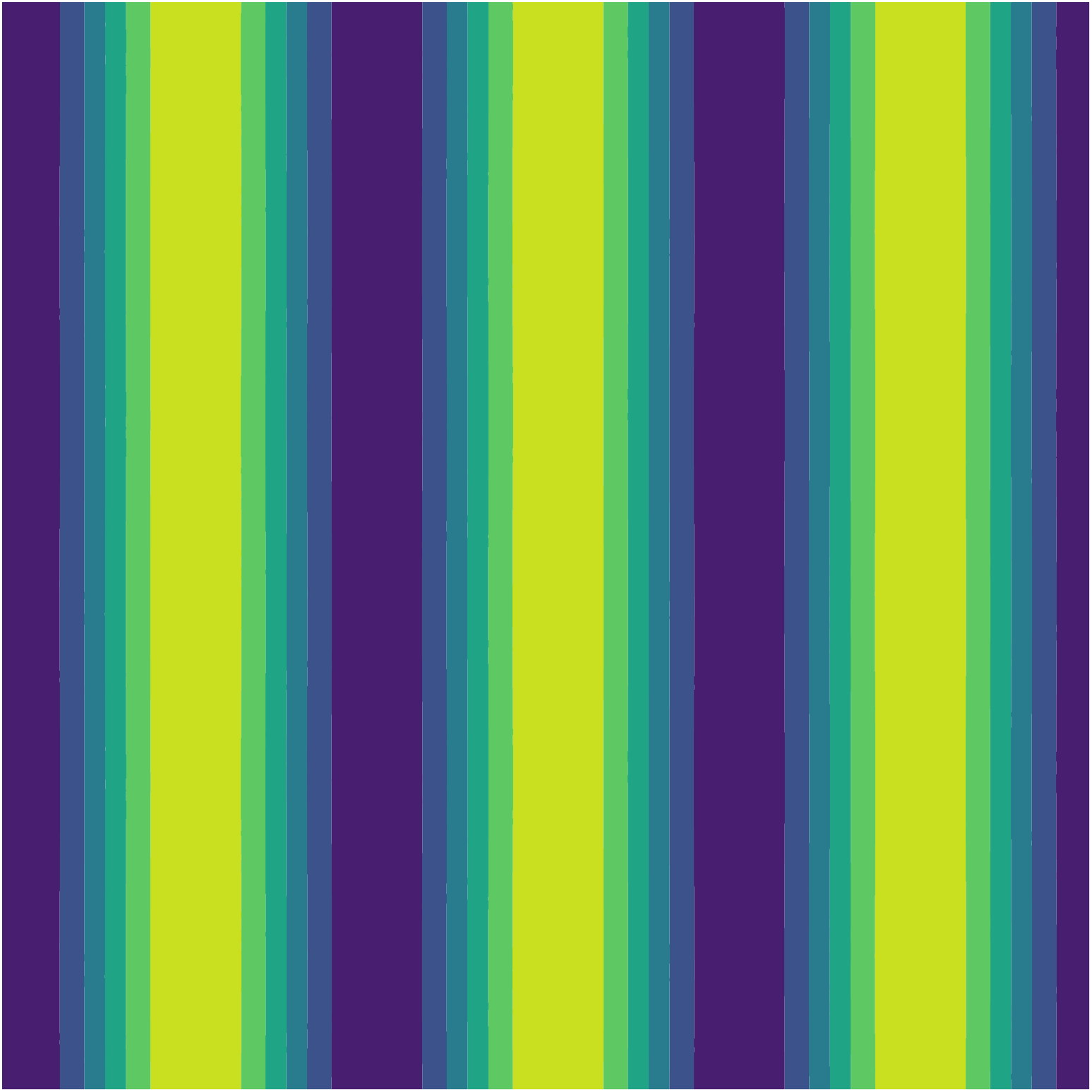}};
      \node at (6,0) {\includegraphics[height=2.7cm]{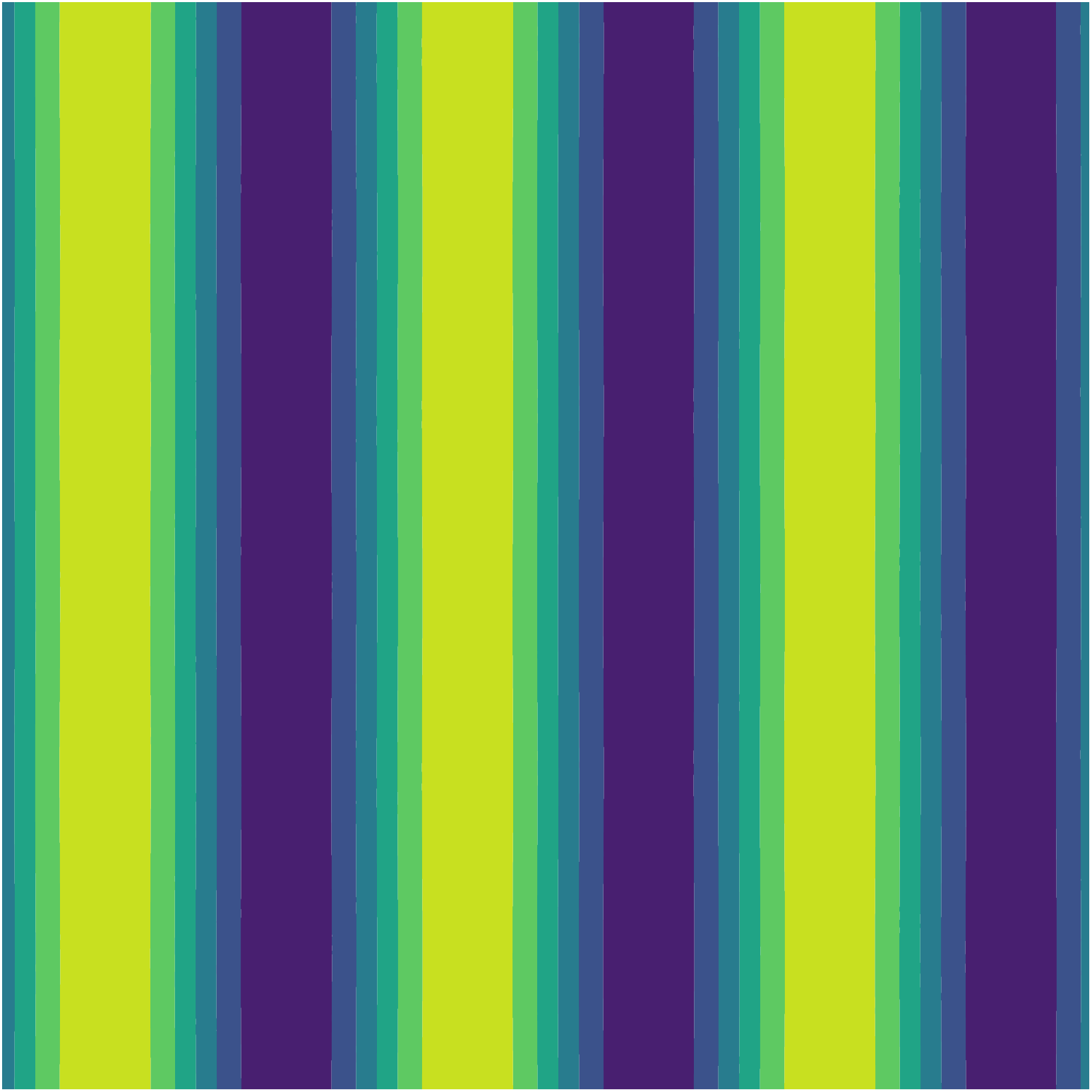}};
      \node at (9,0) {\includegraphics[height=2.7cm]{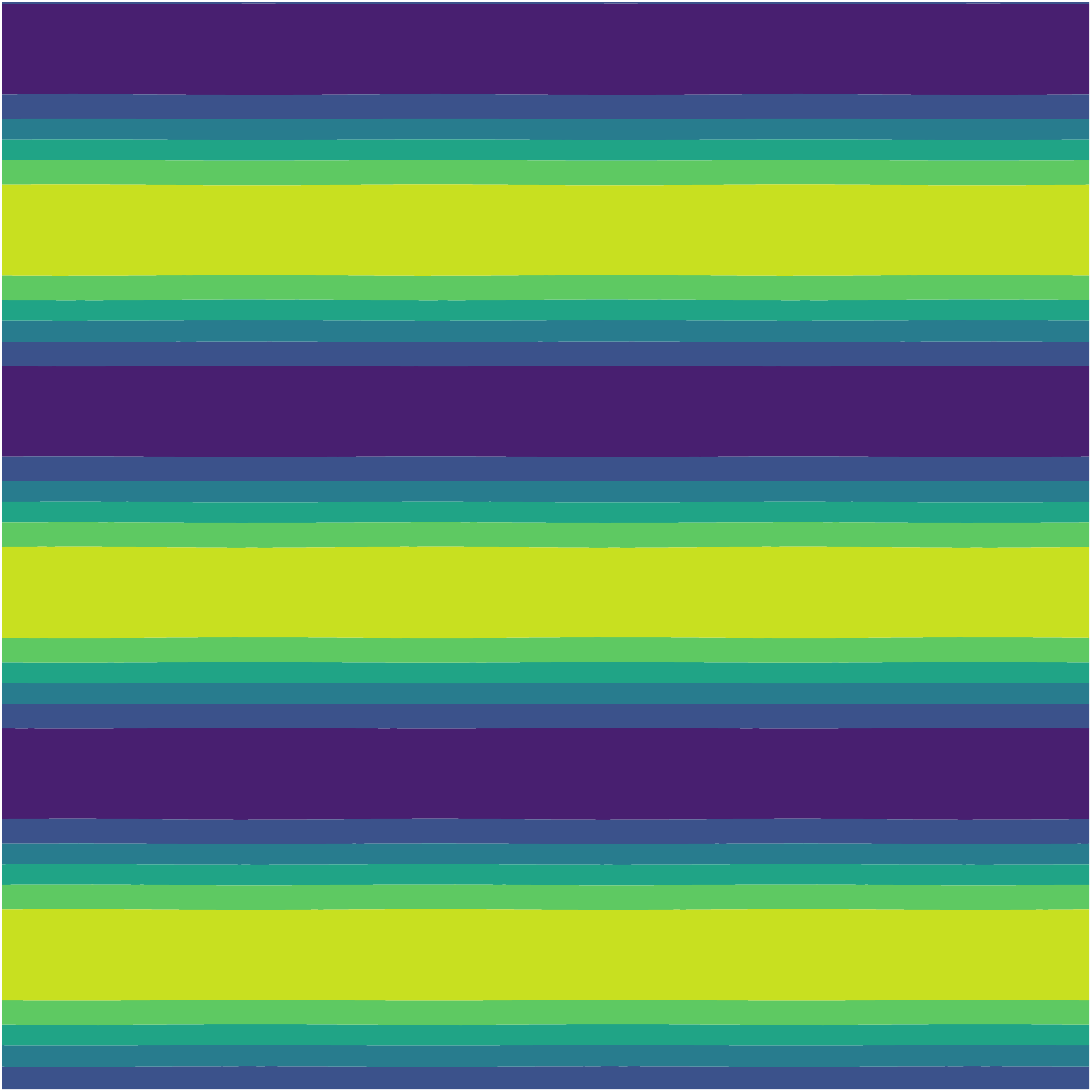}};
      \node at (12,0) {\includegraphics[height=2.7cm]{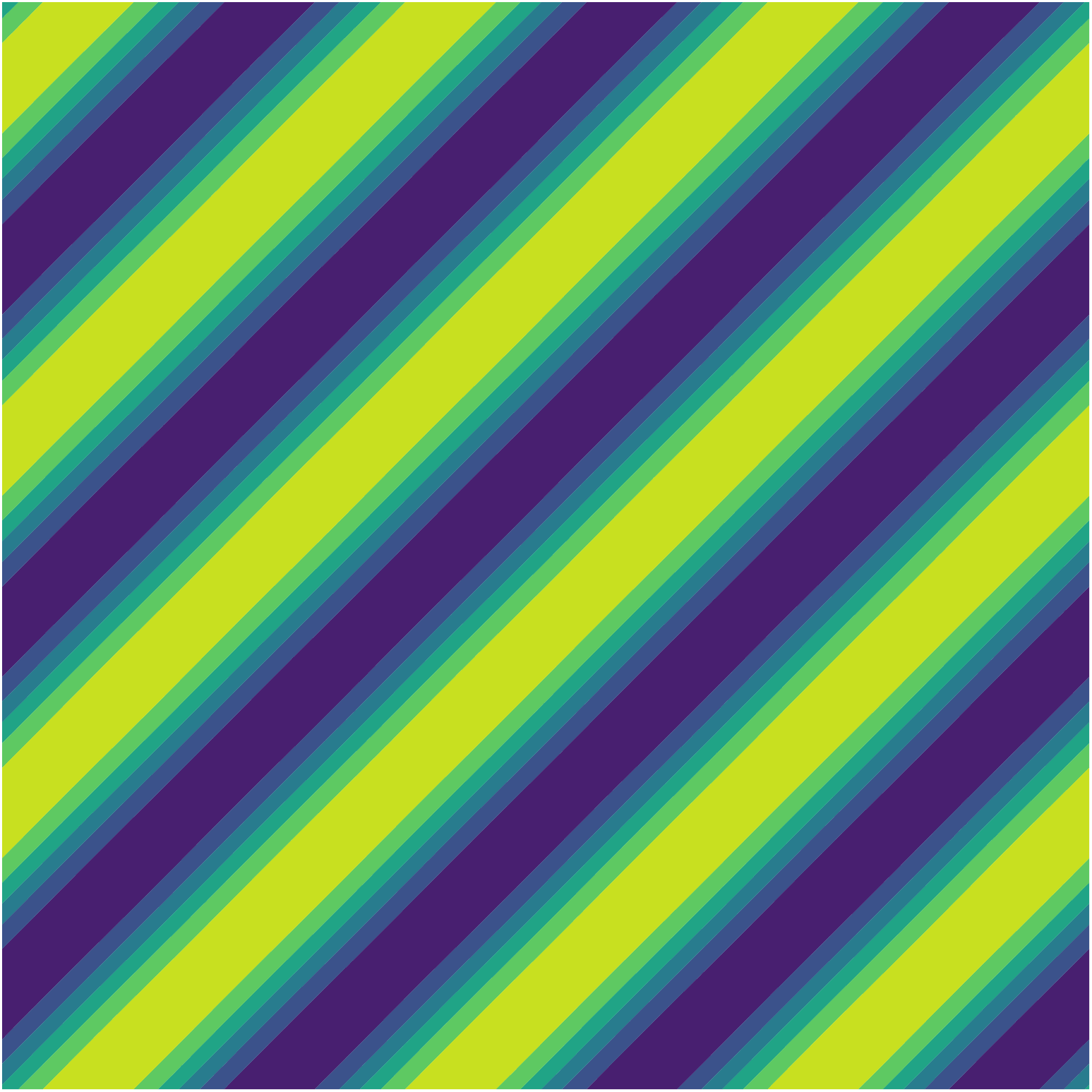}};
      \node at (0,-1.7) {\scriptsize 39.48};
      \node at (3,-1.7) {\scriptsize 39.48};
      \node at (6,-1.7) {\scriptsize 39.48};
      \node at (9,-1.7) {\scriptsize 39.48};
      \node at (12,-1.7) {\scriptsize 78.95};
    \end{scope}
  \end{tikzpicture}
  \caption{The first five non-constant basis functions ${e_2,\ldots,e_{6}}$
    in \cref{theorem:spectral}, with their eigenvalues $\lambda_i$, for the group \texttt{p1}.
    The eigenbasis in this case is precisely the standard Fourier basis on $\mathbb{R}^2$.}
  \label{fig:fourier:p1}
\end{figure}
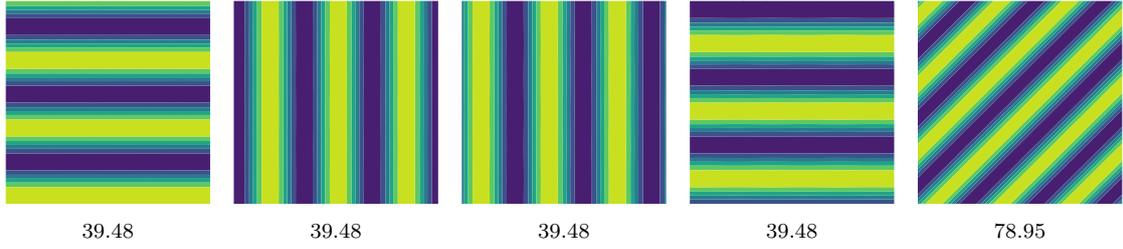

\subsection{Relationship to Fourier series}

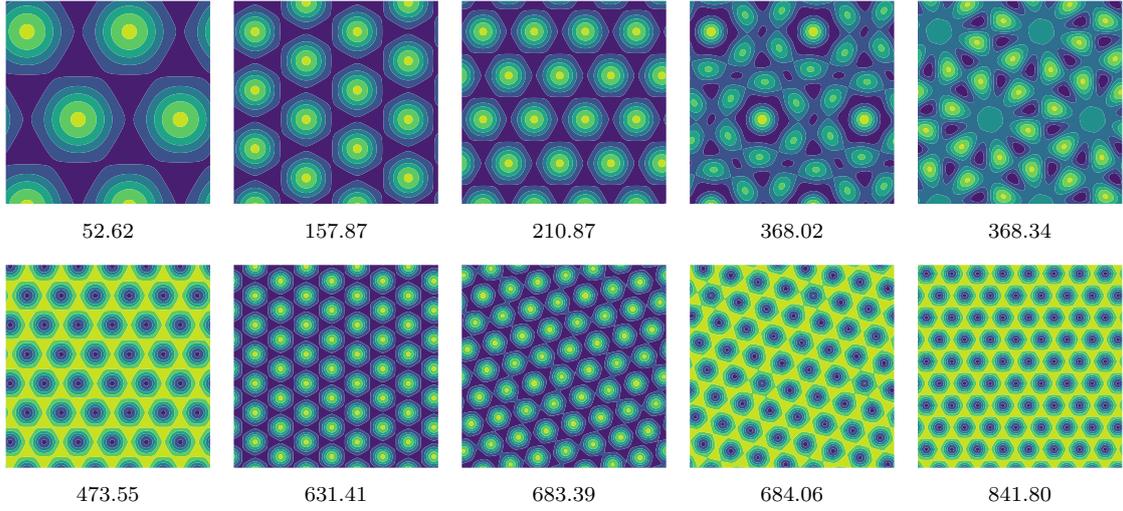
\begin{figure}
  \begin{tikzpicture}
    \begin{scope}
      \node at (0,0) {\includegraphics[height=2.7cm]{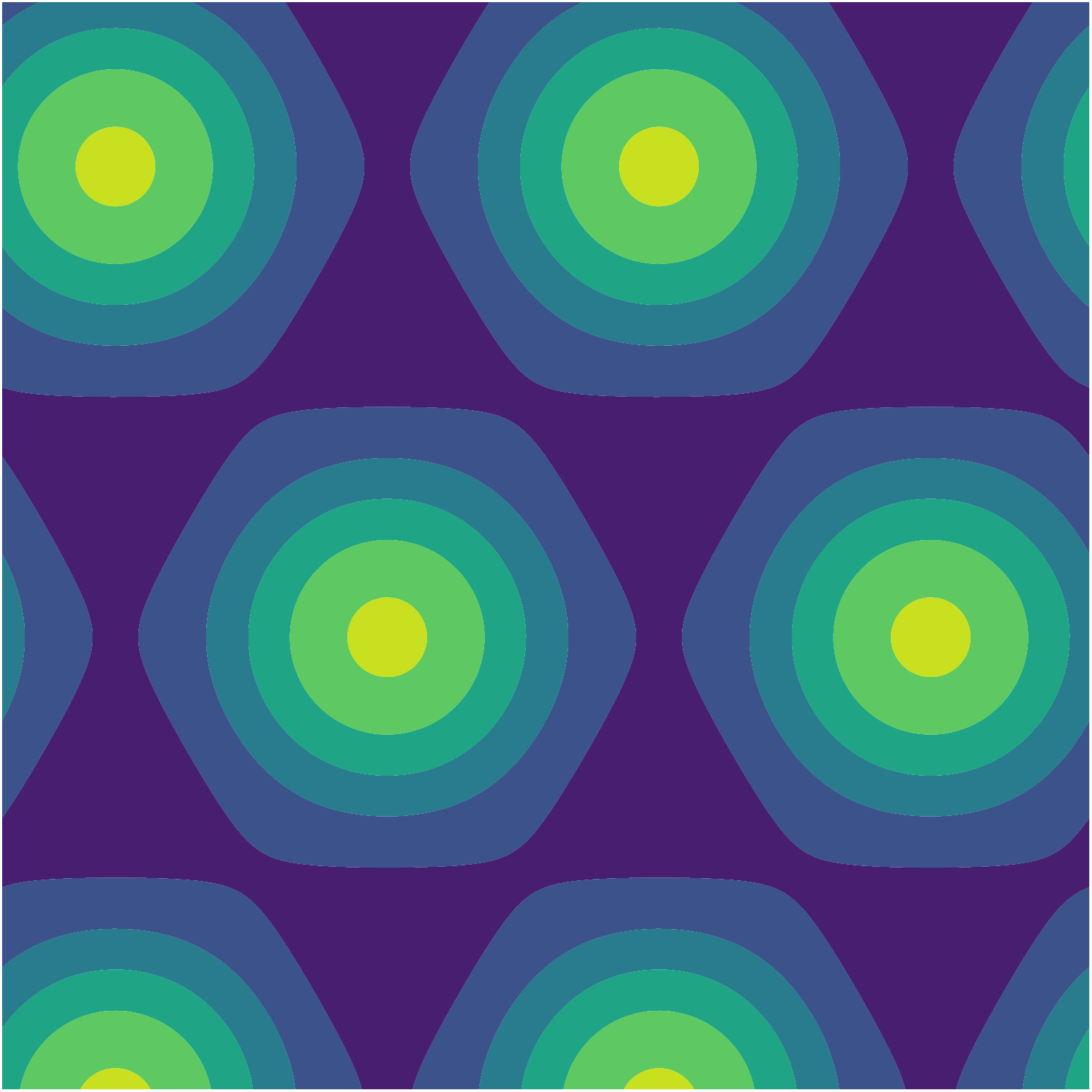}};
      \node at (3,0) {\includegraphics[height=2.7cm]{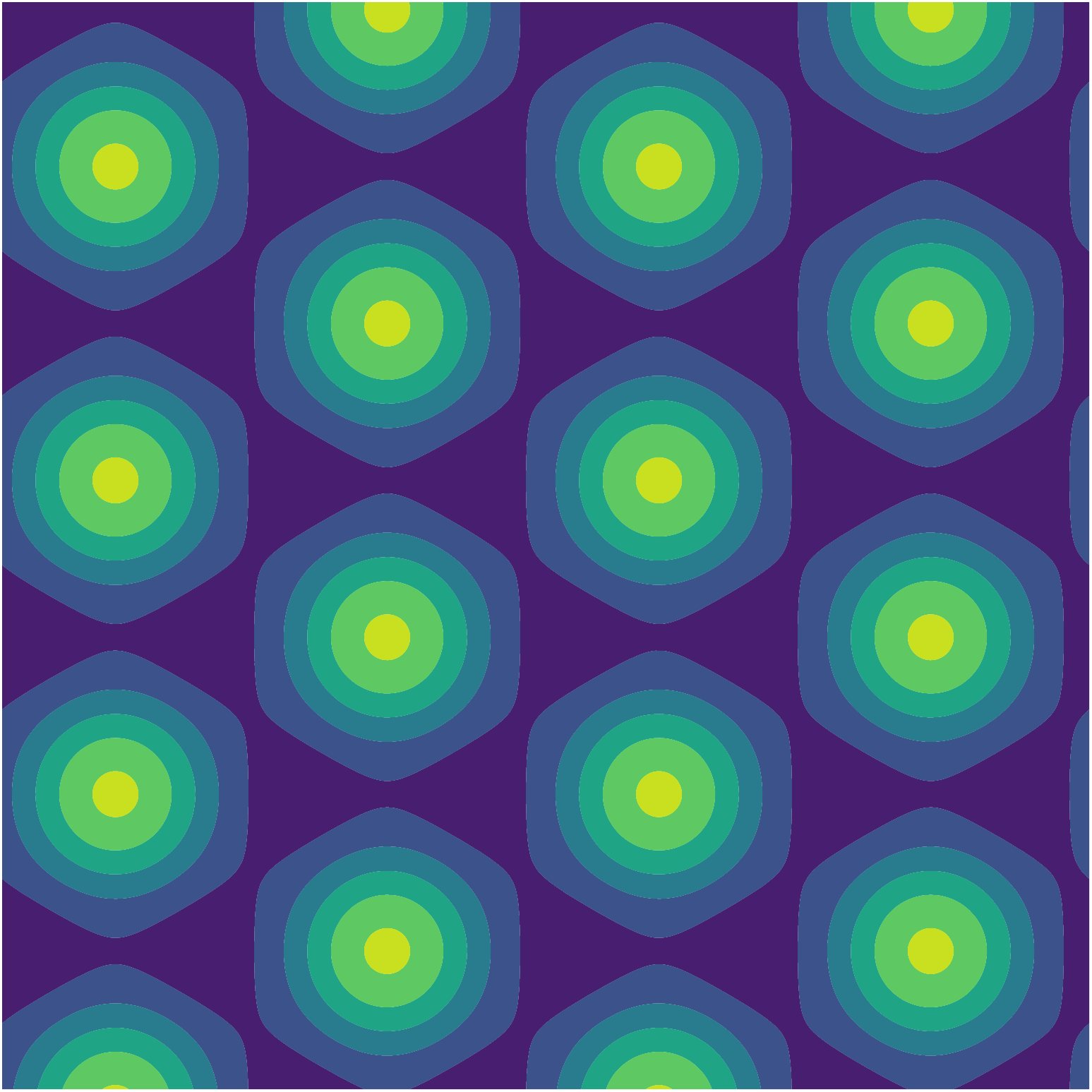}};
      \node at (6,0) {\includegraphics[height=2.7cm]{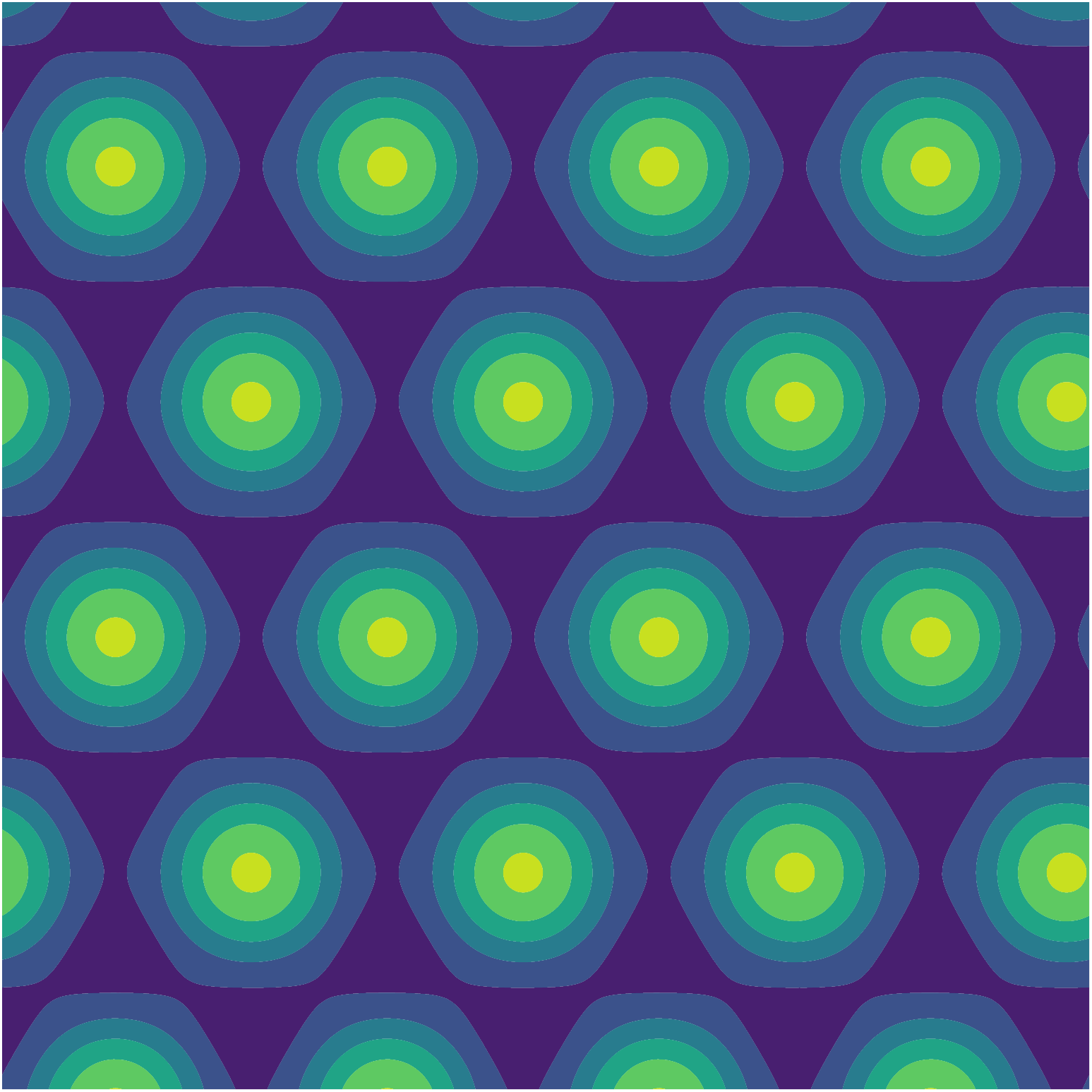}};
      \node at (9,0) {\includegraphics[height=2.7cm]{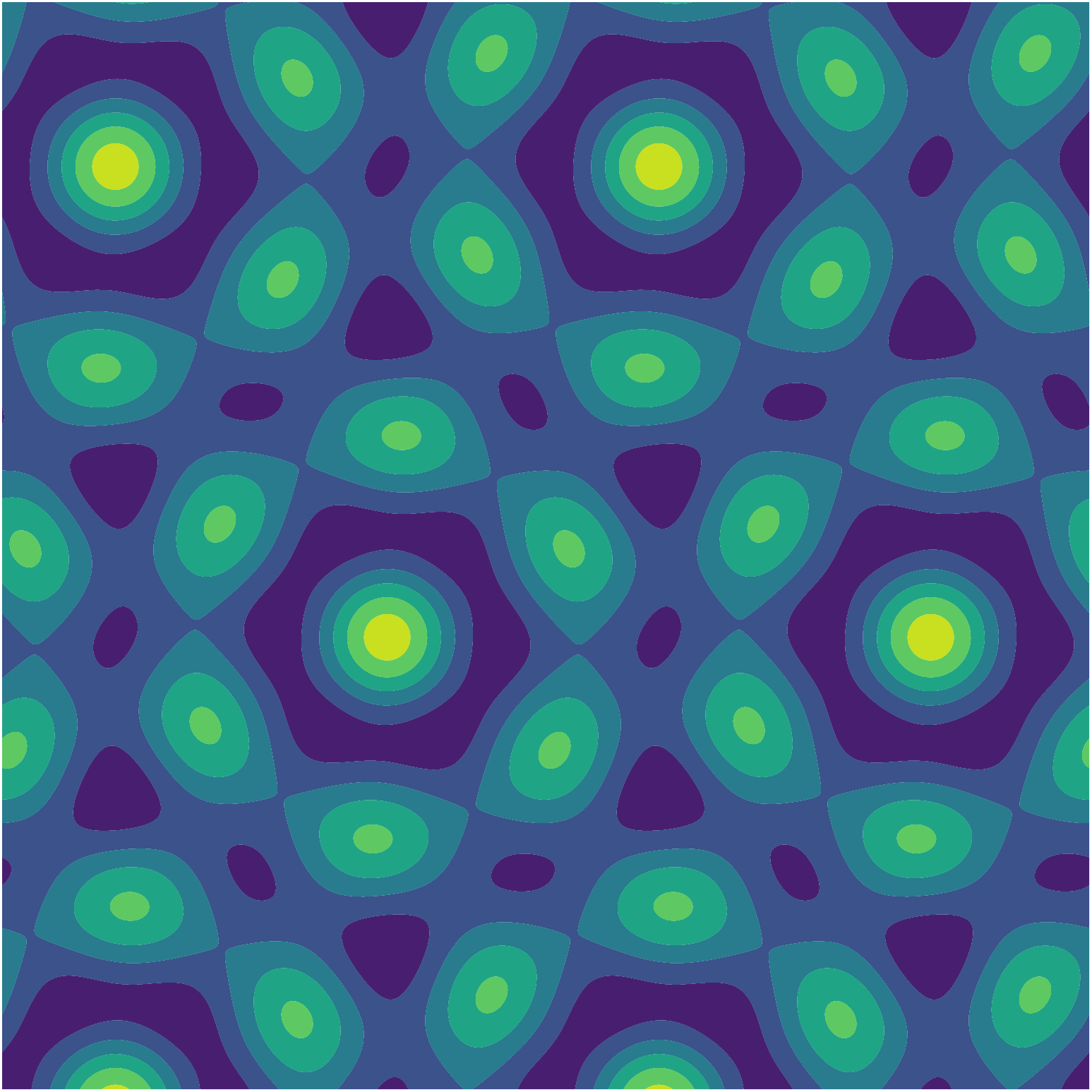}};
      \node at (12,0) {\includegraphics[height=2.7cm]{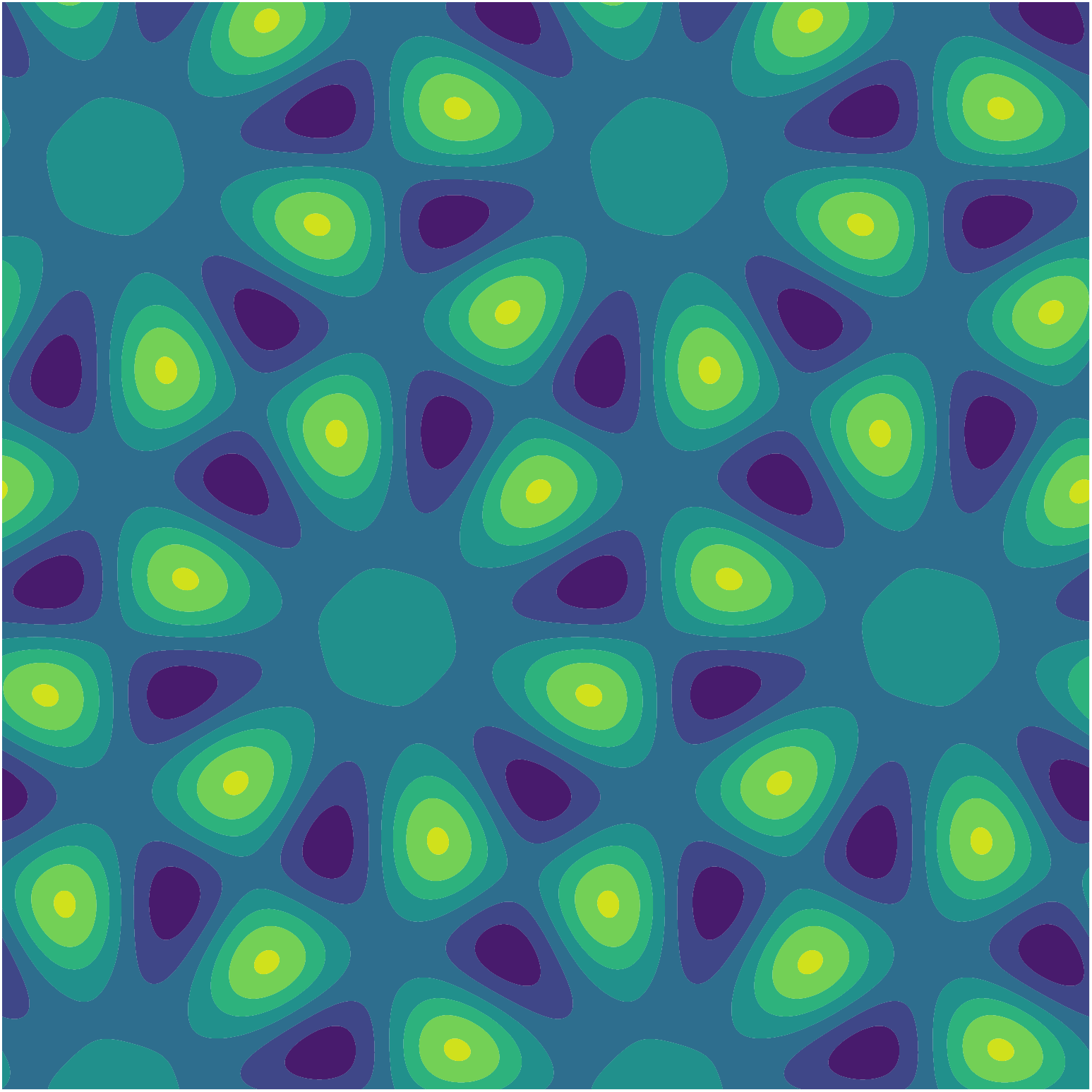}};
      \node at (0,-1.7) {\scriptsize 52.62};
      \node at (3,-1.7) {\scriptsize 157.87};
      \node at (6,-1.7) {\scriptsize 210.87};
      \node at (9,-1.7) {\scriptsize 368.02};
      \node at (12,-1.7) {\scriptsize 368.34};
    \end{scope}
    \begin{scope}[yshift=-3.5cm]
      \node at (0,0) {\includegraphics[height=2.7cm]{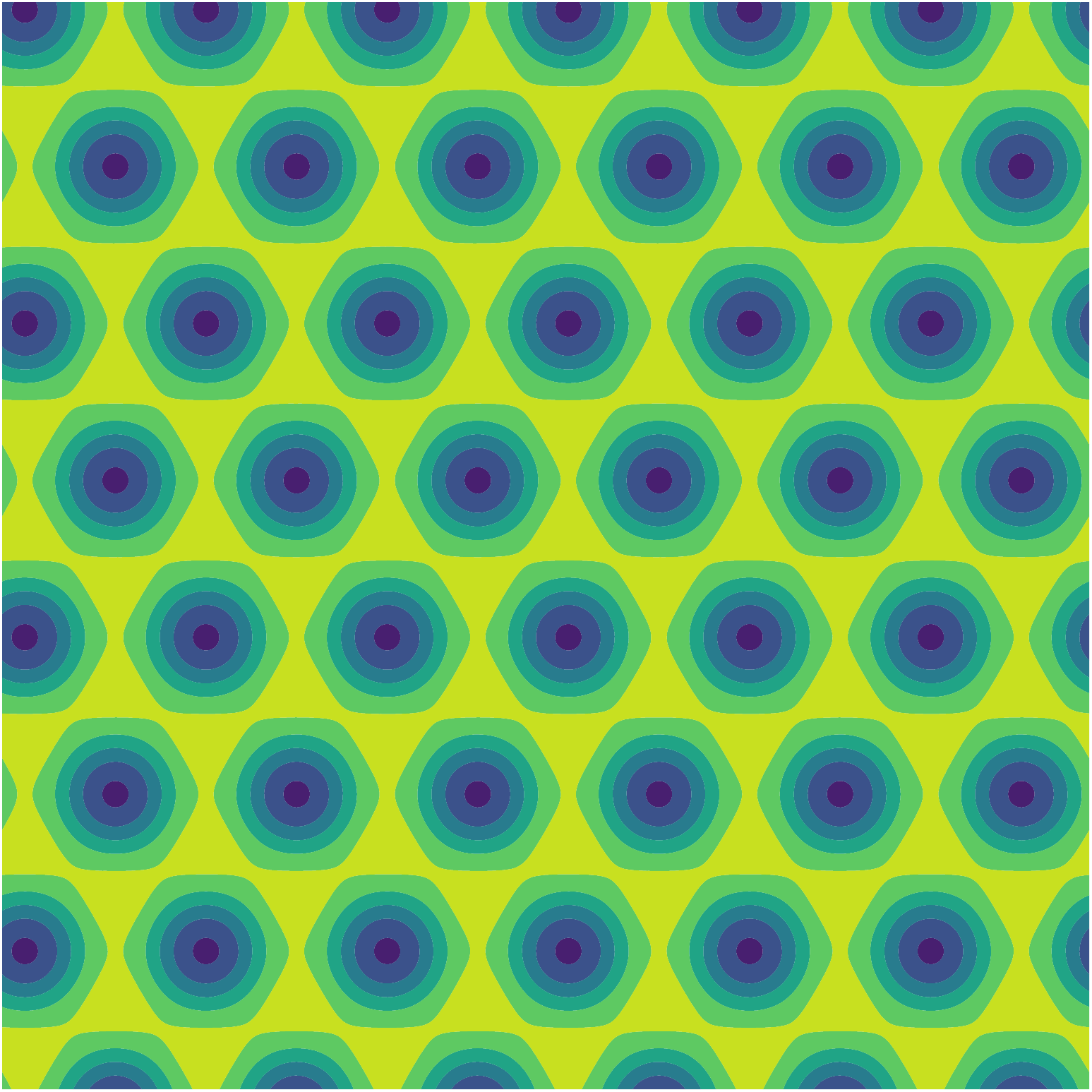}};
      \node at (3,0) {\includegraphics[height=2.7cm]{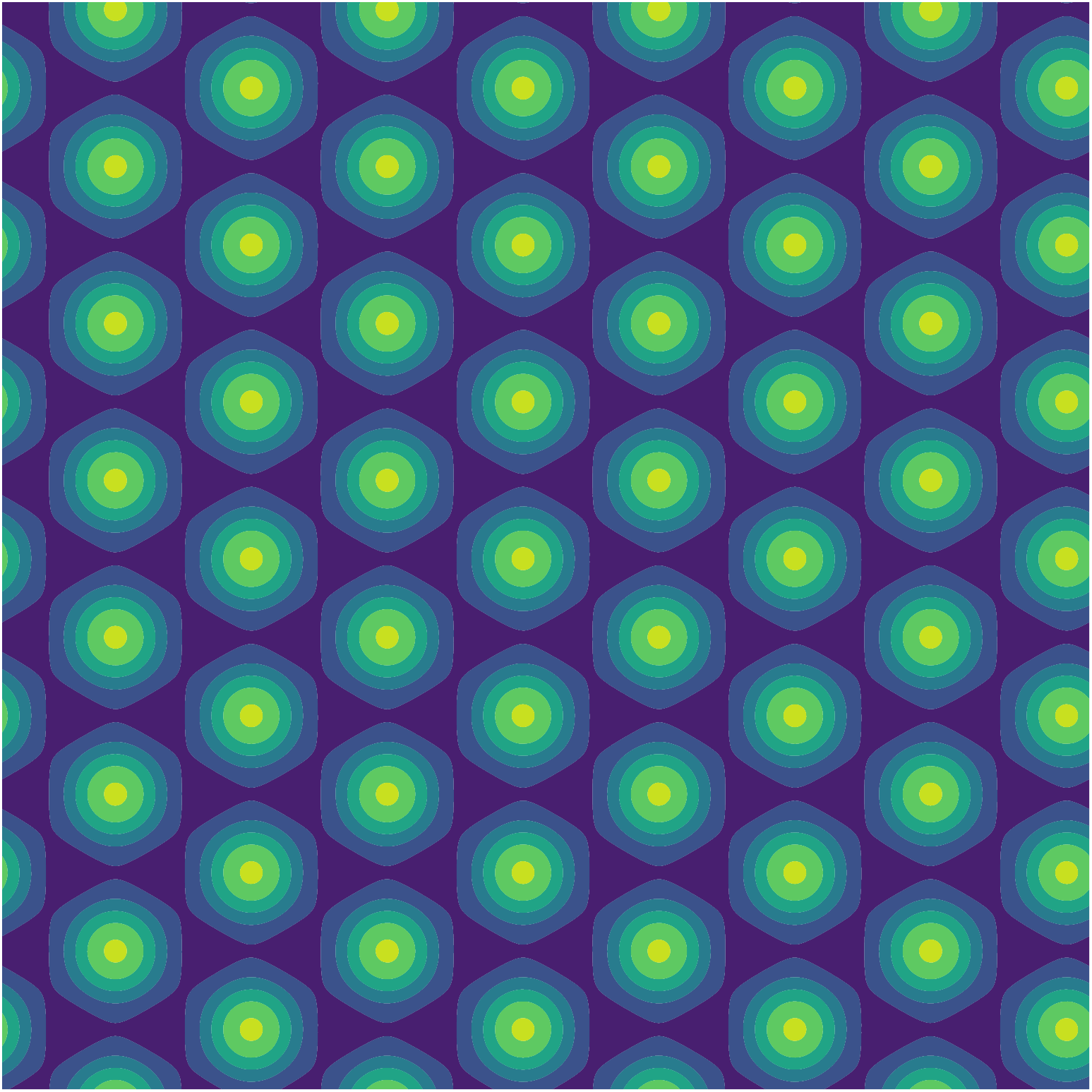}};
      \node at (6,0) {\includegraphics[height=2.7cm]{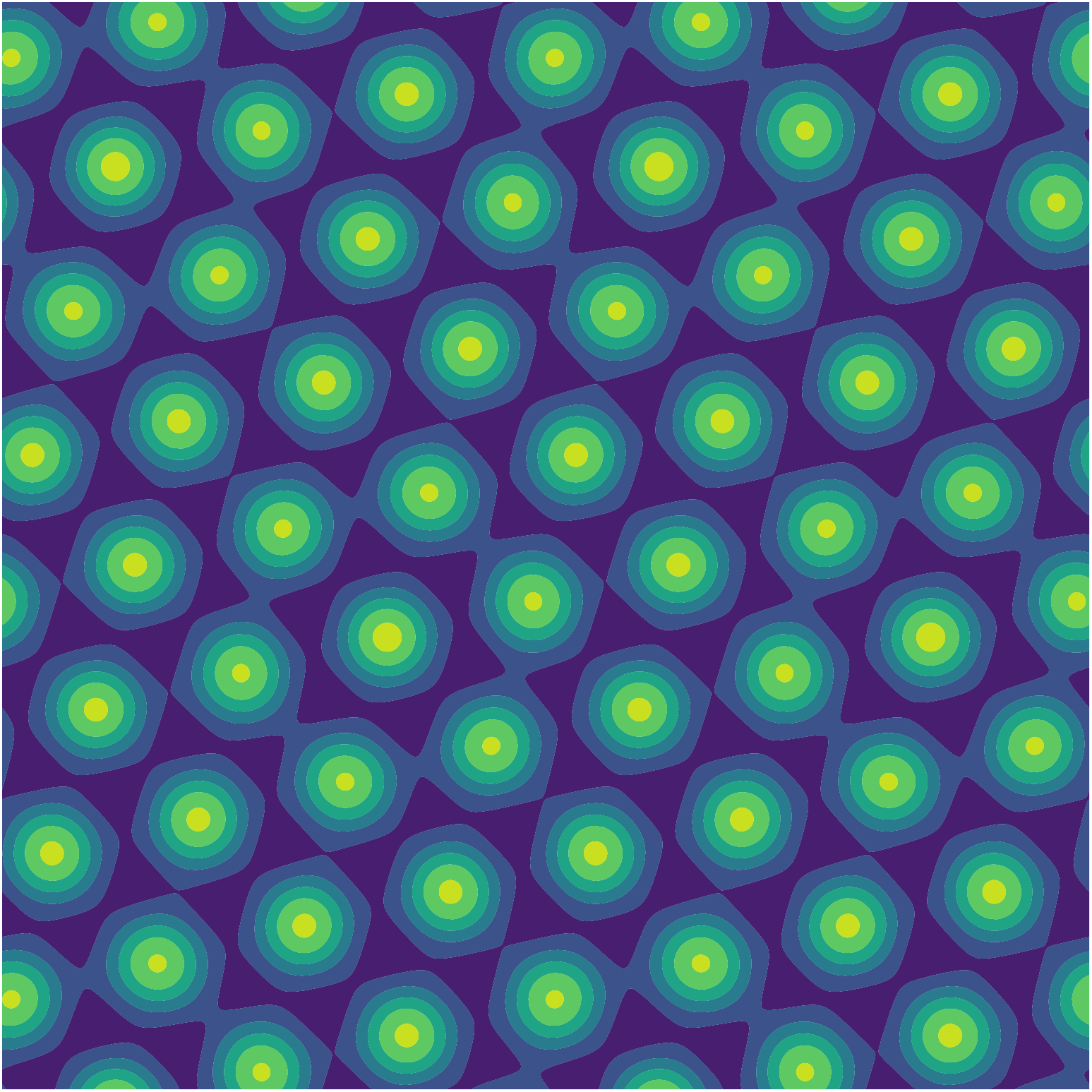}};
      \node at (9,0) {\includegraphics[height=2.7cm]{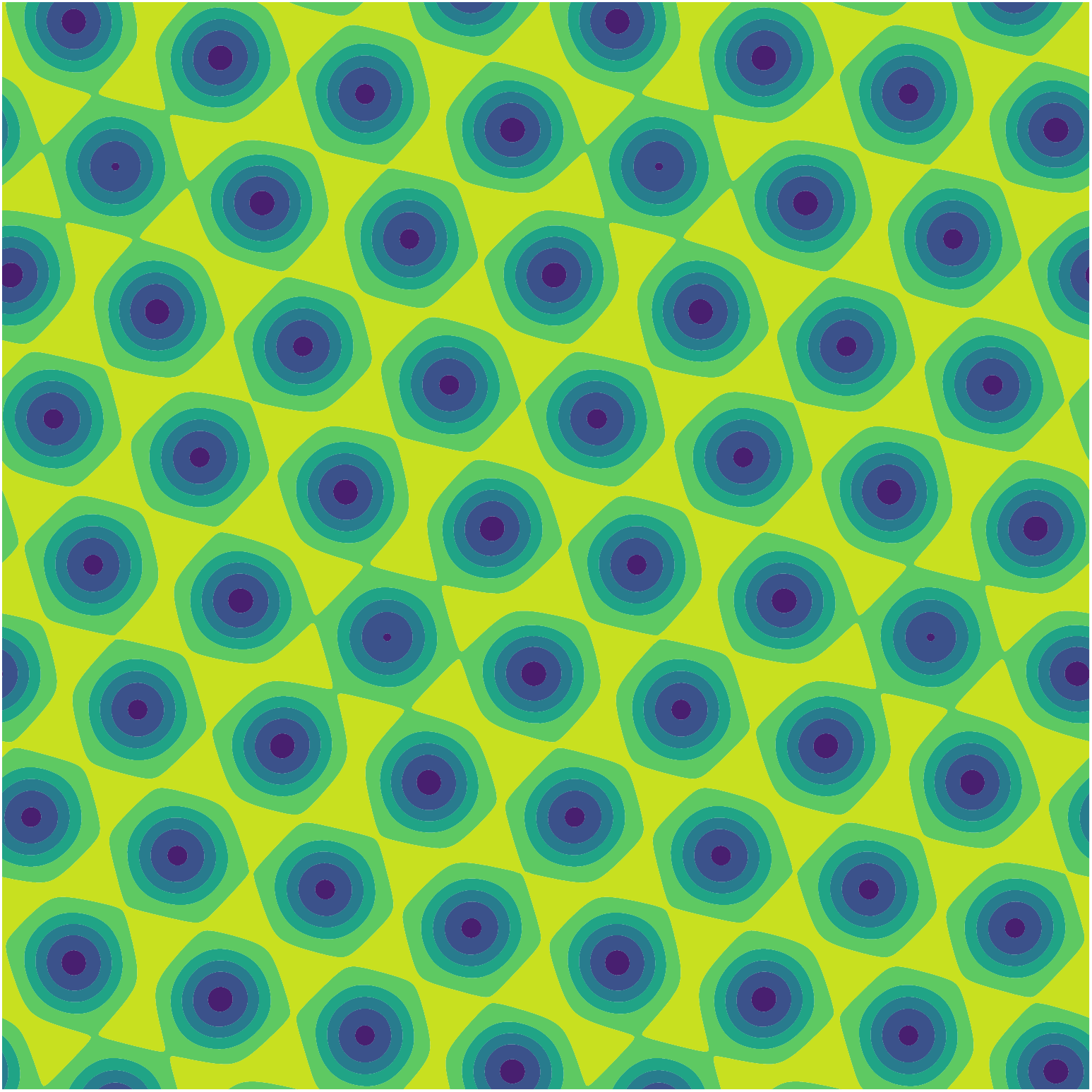}};
      \node at (12,0) {\includegraphics[height=2.7cm]{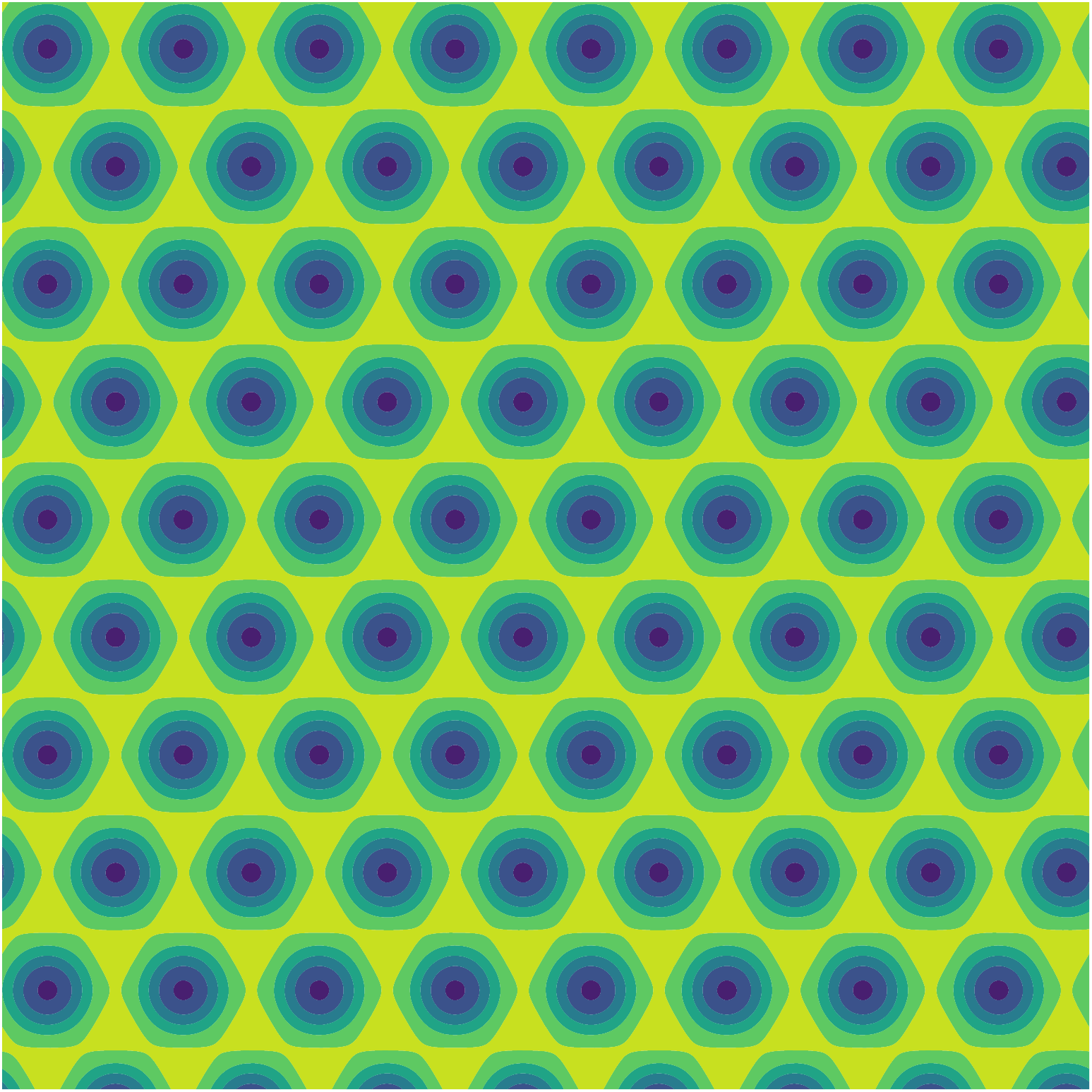}};
      \node at (0,-1.7) {\scriptsize 473.55};
      \node at (3,-1.7) {\scriptsize 631.41};
      \node at (6,-1.7) {\scriptsize 683.39};
      \node at (9,-1.7) {\scriptsize 684.06};
      \node at (12,-1.7) {\scriptsize 841.80};
    \end{scope}
  \end{tikzpicture}
  \caption{The first ten non-constant basis functions ${e_2,\ldots,e_{11}}$
    in \cref{theorem:spectral}, with their approximate eigenvalues $\lambda_i$, for the group \texttt{p6}.
  In this case, the mulitplicities $k(\lambda_i)$ are not $2n=4$ as for the standard Fourier transform.}
  \label{fig:fourier:p6}
\end{figure}

The standard Fourier bases for periodic functions on $\mathbb{R}^n$
can be obtained as the special cases of \cref{theorem:spectral}
for shift groups: Fix some edge width ${c>0}$, and choose $\Pi$ and $\mathbb{G}$ as
\begin{equation*}
  \Pi\;=\;[0,c]^n
  \qquad\text{ and }\qquad
  \group=\bigl\{x\mapsto x+c(i_1,\ldots,i_n)^{\trans}\,\big|\,i_1,\ldots,i_n\in\mathbb{Z}\bigr\}\;.
\end{equation*}
For these groups, all eigenvalue multiplicities are ${k(\lambda_i)=2n}$ for each ${i\in\mathbb{N}}$.
For ${n=2}$, the group $\group$ is \texttt{p1} (see \cref{fig:wallpaper}). Its eigenfunctions
are shown in \cref{fig:fourier:p1}.

To clarify the relationship in more detail, consider the case
${n=1}$: Since $\Delta$ is a second derivative, the functions
${e(x)=\cos(\nu x)}$ and ${e(x)=\sin(\nu x)}$ satisfy
\begin{equation*}
  \Delta e(x)\;=\;-\nu^2e(x)\qquad\text{ for each }\nu\geq 0\;,
\end{equation*}
and are hence eigenfunctions of ${-\Delta}$ with eigenvalue ${\lambda=\nu^2}$.
For this choice of $\Pi$ and $\group$, the invariance constraint
in \eqref{sturm:liouville} holds iff ${e(x)=e(x+c)}$ for every ${x\in\mathbb{R}}$.
That is true iff
\begin{equation*}
  \nu(x+c)\;=\;\nu x+2\pi(i-1)\quad\text{ for some $i\in\mathbb{N}$,
    and hence }\quad
  \lambda_i\;=\;\nu^2\;=\;\Bigl(\mfrac{2\pi(i-1)}{c}\Bigr)^2\;.
\end{equation*}
The eigenspaces are therefore the two-dimensional vector spaces
\begin{equation*}
    \mathcal{V}(\lambda_i)\;=\;\text{span}\braces{\sin(\sqrt{\lambda_i}\argdot),\cos\sqrt{\lambda_i}\argdot)}
    \quad\text{ with }\quad
    k(\lambda_i)\;=\;2
    \qquad\text{ for all }i\in\mathbb{N}\;.
\end{equation*}
Any continuous function $f$ that is $\group$-invariant (or, equivalently, $c$-periodic) can be expanded as
\begin{equation}
  \label{eq:fourier:series:1d}
    f(x)\;=\;\msum_{i\in\mathbb{N}}a_i\cos(\sqrt{\lambda_i}x)+b_i\sin(\sqrt{\lambda_i}x)\;.
\end{equation}
In the notation of \cref{theorem:spectral}, the coefficients are
${c_{2i}=a_i}$ and ${c_{2i+1}=b_i}$, and
\begin{equation*}
  e_{2i}(x)=\cos(\sqrt{\lambda_{2i}}x)
  \quad\text{ and }\quad
  e_{2i+1}(x)=\sin(\sqrt{\lambda_{2i+1}}x)
  \;.
\end{equation*}
Note that the unconstrained equation has solutions for all $\lambda$ in the uncountable set ${[0,\infty)}$.
The invariance constraint limits possible values to the countable set ${\lambda_1,\lambda_2,\ldots}$.
If $f$ was continuous but not invariant, the expansion \eqref{eq:fourier:series:1d} would hence require an integral
on the right. Since $f$ is invariant, a series suffices.

\begin{remark}[Multiplicities and real versus complex coefficients]
  Fourier series, in particular in one dimension, are often written using complex-valued functions as
    \begin{equation*}
      f(x)\;=\;\msum_{i\in\mathbb{N}}\gamma_i\exp(J\lambda_i x)\qquad\text{ where }
      \gamma_i\in\mathbb{C}\text{ and }J:=\sqrt{-1}\;.
    \end{equation*}
    Since Euler's formula ${\exp(Jx)=\cos(x)+J\sin(x)}$ shows
    \begin{equation*}
      \gamma_i\exp(J\sqrt{\lambda_i})
      \;=\;
      a_i\cos(\sqrt{\lambda_i}x)+b_i\sin(\sqrt{\lambda_i})
      \quad\text{ for }(a_i-Jb_i)=\gamma_i\;,
  \end{equation*}
    that is equivalent to \eqref{eq:fourier:series:1d}. The complex plane $\mathbb{C}$ is not inherent to the Fourier
    representation, but rather a convenient way
    to parameterize the two-dimensional eigenspace $\mathcal{V}(\lambda_i)$.
    For general crystallographic groups, the complex representation is less useful,
    since the multiplicities $k(\lambda_i)$ may not be even, as can be seen in \cref{fig:fourier:p6}.
\end{remark}

\begin{figure}
  \begin{tikzpicture}
    \begin{scope}
      \node at (0,0) {\includegraphics[height=2.7cm]{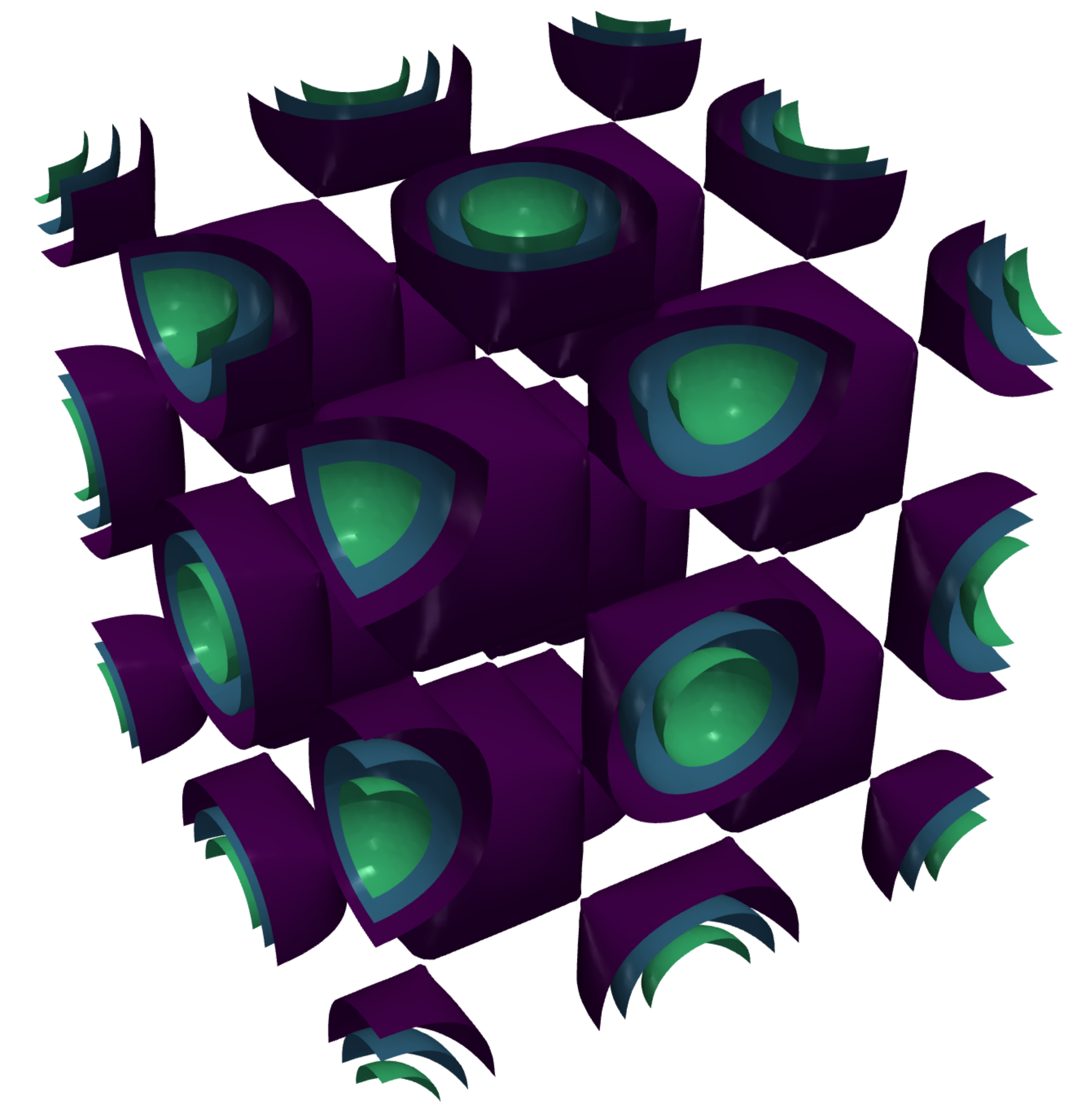}};
      \node at (3,0) {\includegraphics[height=2.7cm]{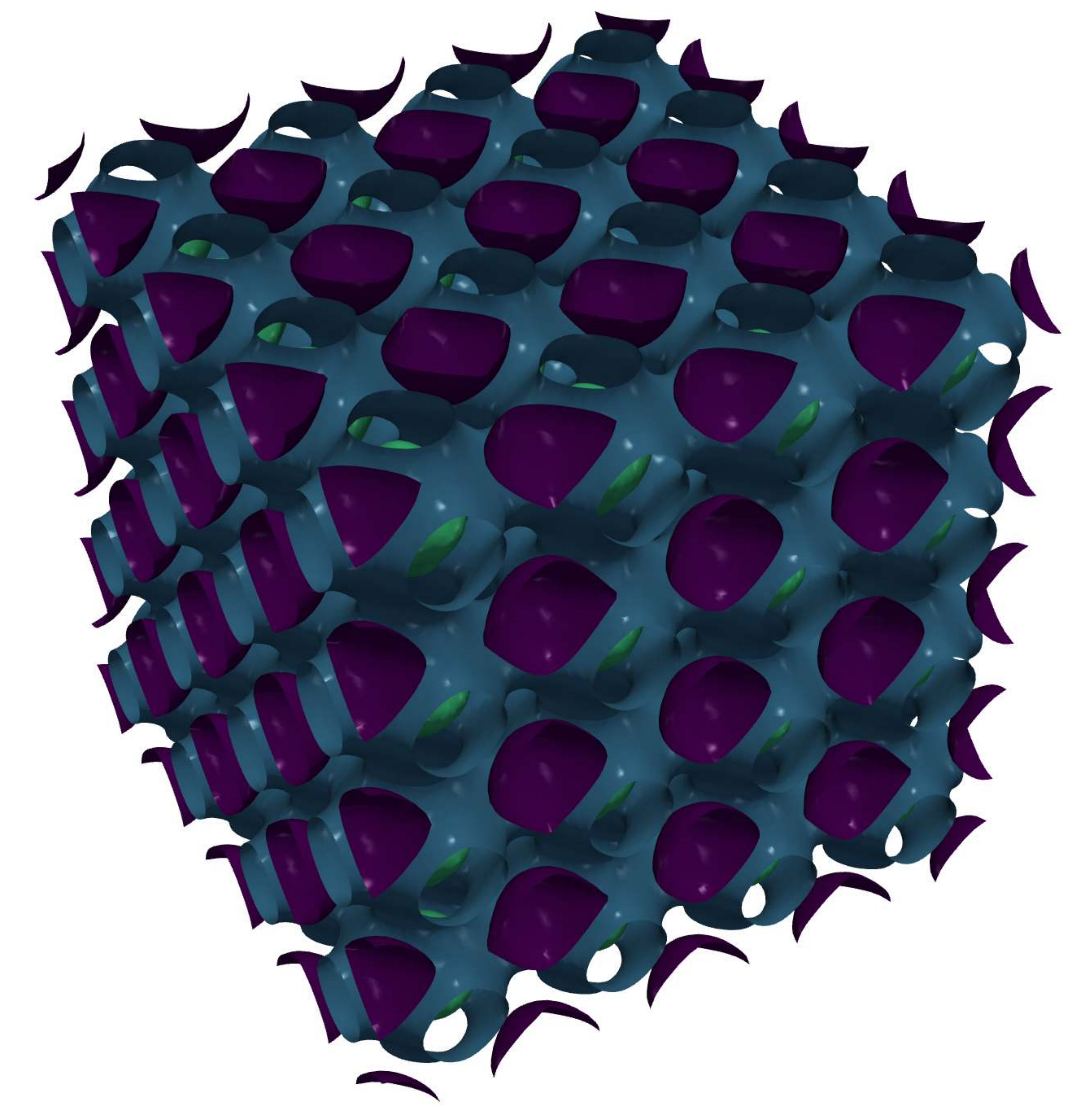}};
      \node at (6,0) {\includegraphics[height=2.7cm]{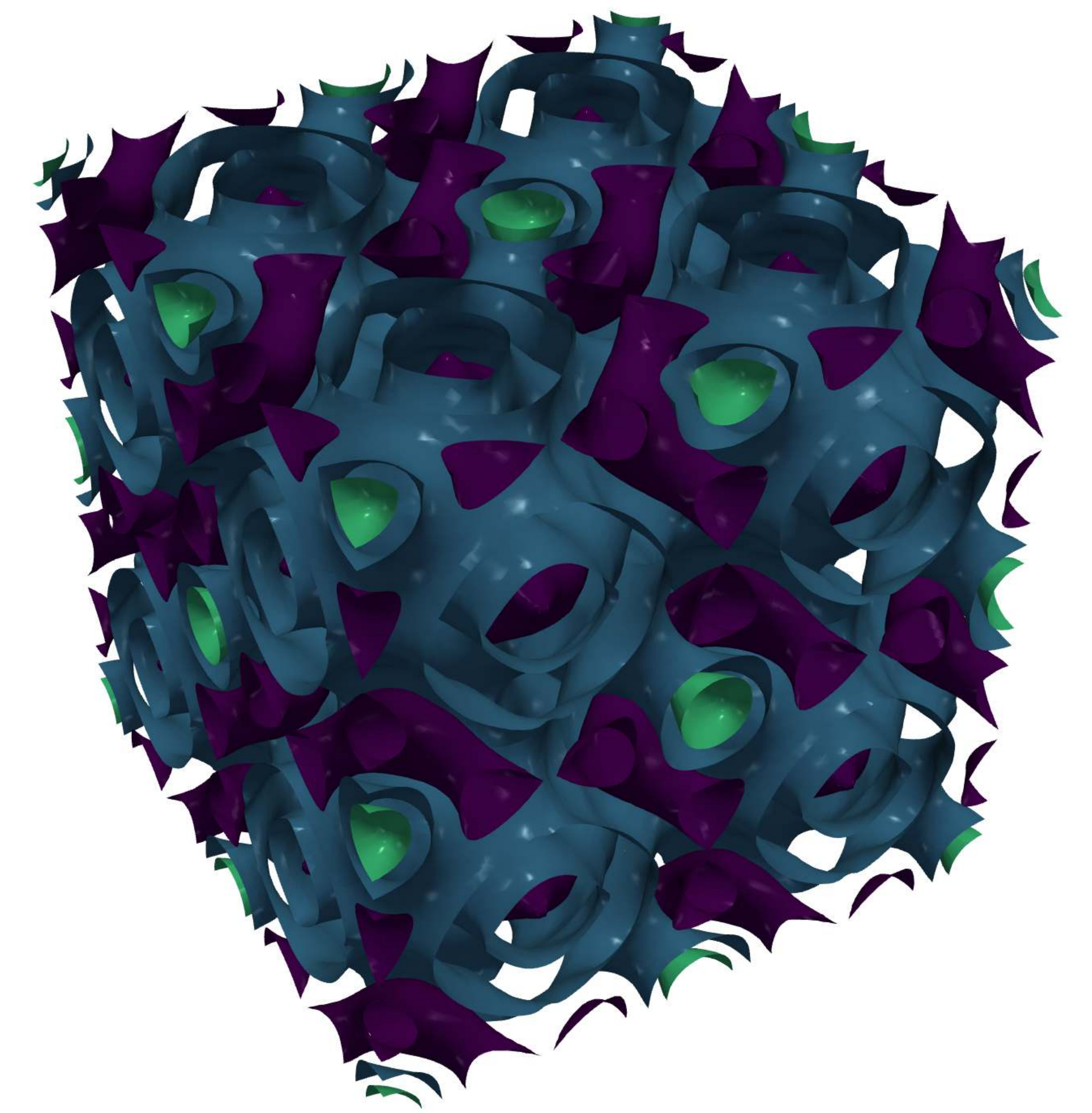}};
      \node at (9,0) {\includegraphics[height=2.7cm]{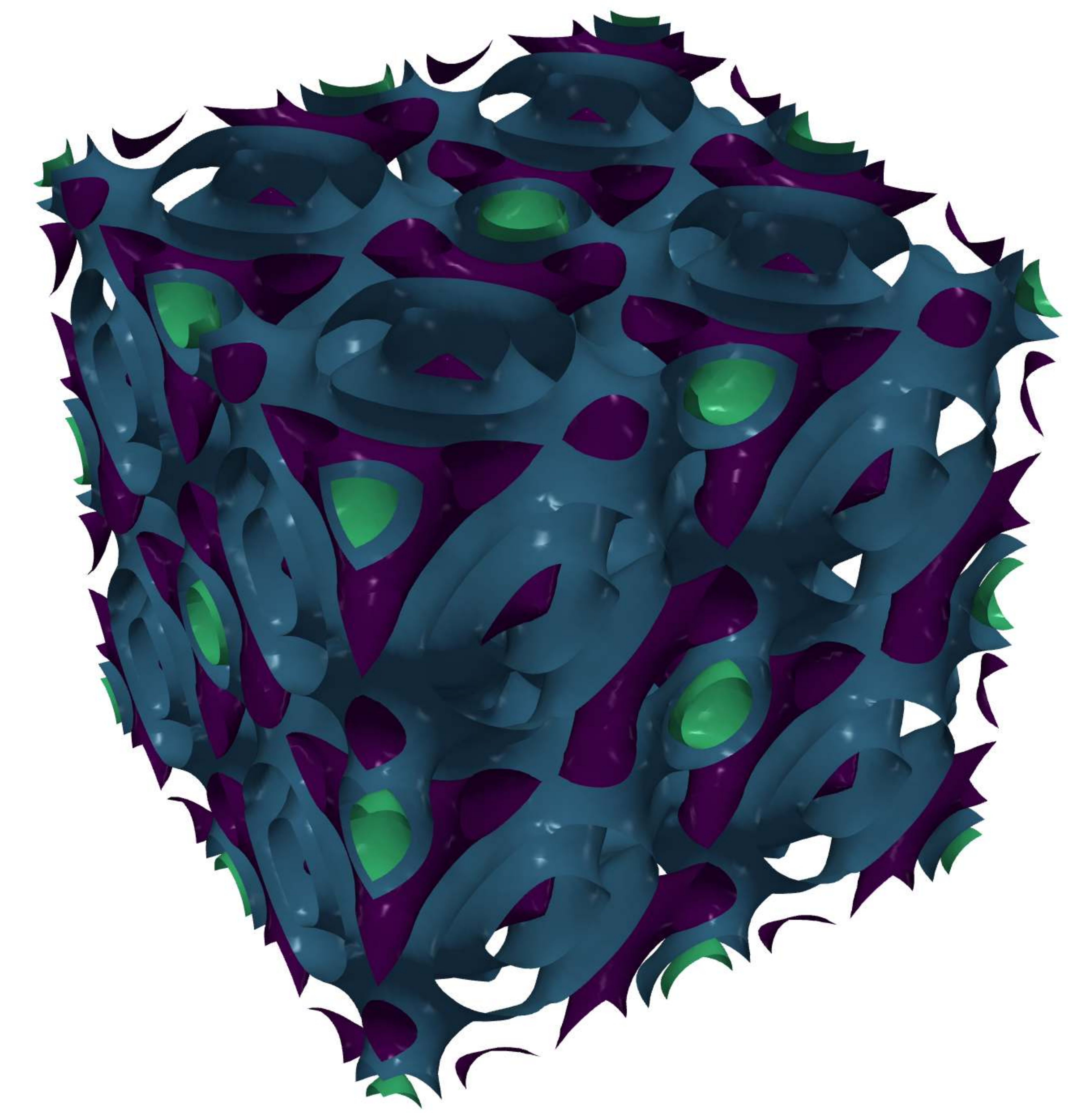}};
      \node at (12,0) {\includegraphics[height=2.7cm]{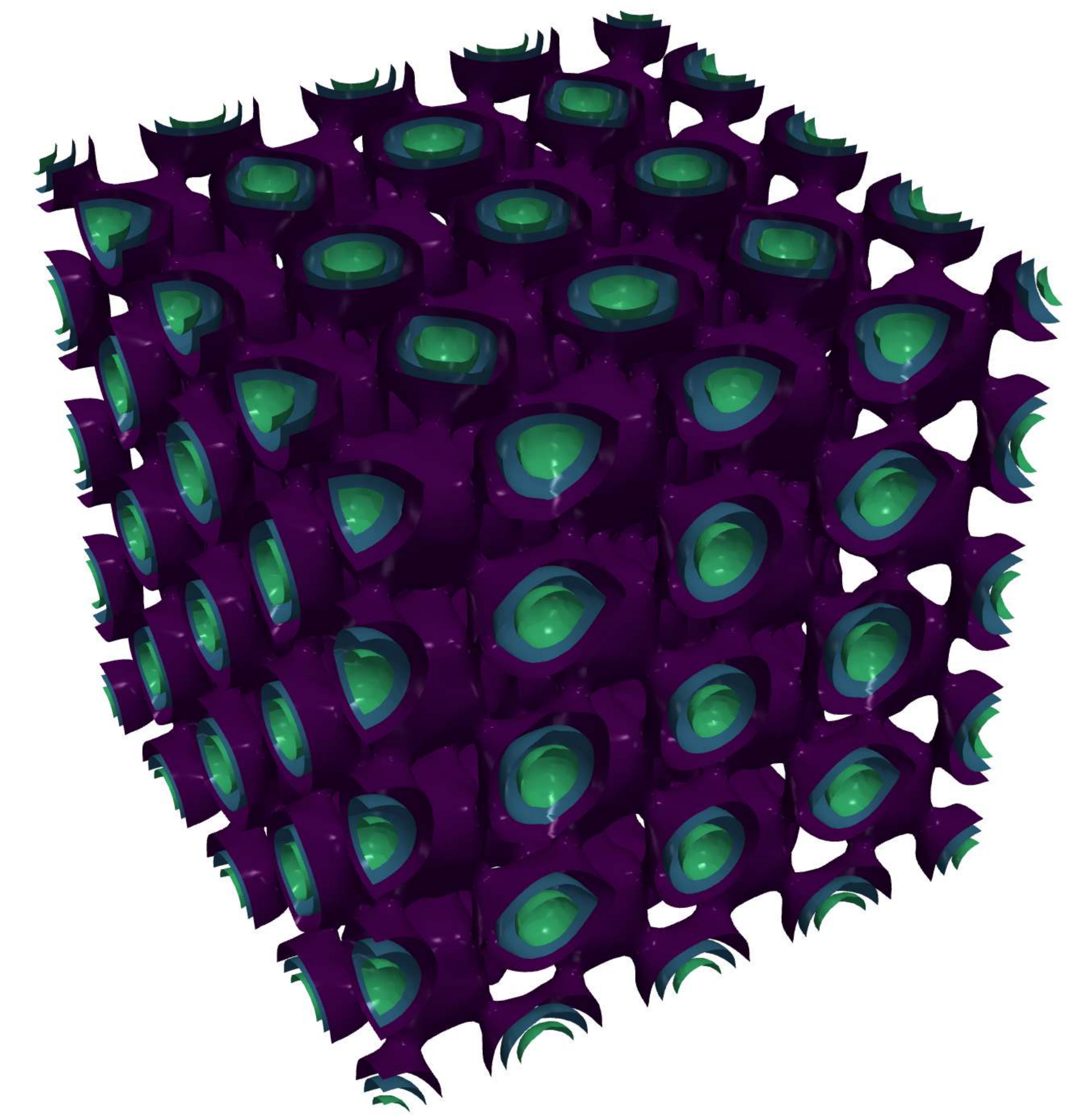}};
      \node at (0,-1.7) {\scriptsize 78.94};
      \node at (3,-1.7) {\scriptsize 158.46};
      \node at (6,-1.7) {\scriptsize 237.30};
      \node at (9,-1.7) {\scriptsize 238.46};
      \node at (12,-1.7) {\scriptsize 323.08};
    \end{scope}
    \begin{scope}[yshift=-3.5cm]
      \node at (0,0) {\includegraphics[height=2.7cm]{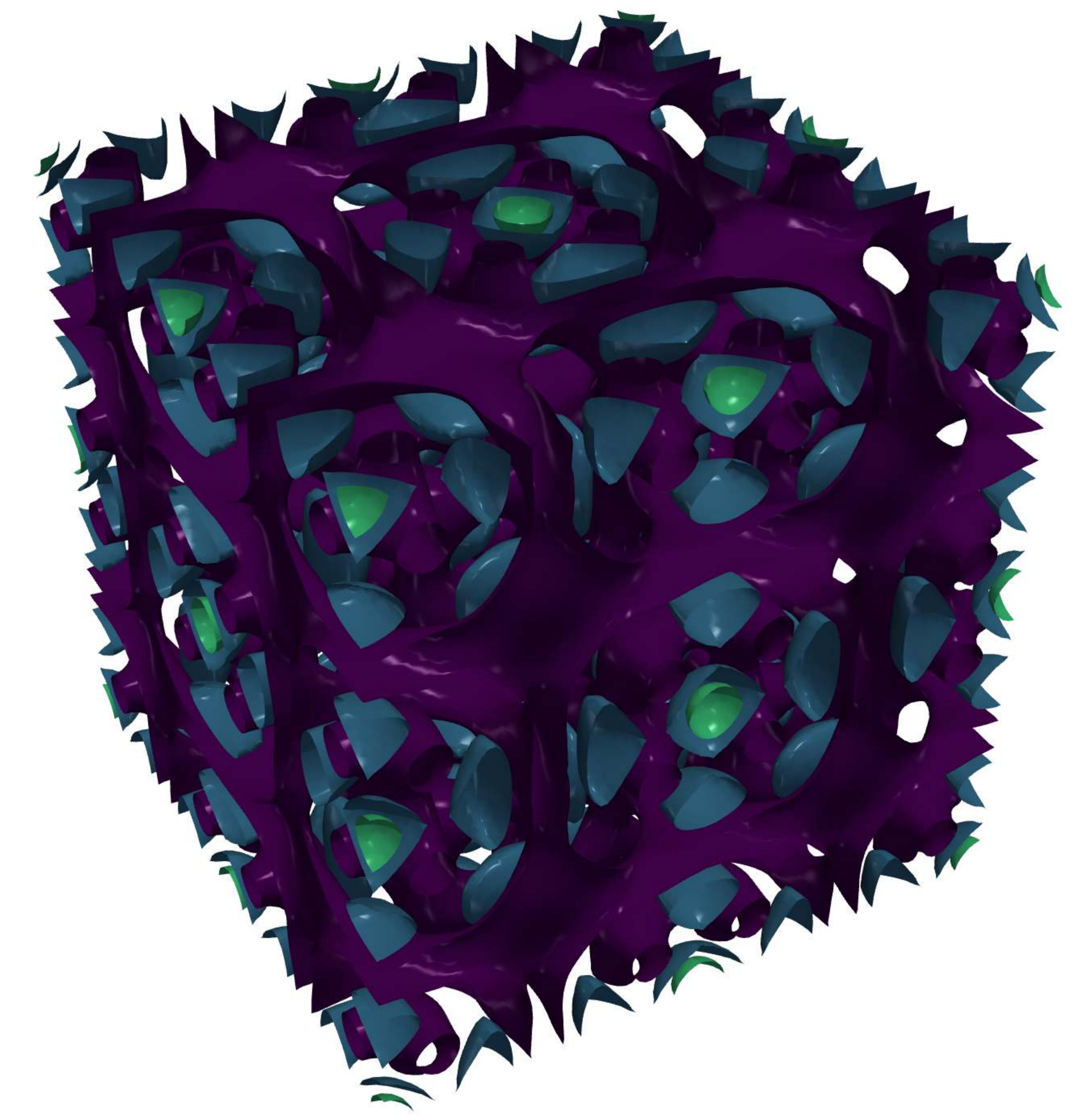}};
      \node at (3,0) {\includegraphics[height=2.7cm]{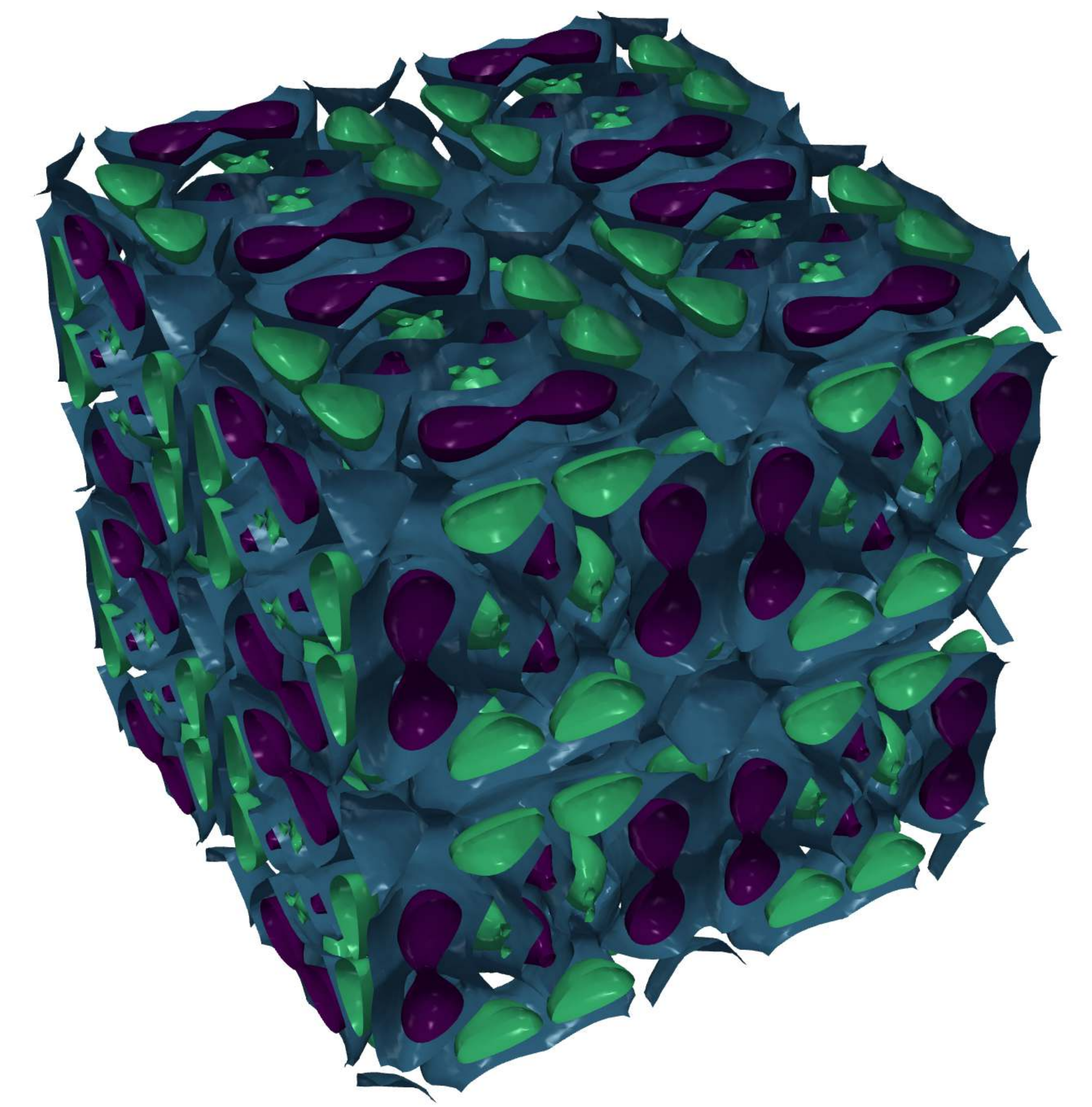}};
      \node at (6,0) {\includegraphics[height=2.7cm]{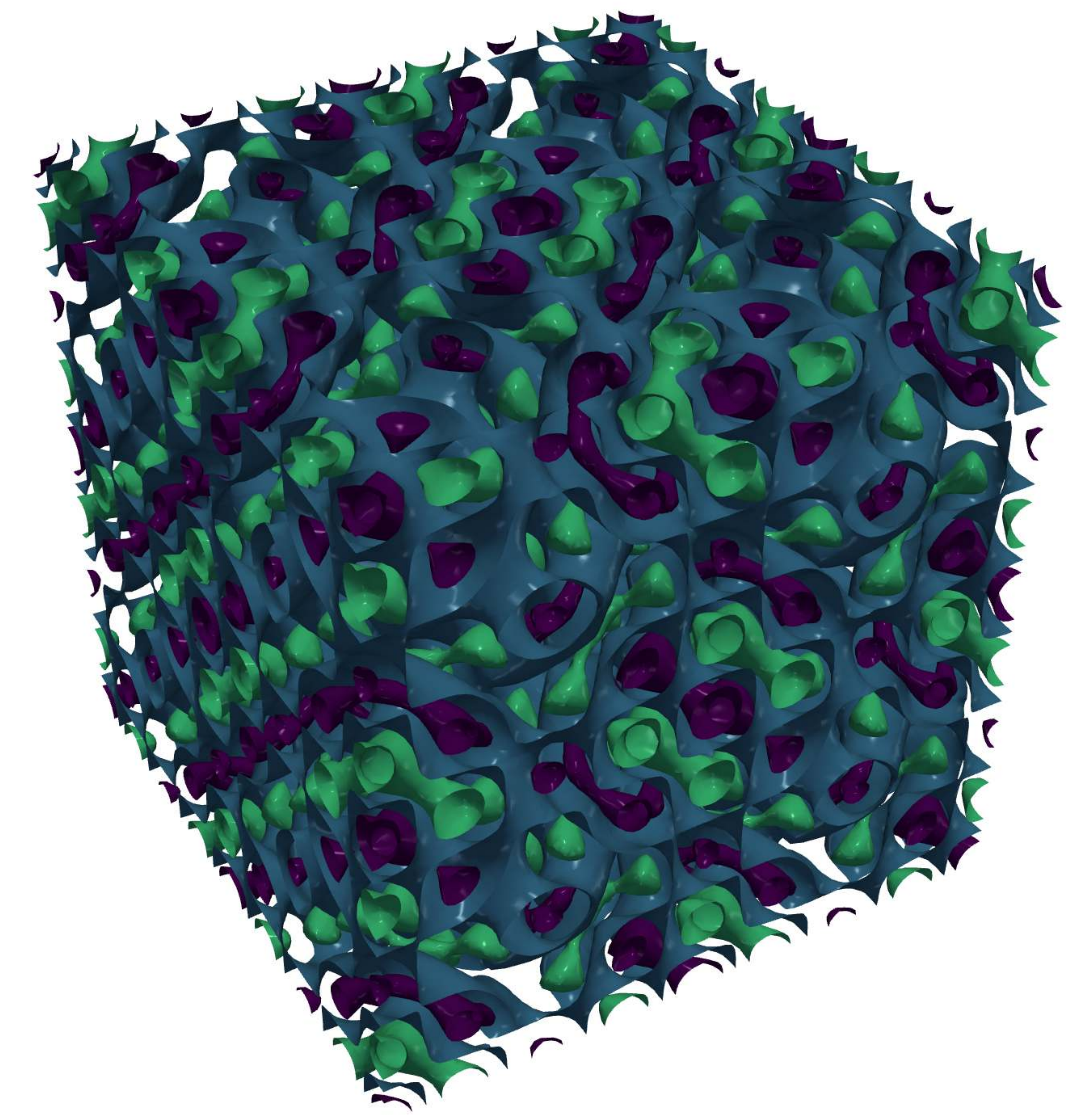}};
      \node at (9,0) {\includegraphics[height=2.7cm]{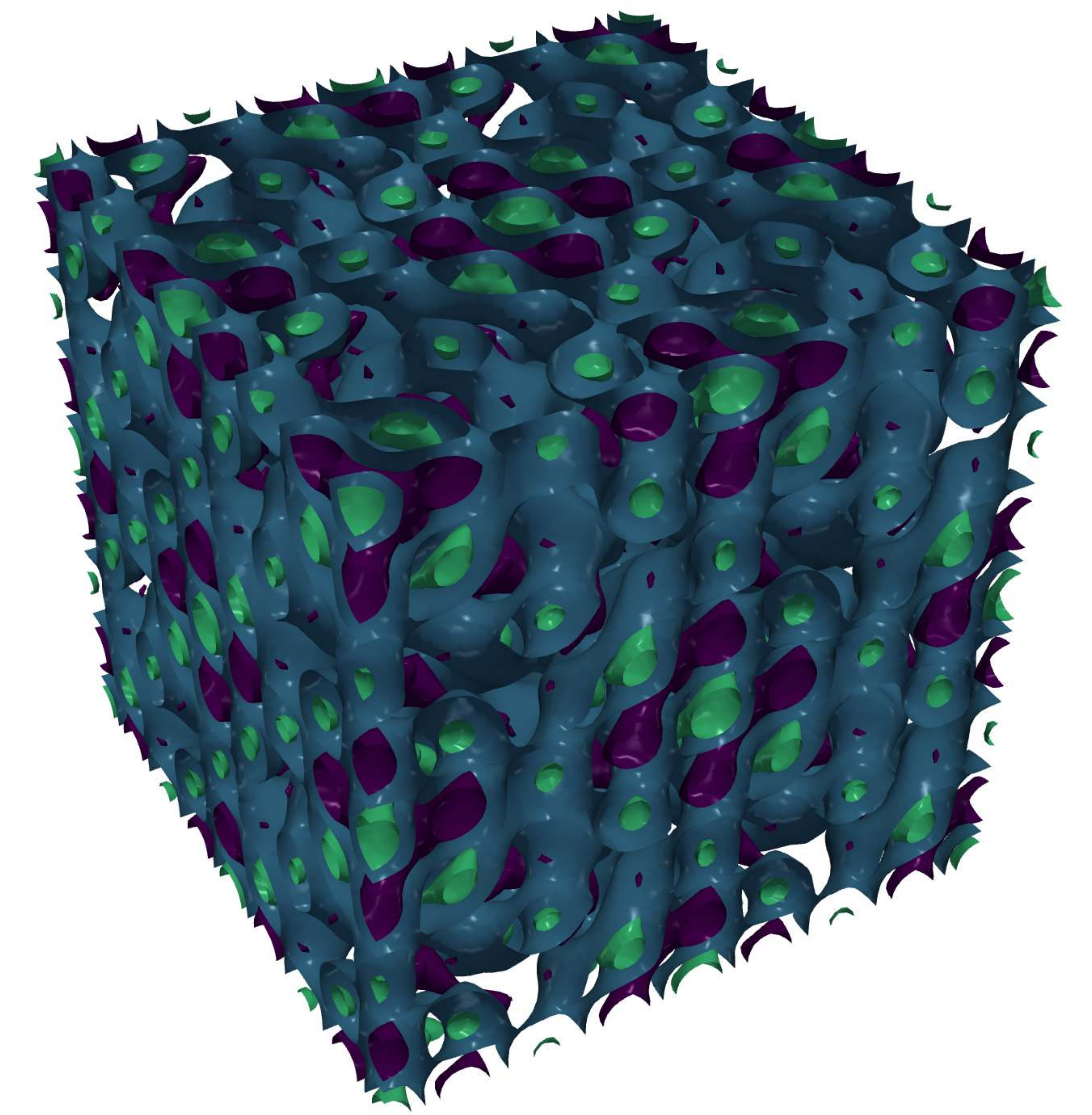}};
      \node at (12,0) {\includegraphics[height=2.7cm]{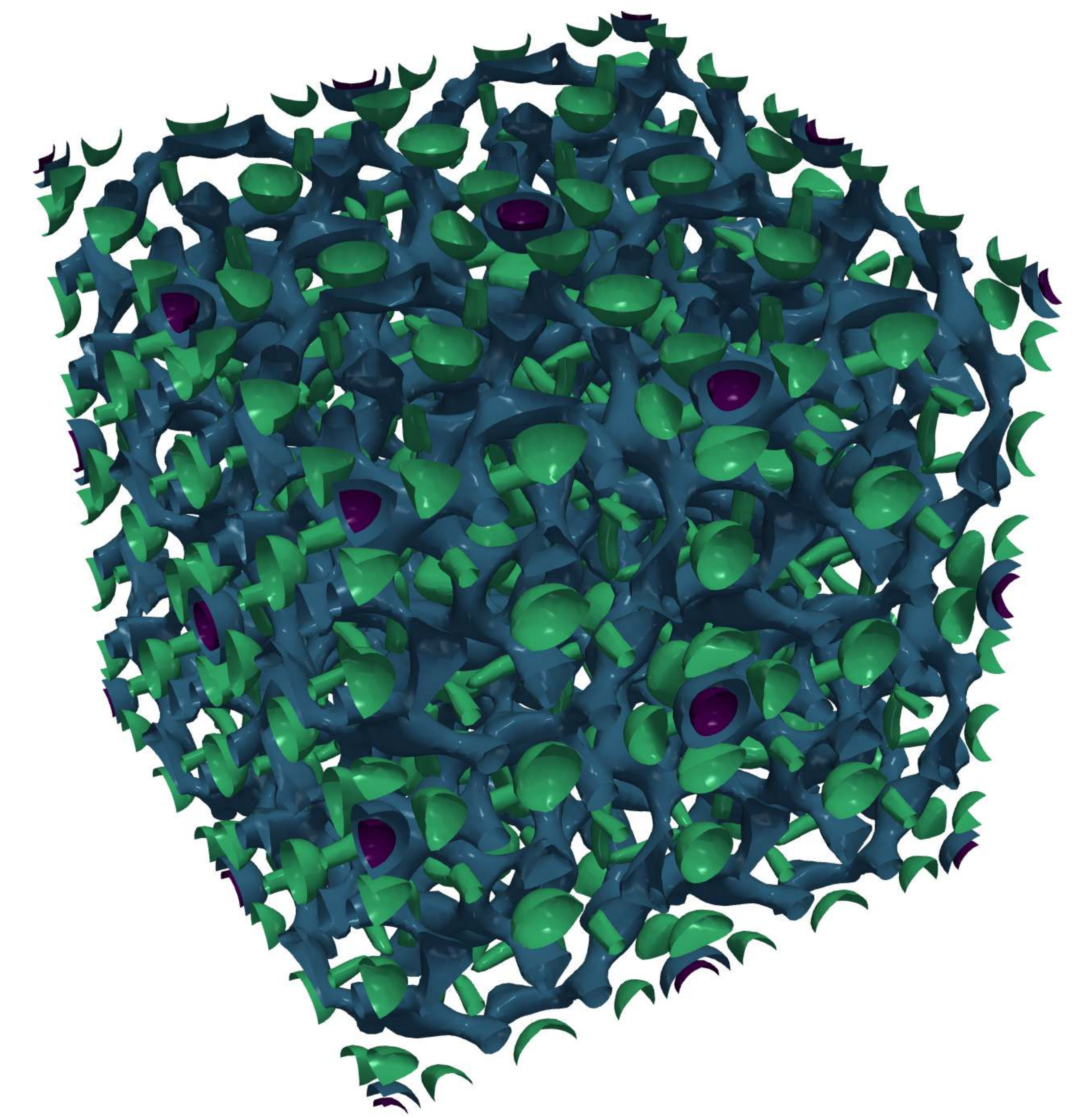}};
      \node at (0,-1.7) {\scriptsize 395.37};
      \node at (3,-1.7) {\scriptsize 438.02};
      \node at (6,-1.7) {\scriptsize 489.34};
      \node at (9,-1.7) {\scriptsize 490.42};
      \node at (12,-1.7) {\scriptsize 576.43};
    \end{scope}
  \end{tikzpicture}
    \caption{The first ten non-constant basis functions ${e_2,\ldots,e_{11}}$
    in \cref{theorem:spectral}, with their approximate eigenvalues $\lambda_i$, for the group \texttt{I23} on $\mathbb{R}^3$.
  }
  \label{fig:fourier:I23}
\end{figure}
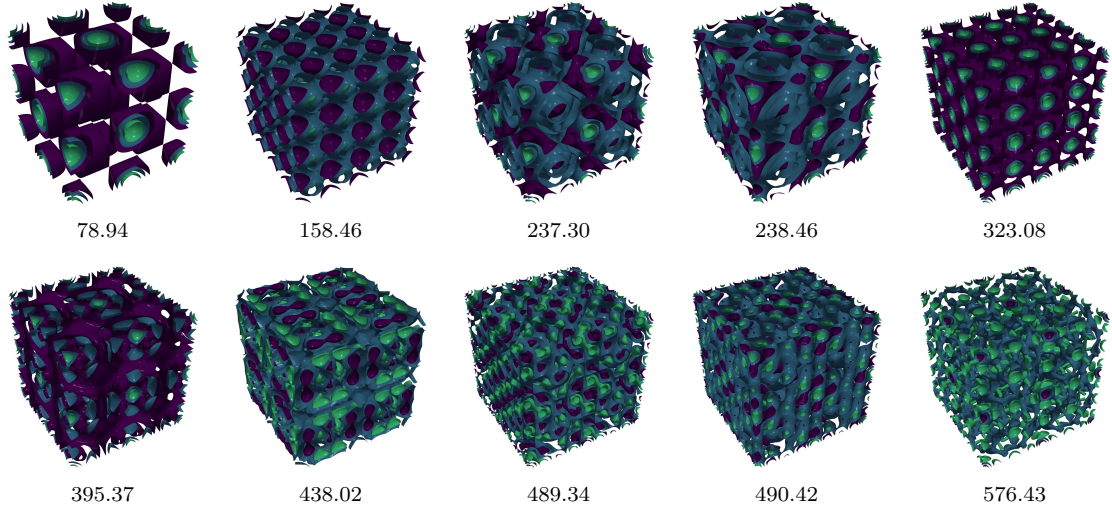

\subsection{Spectral algorithms}
\label{sec:spectral:algorithms}

The eigenfunctions in \cref{theorem:spectral} can be
approximated by eigenvectors of a suitable graph Laplacian of the orbit graph as follows.
We first compute an orbit graph ${\mathcal{G}=(\Gamma,E)}$ as described in
\cref{sec:orbit:graphs}. We weight each edge $(x,y)$ of the graph by
\begin{align}
  w(x,y) &=
  \begin{cases}
    \exp\Bigl(-\frac{d^2_\group(\group(x), \group(y))}{2\epsilon^2}\Bigr) &
    \text{ if }(x,y)\in E\\
    0 &\text{ otherwise}
  \end{cases}
\;.
\label{eqn:edge-weights}
\end{align}
The normalized Laplacian of the weighted graph is
\begin{align}
L = \mathbb{I} - D^{-1}W \qquad\text{ where }W_{xy}=w(x,y)\,,
\label{eqn:graph-laplacian}
\end{align}
and ${D}$ is the diagonal matrix containing the sum of each row of~$W$.
See e.g., \citet{chung1997spectral} for more on the matrix $L$.
Our estimates of the eigenvalues and -functions of~$\Delta$ are the eigenvalues and eigenvectors of~$L$,
\begin{equation*}
  \hat{\lambda}_i\;:=\;\text{$i$th eigenvalue of }L
  \quad\text{ and }\quad
  \hat{e}_i\;:=\;\text{$i$th eigenvector of }L\,.
\end{equation*}
These approximate the spectrum of $\Delta$ in the sense that
\begin{equation*}
  \lambda_i
  \;\approx\;
  \mfrac{2}{\epsilon}\,\hat{\lambda}_i
  \qquad\text{ and }\qquad
  e_i(x)
  \;\approx\;
  \mfrac{2}{\epsilon}\,(\hat{e}_{i})_x
  \qquad\text{ for }x\in\Gamma\;,
\end{equation*}
see \citet{singer2006graph}. Once an eigenvector $\hat{e}_i$ is computed, values of
$\ef_i$ at points ${x\not\in\Gamma}$ can be estimated using standard interpolation methods.
\begin{algorithm}[Computing Fourier basis]{\ }\\[.5em]
  \label{algorithm:spectral}
  \begin{tabular}{cl}
  1.) & Construct the orbit graph~$(\Gamma,E)$.\\
  2.) & Compute the normalized Laplacian matrix~$L$ according to \eqref{eqn:graph-laplacian}.\\
  3.) & Compute eigenvectors~$\hat{\ef}_i$ and eigenvalues~$\hat{\lambda}_i$ of~$L$.\\
  4.) & Return eigenvalues and interpolated eigenfunctions.
  \end{tabular}
\end{algorithm}
Alternatively, the basis can be computed using a Galerkin approach, which
is described in \cref{sec:linear-reps-revisited}. The functions in Figures \ref{fig:fourier:p1}, \ref{fig:fourier:p6}
and \ref{fig:fourier:I23} are computed using the Galerkin method.

\begin{remark}[Reflections and Neumann boundary conditions]
  \label{remark:neumann}
  The orbit graph automatically enforces the boundary condition \eqref{pbc}, since it measures distance in terms of $d_\group$.
  The exception are group elements that are reflections, since these imply an additional property that the graph does not resolve:
  If $\phi$ is a reflection over a facet $S$, $x$ a point on $S$ (and hence $\phi x=x$), and $f$ a $\phi$-invariant smooth function,
  we must have ${\nabla f(x)=-\nabla f(\phi x)}$, and hence $\nabla f=0$ on $S$. In the parlance of PDEs, this is a \kword{Neumann boundary condition},
  and can be enforced in several ways:\\
  1) For each point ${x_j\in\Gamma}$ that is on $S$, add a point $x_{j'}$ to $\Gamma$ and
  the edge $(x_j,x_j')$ to $E$. Then constrain each eigenvector $e_i$ in \cref{algorithm:spectral} to satisfy ${e_{ij}=e_{ij'}}$.
  This approach is common in spectral graph theory (e.g., \citet{chung1997spectral}).
  \\
  2) Alternatively, one may symmetrize the orbit graph: For vertext $x_j$ that is close to $S$, add its reflection ${x_{j'}:=\phi(x_j)}$ to
  $\Gamma$. Now construct the edge set according to $d_\group$ using the augmented vertex set, and again constrain eigenvectors to
  satisfy ${e_{ij}=e_{ij'}}$.
  Either constrained eigenvalue problem can be solved using techniques of \citet{golub1973some}.
\end{remark}

\section{Nonlinear representation: Factoring through an orbifold}
\label{sec:nonlinear}

\begin{figure}[t]
  \begin{center}
    \begin{tikzpicture}
      \begin{scope}
      \node at (.1,0) {\includegraphics[height=2.794cm]{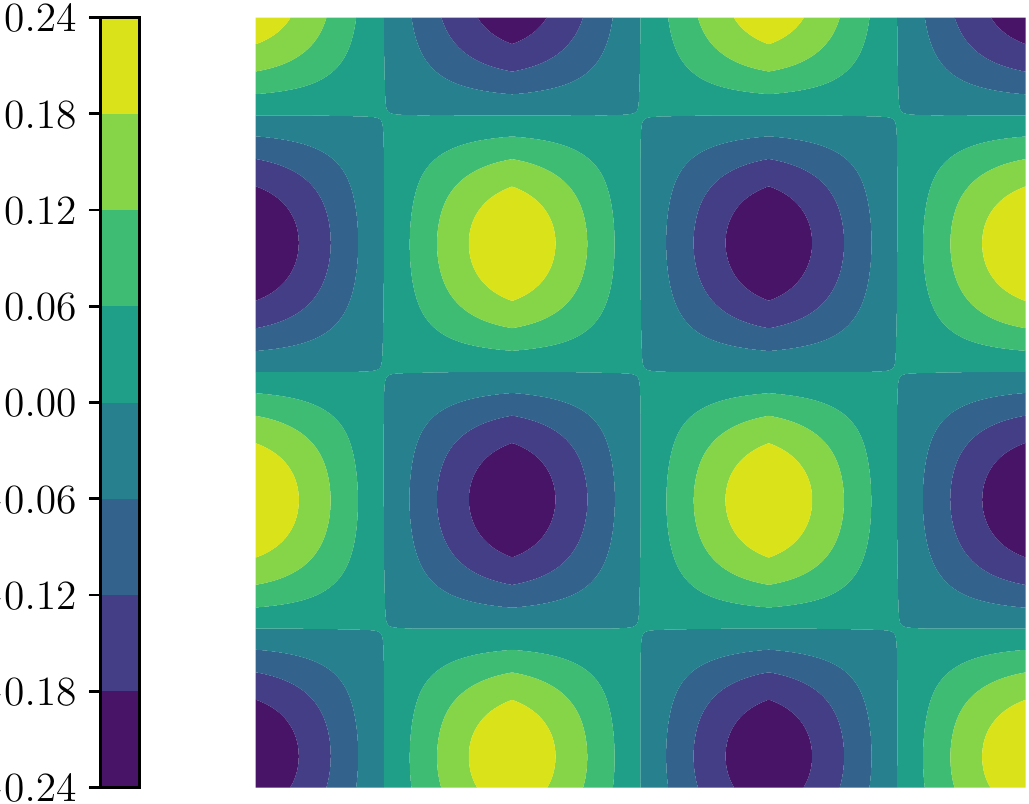}};
      \node at (3.5,0) {\includegraphics[height=2.7cm]{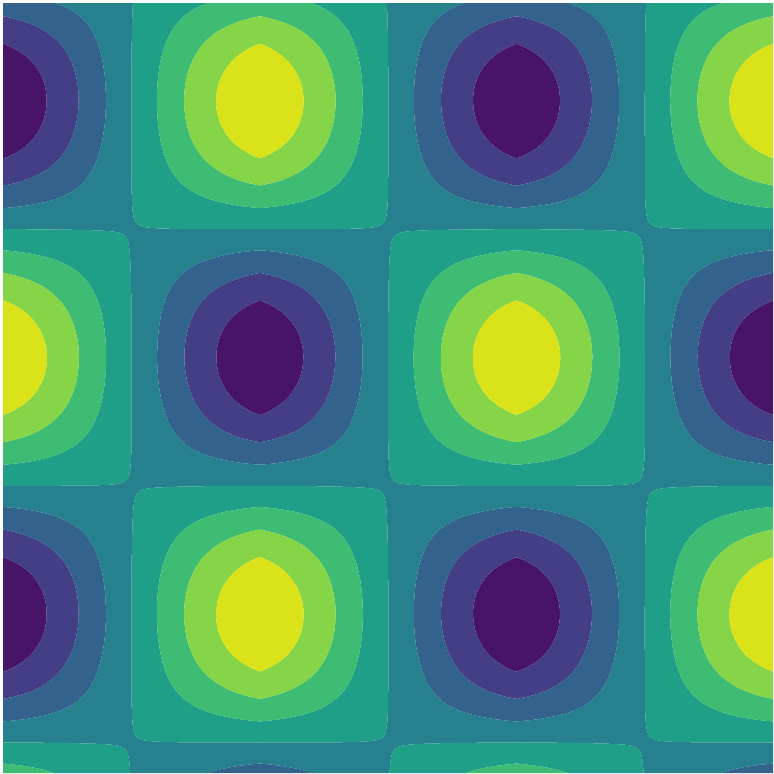}};
      \node at (6.5,0) {\includegraphics[height=2.7cm]{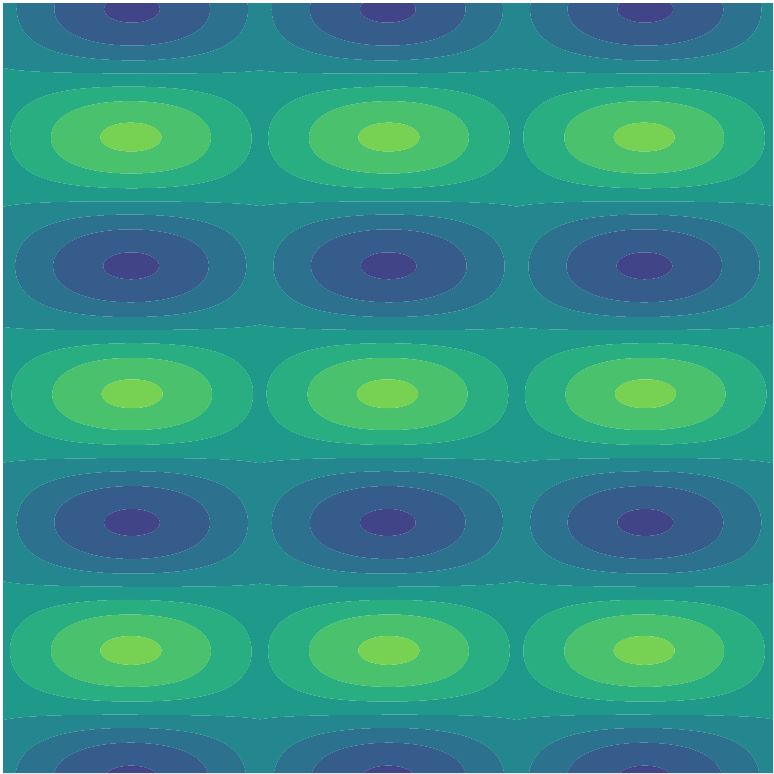}};
      \node at (9.5,0) {\includegraphics[height=2.7cm]{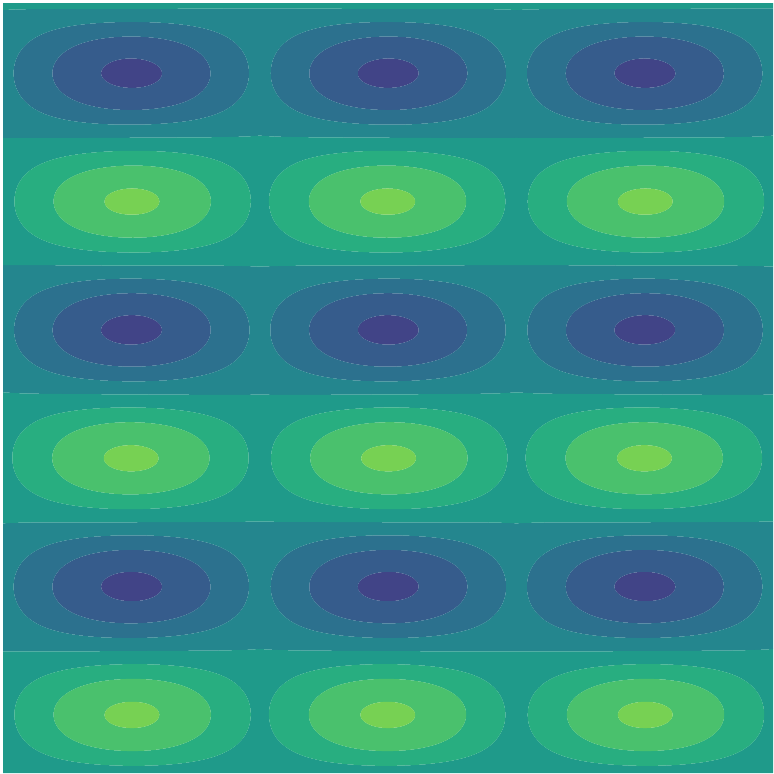}};
      \end{scope}
      \begin{scope}[yshift=-3cm]
      \node at (.1,0) {\includegraphics[height=2.794cm]{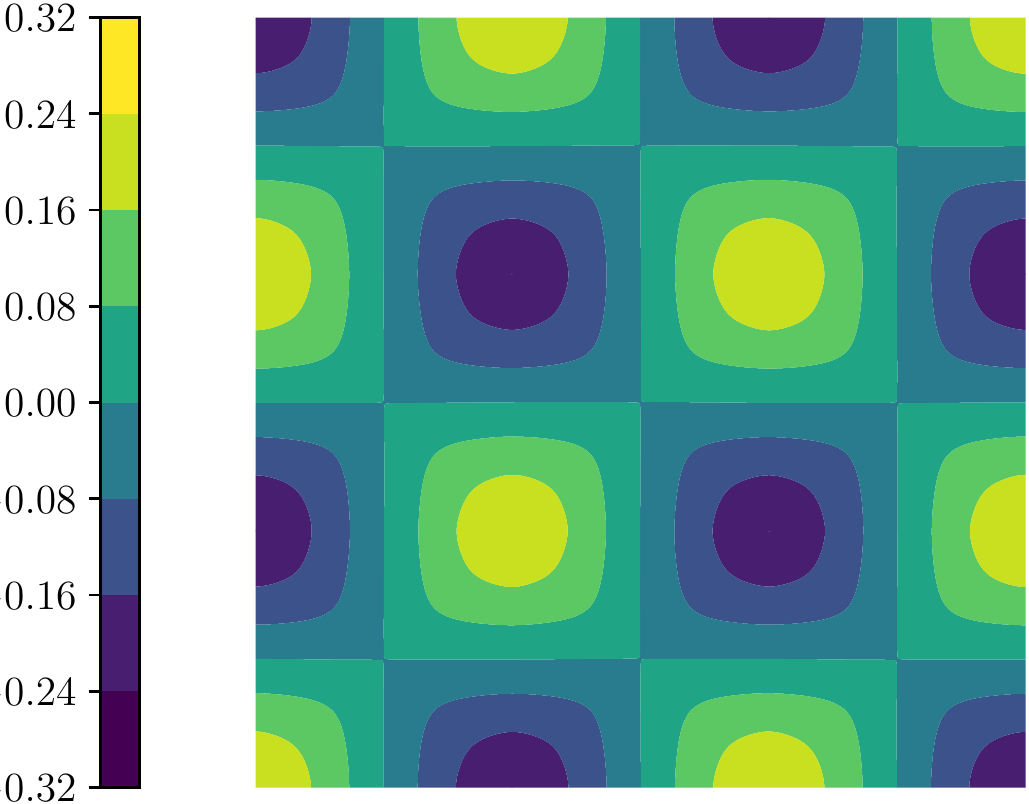}};
      \node at (3.5,0) {\includegraphics[height=2.7cm]{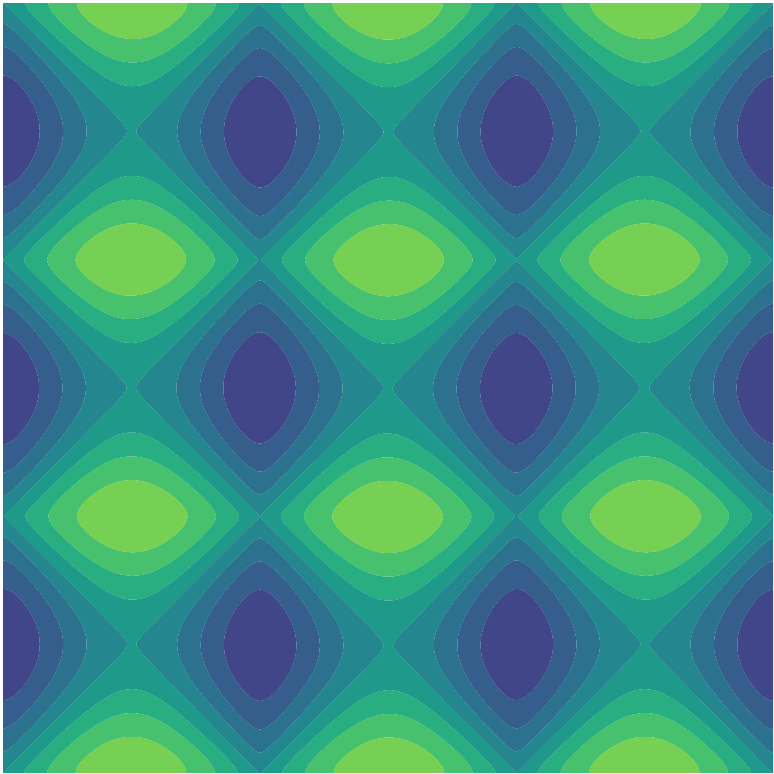}};
      \node at (6.5,0) {\includegraphics[height=2.7cm]{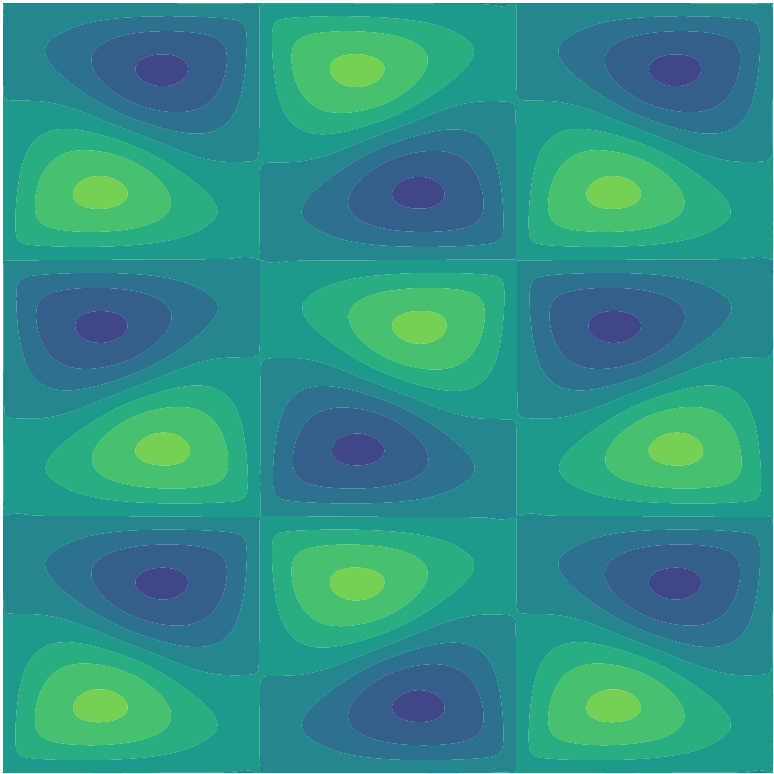}};
      \node at (9.5,0) {\includegraphics[height=2.7cm]{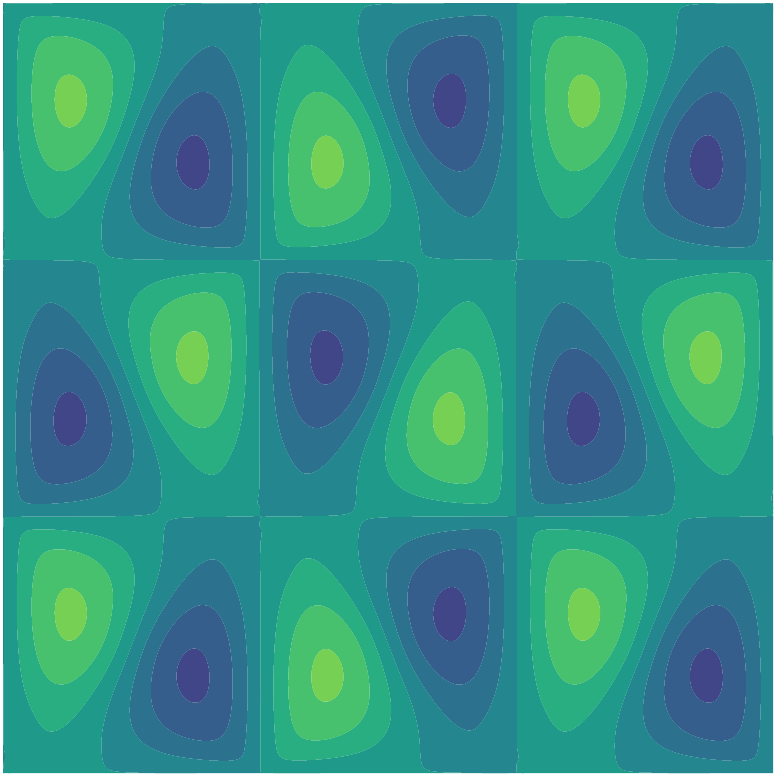}};
      \end{scope}
    \end{tikzpicture}
  \end{center}
  \vspace{-1em}
\caption{Visualizations of four-dimensional orbifold embeddings for wallpaper groups \texttt{cm} (top) and \texttt{p2gg} (bottom).  Compare the \texttt{p2gg} embedding to the three-dimensional visualization in Figure \ref{fig:orbifolds}.}
\label{fig:plane-embeddings}
\end{figure}

We now generalize MacKay's construction, as sketched in the introduction, from shifts to
crystallographic groups. The construction defines a map
\begin{equation*}
  \rho|_\Pi:\,\Pi\rightarrow\Omega\subset\mathbb{R}^N
  \quad\text{ which in turn defines }\quad
  \rho:=\rho_\Pi\circ p:\mathbb{R}^n\rightarrow\Omega\;.
\end{equation*}
In MacKay's case, $\Pi$ is an interval and ${\Omega\in\mathbb{R}^2}$ a circle.
The circle can be obtained from~$\Pi$ by ``gluing'' the ends of the interval to each other.
To generalize this idea, we proceed as follows: Starting with the polytope $\Pi$,
we find any pair of points $x$ and $y$ on the same orbit of $\group$, and ``bend'' $\Pi$
so that we can glue $x$ to $y$. That results in a surface $\Omega$ in $\mathbb{R}^N$, where ${N\geq n}$
since we have bent $\Pi$.
If we denote the point on $\Omega$ that corresponds to ${x\in\Pi}$ by
${\rho|_\Pi(x)}$, we obtain the maps above. We first show how to implement this construction numerically,
and then consider its mathematical properties.
In mathematical terms, the surface $\Omega$ is an \kword{orbifold}, a concept that generalizes
the notion of a manifold. The term~$\group$-orbifold is made precise in
in \cref{sec:orbifolds}, but can be read throughout this section as a surface in $\mathbb{R}^N$
that is ``smooth almost everywhere''.

\subsection{Gluing algorithms }
\label{sec:embedding:algorithm}

The gluing algorithm constructs numerical approximations $\widehat{\rho}$ of $\rho$ and
$\widehat{\Omega}$ of $\Omega$. Here, $\widehat{\Omega}$ is a surface in $\widehat{N}$ dimensions,
where (as we explain below) ${\widehat{N}}$ may be larger than $N$.
As in the linear formulation of \cref{sec:linear}, we start with the orbit graph ${\mathcal{G}=(\Gamma,E)}$,
but in this case weight the edges to obtain a weighted graph
\begin{equation*}
  \mathcal{G}_w\;=\;(\Gamma,E_w)
  \qquad\text{ with }\qquad
  \text{weight}(x,y):=\begin{cases}d_\group(x,y) & \text{if }(x,y)\in E\\
  0 & \text{if }(x,y)\not\in E
  \end{cases}\;.
\end{equation*}
The weighted graph provides approximate distances in quotient space.
The surface $\widehat{\Omega}$ is constructed from this graph by multidimensional scaling
(MDS) \cite{kruskal1964multidimensional}.
MDS proceeds as follows: Let $R$ be the matrix of squared geodesic distances, with entries
\begin{equation*}
  R_{ij}\;=\;(\text{weighted path length from $x_i$ to $x_j$ in $\mathcal{G}_w$})^2\;.
\end{equation*}
Let ${0<\delta_1\leq\ldots\leq\delta_{|\Gamma|}}$ be the eigenvalues and ${v_1,\ldots,v_{|\Gamma|}}$ the eigenvectors of the matrix
\begin{align*}
  \tilde{R} &= -\frac{1}{2}\Bigl(\mathbb{I} - \mfrac{1}{|\Gamma|}\mathbb{J}\Bigr)\,R\,\Bigl(\mathbb{I} - \mfrac{1}{|\Gamma|}\mathbb{J}\Bigr)
  \qquad\text{ where }
  \mathbb{I}=\text{diag}(1,\ldots,1)\text{ and }
  \mathbb{J}={\scriptsize
  \begin{pmatrix} 1 & \cdots & 1\\ \svdots & & \svdots\\ 1 & \cdots &1
  \end{pmatrix}}\;.
\end{align*}
The embedding of each point ${x_i}$ in the $\varepsilon$-net $\Gamma$ is
then given by
\begin{equation*}
  \widehat{\rho}(x_i)\;:=\;
\left(\begin{matrix}
	\sqrt{\delta_{|\Gamma|}}\,v_{|\Gamma|,i}\\[.1em]
	\sqrt{\delta_{|\Gamma|-1}}\,v_{|\Gamma|-1,i}\\
	\vdots\\
	\sqrt{\delta_{\smash{|\Gamma|-\widehat{N}}}}\,v_{|\Gamma|-\widehat{N},i}
\end{matrix}\right)\,.
\end{equation*}
The dimension $\widehat{N}$ is chosen to minimize error in the distances.
From ${\smash{\widehat{\rho}}(x_1),\ldots,\smash{\widehat{\rho}}(x_{|\Gamma|})}$, the surface $\smash{\widehat{\Omega}}$ and the map
$\smash{\widehat{\rho}|_\Pi}$ are obtained by interpolation.
\begin{figure}[t]
  \begin{center}
    \begin{tikzpicture}
      \begin{scope}
        \node at (0,0) {\includegraphics[width=3cm]{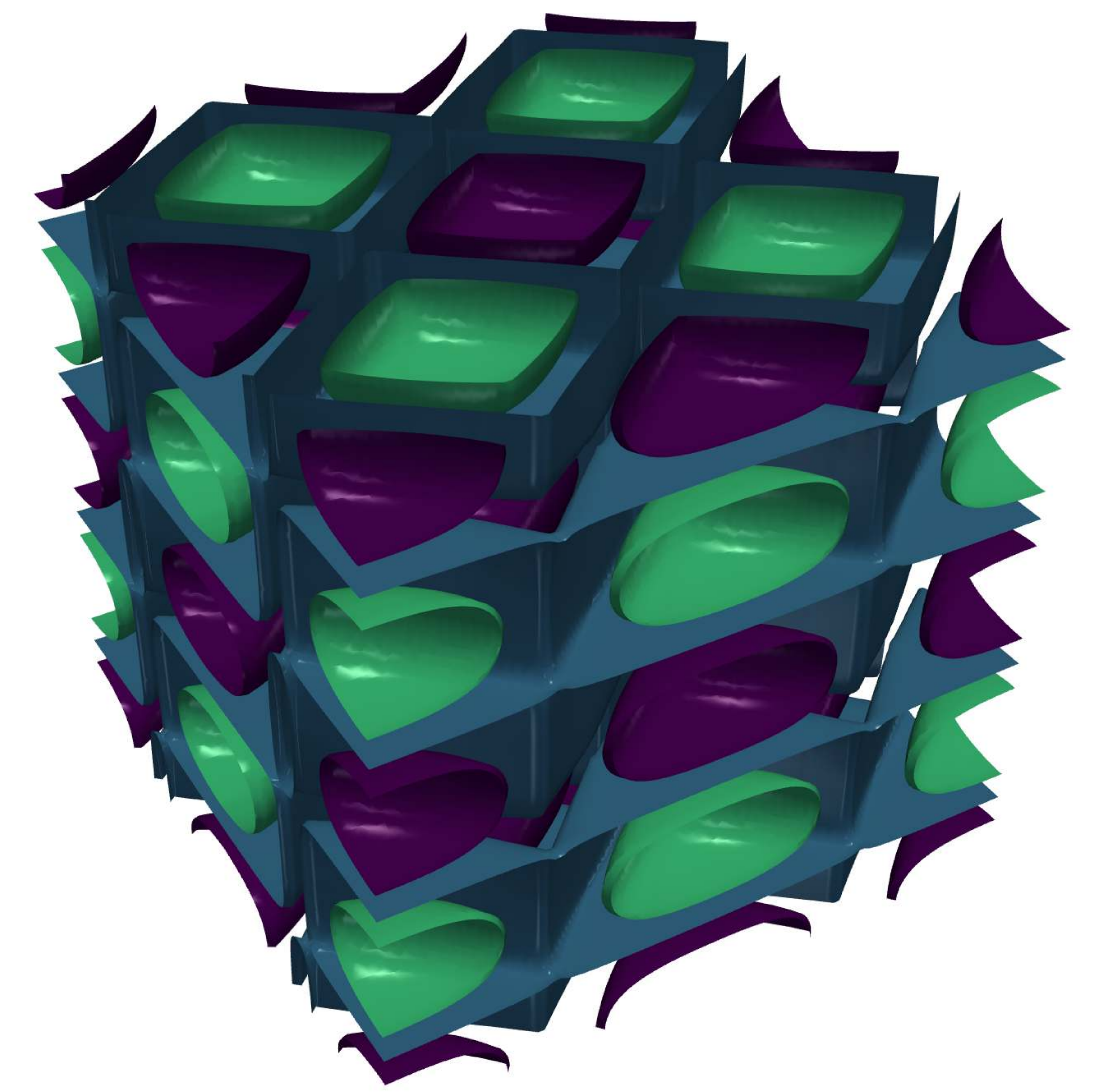}};
        \node at (3,0) {\includegraphics[width=3cm]{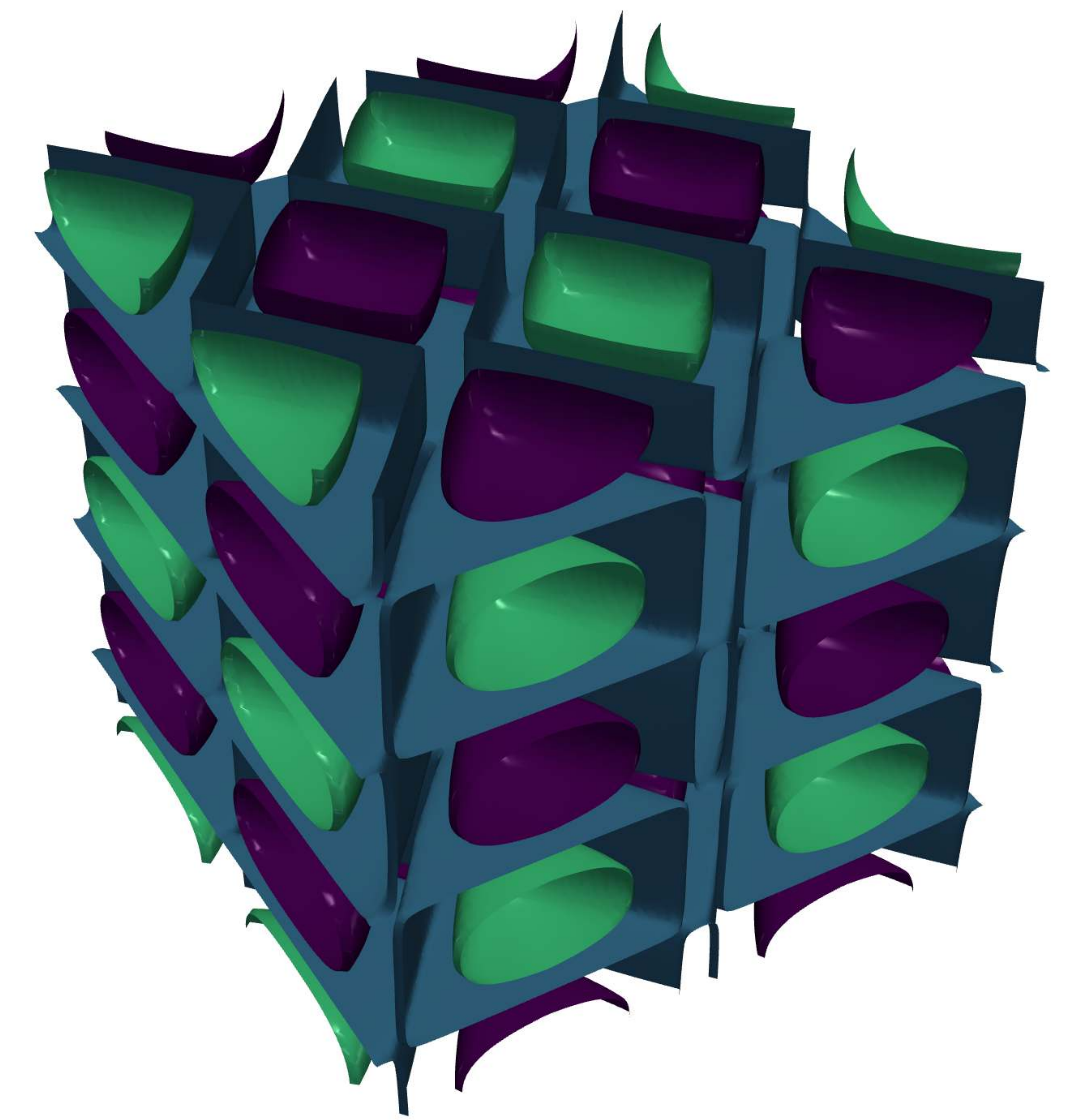}};
        \node at (6,0) {\includegraphics[width=3cm]{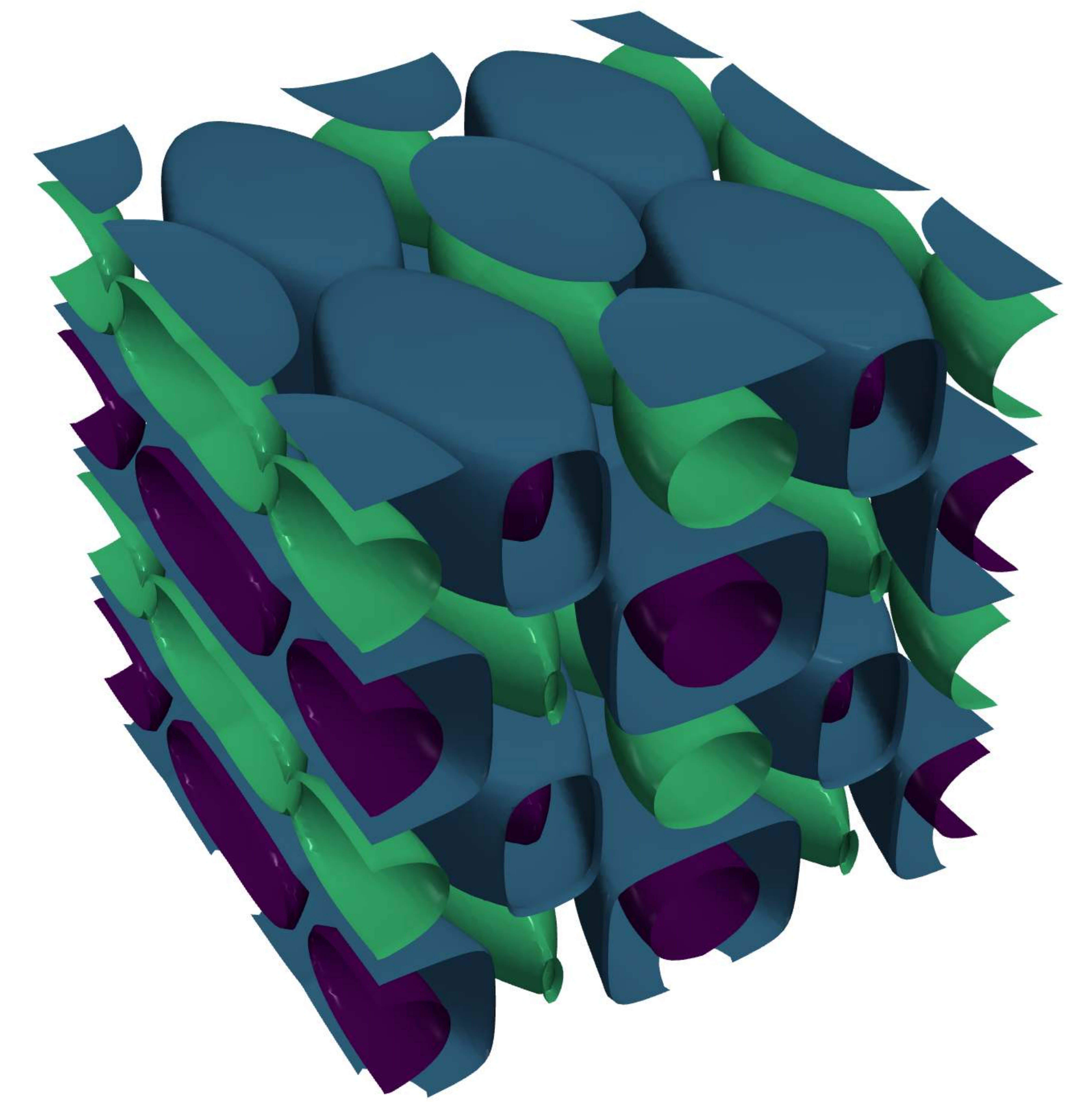}};
        \node at (9,0) {\includegraphics[width=3cm]{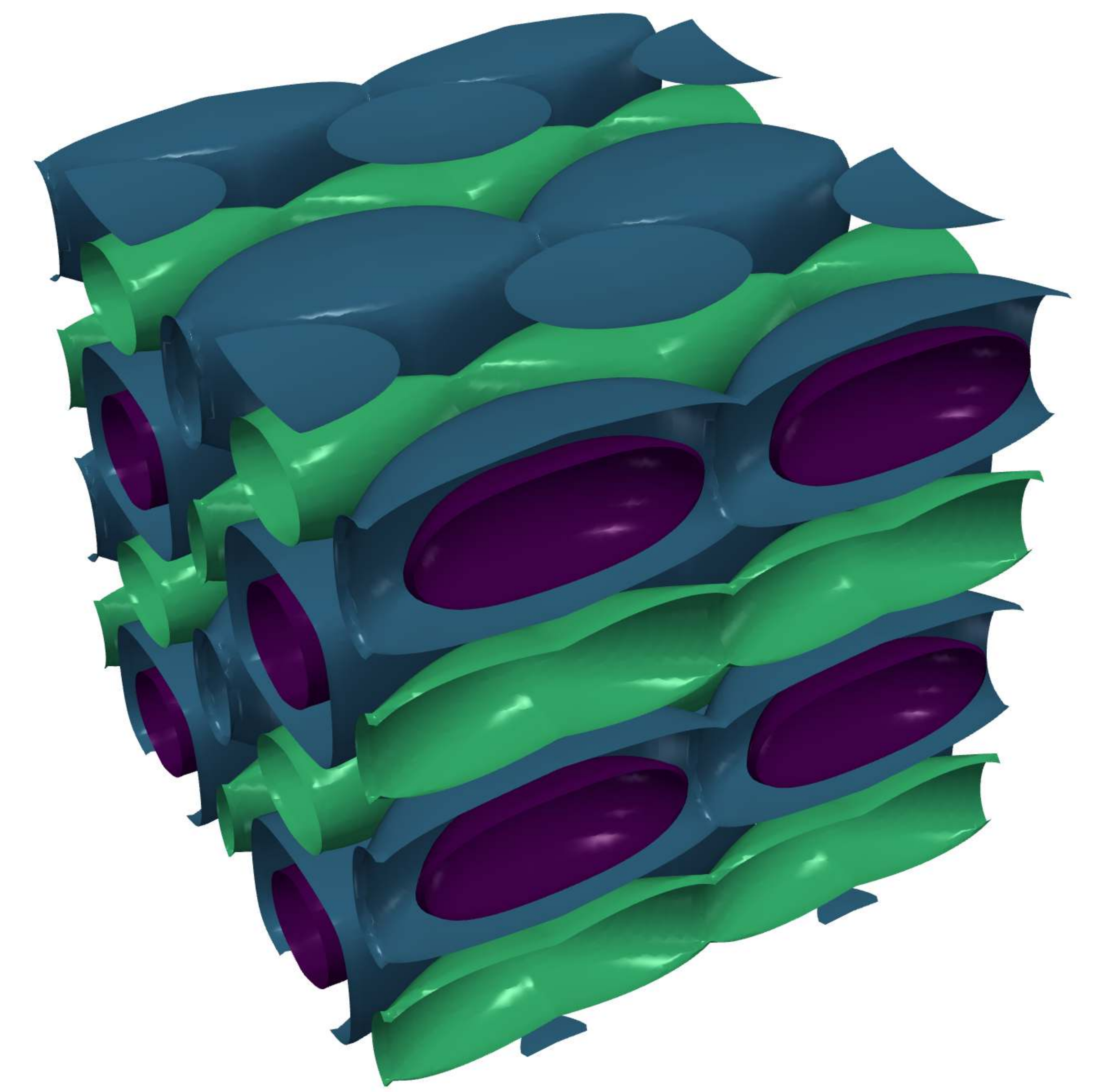}};
        \node at (12,0) {\includegraphics[width=3cm]{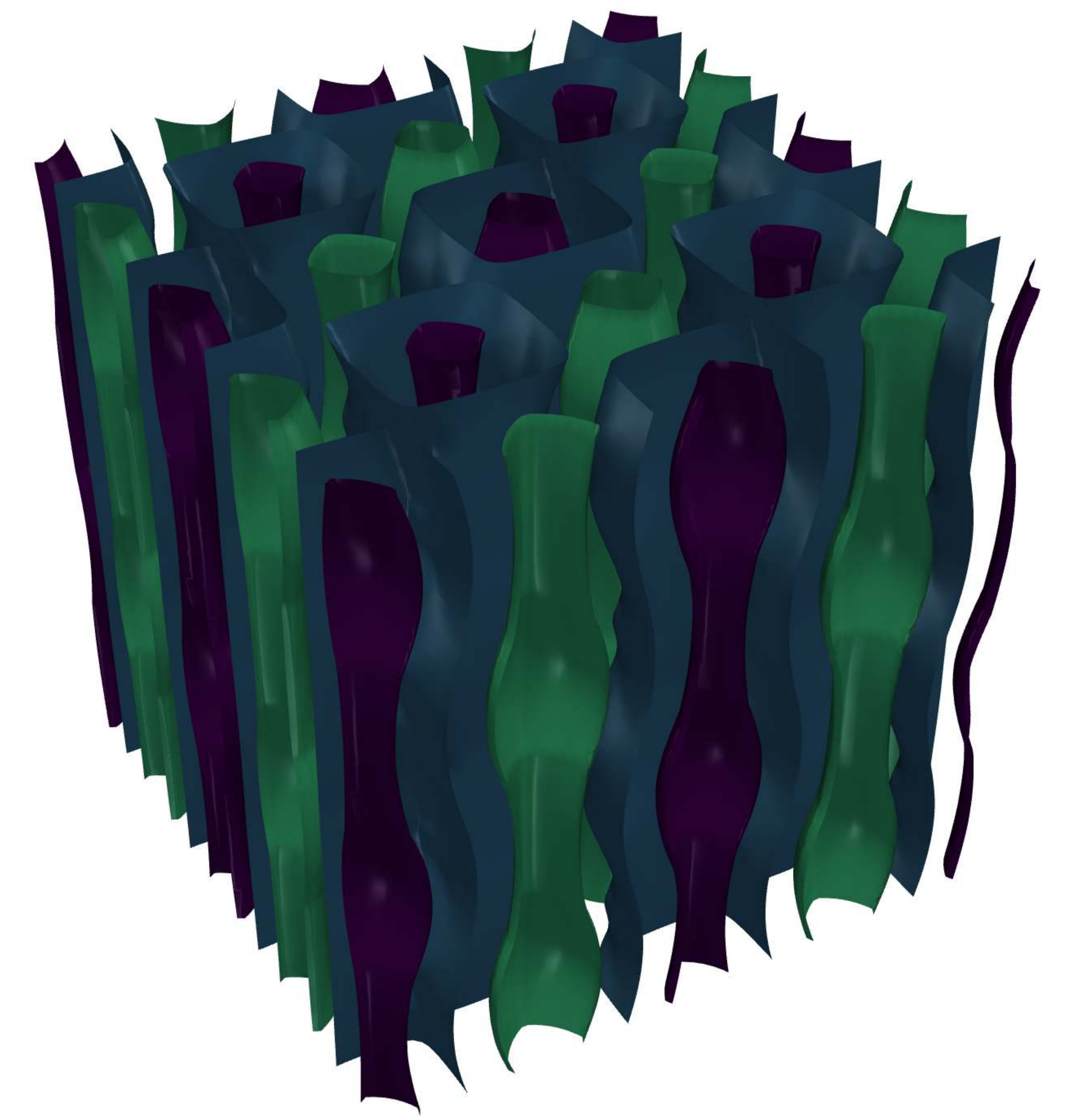}};
      \end{scope}
      \begin{scope}[yshift=-3.5cm]
        \node at (0,0) {\includegraphics[width=3cm]{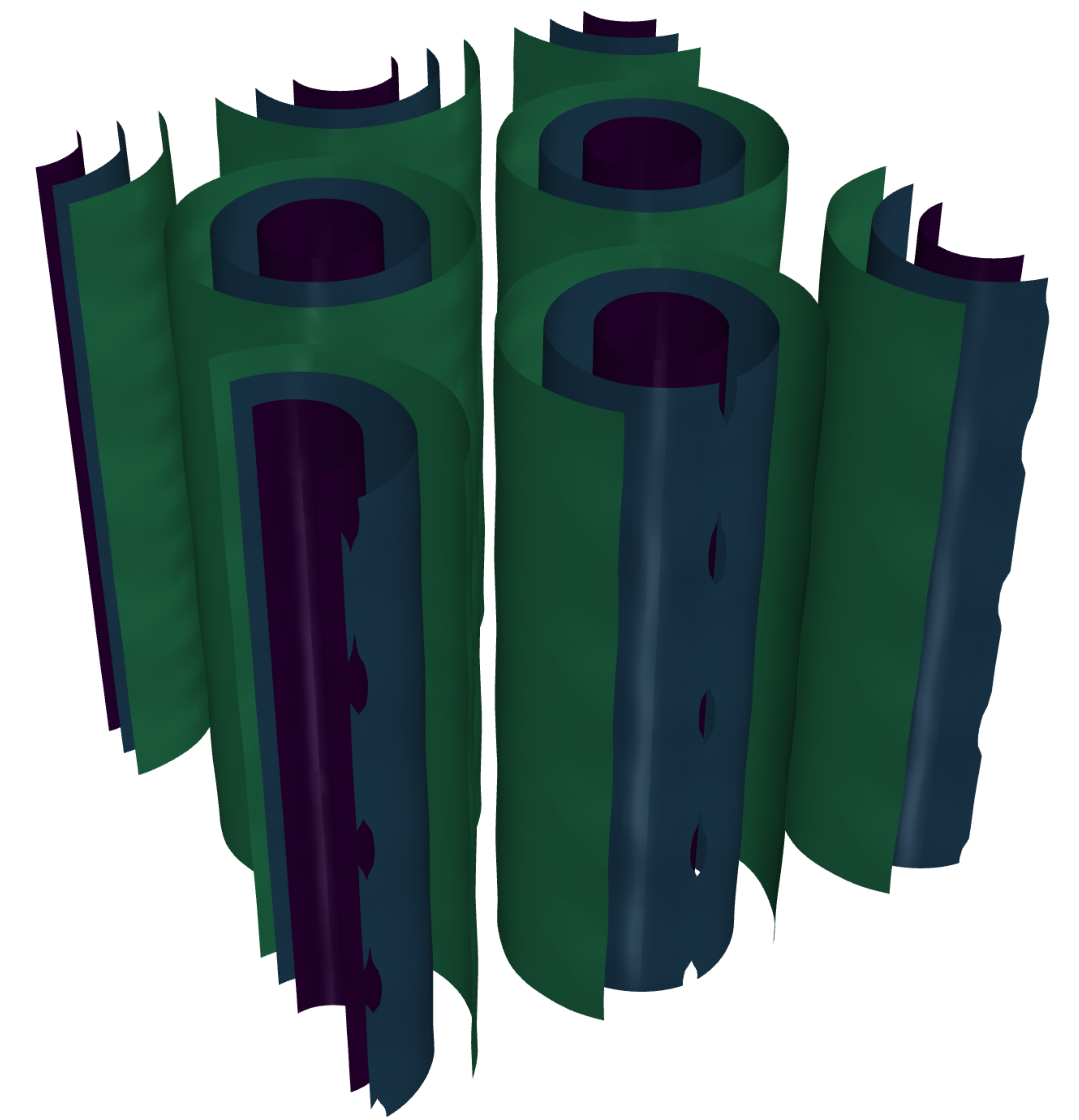}};
        \node at (3,0) {\includegraphics[width=3cm]{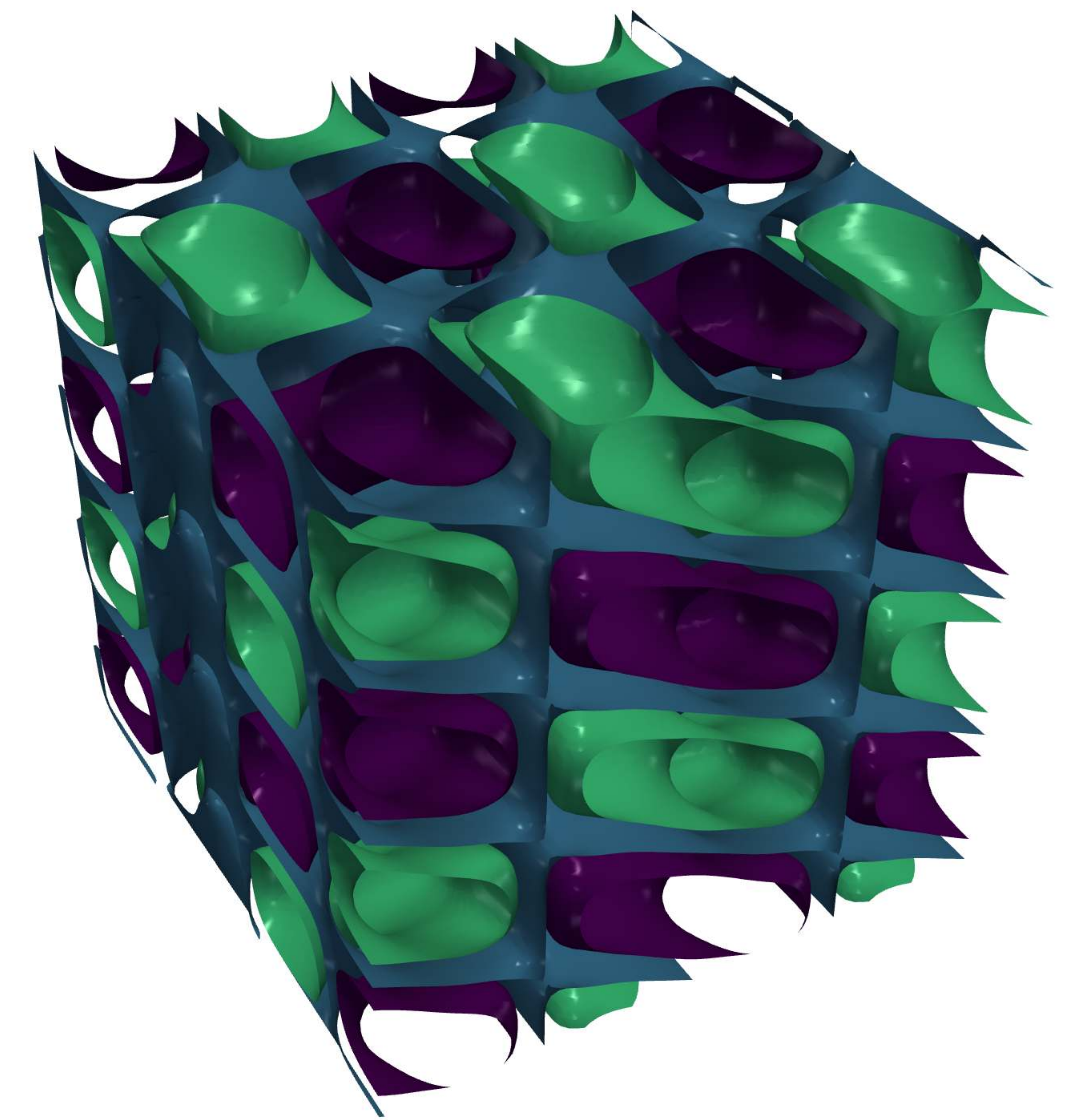}};
        \node at (6,0) {\includegraphics[width=3cm]{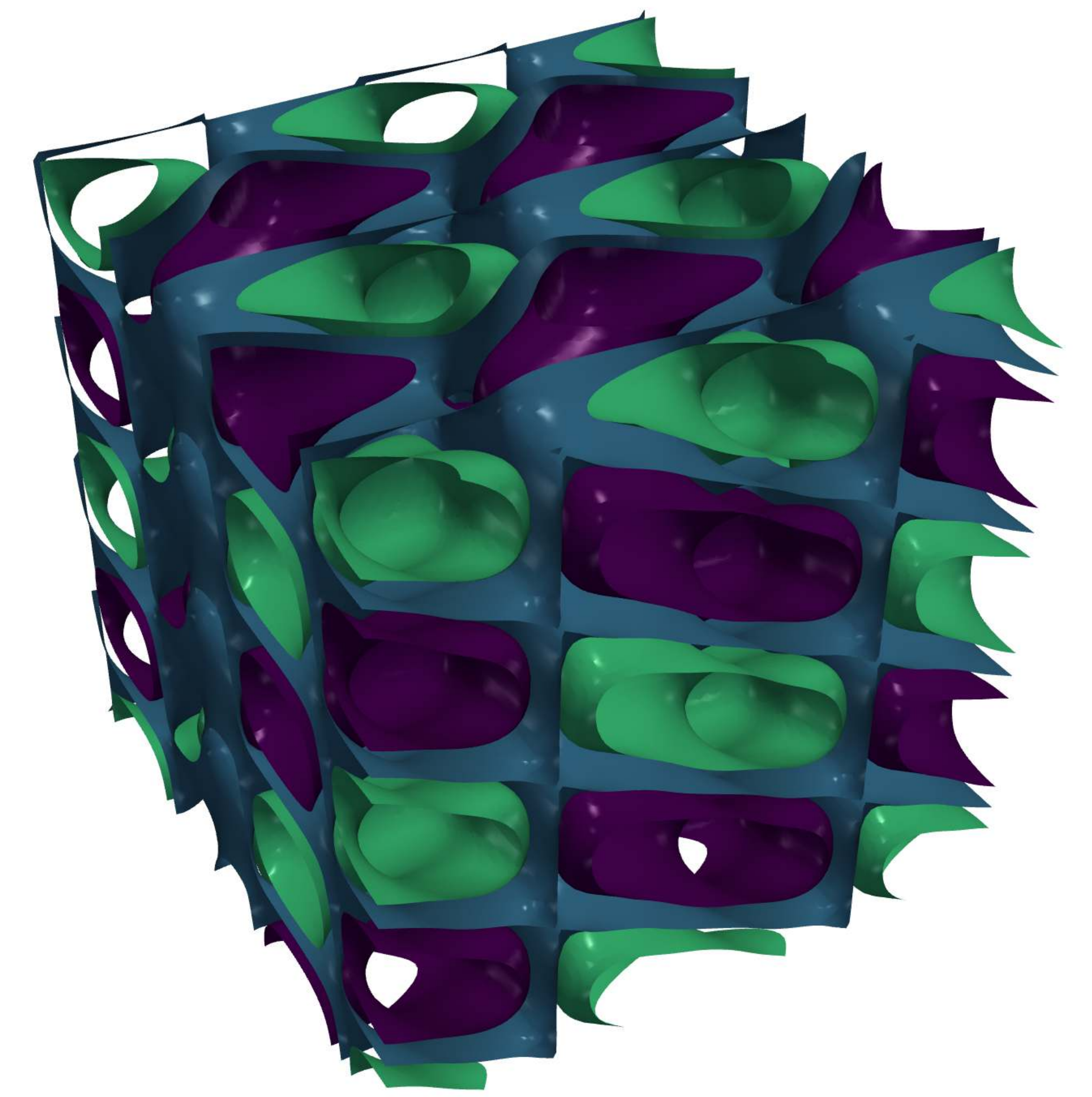}};
        \node at (9,0) {\includegraphics[width=3cm]{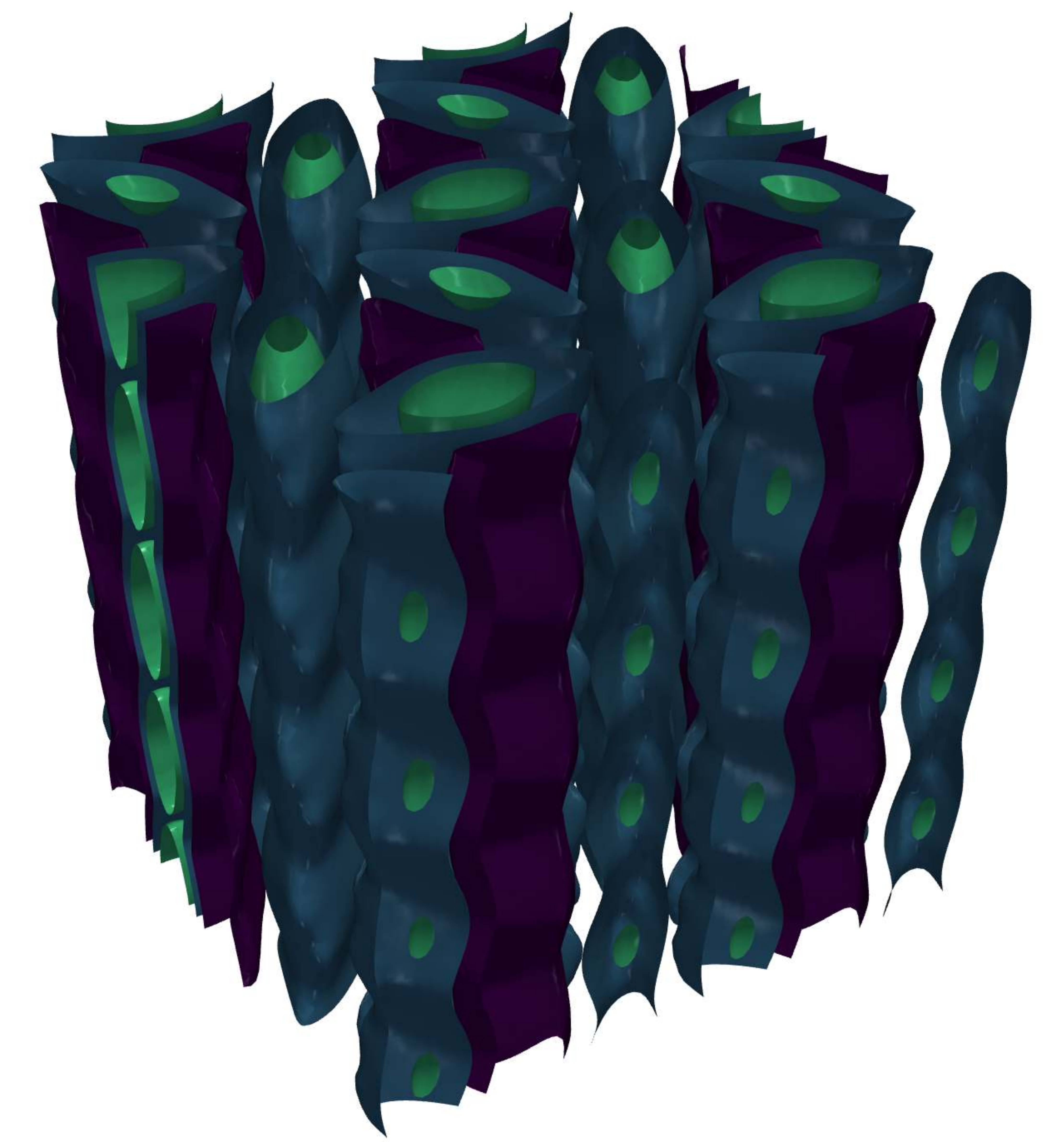}};
        \node at (12,0) {\includegraphics[width=3cm]{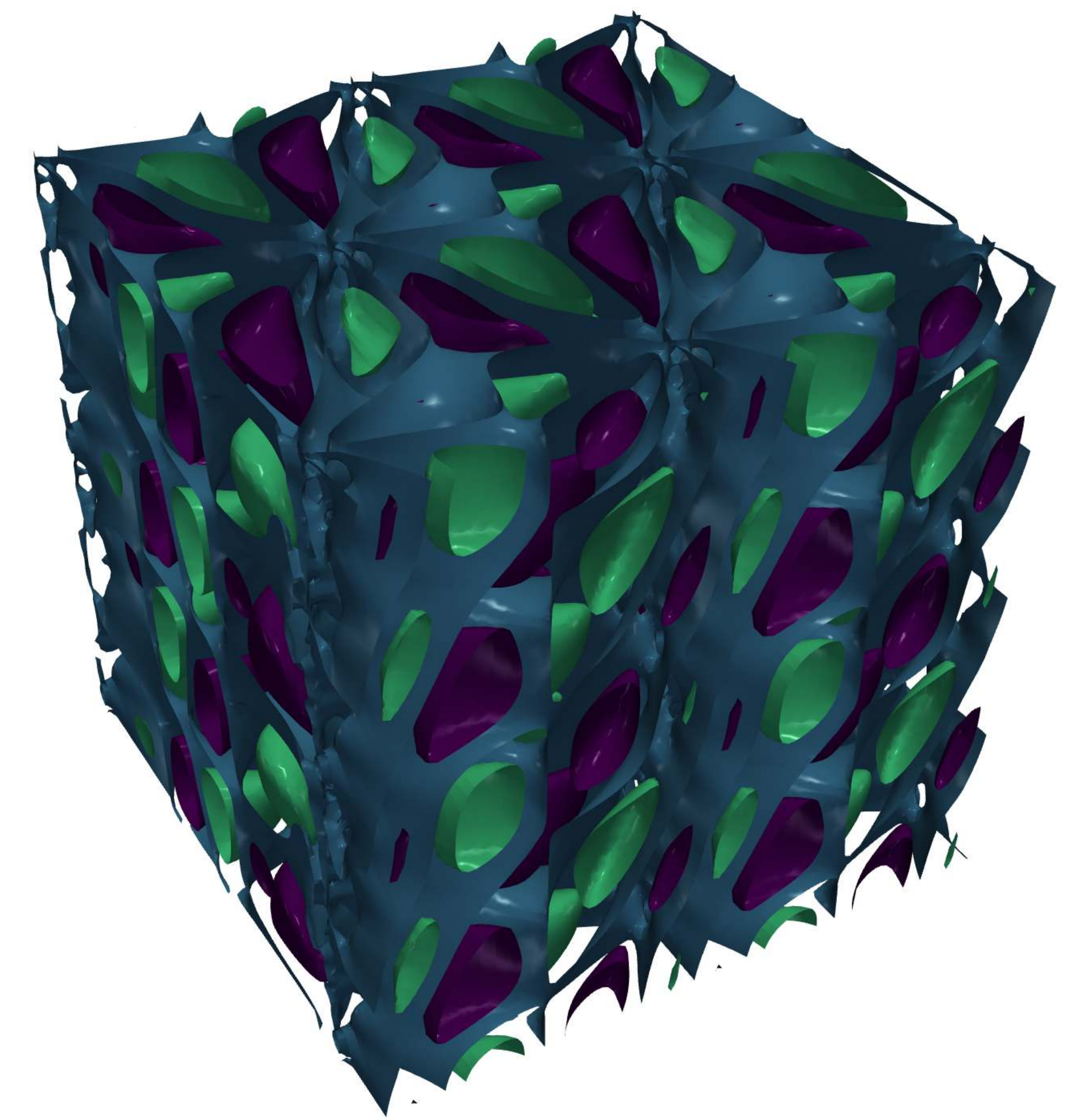}};
      \end{scope}
    \end{tikzpicture}
  \end{center}
    \vspace{-1em}
\caption{Visualizations of five-dimensional orbifold embeddings for \texttt{Imm2} (top) and \texttt{P6$_{\texttt{5}}$} (bottom). }
\label{fig:space-embeddings}
\end{figure}

\begin{figure}[t]
  \begin{center}
    \begin{tikzpicture}
      \begin{scope}
        \node at (0,0)   {\includegraphics[width=4cm]{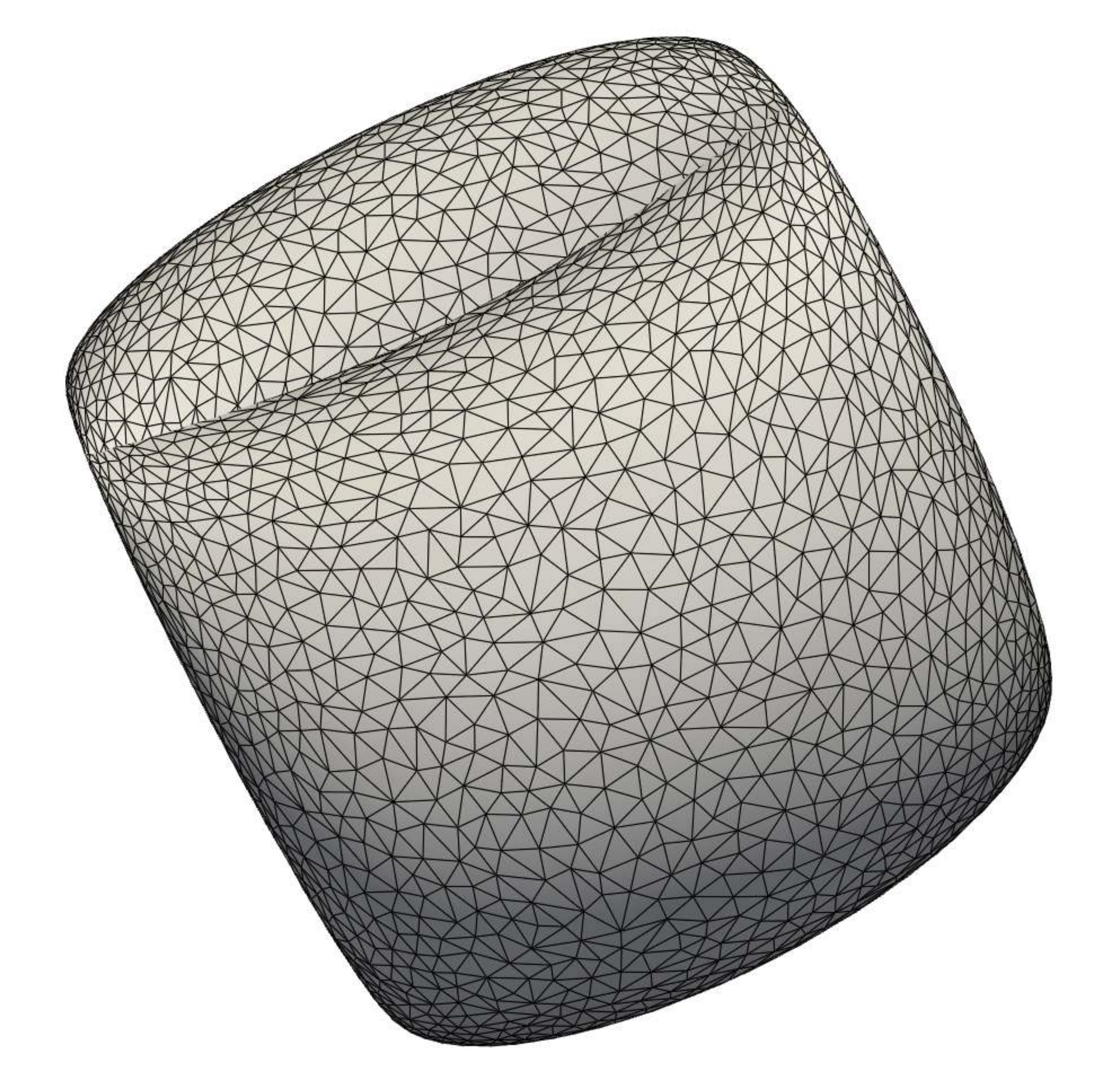}};
        \node at (5.5,0) {\includegraphics[width=4cm]{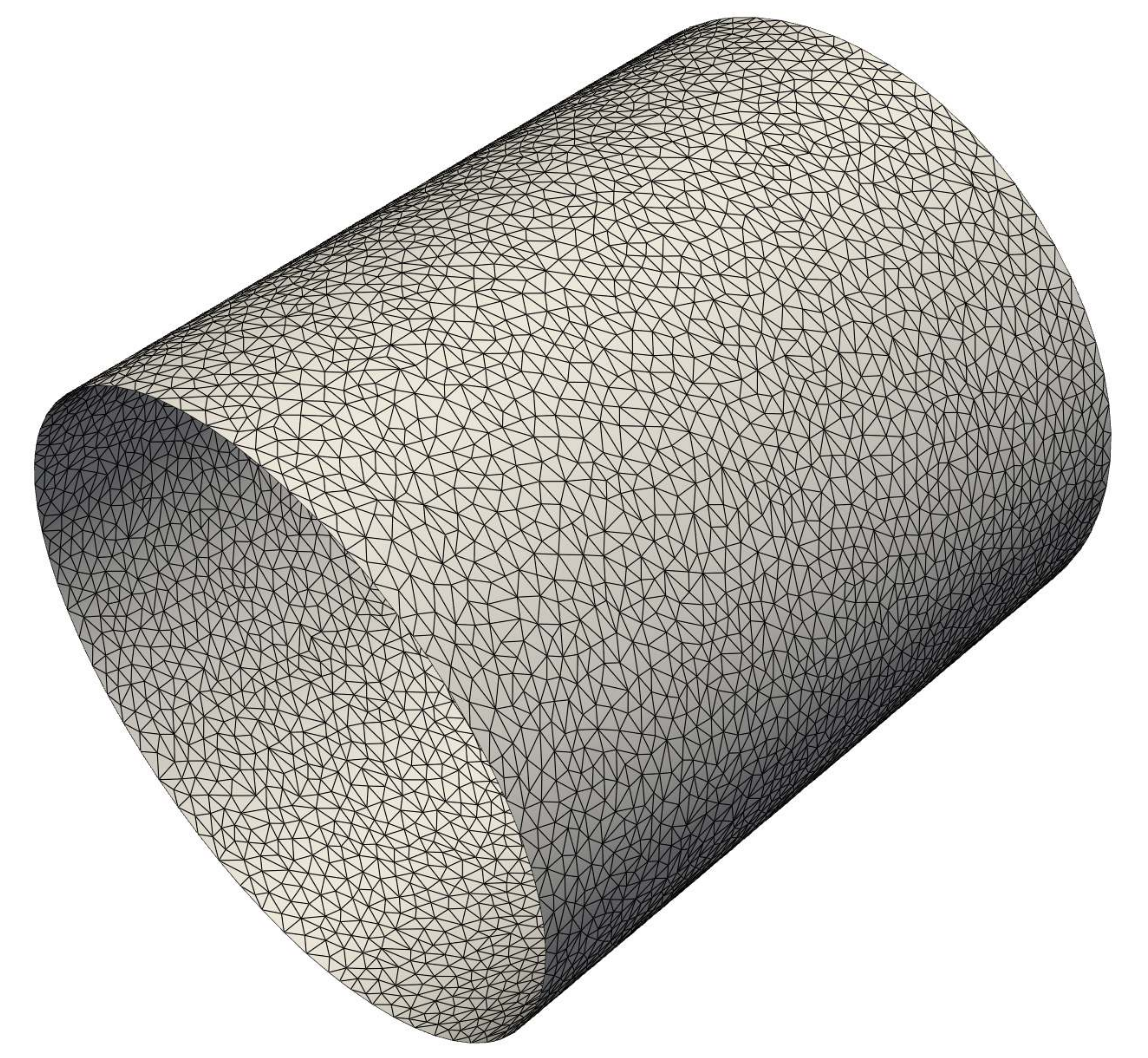}};
        \node at (11,0)  {\includegraphics[width=4cm]{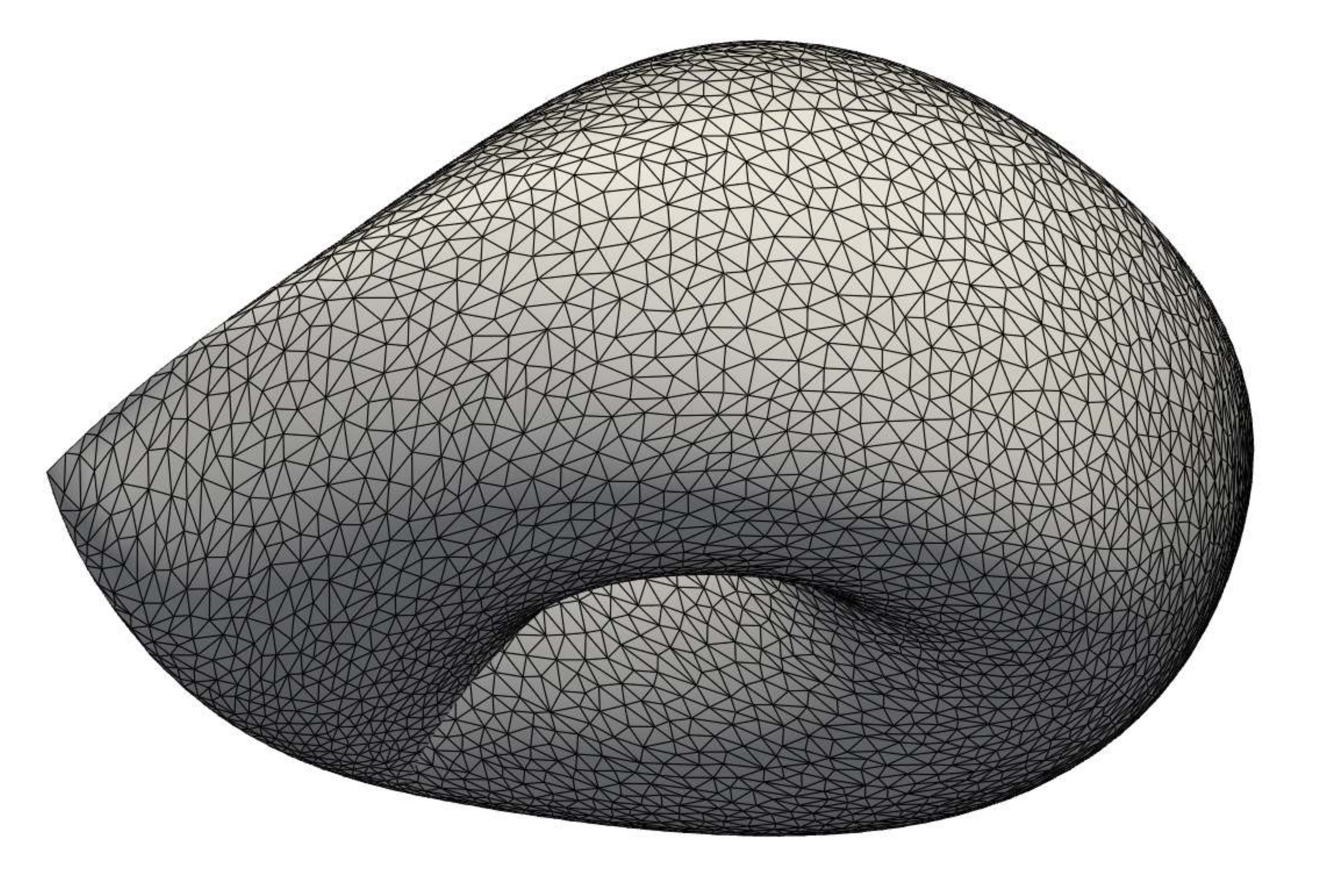}};
        \node at (0,2.2)  {\scriptsize p1};
        \node at (5.5,2.2){\scriptsize pm};
        \node at (11,2.2) {\scriptsize p2gg};
      \end{scope}
      \begin{scope}[yshift=-4.6cm]
        \node at (0,0)   {\includegraphics[height=4cm]{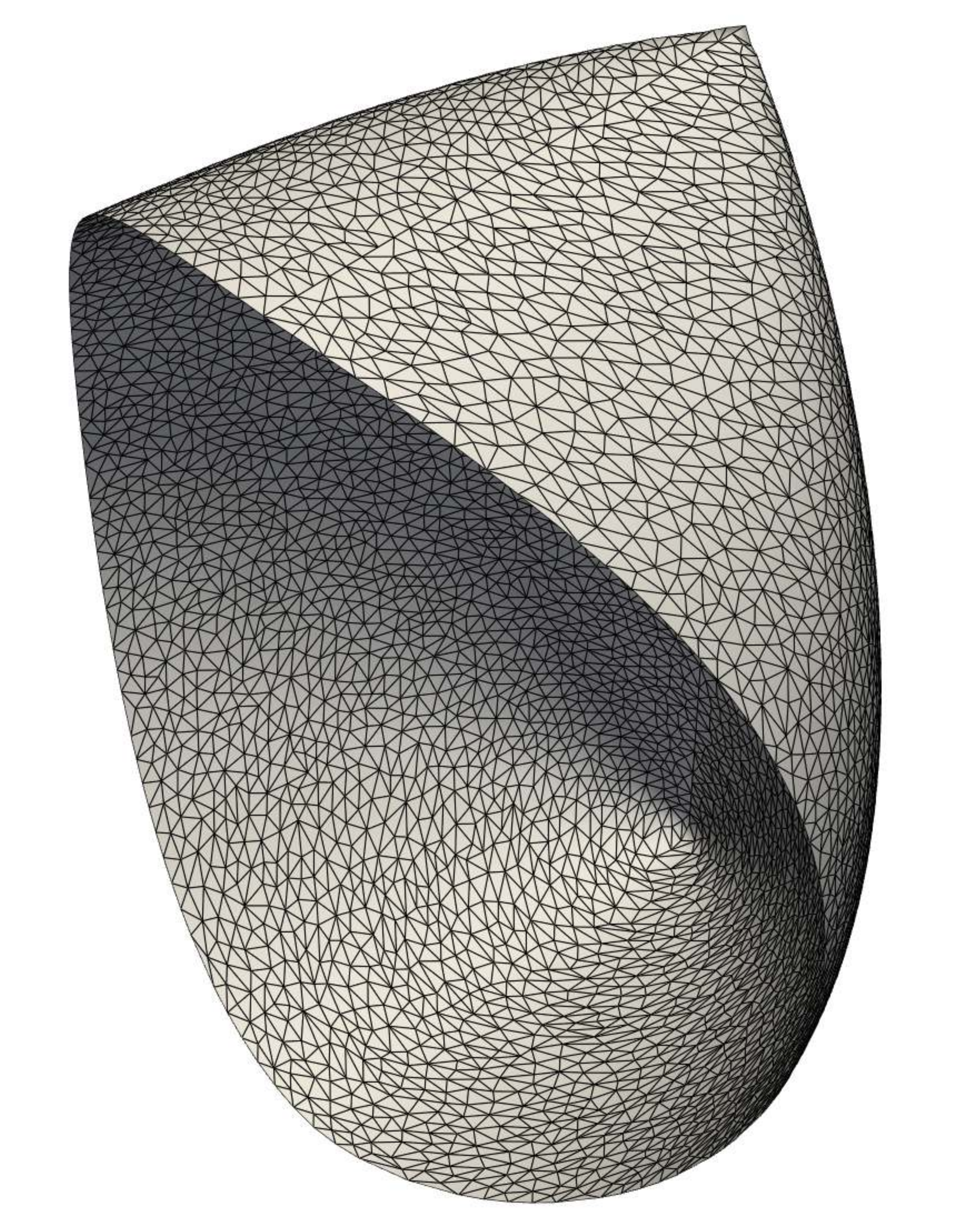}};
        \node at (5.5,0) {\includegraphics[height=4cm]{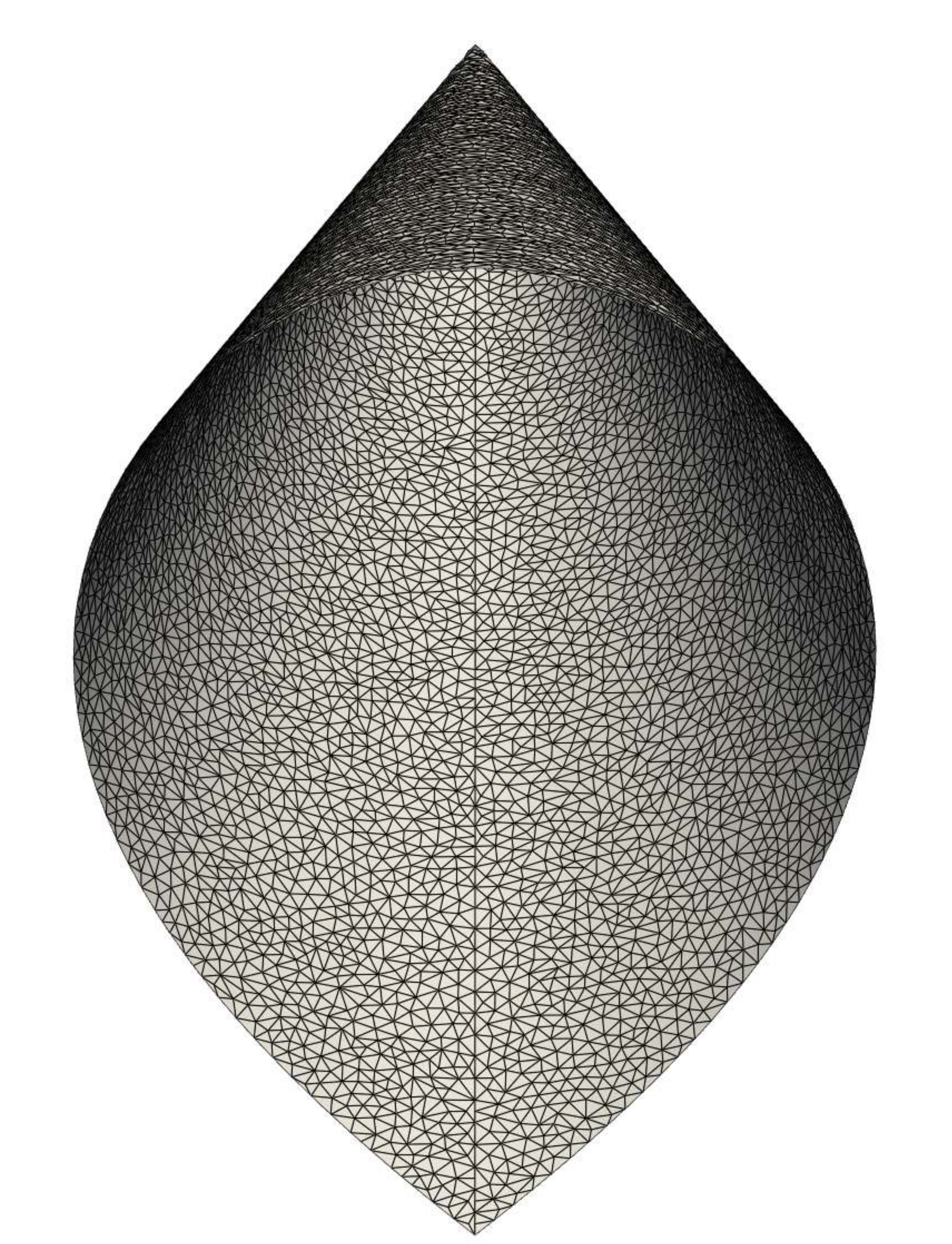}};
        \node at (11.5,-.7)  {\includegraphics[width=6cm,angle=225]{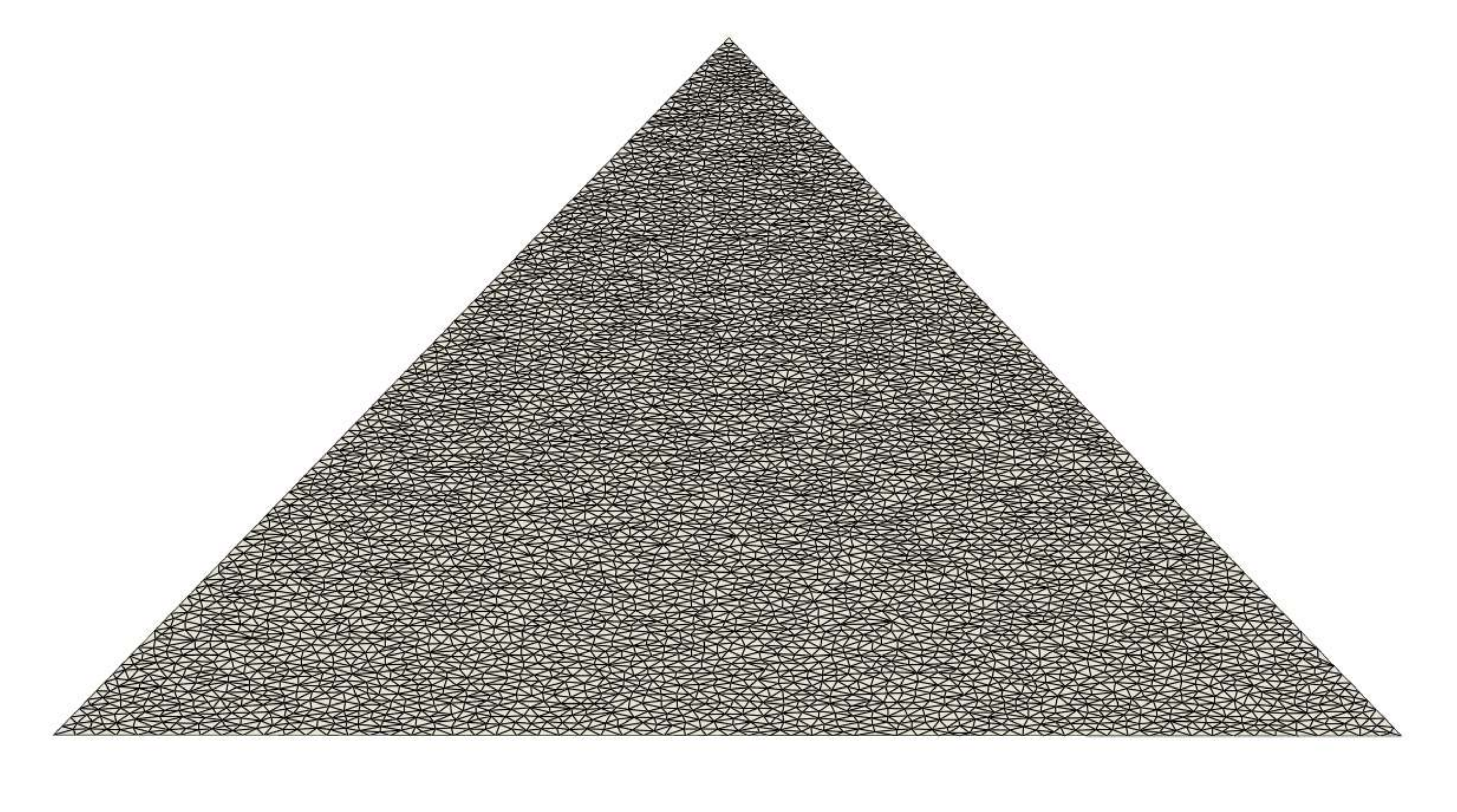}};
        \node at (0,-2.3)   {\scriptsize p2mg};
        \node at (5.5,-2.3) {\scriptsize p4gm};
        \node at (11,-2.3)  {\scriptsize p3m1};
      \end{scope}
    \end{tikzpicture}
  \end{center}
  \vspace{-1.8cm}
  \caption{Visualization of several plane group orbifolds in three dimensions, computed as
    embeddings of orbit graphs as described in \cref{sec:embedding:algorithm}. The mesh on each orbifold is the image of
    the orbit graph.
    Note that reflections result in boundaries for \texttt{pm}, \texttt{p2mg}, \texttt{p4mg}, and \texttt{p3m1}, and that in particular all of the
    edges for \texttt{p3m1} are boundaries.  The self-intersections visible for \texttt{p2gg} are an artifact of visualizing in three dimensions---for a sufficiently high dimension, the surface does not self-intersect.}
\label{fig:orbifolds}
\end{figure}

\begin{algorithm}[Gluing with multidimensional scaling]{\ }\\[.5em]
  \label{algorithm:MDS}
  \begin{tabular}{cl}
  1.) & Construct the weighted orbit graph $\mathcal{G}_w$.\\
  2.) & Compute the eigenvalues $\delta_i$ and eigenvectors~$v_i$ of~$\tilde{R}$.\\
  3.) & Compute vertex embeddings ${\smash{\widehat{\rho}}(x_1),\ldots,\smash{\widehat{\rho}}(x_{|\Gamma|})}$.\\
  4.) & Return interpolated vertex embeddings.
  \end{tabular}
\end{algorithm}
Once $\smash{\widehat{\rho}|_\Pi}$ can be computed, we can also compute
  $\widehat{\rho}:=\widehat{\rho}\,|_\Pi\circ p$,
since the projector $p$ can be evaluated using \cref{algorithm:projector}.
\begin{remark}
  The procedure satisfies two desiderata for constructing the orbifold map:
  1) facets to be glued will be brought together, and 2) distances between interior points in~$\Pi$ will be approximately preserved.
  The embedding is unique up to isometric transformations.
  The embedding step is similar to the Isomap \cite{tenenbaum2000global} algorithm, but
  unlike Isomap embeds into a higher-dimensional space rather than a lower-dimensional one.
\end{remark}

\subsection{Example: Invariant neural networks}

Given $\group$ and $\Pi$, compute $\widehat{\rho}$ and $\widehat{\Omega}$ using \cref{algorithm:MDS}.
Choose a neural network
\begin{equation*}
  h_\theta:\mathbb{R}^{\smash{\widehat{N}}}\rightarrow\mathbb{R}
  \quad\text{ with parameter vector }\theta\text{ and set }\quad
  f_\theta:=h_\theta\circ\smash{\widehat{\rho}}\;.
\end{equation*}
Then $f_\theta$ is a real-valued neural network on $\mathbb{R}^n$.
\cref{fig:space-nn} shows examples of $f_\theta$ for ${n=3}$, where
$h_{\theta}$ has three hidden layers of ten units each, with
rectified linear (relu) activations, although the input dimension $\widehat{N}$ may vary according to the choice of
$\group$ and $\Pi$. The parameter vector is generated at random.
\begin{remark}
  Since most ways of performing interpolation in the construction of $\smash{\widehat{\rho}}$ are
  amenable to automatic differentiation tools, this representation is easy to incorporate into machine learning pipelines.
  Moreover, universality results for neural networks (e.g., \citet{hornik1989multilayer}) carry over:
  If a class of neural networks $h_\theta$ approximates to arbitrary precision in~${\C(\mathbb{R}^N)}$, the
  the resulting functions $f_\theta$ approximate to arbitrary precision in~$\C_\group$ (though the approximation rate may change under composition with $\rho$). See \cref{corollary:C:isometries}.
\end{remark}

\begin{figure}[t]
  \begin{center}
    \begin{tikzpicture}
      \node at (0,0) {\includegraphics[width=3.7cm]{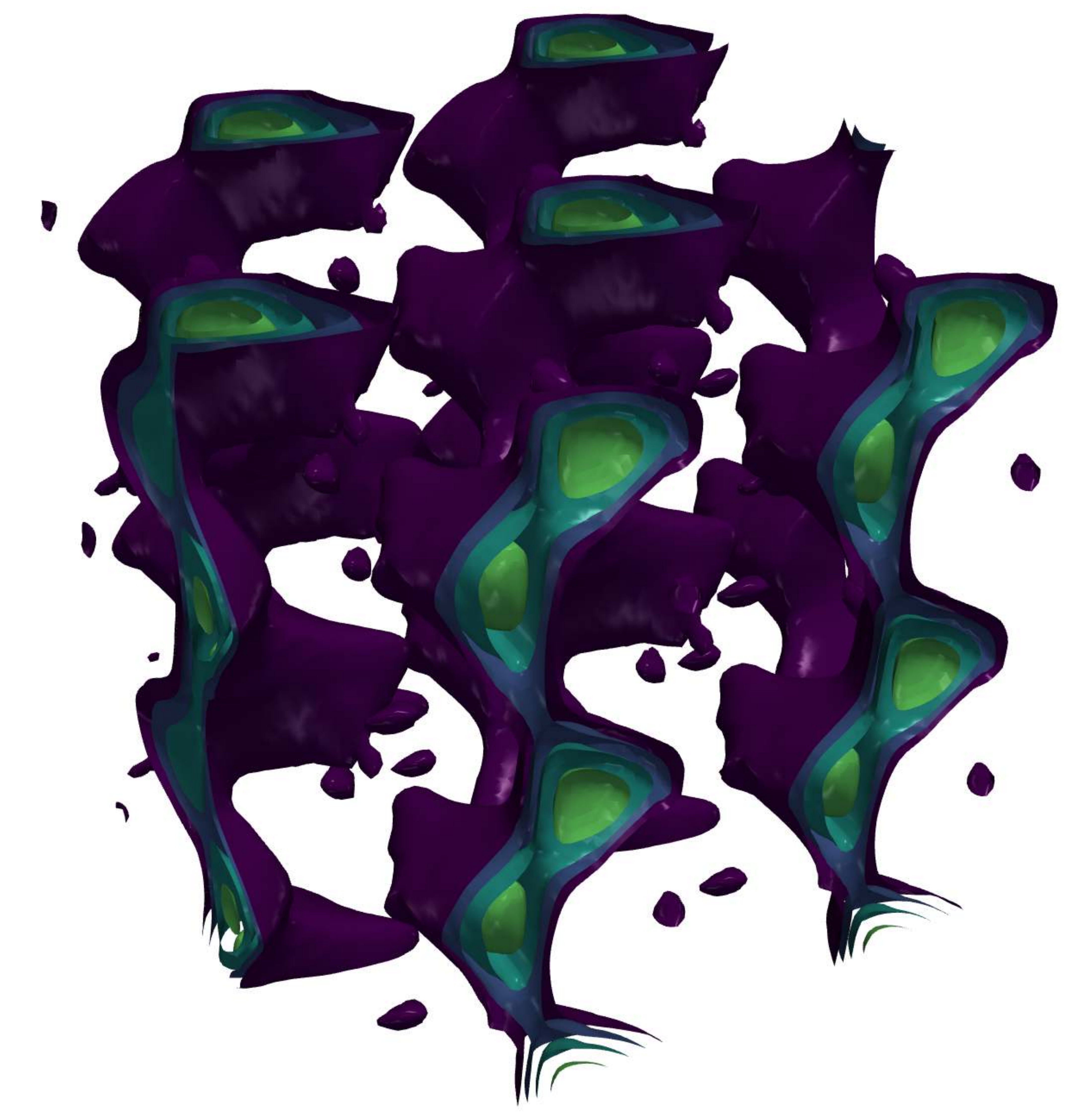}};
      \node at (3.8,0) {\includegraphics[width=3.7cm]{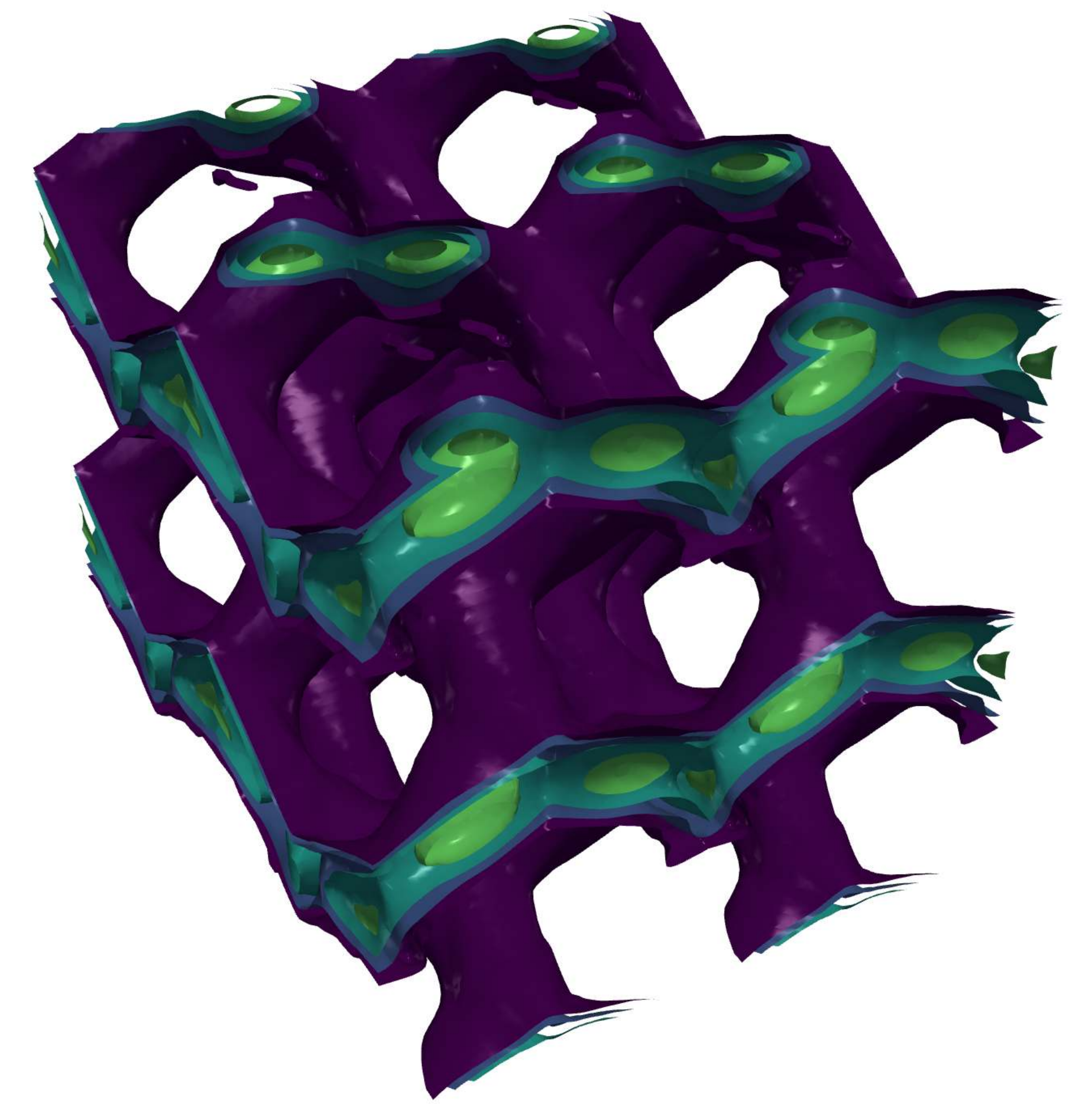}};
      \node at (7.7,0) {\includegraphics[width=3.7cm]{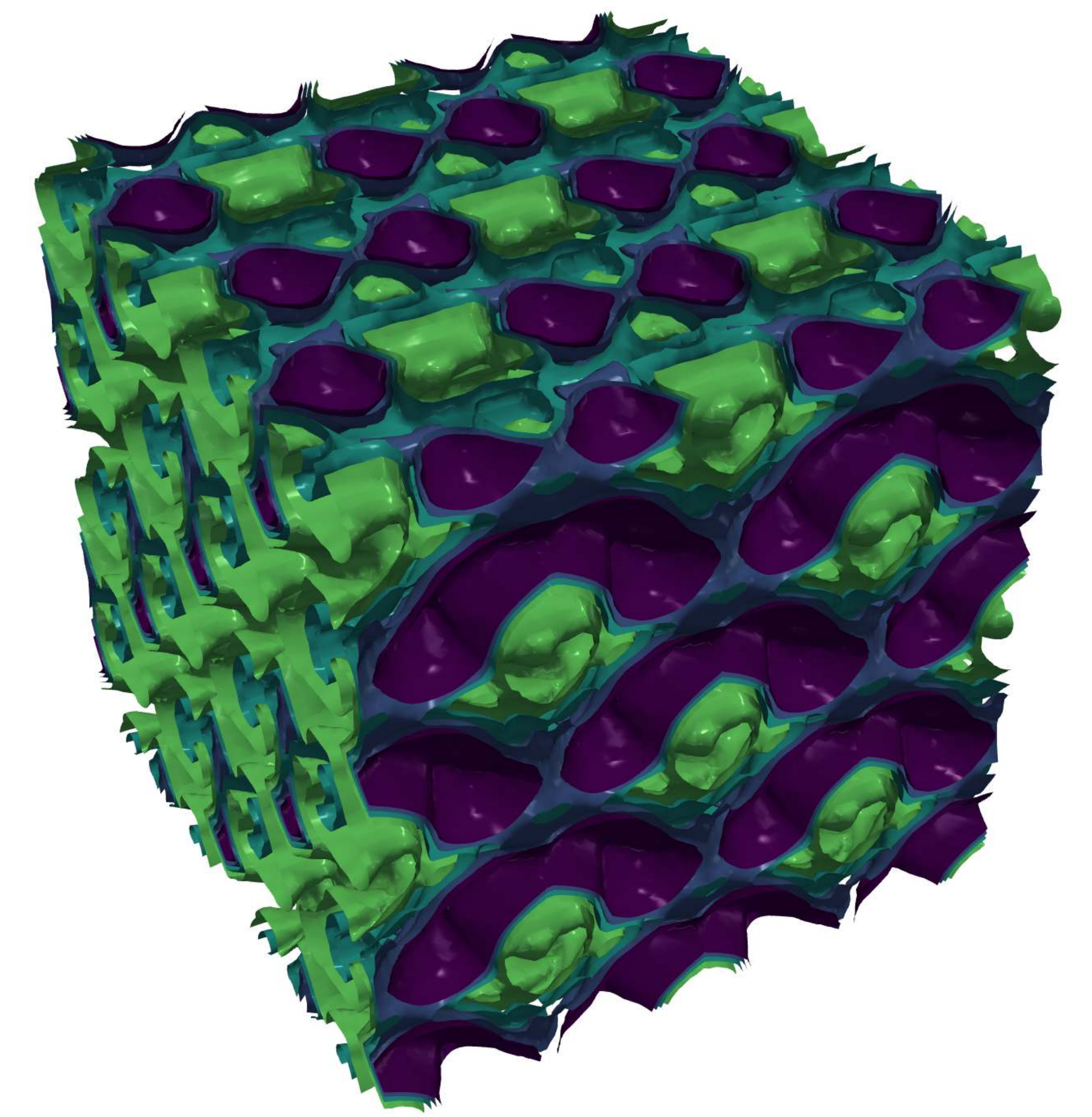}};
      \node at (11.6,0) {\includegraphics[width=3.7cm]{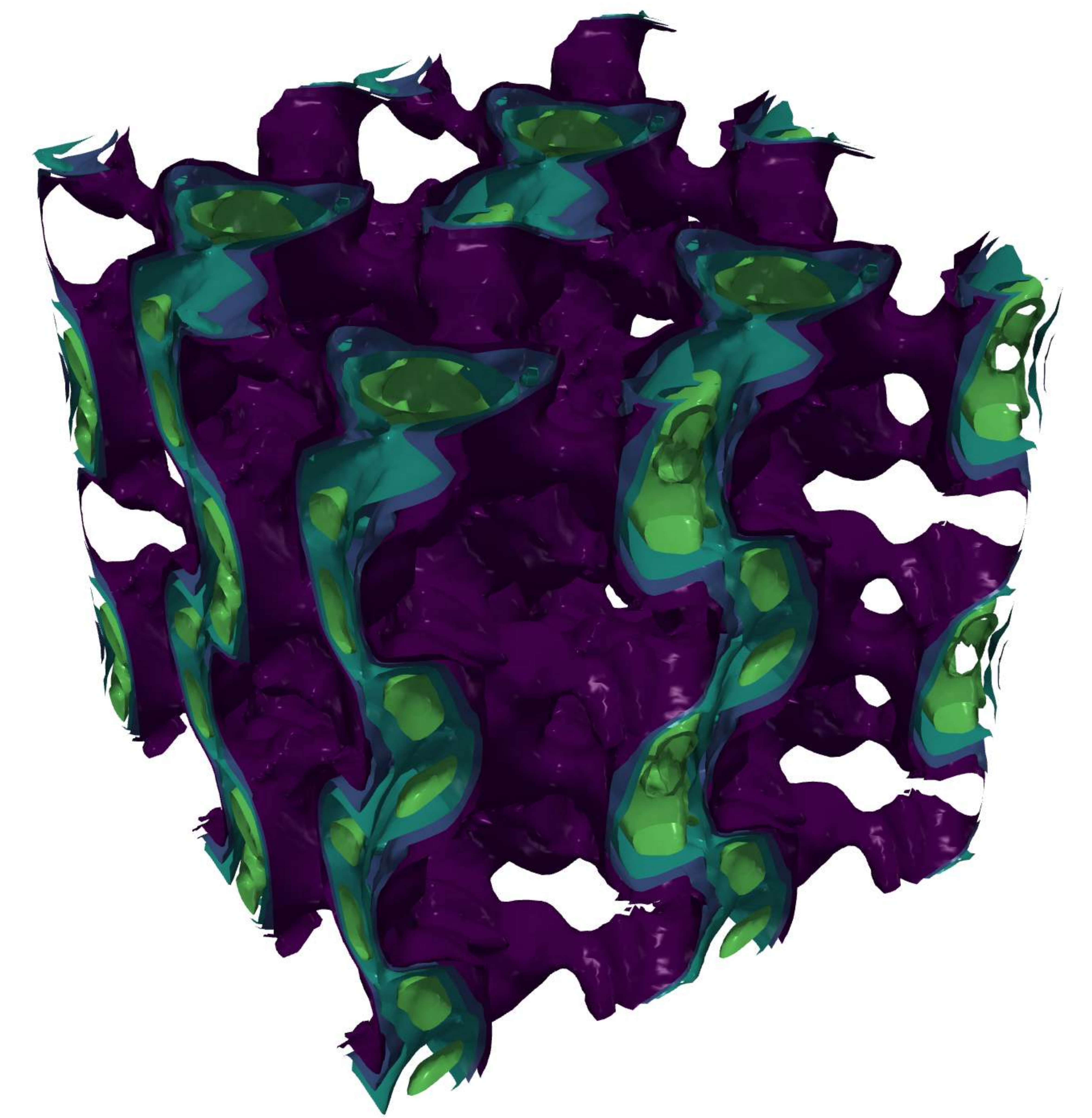}};
      \node at (0,-2.3)  {\scriptsize \texttt{P3\textsubscript{2}}};
      \node at (3.8,-2.3){\scriptsize \texttt{Pnn2}};
      \node at (7.7,-2.3) {\scriptsize \texttt{Aea2}};
      \node at (11.6,-2.3) {\scriptsize \texttt{P4\textsubscript{1}2\textsubscript{1}2}};
    \end{tikzpicture}
  \end{center}
  \vspace{-1.5em}
\caption{Four random neural networks for four different space groups, each with three hidden layers of ten units each and relu activations.}
\label{fig:space-nn}
\end{figure}

\begin{figure}[b]
    \caption{Dirichlet domains. {\em Left}: The set $R_{\phi}(x)$, shown in gray,
    is the closed half space of all points at least as
    close to $x$ as to $\phi x$. (In machine learning terms, the boundary of $R_\phi(x)$
    orthogonally intersects the connecting line between $x$ and $\phi x$ at its center, and is hence
    the support vector classifier with two support vectors $x$ and $\phi x$.) {\em Right}: A Dirichlet domain
    $\mathbf{D}(x)$ defined by the group generated by a vertical shift $\phi$ and a diagonal shift $\psi$ in the plane.
    }
    \vspace{-.9cm}
  \begin{center}
  \begin{tikzpicture}
    \begin{scope}
      \draw[gray!40!white,fill] (-2.5,-1.5) rectangle (2.5,0);
    \node[circle,scale=.3,fill,label=right:{\scriptsize $x$}] (a) at (0,-1) {};
    \node[circle,scale=.3,fill,label=right:{\scriptsize $\phi x$}] (b) at (0,1) {};
    \draw[dotted] (a)--(b);
    \draw[thick] (-2.5,0)--(2.5,0);
    \node at (-1.8,-.8) {\scriptsize $R_\phi(x)$};
    \end{scope}
    \begin{scope}[xshift=7cm]
      \newdimen\R
      \R=.75cm
      \node[circle,scale=.3,fill,label=below:{\scriptsize $x$}] (a) at (0,0) {};
      \draw (0:\R) \foreach \x in {60,120,...,360} {  -- (\x:\R) };
      \foreach \x/\l/\p in {
        60/{}/above,
        120/{}/above,
        180/{}/left,
        240/{}/below,
        300/{}/below,
        360/{}/right
      }
      \node[inner sep=1pt,circle] at (\x:\R) {};
      \node[circle,scale=.3,fill,label=right:{\scriptsize $\phi x$}] at ($(a.30)+(30:1.732*\R)$) {};
      \node[circle,scale=.3,fill,label=right:{\scriptsize $\psi x$}] at ($(a.90)+(90:1.732*\R)$) {};
      \node[circle,scale=.3,fill,label=left:{\scriptsize $\phi^{-1}\psi x$}] at ($(a.150)+(150:1.732*\R)$) {};
      \node[circle,scale=.3,fill,label=left:{\scriptsize $\phi^{-1} x$}] at ($(a.210)+(210:1.732*\R)$) {};
      \node[circle,scale=.3,fill,label=left:{\scriptsize $\psi^{-1} x$}] at ($(a.270)+(270:1.732*\R)$) {};
      \node[circle,scale=.3,fill,label=right:{\scriptsize $\phi\psi^{-1} x$}] at ($(a.330)+(330:1.732*\R)$) {};
    \end{scope}
  \end{tikzpicture}
  \end{center}
  \label{fig:dirichlet:domain}
\end{figure}
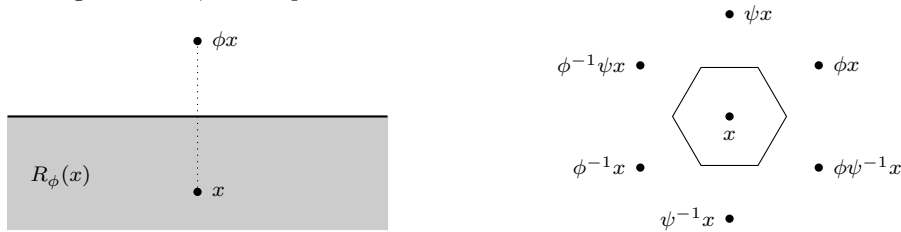

\subsection{Exact tilings}

Although the properties of general orbifolds constitute one of the more demanding problems
of modern mathematics, orbifolds of crystallographic groups are particularly well-behaved, and are well-understood.
That we can draw directly on this theory is due to the fact that it uses a notion of gluing very similar to
that employed by our algorithms as a proof technique \citep{Bonahon:2009,Ratcliffe:2006}.
The two notions align under an additional condition:
A convex polytope $\Pi$ is \kword{exact} for $\group$ if $\group$ tiles with $\Pi$, and if each
face $S$ of $\Pi$ can be represented as
\begin{equation*}
  S\;=\;\Pi\cap\phi\Pi\qquad\text{ for some }\phi\in\group\;.
\end{equation*}
Not every $\Pi$ with which $\group$ tiles is exact---in \cref{fig:wallpaper}, for example, the polytopes shown
for \texttt{pg} and \texttt{p3} are not exact, though all others are. However, given $\Pi$ and $\group$, we can always
construct an exact surrogate as follows:
Choose any point ${x\in\mathbb{R}}$ that is not a fixed point for any ${\phi\in\group\setminus\braces{\Id}}$.
If $\group$ is crystallographic, that is true for every point in the interior of~$\Pi$. For each~${\phi\in\group}$, the set
\begin{equation*}
  R_{\phi}(x)
  \;:=\;
  \braces{y\in\mathbb{R}^n\,|\,d_n(y,x)\leq d_n(y,\phi x)}\;,
\end{equation*}
is a half-space in $\mathbb{R}^n$ (see \cref{fig:dirichlet:domain}/left). The intersection
\begin{equation*}
  \mathbf{D}(x)
  \;:=\;
  \medcap\nolimits_{\phi\in\group} R_\phi(x)
  \;=\;
  \medcap\nolimits_{\phi\,|\,\phi\Pi\cap\Pi\neq\varnothing} R_\phi(x)
\end{equation*}
of these half-spaces is called a \kword{Dirichlet domain} for $\group$ (\cref{fig:dirichlet:domain}/right).
\begin{fact}[{\citep[][6.7.4]{Ratcliffe:2006}}]
  If $\group$ is crystallographic,
  $\mathbf{D}(x)$ is an exact convex polytope for $\group$.
\end{fact}
\begin{example}
For illustration, consider the group \texttt{pg}: We start with a rectangle $\Pi$.
The group is generated by two glide reflections $\phi$ and $\psi$, each of which shifts
$\Pi$ horizontally and then reflects it about one of its long edges (\cref{fig:dirichlet:pg}/left).
Exactness fails because the set ${\Pi\cap\phi\Pi}$, marked in black, is not a complete edge of $\Pi$.
A Dirichlet domain for this tiling differs significantly from $\Pi$ (\cref{fig:dirichlet:pg}/right).
Although substituting $\mathbf{D}(x)$ for $\Pi$ changes the look of the tiling,
it does not change the group---that is, we still work with the same set of transformations (rather than
another group in the same isomorphism class), and the axes of reflections are still
defined by the faces of $\Pi$ rather than those of $\mathbf{D}(x)$.
\end{example}
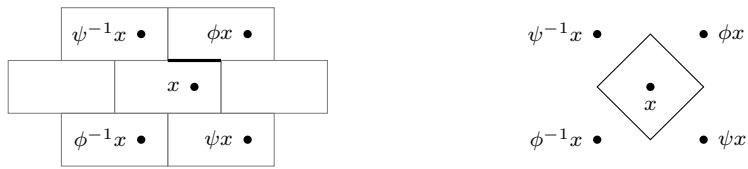
\begin{figure}[b]
  \caption{The Dirichlet domain of a tiling may differ from the convex polytope defining the tiling.
    \emph{Left}: A group generated by two glide reflections $\phi$ and $\psi$ that
    tiles with a rectangle $\Pi$. Both glides shift the rectangle
    and then reflect it about one of its long edges. The set ${\Pi\cap\phi\Pi}$, marked black,
    is not a face of $\Pi$. \emph{Right}: The Dirichlet domain $\mathbf{D}(x)$. When $\group$ tiles with
    $\mathbf{D}(x)$, the axes of reflection are still defined by the edges of $\Pi$.
  }
  \vspace{-.5cm}
  \begin{center}
    \begin{tikzpicture}
      \begin{scope}[scale=.7]
        \draw[gray] (0,1.5) rectangle (2,.5);
        \draw[gray] (0,-1.5) rectangle (2,-.5);
        \draw[gray] (0,-1.5) rectangle (-2,-.5);
        \draw[gray] (0,1.5) rectangle (-2,.5);
        \draw[gray] (-1,-.5) rectangle (-3,.5);
        \draw[gray] (1,-.5) rectangle (3,.5);
        \draw[gray] (1,-.5) rectangle (-1,.5);
        \draw[very thick] (0,.5)--(1,.5);
        \node[circle,scale=.3,fill,label=left:{\scriptsize $x$}] (a) at (.5,0) {};
        \node[circle,scale=.3,fill,label=left:{\scriptsize $\phi x$}] at ($(a)+(1,1)$) {};
        \node[circle,scale=.3,fill,label=left:{\scriptsize $\psi x$}] at ($(a)+(1,-1)$) {};
        \node[circle,scale=.3,fill,label=left:{\scriptsize $\psi^{-1} x$}] at ($(a)+(-1,1)$) {};
        \node[circle,scale=.3,fill,label=left:{\scriptsize $\phi^{-1} x$}] at ($(a)+(-1,-1)$) {};
      \end{scope}
      \begin{scope}[xshift=6cm,scale=.7]
        \node[circle,scale=.3,fill,label=below:{\scriptsize $x$}] (a) at (.5,0) {};
        \node[circle,scale=.3,fill,label=right:{\scriptsize $\phi x$}] at ($(a)+(1,1)$) {};
        \node[circle,scale=.3,fill,label=right:{\scriptsize $\psi x$}] at ($(a)+(1,-1)$) {};
        \node[circle,scale=.3,fill,label=left:{\scriptsize $\psi^{-1} x$}] at ($(a)+(-1,1)$) {};
        \node[circle,scale=.3,fill,label=left:{\scriptsize $\phi^{-1} x$}] at ($(a)+(-1,-1)$) {};
        \draw ($(a)+(-1,0)$)--($(a)+(0,1)$)--($(a)+(1,0)$)--($(a)+(0,-1)$)--cycle;
      \end{scope}
    \end{tikzpicture}
  \end{center}
  \label{fig:dirichlet:pg}
\end{figure}

\subsection{Properties of embeddings}

\cref{algorithm:MDS} can be interpreted as computing a numerical approximation $\widehat{\rho}$
to a ``true'' embedding map $\rho$, namely the map in \eqref{intro:embedding:map} in the introduction.
Our main result on the nonlinear representation, \cref{theorem:embedding} below,
shows that this map indeed exists for every crystallographic group, and describes some of
its properties. The proof of the theorem shows that $\rho$ and the set $\Omega$ can be constructed
by the following abstract gluing algorithm.
\graybox{
  {\bf Abstract gluing construction.}\\[.2em]
  \begin{tabular}{cl}
    1.) & Glue: Identify each ${x\in\partial\Pi}$ with the unique point ${y\in\partial\Pi}$ satisfying ${x\sim y}$.
    \\
    2.) & Equip the glued set $M$ with metric $d_\group$.\\
    3.) & Embed the metric space $(M,d_\group)$ as a subset ${\Omega\subset\mathbb{R}^N}$ for some ${N\in\mathbb{N}}$.\\
    4.) & For each ${x\in\Pi}$, define $\rho|_{\Pi}(x)$ as the representative of $x$ on $\Omega$.\\
    5.) & Set ${\rho:=\rho|_\Pi\circ p}$.
  \end{tabular}
}
Since $\Pi$ contains at least one point of each orbit, and the gluing step identifies all points identifies all points
on the same orbit with each other, the glued set $M$ can be regarded as the quotient set ${\mathbb{R}^n/\group}$.
Recall that an \kword{embedding} is a map ${M\rightarrow\Omega\subset\mathbb{R}^N}$ that is a homeomorphism
(a continuous bijection with continuous inverse) of the metric spaces $(M,d_\group)$ and $(\Omega,d_N)$.

The state the theorem, we need one additional bit of terminology: The \kword{stabilizer} of $x$ in $\group$ is the set of all
$\phi$ that leave~$x$ invariant,
\begin{equation*}
  \Stab(x)\;:=\;\braces{\phi\in\group\,|\,\phi x=x}\;
\end{equation*}
see \citet{Vinberg:Shvarsman:1993,Ratcliffe:2006,Bonahon:2009}.
We explain the role of the stabilizer in more detail in the next subsection.
\begin{theorem}
  \label{theorem:embedding}
  Let $\group$ be a crystallographic group that tiles $\mathbb{R}^n$ with an exact convex
  polytope $\Pi$. Then the set $M$ constructed by
  gluing is a compact $\group$-orbifold that
  is isometric to $\mathbb{R}^n/\group$. This orbifold can be
  embedded into $\mathbb{R}^N$ for some
  \begin{equation*}
    n\;\leq\;N\;<\;2(n+\max_{x\in\Pi}|\Stab(x)|)\;<\;\infty\;,
  \end{equation*}
  that is, there is compact subset ${\Omega\subset\mathbb{R}^N}$ such that the
  metric space ${(\Omega,d_N)}$ is homeomorphic to ${(\mathbb{R}^n/\group,d_\group)}$.
  In particular, every point ${x\in\Pi}$ is represented by one and only one point $\rho_{\Pi}(x)$.
  We can hence define a map
  \begin{equation*}
    \rho:\mathbb{R}^n\rightarrow\Omega\subset\mathbb{R}^N
    \qquad
    \text{ as }
    \qquad
    \rho(x)\;:=\;\rho_{\Pi}(p(x))\;.
  \end{equation*}
  The map $\rho$ is continuous, surjective, and $\group$-invariant.
  A function ${f:\mathbb{R}^n\rightarrow Y}$, with values
  in some topological space $Y$, is $\group$-invariant and continuous if and only if
  \begin{equation*}
    f=h\circ\rho\qquad\text{ for some continuous }h:\mathbb{R}^N\rightarrow Y\;.
  \end{equation*}
  $\Omega$ is smooth almost everywhere, in the sense that
  \begin{equation*}
    \vol_n\braces{x\in\Pi\,|\,\Omega\cap B_{\varepsilon}(\rho(x))\text{ is not a manifold for any }\varepsilon>0}\;=\;0
  \end{equation*}
  where $B_\varepsilon(z)$ denotes the open Euclidean metric ball of radius $\varepsilon$ centered at ${z\in\mathbb{R}^N}$.
\end{theorem}
\begin{proof}
  See \cref{appendix:proofs:embedding}.
\end{proof}
\begin{remark}
  (a) Note carefully what the theorem does and does not show about the
  embedding algorithm in \cref{sec:embedding:algorithm}: It does say that the
  glued set constructed by the algorithm discretizes an orbifold, and that an $N$-dimensional
  embedding of this orbifold exists. It does
  \emph{not} show that the embedding computed by MDS matches this dimension---indeed, since
  MDS attempts to construct an embedding that is also isometric (rather than just homeomorphic), we
  must in general expect the MDS embedding dimension to be larger, and we have at present no proof
  that an isometric embedding always exists.
  \\[.2em]
  (b) If the tiling defined by $\group$ and $\Pi$ is not exact,
  we can nonetheless define an embedding~$\rho$ that represents continuous functions that are invariant
  functions with respect to this tiling: Construct a Dirichlet domain $\mathbf{D}$, and then construct $\rho$ by applying the gluing
  algorithm to $\mathbf{D}$. Functions constructed as ${h\circ\rho}$ are then invariant for the tiling $(\group,\Pi)$.
\end{remark}
We have now seen different representations of continuous $\group$-invariant functions on $\mathbb{R}^n$, respectively
by continuous functions on $\Pi$, on the abstract space $\mathbb{R}^n/\group$, and on $\Omega$. On $\Pi$, we must explicitly
impose the periodic boundary condition, so we are using the set
\begin{equation*}
  \C_{\text{\rm pbc}}(\Pi)\;:=\;\braces{\hat{f}\in\C(\Pi)\,|\,\hat{f}\text{ satisfies \eqref{pbc}}}\;.
\end{equation*}
In these representations, the projector $p$, the quotient map $q$,
and the embedding map $\rho$ play very similar roles.
We can make that observation more rigorous:
\begin{corollary}
  \label{corollary:C:isometries}
  Given a crystallographic group $\group$ that tiles with a convex polytope $\Pi$, consider the maps
  \begin{align*}
    I_\Pi:\C_{\text{\rm pbc}}(\Pi)&\rightarrow\C_\group
    &\text{ and }&&
    I_{\mathbb{R}^n/\group}:{\C(\mathbb{R}^n/\group)}&\rightarrow\C_\group
    &\text{ and }&&
    I_\Omega:\C(\Omega)&\rightarrow\C_\group\\
    \hat{f}&\mapsto\hat{f}\circ p
    &&&
    \hat{g}&\mapsto\hat{g}\circ q
    &&&
    \hat{h}&\mapsto\hat{h}\circ\rho
  \end{align*}
  where $I_\Omega$ is only defined if $\Pi$ is exact. Equip all spaces with the supremum norm.
  Then~$I_\Pi$ and~$I_{\mathbb{R}^n/\group}$ are isometric isomorphisms, and if $\Pi$
  is exact, so is $I_\Omega$. In particular, $\C_\group$ is always a separable Banach space.
\end{corollary}
\begin{proof}
  By \cref{lemma:projector}, \eqref{eq:bijection:f:quotient} and \cref{theorem:embedding}, all three maps are bijections.
  We also have
    \begin{equation*}
    \|\hat{f}\|_{\sup}
    \;=\;
    \sup\nolimits_{x\in\Pi}|\hat{f}(x)|
    \;=\;
    \sup\nolimits_{x\in\mathbb{R}^n}|\hat{f}(p(x))|
    \;=\;
    \|\hat{f}\circ p\|_{\sup}
    \quad\text{ for }\hat{f}\in\C_{\text{\rm pbc}}(\Pi)\;,
    \end{equation*}
    and the same holds mutatis mutandis on $\mathbb{R}^n/\group$ and $\Omega$, so all
    maps are isometries. Since~$\Pi$ is compact, $\C(\Pi)$ is separable
    \citep[][3.99]{Aliprantis:Border:2006}. The same hence holds for the closed subspace~$\C_{\text{\rm pbc}}(\Pi)$,
    and by isometry for $\C_\group$.
\end{proof}

\subsection{Why the glued surface may not be smooth}

\begin{figure}[b]
  \caption{Fixed points become non-smooth under gluing. \emph{Left}: A triangular polytope $\Pi$ is transformed by a rotation $\phi$, which rotates by $120^\circ$ around the vertex
    $x$ of $\Pi$. Note that ${\phi^2=\phi^{-1}}$ and ${\phi^3=\Id}$.
    \emph{Middle}: The gluing defined by $\phi$ identifies the edge $(x,z)$ and the edge $(x,y)$.
    \emph{Right}: The resulting glued surface is a cone, with $x$ mapped to the tip. Locally around $x$, the cone is not a manifold.
  }
  \vspace{-1cm}
  \begin{center}
    \includegraphics{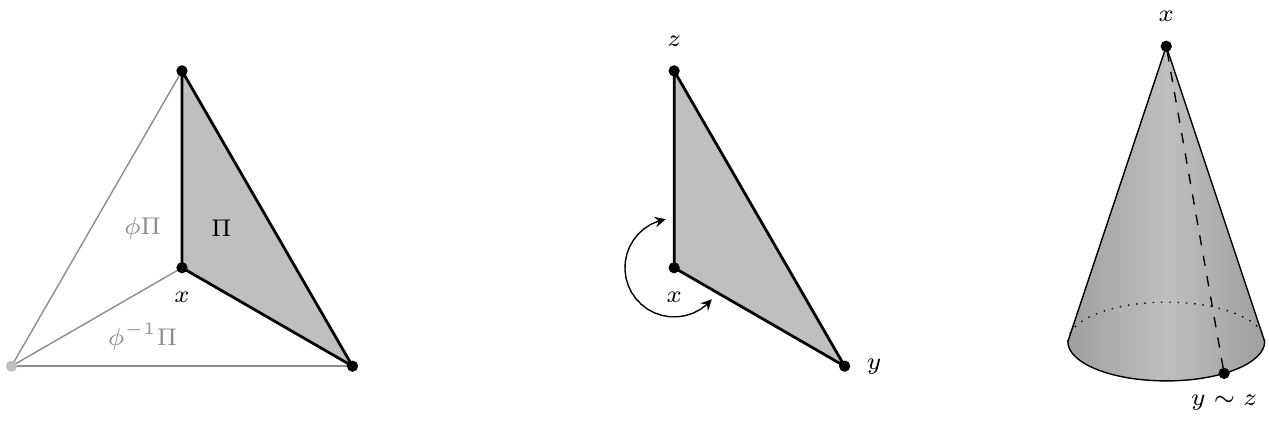}
  \end{center}
  \label{fig:cone}
\end{figure}

Whether or not the glued surface is smooth depends on whether the transformations in~$\group$ leave any points invariant. It is a known fact in geometry (and made precise in the proof of \cref{theorem:embedding}) that
\begin{align*}
  \text{ glued surface is a manifold }
  \qquad&\Leftrightarrow\qquad
  \phi x\;\neq\;x\quad\text{for all }\phi\in\group\setminus\braces{\Id}\text{ and }x\in\mathbb{R}^n\;.\\
\intertext{That can be phrased in terms of the stabilizer as}
  \text{ glued surface is a manifold }
  \qquad&\Leftrightarrow\qquad
  \Stab(x)\;=\;\braces{\Id}\qquad\text{ for all }x\in\mathbb{R}^n\;.
\end{align*}
It is straightforward to check that $\Stab(x)$ is a group \citep{Vinberg:Shvarsman:1993}.
Since each $\phi$ is an isometry,
and shifts of $\mathbb{R}^n$ have no fixed points, $\phi(x)=x$ can only hold if ${b_\phi=0}$.
Thus, $\Stab(x)$ is always a subset of the point group $\group_o$ (in the terminology of \cref{sec:isometries}),
which means it is finite. To illustrate its effect on the surface, consider the following examples.
\begin{example}
(a) Recall that MacKay's construction \citep{mackay1998introduction}, as sketched in the introduction,
can be translated to crystallographic groups by setting ${\Pi=[0,1]}$ and choosing $\group$ as shifts. In this case,
${\Stab(x)=\braces{\Id}}$ for each ${x\in\mathbb{R}}$, and the glued surface is a circle, which is indeed a manifold.
The two-dimensional analogue is to choose ${\Pi=[0,1]^2}$ and $\group$ as the group \texttt{p1} in \cref{fig:wallpaper},
in which case the glued surface is a torus as shown in \cref{fig:orbifolds}, and hence again a manifold.\\[.2em]
(b) Now suppose $\Pi$ is a triangle, $x$ one of its corners, and $\phi$ a $120^\circ$ rotation
around $x$, as illustrated in \cref{fig:cone}. Then ${\Stab(x)=\braces{\Id,\phi,\phi^2=\phi^{-1}}}$, and the glued surface~${\Omega=\rho(\Pi)}$ is a cone with $\rho(x)$ as its tip.
That means $\Omega$ is not a manifold, because no neighborhood of the tip can be mapped isometrically to a neighborhood in $\mathbb{R}^2$.
\end{example}

\section{Invariant kernels}
\label{sec:kernels}

\begin{figure}
  \begin{center}
    \begin{tikzpicture}
      \node at (-.5,0) {\includegraphics[height=3cm]{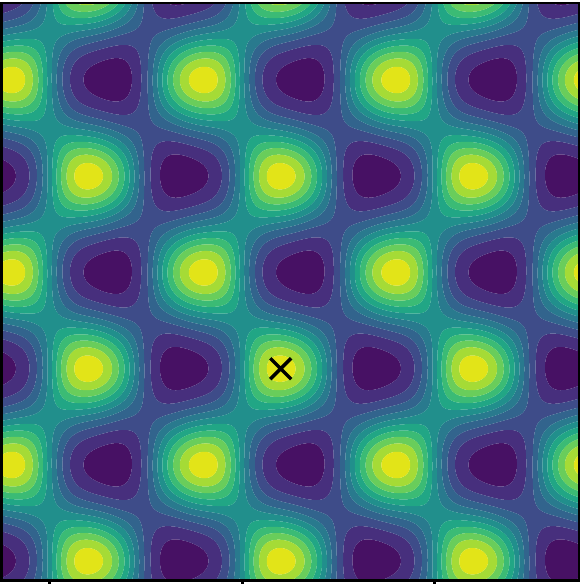}};
      \node at (3.2,0) {\includegraphics[height=3cm]{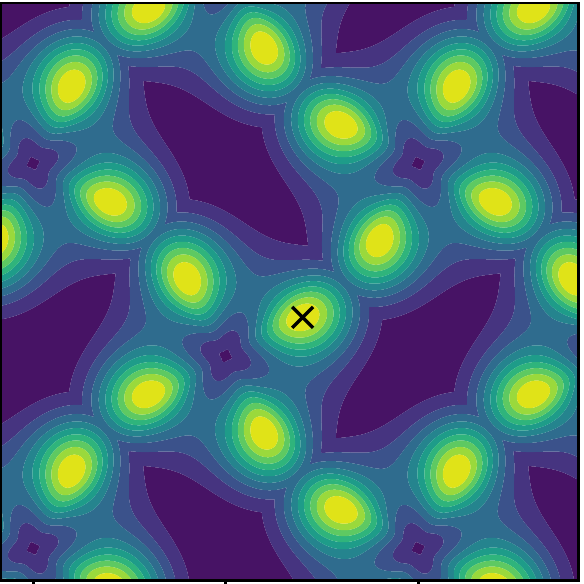}};
      \node at (7,0) {\includegraphics[height=3.5cm]{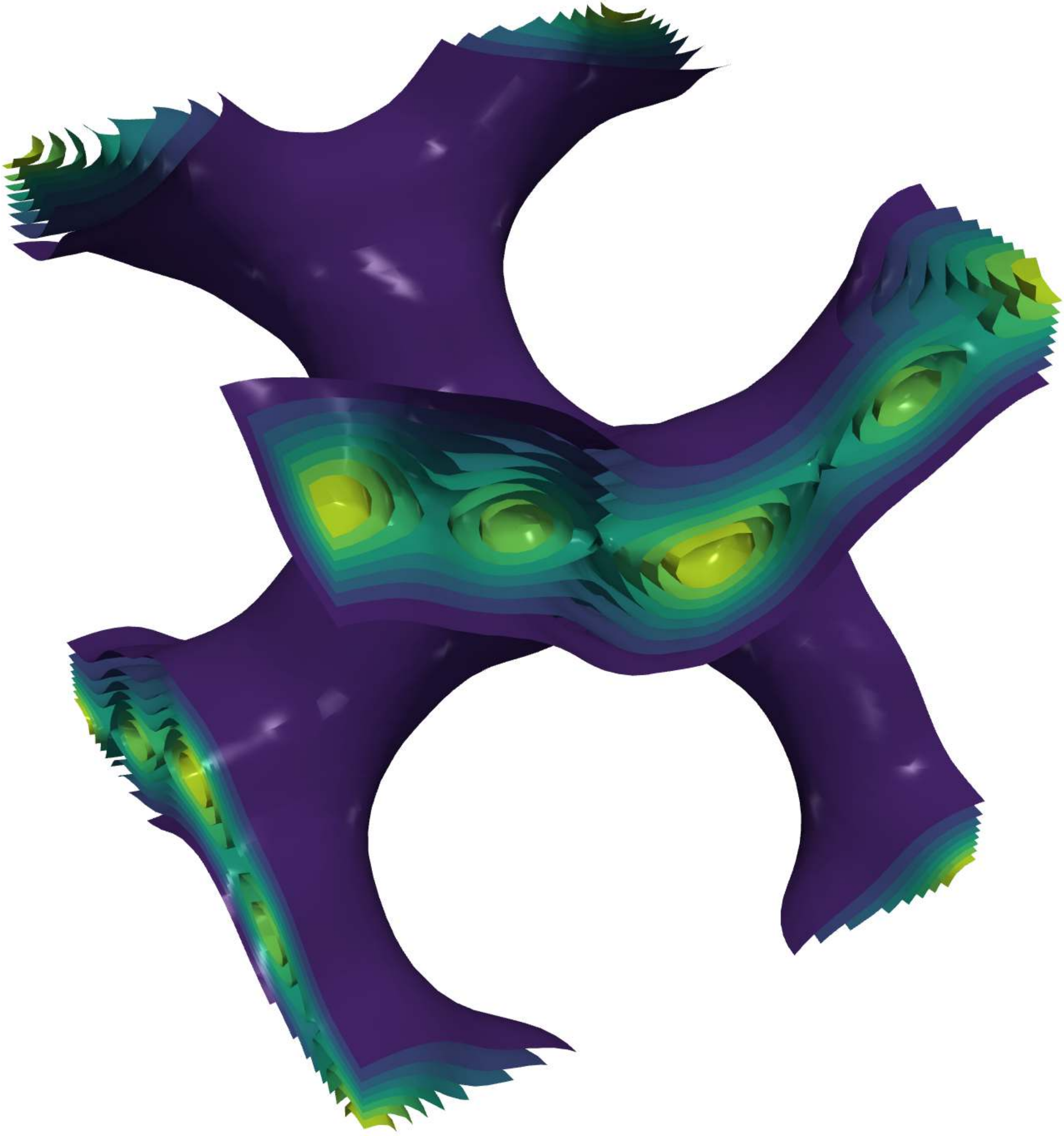}};
      \node at (10.5,0) {\includegraphics[height=3.5cm]{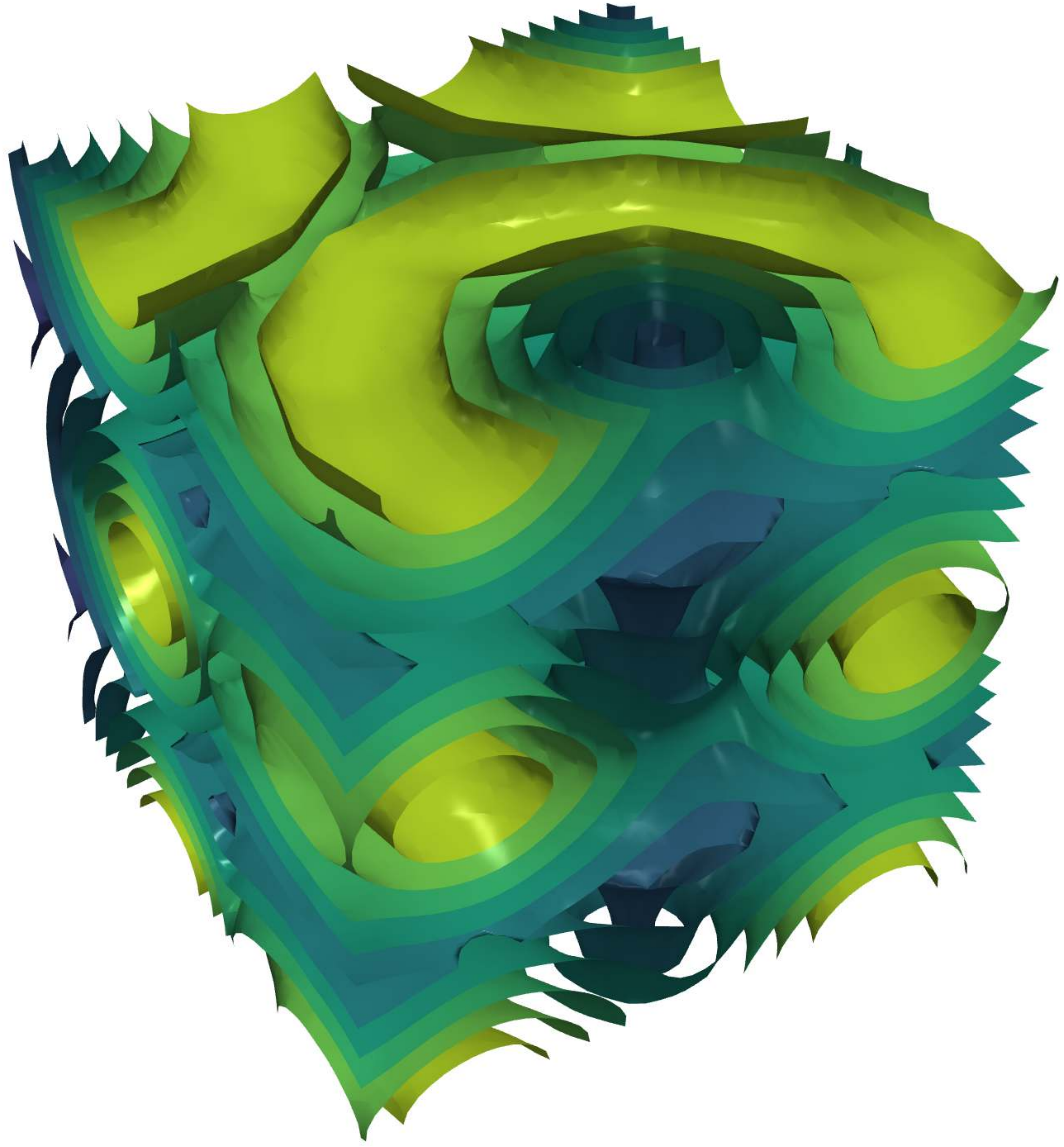}};
    \end{tikzpicture}
    \vspace{-1em}
  \end{center}
  \caption{Invariant kernel functions using the kernel in \eqref{eqn:covariance-kernel},
    for groups on $\mathbb{R}^2$ and $\mathbb{R}^3$: \texttt{pg} with $\ell=0.1$, \texttt{p4gm} with $\ell=0.01$,
    \texttt{I4\textsubscript{1}} with $\ell=0.01$, and \texttt{P6\textsubscript{3}/m} with $\ell=0.05$ (from left to right).}
  \label{fig:kernel}
\end{figure}

Throughout this section, ${\kappa:\mathbb{R}^n\times\mathbb{R}^n\rightarrow\mathbb{R}}$ is a kernel, i.e., a
positive definite function, and~$\mathbb{H}$ is its reproducing kernel Hilbert space, or RKHS.
\cref{sec:function:spaces} reviews definitions.
We consider kernels that are $\group$-invariant in both arguments in the sense of \eqref{eq:invariant:k:arguments},
that is,
\begin{equation*}
  \label{invariant:kernel}
  \kappa(\phi x,\psi y)\;=\;\kappa(x,y)\qquad\text{ for all }\phi,\psi\in\group\text{ and all }x,y\in\mathbb{R}^n\;.
\end{equation*}
That is the natural notion of invariance for most purposes, since
such kernels are precisely those that define spaces of $\group$-invariant
functions:
\begin{proposition}
  \label{result:RKHS:invariant:elements}
  If and only if $\kappa$ is $\group$-invariant in each argument,
  all functions ${f\in\mathbb{H}}$ are $\group$-invariant.
  If $\kappa$ is also continuous, all ${f\in\mathbb{H}}$ are continuous, and hence
  ${\mathbb{H}\subset\C_{\group}}$.
\end{proposition}
\cref{theorem:embedding} implies that, to define an invariant kernel, we can start with any
kernel~$\hat{\kappa}$ on the embedding space $\mathbb{R}^N$, and compose it with the embedding
map $\rho$:
\begin{corollary}
  \label{result:constructing:invariant:kernel}
  Let $\hat{\kappa}$ be a kernel on ${\mathbb{R}^N}$. Then the function
  \begin{equation*}
    \kappa(x,y):=\hat{\kappa}(\rho(x),\rho(y))
    \quad\text{ or in short }\quad
    \kappa=\hat{\kappa}\circ(\rho\otimes\rho)
  \end{equation*}
  is a kernel on ${\mathbb{R}^n}$ that is $\group$-invariant
  in both arguments. If $\hat{\kappa}$ is continuous, so is $\kappa$.
\end{corollary}
That follows immediately from \cref{theorem:embedding} and the fact that
the restriction of a kernel to a subset is again a kernel \citep{Steinwart:Christmann:2008}.
\begin{example}
  Suppose $\hat{\kappa}$ is an radial basis function (RBF) kernel with length scale $\ell$ on~$\mathbb{R}^N$,
  and hence of the form ${\hat{\kappa}(z,z')=\exp(-\|z-z'\|^2/\ell^2)}$. Then $\kappa$ is simply
\begin{equation*}
  \label{eqn:covariance-kernel}
  \kappa(x,y)
  \;=\;
  \hat{\kappa}(\rho(x),\rho(z))
  \;=\;
  \exp\Bigl(-\frac{||\rho(x)-\rho(y)||^2}{2\ell^2}\Bigr)\;.
\end{equation*}
\cref{fig:kernel} illustrates this kernel the two-dimensional groups
\texttt{pg} and \texttt{p4gm} and the three-dimensional groups \texttt{I4$_\texttt{1}$} and \texttt{P6$_\texttt{3}$/m}.
\end{example}
Once we have constructed an invariant kernel, its application to machine learning problems is
straightforward. That becomes obvious if we define
${\Phi(x)\;=\;\kappa(x,\argdot)}$, often called the \kword{feature map} of $\kappa$
\citep{Steinwart:Christmann:2008}.
Using the definition of the scalar product on $\mathbb{H}$ and the reproducing property (see \cref{sec:RKHS}),
we then have
\begin{equation*}
  \Phi:\mathbb{R}^n\rightarrow\mathbb{H}
  \qquad\text{ and }\qquad
  \kappa(x,y)\;=\;\sp{\Phi(x),\Phi(y)}_{\mathbb{H}}\;.
\end{equation*}
If $\kappa$ is $\group$-invariant, then $\Phi$ is also $\group$-invariant by construction.
Recall that most kernel methods in machine learning are derived by substituting
a Euclidean scalar product by~$\sp{\Phi(x),\Phi(y)}_{\mathbb{H}}$, thereby making a linear
method nonlinear. Using a $\group$-invariant kernel results in a $\group$-invariant
method.
\begin{example}[Invariant SVM]
  A support vector machine (SVM) with kernel $\kappa$ is determined by two finite sets of points $\mathcal{X}$ and
  $\mathcal{Y}$ in $\mathbb{R}^n$. To train the SVM, one maps these points into $\mathbb{H}$ via $\Phi$, finds the shortest
  connecting line between the convex hulls of $\Phi(\mathcal{X})$ and $\Phi(\mathcal{Y})$, and determines
  a hyperplane $F$ that is orthogonal to this line and intersects its center---equivalently, in dual formulation, the
  unique hyperplane that separates the convex hulls of $\Phi(\mathcal{X})$ and $\Phi(\mathcal{Y})$
  and maximizes the $\mathbb{H}$-norm distance to both. The set of points $x$ in $\mathbb{R}^n$ whose
  image $\Phi(x)$ lies on $F$ is the \kword{decision surface} of the SVM in $\mathbb{R}^n$.
  The hyperplane can be specified by two functions $g$ (an offset vector) and $h$ (a normal vector) in $\mathbb{H}$:
  A function ${f\in\mathbb{H}}$ lies on $F$ if and only if
  \begin{equation*}
    \sp{f-g,h}_{\mathbb{H}}
    \;=\;
    0
    \qquad\text{ or equivalently }\qquad
    \sp{f,h}_{\mathbb{H}}
    \;=\;
    \sp{g,h}_{\mathbb{H}}\;.
  \end{equation*}
  Let $x$ be a point in $\mathbb{R}^n$. If $y$ and $z$ are points with ${g=\Phi(y)}$ and ${h=\Phi(z)}$, then
  \begin{equation*}
    x\text{ is on decision surface }\qquad\Longleftrightarrow\qquad\kappa(x,z)\;=\;\kappa(y,z)\;.
  \end{equation*}
  Since invariance of $\kappa$ implies ${\kappa(\phi x,z)=\kappa(x,z)}$, that shows the decision surface is $\group$-invariant.
  \cref{fig:svm} shows examples.
  In these figures the data were randomly generated with regions assigned labels using a random function generated as in \cref{sec:gp}.
  The support vectors are highlighted and illustrate the effects of symmetry constraints: the decision surface can be determined by data observed far away.
\end{example}
Two of the most important results on kernels are Mercer's theorem
and the compact inclusion theorem \citep[][Chapter 4]{Steinwart:Christmann:2008}. The latter shows the inclusion map
${\mathbb{H}\hookrightarrow\C}$ is compact, and is used in turn to establish good
statistical properties of kernel methods, such as oracle inequalities and finite covering numbers
\citep{Steinwart:Christmann:2008}. Both results assume that $\kappa$ has
compact support. If $\kappa$ is invariant under a crystallographic group,
its support is necessarily non-compact, but the next result shows that
versions of both theorems hold nonetheless:
\begin{proposition}
  \label{result:kernel:compactness}
  If $\kappa$ is continuous and $\group$-invariant in both arguments, the
  inclusion map~${\mathbb{H}\hookrightarrow\C_{\group}}$ is compact.
  There exist functions ${f_1,f_2,\ldots\in\mathbb{H}}$ and
  scalars ${c_1\geq c_2\geq\ldots>0}$ such that
  \begin{equation*}
    \kappa(x,y)\;=\;\tsum_{i\in\mathbb{N}}c_if_i(x)f_i(y)\qquad\text{ for all }x,y\in\mathbb{R}^n\;,
  \end{equation*}
  and the scaled sequence ${(\sqrt{c_i}f_i)}$ is an orthonormal basis of $\mathbb{H}$.
  With this basis,
  \begin{equation*}
    \mathbb{H}\;=\;
    \braces{\;f\!=\!\tsum_{i\in\mathbb{N}}a_i\sqrt{c_i}f_i\,|\,
      a_1,a_2,\ldots\in\mathbb{R}\text{ with }\tsum_{i}|a_i|^2<\infty}\;,
  \end{equation*}
  where each series converges in $\mathbb{H}$ and hence (by compactness of inclusion)
  also uniformly.
\end{proposition}
Intuitively, that is the case because every $\group$-invariant kernel is the pullback
of a kernel on $\Omega$, and $\Omega$ is compact.
\cref{fig:svm} shows an application of such a kernel to generate a two-class classifier
with an $\group$-invariant decision surface.
\begin{figure}
  \begin{center}
    \begin{tikzpicture}
      \node at (0,0) {\includegraphics[width=4cm]{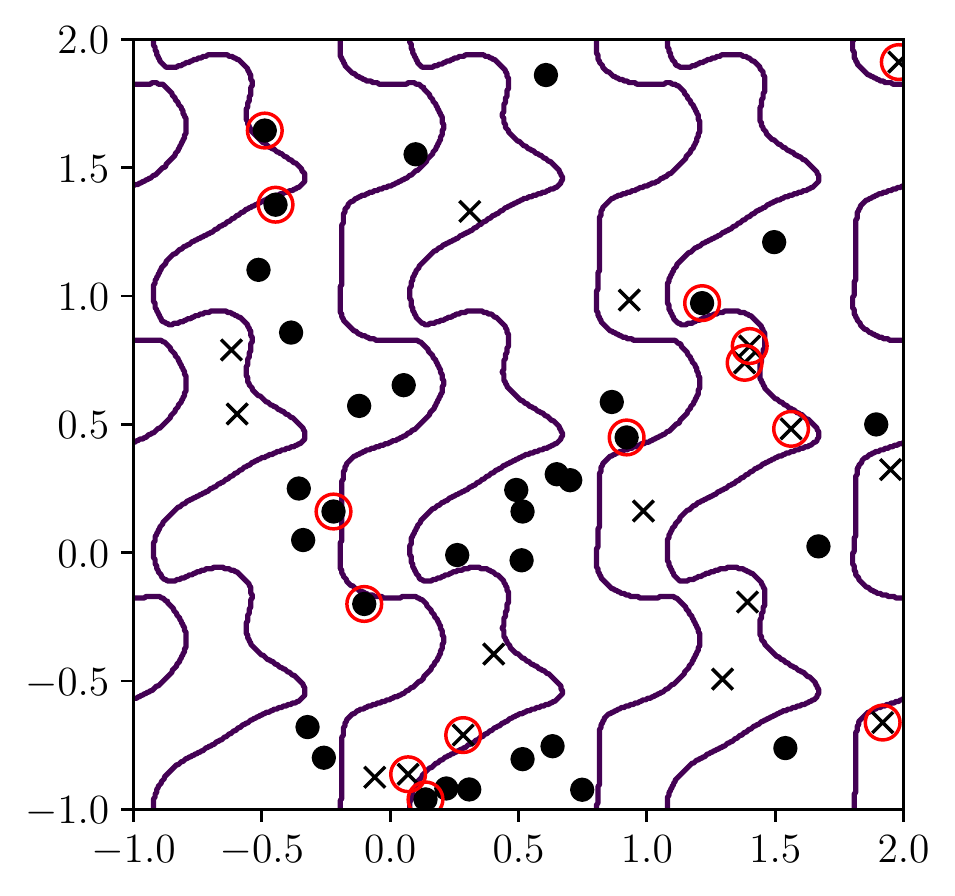}};
      \node at (5,0) {\includegraphics[width=4cm]{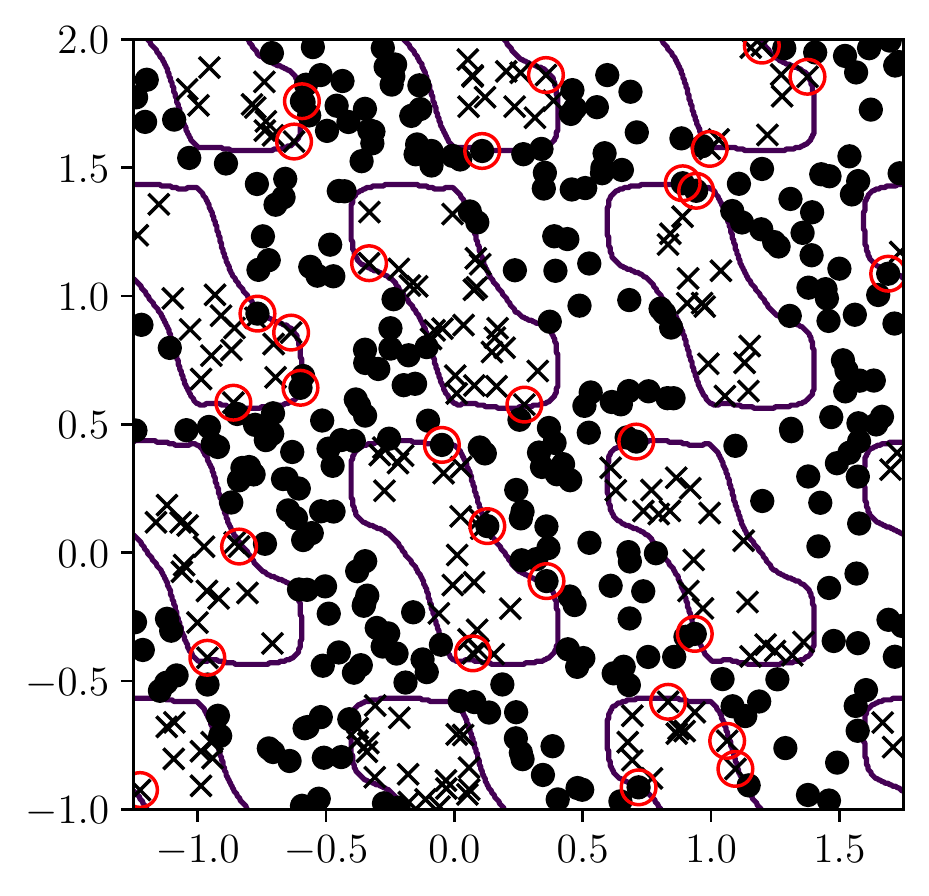}};
      \node at (10,0) {\includegraphics[width=4cm]{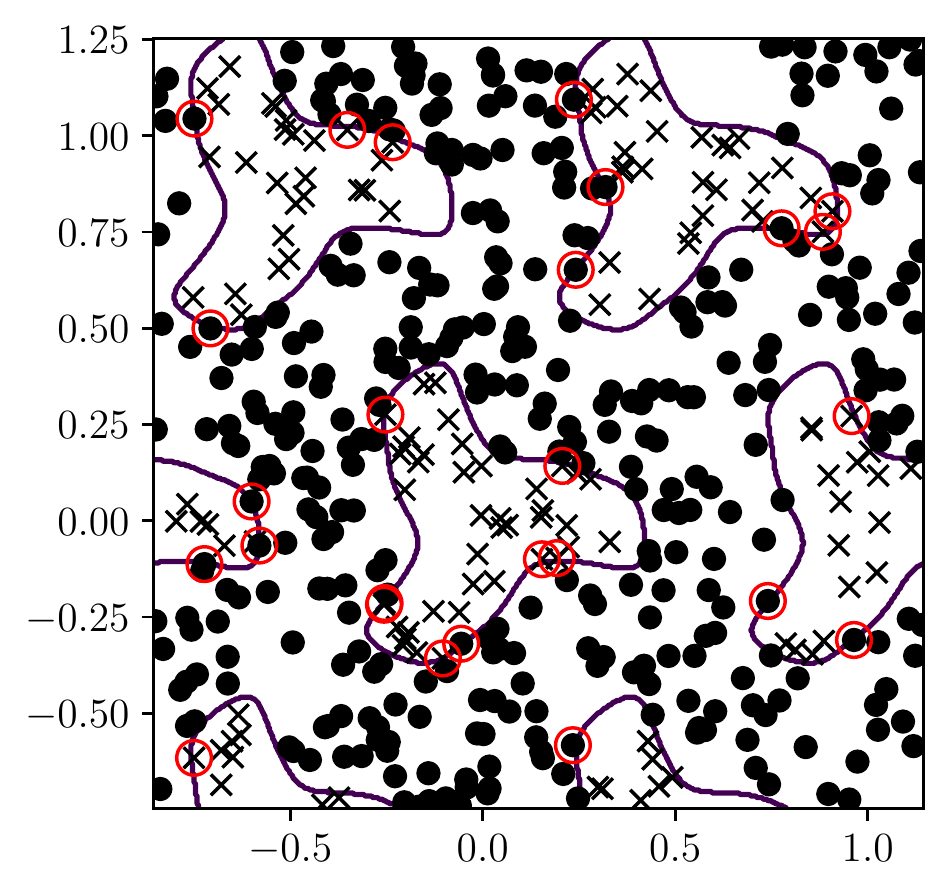}};
    \end{tikzpicture}
    \vspace{-.5cm}
  \end{center}
  \caption{Support vector machine decision surfaces for $\group$-invariant kernels constructed as in
    \cref{{result:constructing:invariant:kernel}}. The groups used are \texttt{p1} (left), \texttt{p2} (middle),
    and \texttt{p3} (right). Support vectors are highlighted with red circles.}
  \label{fig:svm}
\end{figure}

\section{Invariant Gaussian processes}
\label{sec:gp}

We now consider the problem of generating random functions
${F:\mathbb{R}^n\rightarrow\mathbb{R}}$ such that each instance of $F$ is
continuous and $\group$-invariant with probability 1. That can be done linearly
using the generalized Fourier representation, by generating the coefficients $c_i$
in \cref{theorem:spectral} at random. Here, we consider the nonlinear representation
instead: If we set
\begin{equation*}
  F\;:=\;H\circ\rho\qquad\text{ for a random continuous function }H:\mathbb{R}^N\rightarrow\mathbb{R}\;,
\end{equation*}
\cref{theorem:embedding} implies that $F$ is indeed continuous and $\group$-invariant with probability $1$,
and hence a random element of
$\C_{\group}$. Conversely, the result also implies that every random
element of $\C_{\group}$ is of this form, for some random element $H$ of $\C(\mathbb{R}^N)$.

\subsection{Almost surely invariant processes}

Recall that a random function ${F:M\subseteq\mathbb{R}^n\rightarrow\mathbb{R}}$ is a
\kword{Gaussian process} if
the joint distribution of the random vector ${(F(x_1),\ldots,F(x_k))}$ is
Gaussian for any finite set of points ${x_1,\ldots,x_k\in M}$.
The \kword{mean} and \kword{covariance function} of a Gaussian process are defined as
\begin{equation*}
  \mu(x)\;:=\;\mean[F(x)]
  \quad\text{ and }\quad
  \kappa(x,y)\;:=\;\mean[(F(x)-\mu(x))(F(y)-\mu(y)]
  \quad\text{ for }x,y\in M\;.
\end{equation*}
The covariance function is always positive definite, and hence a kernel on
$M$. The distribution of a Gaussian process is completely determined by $\mu$ and $\kappa$,
and conditions for~$F$ to satisfy continuity or stronger regularity conditions can
be formulated in terms of $\kappa$. See e.g., \citet{Marcus:Rosen:2006} for more
background.
\begin{proposition}
  \label{result:invariant:GP}
  Let $H$ be a continuous Gaussian process on $\mathbb{R}^N$, with mean $\mu$ and covariance function $\kappa$.
  Then ${F:=H\circ\rho}$ is a continuous random function on $\mathbb{R}^n$, and is $\group$-invariant
  with probability 1. Consider any finite set of points
  \begin{equation*}
    x_1,\ldots,x_k\in\mathbb{R}^n\quad\text{ such that }
    x_i\not\sim x_j\text{ for all distinct }i,j\leq k\;.
  \end{equation*}
  Then ${(F(x_1),\ldots,F(x_k))}$ is a Gaussian random vector, with mean and covariance
  \begin{equation*}
    \mean[F(x_i)]=\mu(\rho(x_i)) \quad\text{ and }\quad
    \text{\rm Cov}[F(x_i),F(x_j)]=\kappa(\rho(x_i),\rho(x_j))
    \quad\text{ for }i,j\leq k\;.
  \end{equation*}
\end{proposition}

Clearly, $F$ cannot be a Gaussian process on $\mathbb{R}^n$: Since $F$ is invariant,
$F(x)$ completely determines $F(\phi(x))$, so ${(F(x),F(\phi x))}$ cannot be jointly Gaussian.
Put differently, conditioning $F$ on its values on $\Pi$ renders $F$ non-random.
Loosely speaking, the proposition hence says that $F$ is ``as Gaussian'' as a $\group$-invariant
random function can be. \cref{fig:gp} illustrates random functions generated by such a process.

\begin{example}
  \label{example:gp}
  The construction of \citet{mackay1998introduction} described in the introduction
  was designed specifically for Gaussian processes, to generate periodic functions at random.
  We can now generalize these processes from periodicity to crystallographic invariance:
  Given~$\group$ and~$\Pi$, construct the embedding map ${\rho:\mathbb{R}^n\rightarrow\mathbb{R}^N}$.
  Choose $\hat{\kappa}$ as the RBF kernel~\eqref{eqn:covariance-kernel} on~$\mathbb{R}^N$, and $\hat{\mu}$ as the constant function $0$ on $\mathbb{R}^N$.
  Then generate $F$ as
  \begin{equation*}
    H\sim\text{GP}(\hat{\mu},\hat{\kappa})
    \qquad\text{ and }\qquad
    F\;:=\;H\circ\rho\;.
  \end{equation*}
  For visualization, draws can be approximated by the randomized feature scheme of
  \citet{rahimi2007random}. \cref{fig:gp} shows examples for $\group$ chosen as
  \texttt{p2} and \texttt{p31m} on $\mathbb{R}^2$, and for \texttt{P-6} and \texttt{P422} on $\mathbb{R}^3$.
\end{example}

\begin{figure}
  \begin{center}
    \begin{tikzpicture}
      \node at (0,0) {\includegraphics[width=3cm]{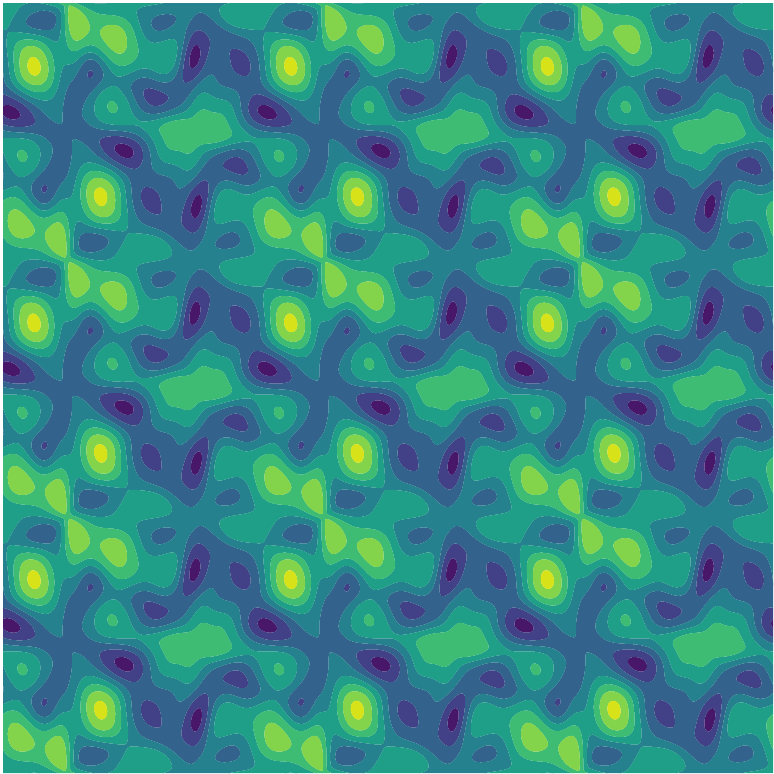}};
      \node at (3.7,0) {\includegraphics[width=3cm]{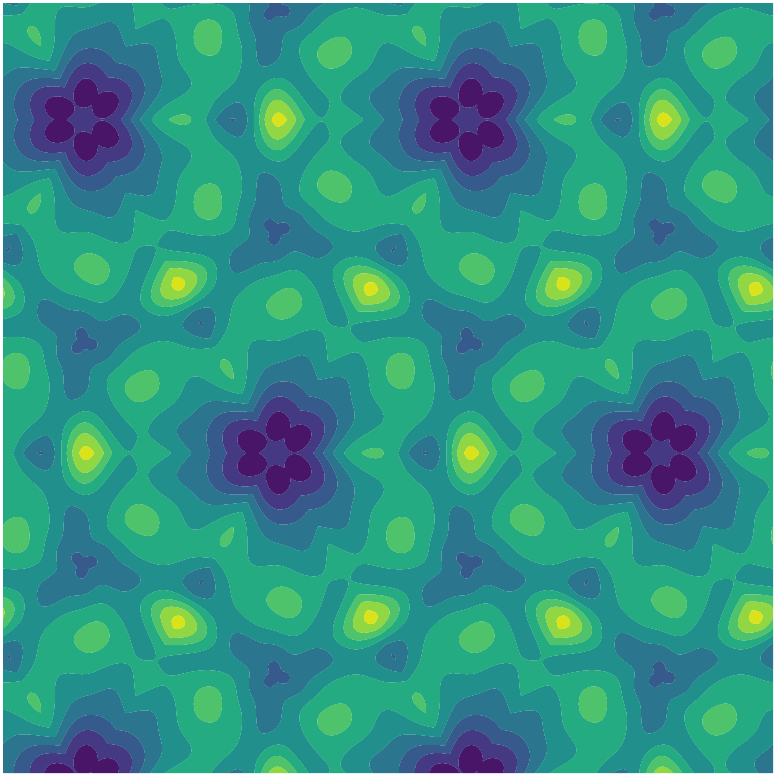}};
      \node at (7.4,0) {\includegraphics[width=3.5cm]{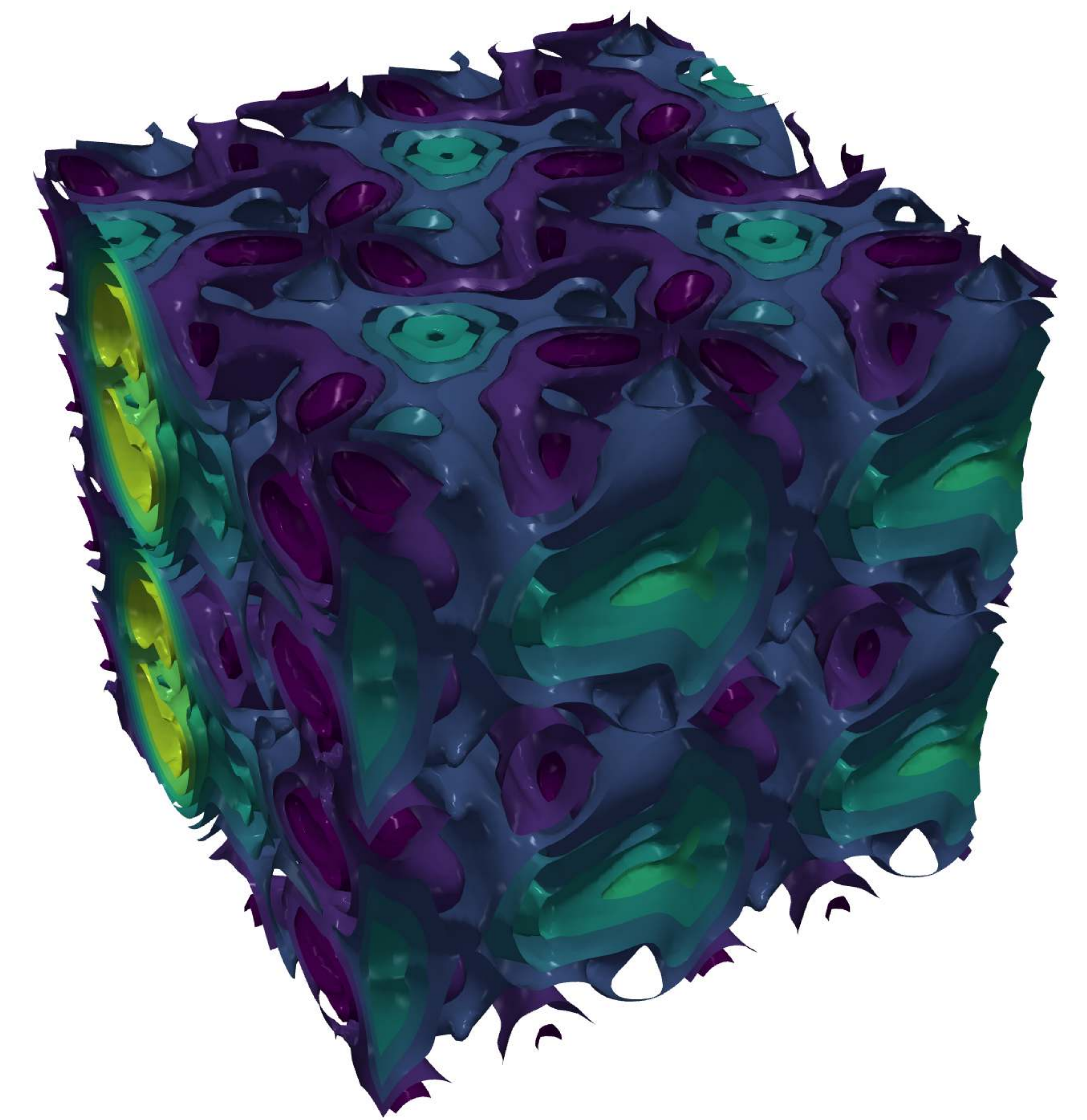}};
      \node at (11.1,0) {\includegraphics[width=3.5cm]{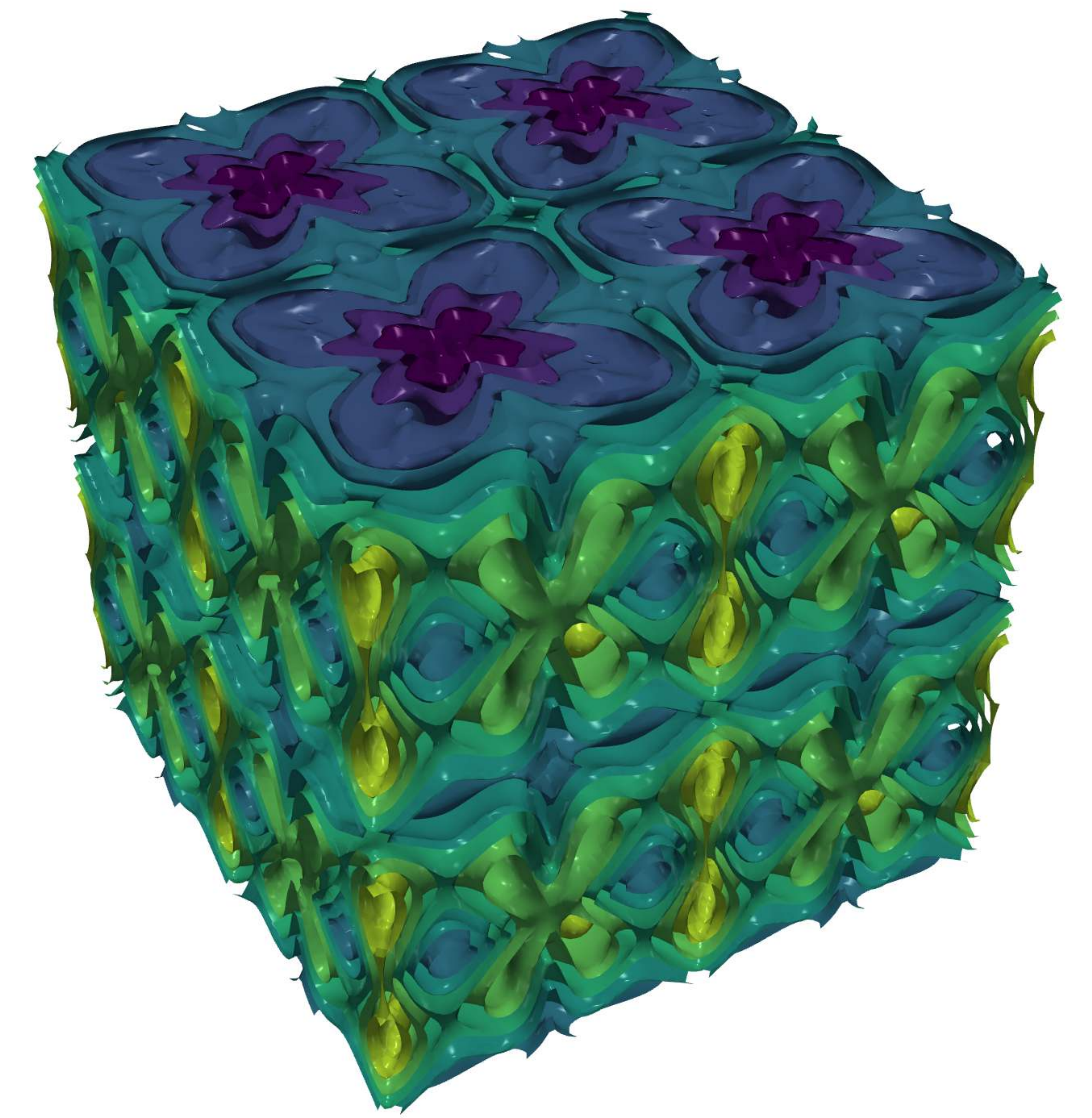}};
    \end{tikzpicture}
    \vspace{-.7cm}
  \end{center}
  \caption{Random invariant functions on $\mathbb{R}^2$ and $\mathbb{R}^3$,
    generated by Gaussian processes as described in \cref{example:gp}.
    The groups are, from left to right, \texttt{p2}, \texttt{p31m}, \texttt{P-6}, and \texttt{P422}.
  }
  \label{fig:gp}
\end{figure}

\subsection{Distributionally invariant processes}

Another type of invariance that random functions can satisfy is \kword{distributional $\group$-invariance},
which holds if
\begin{equation*}
  F\;\equdist\;F\circ\phi\qquad\text{ for all }\phi\in\group\;.
\end{equation*}
Here, $\smash{\equdist}$ denotes equality in distribution.
That is equivalent to requiring that
the distribution $P$ of $F$ satisfies ${P(\phi A)=P(A)}$ for every measurable set $A$.
For crystallographic groups, distributionally invariant Gaussian processes can
be constructed by factoring the parameters, rather than
the random function $F$, through the embedding in \cref{theorem:embedding}:\nolinebreak
\begin{corollary}
  Let $\mu$ be a real-valued function and $\kappa$ a kernel on $\mathbb{R}^N$.
  If $F$ is the Gaussian process on $\mathbb{R}^n$ with mean ${\mu\circ\rho}$
  and covariance function ${\kappa\circ(\rho\otimes\rho)}$, then
  $F$ is distributionally $\group$-invariant, i.e.\ ${F\circ\phi\equdist F}$ for all ${\phi\in\group}$.
\end{corollary}
Almost sure invariance implies distributional invariance; distributional invariance is
typically a much weaker property.
Frequently encountered examples of distributional invariance are all forms of
stationarity (distributional invariance under shift groups) and of
exchangeability  (permutation groups).

\section{The Laplace operator on invariant functions}
\label{sec:laplacian}

The results in this section describe the behavior of the Laplace operator on
$\group$-invariant functions. All of these are ingredients in the proof of the Fourier representation.
We first describe the transformation behavior of differentials of invariant
functions, in \cref{sec:differentials}. Gradients turn out to be invariant under shifts and
equivariant under orthogonal transformations.
Gradient vector fields, and more generally vector fields
with the same transformation behavior as gradients, have a cancellation property---their integral
orthogonal to the tile boundary vanishes (\cref{sec:flux}).
We then define the relevant solution space for the spectral problem, which has Hilbert space
structure (so that we can define orthogonality and self-adjointness) but has smoother elements than
$\L_2$, in \cref{sec:HG}. Once the Laplacian has been properly defined on this space, we can
use the cancellation property to show it is self-adjoint.

\subsection{Differentials and gradients of invariant functions}
\label{sec:differentials}
Given a differentiable function ${f:\mathbb{R}^n\rightarrow\mathbb{R}^m}$, denote the differential at $x$ as
\begin{equation*}
  Df(x)
  \;=\;
  \left(\mfrac{\partial f_j}{\partial x_i}\right)_{i\leq n,j\leq m}
  \;\in\mathbb{R}^{n\times m}\;.
\end{equation*}
The next result summarizes how invariance of $f$ under a transformation $\phi$ affects $Df$.
Note the order of operations matters: $D(f\circ\phi)$ is the differential of the function $f\circ\phi$, whereas
${(Df)\circ\phi}$ transforms the differential $Df$ of $f$ by $\phi$.
\begin{lemma}
  \label{lemma:differentials}
  If ${f:\mathbb{R}^n\rightarrow\mathbb{R}^m}$ is
  invariant under an isometry ${\phi x=A_\phi x+b_\phi}$ and differentiable, then
  \begin{equation}
    \label{eq:equivariant:differential}
    (Df)(\phi x)\;=\;Df(x)\cdot A_{\phi}^{\trans}\;.
  \end{equation}
  If in particular ${f:\mathbb{R}^n\rightarrow\mathbb{R}}$, the gradient satisfies
  \begin{equation}
    \label{eq:equivariant:gradient}
    \nabla f(\phi x)\;=\;A_{\phi}\cdot\nabla f(x) \qquad\text{ for all }x\in\mathbb{R}^n\;.
  \end{equation}
  The Hessian matrix $H_f$ and the Laplacian satisfy
  \begin{equation}
    \label{eq:invariant:Laplacian}
    H_f(\phi x)\;=\;A_{\phi}H_f(x)A_{\phi}^{\trans}
    \quad\text{ and }\quad
    \Delta f(\phi x)\;=\;\Delta f(x)
    \qquad\text{ for all }x\in\mathbb{R}^n\;.
  \end{equation}
\end{lemma}
\begin{proof}
  Since $\phi$ is affine, its differential ${(D\phi)(x)=A_\phi}$ is constant. The chain rule
  shows
  \begin{equation*}
    D(f\circ\phi)
    \;=\;
    (Df)\circ\phi\cdot (D\phi)
    \;=\;
    ((Df)\circ\phi)\cdot A_\phi\;.
  \end{equation*}
  By invariance, ${f\circ\phi}$ and $f$ are the same function, and hence ${D(f\circ\phi)=Df}$.
  Substituting into the identity above shows \eqref{eq:equivariant:differential}, since ${A_\phi^{\trans}=A_\phi^{-1}}$.
  For ${m=1}$, the transpose ${D^{\trans}=\nabla}$ is the gradient, and \eqref{eq:equivariant:differential} becomes \eqref{eq:equivariant:gradient}. Using \eqref{eq:equivariant:gradient}, the Hessian can be written as
  \begin{equation*}
    H_f\;=\;D(\nabla f)\;=\;D(A_{\phi}\nabla f\circ\phi^{-1})\;.
  \end{equation*}
  Another application of the chain rule then shows
  \begin{align*}
    H_f
    \;=\;
    D(A_{\phi}\nabla f)\circ\phi^{-1}\cdot D\phi^{-1}
    \;=\;
    A_{\phi}(D\nabla f)\circ\phi^{-1}\cdot A_{\phi^{-1}}
    \;=\;
    A_{\phi}(H_f\circ \phi^{-1})A_{\phi}^{\trans}\;,
  \end{align*}
  which is the first statement in \eqref{eq:invariant:Laplacian}. Since the Laplacian is the trace of $H_f$, and
  the trace in invariant under change of basis, that implies
  \begin{equation*}
    \Delta f(\phi x)
    \;=\;
    \text{tr}(H_f(\phi x))
    \;=\;
    \text{tr}(A_{\phi}H_f(x)A_{\phi}^{\trans})
    \;=\;
    \text{tr}(H_f(x))
    \;=\;
    \Delta f(x)\;.\qedhere
  \end{equation*}
\end{proof}

\subsection{Flux through the tile boundary}
\label{sec:flux}

The next result is the key tool we use to prove self-adjointness of the Laplacian.
We have seen above that the gradient of a $\group$-invariant function transforms
under $\group$ according to \eqref{eq:equivariant:gradient}.
We now abstract from the specific function $\nabla f$,
and consider any vector field~${F:\Pi\rightarrow\mathbb{R}^n}$
that transforms like the gradient on the tile boundary, i.e.
\begin{equation}
  \label{periodic:condition:field}
  F(y)\;=\;A_{\phi}F(x)\qquad\text{ whenever }y=\phi x\;.
\end{equation}
For a polytope $\Pi$ with facets ${S_1,\ldots,S_k}$, we define the
normal field on the boundary as
\begin{equation*}
  \n_{\Pi}:\partial\Pi\rightarrow\mathbb{R}^n
  \qquad\text{ given by }\qquad
  \n_{\cF}(x)\;:=\;\begin{cases}
  \n_i & \text{ if }x\in S_i^{\circ}\\
  0 & \text{ otherwise }
  \end{cases}
\end{equation*}
where $\n_i$ is the unit normal vector of the facet $S_i$, directed outward with respect to $\Pi$.
In vector analysis, the projection~${F^{\trans}\n_{\Pi}}$ of a vector field onto the direction orthogonal
to~$\partial\Pi$ is known as the \kword{flux} of~$F$ through the boundary.
\begin{proposition}[Flux]
  \label{lemma:gradient:field}
  Let $\group$ be a crystallographic group that tiles $\mathbb{R}^n$ with a convex polytope $\Pi$.
  If a vector field ${F:\Pi\rightarrow\mathbb{R}^n}$ is integrable on $\partial\Pi$
  and satisfies \eqref{periodic:condition:field}, then
  \begin{equation*}
    \mint_{\partial\cF}F(x)^{\trans}\n_{\cF}(x)\vol_{n-1}(dx)\;=\;0    \;.
  \end{equation*}
\end{proposition}

\begin{proof}
  See \cref{appendix:proofs:flux}.
\end{proof}

\subsection{The Sobolev space of invariant functions}
\label{sec:HG}

The proof of \cref{theorem:spectral} follows a well-established strategy in spectral theory:
The relevant spectral results hold for self-adjoint operators, and self-adjointness can only
be defined with respect to an inner product. Since the space $\C^2$ on which the Laplace
operator is defined is a Banach space, but has no inner product, one must hence first embed the problem into
a suitable Hilbert space. For the Laplacian, this is generally a first-order Sobolev space;
see \cref{sec:function:spaces} for a review of definitions, and \citet{Brezis:2011,Mazya:2011,McLean:2000} for more
on spectral theory and the general approach.

In our case, we proceed as follows:
Since invariant functions are completely determined by their values on $\Pi$, we can equivalently
solve the problem on the bounded domain $\Pi$ rather than the unbounded domain $\mathbb{R}^n$.
That gives us access to a number of results specific to bounded domains.
We also observe that the invariance constraint ${e=e\circ\phi}$ is a linear constraint---if two functions
satisfy it, so do their linear combinations---so the feasible set of this constraint is a vector space,
and we can encode the constraint by restriction to a suitable subspace.
We start with the vector space
\begin{equation}
  \label{definition:space:H}
  \mathcal{H}
  \;:=\;
  \braces{f|_{\Pi^{\circ}}\,|\,f:\mathbb{R}^n\rightarrow\mathbb{R}\text{ infinitely often differentiable and $\group$-invariant}}\;.
\end{equation}
The elements of $\mathcal{H}$ are hence infinitely often differentiable on $\Pi^{\circ}$, and their
continuous extensions to the closure $\Pi$ satisfy the periodic boundary condition \eqref{pbc}.
We then define the Sobolev space of candidate solutions as
\begin{equation*}
  \H_{\group}\;:=\;\text{ closure of $\mathcal{H}$ in }\H^1(\Pi^{\circ})\;,
\end{equation*}
equipped with the norm and inner product of ${\H^1(\Pi^\circ)}$.
As a closed subspace of a Hilbert space, it is a Hilbert space.

\subsection{The Laplace operator on $\H_{\group}$}

We now have to extend $\Delta$ to all elements of $\H_\group$. In general,
a linear operator $\Lambda$ on a closed subspace ${V\subset\H^1(\Pi^\circ)}$ is an
\kword{extension} of $\Delta$ to $V$ if it satisfies
\begin{equation}
  \label{extends}
  \Lambda f\;=\;\Delta f\qquad\text{ for all }f\in V\cap\C^2(\Pi^\circ)\;.
\end{equation}
The extended operator is \kword{self-adjoint} on $V$ if
\begin{equation*}
  \sp{\Lambda f,h}_{\H^1}\;=\;\sp{f,\Lambda h}_{\H^1}\qquad\text{ for all }f,h\in V\;.
\end{equation*}
To prove self-adjointness, one decomposes $\Lambda$ as
\begin{equation*}
  \sp{-\Delta f,h}_{\L_2}
  \,=\,
  (\text{integral over $\Gamma^{\circ}$ that is symmetric in $f$ and $h$})
  \;-\;
  (\text{integral over }\partial\Gamma)\;.
\end{equation*}
This is the Green identity alluded to in the introduction.
To make it precise, we need two quantities: One is the \kword{energy form}
or \kword{energy product}
\begin{equation}
\label{energy:form}
  a(f,h)\;:=\;\mint_{\Gamma}\nabla f(x)^{\trans}\nabla h(x)\,\vol_n(dx)\;.
\end{equation}
Since it only involves first derivatives, and both appear under the integral, it is
well-defined for any ${f,h\in\H^1(\Pi^{\circ})}$, and is hence a symmetric bilinear form
${a:\H^1\times\H^1\rightarrow\mathbb{R}}$. It is positive definite, since
\begin{equation}
  \label{energy:positive:semidefinite}
  a(f,h)\;=\;\msum_{i\leq n}\sp{\partial_i f,\partial_i h}_{\L_2}
  \quad\text{ and hence }\quad
  a(f,f)\;=\;\msum_{i\leq n}\|\partial_if\|_{\L_2}\;\geq\;0\;.
\end{equation}
Substituting the definition of $a$ into that of the $\H^1$ scalar product
in \eqref{eq:sobolev:sp} shows that
\begin{equation}
  \label{eq:H1:L2:a}
  \sp{f,h}_{\H^1}\;=\;\sp{f,h}_{\L_2}\,+\,a(f,h)
  \qquad\text{ for all }f,h\in\H^1(\Gamma^\circ)\;.
\end{equation}
The second quantity is the \kword{conormal derivative}
\begin{equation*}
  \partial_{\n}f(x)\;:=\;\nabla f(x)^{\trans}\n_\Pi(x)\;.
\end{equation*}
The precise statement of the decomposition above is then as follows.
\begin{fact}[Green's identity]
  \label{green:identity}
  If the domain $\Pi$ is sufficiently regular---in particular, if $\Pi$ is a convex polytope---then
  \begin{equation*}
  \sp{-\Lambda f,h}_{\L_2}
  \;=\;
  a(f,h)
  \,-\,
  \mint_{\partial\Gamma}\partial_{\n}f(x)h(x)\vol_{n-1}(dx)
  \qquad\text{ for }f,h\in\H^1(\intF)\;.
\end{equation*}
\end{fact}
Informally, this shows that $\Delta$ ``behaves self-adjointly'' in the interior
of $\Pi$, where derivatives can be computed in all directions around a point.
At points on $\partial\Pi$, the boundary truncates derivatives in some direction,
and that requires a correction term~$\partial_{\n}f$.\nolinebreak
\begin{theorem}[Properties of the Laplacian]
  \label{theorem:laplacian}
  Let $\group$ be a crystallographic group that tiles~$\mathbb{R}^n$ with a convex polytope $\Pi$.
  Then $\Delta$ has a unique extension to a linear operator $\Lambda$ on~$\H_\group$.
  This operator is self-adjoint and continuous on $\H_\group$, and satisfies
  \begin{equation}
    \label{eq:theorem:laplacian}
    \text{\rm(i) }\;\sp{-\Lambda f,h}_{\L_2}\;=\;a(f,h)
    \qquad\text{ and }\qquad
    \text{\rm(ii) }
    \sp{-\Lambda f,f}_{\H^1}\;\geq\;\|f\|^2_{\H^1}-\|f\|^2_{\L_2}
\end{equation}
  for all ${f,h\in\H_\group}$.
\end{theorem}
The proof uses the flux property to show that crystallographic symmetry
makes the boundary term cancel. Since $a$ is symmetric, that makes $\Lambda$ self-adjoint.
In the parlance of elliptic differential equations, (\ref{eq:theorem:laplacian}ii) says
that $\Lambda$ is \kword{coercive} on $\H_\group$ (see \citep{McLean:2000}).
\begin{proof} See \cref{appendix:proofs:laplacian}.
\end{proof}

\subsection{Linear representations from a nonlinear ansatz}
\label{sec:linear-reps-revisited}

The properties of Laplace operators lead naturally to a class of numerical approximations
known as Galerkin methods (e.g., \citet{Braess:2007}). Using the embedding map
$\rho$, we can derive a Galerkin method that can be used to compute the Fourier basis functions
in \cref{theorem:spectral}---that is, we can use the nonlinear representation approach in the
numerical approximation of the linear representation. The Galerkin method
can be more accurate than the spectral approach in \cref{algorithm:spectral},
and was used to render Figures \ref{fig:fourier:p1}, \ref{fig:fourier:p6}
and \ref{fig:fourier:I23}.

Galerkin methods posit basis functions
${\chi_1,\ldots,\chi_m}$ and approximate an infinite dimensional function space by the finite-dimensional subspace
${\text{span}\braces{\chi_1,\ldots,\chi_m}}$.
In our case, we approximate solutions $e$ of \eqref{sturm:liouville} by approximating their
restrictions $e|_\Pi$. We hence need functions ${\chi_i:\Pi\rightarrow\mathbb{R}}$.
W we start with functions ${\tilde{\chi}_i:\mathbb{R}^N\rightarrow\mathbb{R}}$, and set
${\chi_i:=\tilde{\chi}_i\circ\rho}$. We then assume ${e|_{\Pi}}$ of \eqref{sturm:liouville} is in the
span, and hence of the form
\begin{equation}
  \label{galerkin:ansatz}
  e|_{\Pi}\;=\;\msum_{i\leq m}c_i\chi_i\;.
\end{equation}
If $e$ solves the eigenvalue problem \eqref{sturm:liouville}, $e|_{\Pi}$ satisfies
\begin{align*}
  \sp{-\Delta e|_{\Pi},\chi_j}_{\L_2}\;=\;\lambda\sp{e|_{\Pi},\chi_j}_{\L_2}\qquad\text{ for all }j\leq m\;.
\end{align*}
Applying \eqref{eq:theorem:laplacian} and substituting in \eqref{galerkin:ansatz} shows
\begin{equation*}
  a(e|_{\Pi},\chi_j)\,=\,\lambda\sp{e|_{\Pi},\chi_j}_{\L_2}
  \quad\text{ and }\quad
  \msum_{i}c_ia(\chi_i,\chi_j)\,=\,\lambda\msum_ic_i\sp{\chi_i,\chi_j}_{\L_2}\,.
\end{equation*}
If we define matrices with entries ${A_{ij}:=a(\chi_i,\chi_j)}$
and ${B_{ij}:=\sp{\chi_i,\chi_j}_{\L_2}}$, that becomes
\begin{equation*}
  Ac\;=\;\lambda Bc\qquad\text{ where }c=(c_1,\ldots,c_m)^{\trans}\;.
\end{equation*}
The entries of $A$ and $B$ can be computed with off-the-shelf cubature
methods, and we can then solve for the pair $(\lambda,c)$.
\begin{remark}
  (a) If $\rho$ and the basis functions are implemented with JAX \citep{jax2018github} or a similar
  automatic differentiation tool, the gradients in \eqref{energy:form} are available, which avoids finite
  difference approximation and explicit computation of second derivatives.\\[.2em]
  (b) Neumann boundary conditions for reflections (see \cref{remark:neumann}) can be
  enforced using the methods of \citet{golub1973some}.\\[.2em]
  (c) The basis functions $\tilde{\chi}_i$ can be almost any basis on $\mathbb{R}^N$.
  Figures \ref{fig:fourier:p1}--\ref{fig:fourier:I23} were rendered by placing points ${x_1,x_2,\ldots}$ uniformly on $\Pi$,
  and centering radial basis functions at the points ${\rho(x_i)}$ in $\mathbb{R}^N$.
\end{remark}

\section{Related work and additional references}
\label{sec:related:work}

{\noindent\bf In machine learning}.
There has been substantial work on group invariance and equivariance in machine learning, with a focus on finite and compact groups.
Most salient has been work on approximate translation invariance and equivariance in convolutional neural networks for images \cite{lecun1989handwritten,krizhevsky2017imagenet} and speech \cite{abdel2014convolutional}, although this work has not been framed in a group-theoretic way.
To our knowledge the earliest explicit consideration of compact and finite group structure in machine learning was from a Fourier perspective by \citet{kondor2008group}; this was primarily in the context of Hilbert-space formalisms of learning.
The current perspective on compact and finite group equivariance in deep learning arose largely from \citet{cohen2016group}.
There has been widespread application of machine learning models when group invariance or equivariance is desired, e.g., permutation invariance for sets \cite{zaheer2017deep} and equivariance for neural auction design \cite{rahme2021permutation}.
In the natural sciences, rotation invariance has been used for astronomy \cite{dieleman2015rotation} and $E(3)$ equivariance has proved important for molecular applications \cite{batzner20223}.
Permutation equivariance of transformer architectures plays a crucial role in large language model \cite{vaswani2017attention}.
\\[.5em]
{\noindent\bf In crystallography}.
  Crystallographers have completely described the 17 two-dimensional and 230 three-dimensional
  crystallgraphic groups and various tilings they describe,
  and tabulated many of their properties \citep{Hahn:2011}.
  The emphasis in this work differs somewhat from that in mathematics---in particular,
  work in crystallography emphasizes polytopes $\Pi$ that occur in crystal structures
  (and which are not necessarily exact in the terminology used in \cref{theorem:embedding}),
  whereas more abstract work in geometry tends to work with Dirichlet domains or other exact tilings.
  A long line of work in the context of X-ray crystallography modifies the matrices
  that occur in fast Fourier transforms (FFTs) to speed up computation if a crystallographic symmetry
  is present in the data. This starts with the work of
  \citet{Bienenstock:Ewald:1962} and \citet{TenEyck:1973}, see also \citet{An:Cooley:Tolimieri:1990}.
  The introduction of \citet{Seguel:Burbano:2003} gives an overview.
  This work does not attempt to derive invariant Fourier bases.
\\[.5em]
{\noindent\bf In Fourier and PDE analysis}.
As we have already explained in some detail, the special case of \cref{theorem:spectral}
for ${\Pi=[0,1]^n}$ and ${\group=\mathbb{Z}^n}$ yields
the Fourier transform. For this problem, the periodic boundary condition can be
replaced by a Neumann condition, and spectral problems with Neumann conditions
are standard material in textbooks \citep{Brezis:2011,Lax:2002}.
For shifts that are not axis-parallel, the
periodic boundary condition is known
as a \kword{Born-von Karman boundary condition} \citep{Ashcroft:Mermin:1976}. We are not aware of
extensions to crystallographic groups.
An introduction to the PDE techniques used in our proofs
can be found in \citet{Brezis:2011}. The conditions imposed there are too restrictive
for our problems, however; a treatment general enough to cover all results we use
is given by \citet{McLean:2000}.
\\[.5em]
{\noindent\bf In geometry}.
\citet[e.g.,][]{Thurston:1997} coined the term orbifold in the
1970s. Commonly cited references include \citet{Scott:1983,Bonahon:Siebenmann:1985,Thurston:1997};
\citet{Apanasov:2000} has a detailed bibliography.
These all focus on general groups, however, for which the theory is much harder than in our case.
The quotient space structure of crystallographic groups was already understood much earlier
by the G\"{o}ttingen and Moscow schools \citep{Vinberg:Shvarsman:1993}.
A readable introduction to isometry groups and their quotients is given by \citet{Bonahon:2009}.
The comprehensive account of \citet{Ratcliffe:2006} is
more demanding, but covers all results needed in our proofs.
\citet{Vinberg:Shvarsman:1993} cover the
geometric aspects of crystallographic groups. \citet*{conway2016symmetries}
explain the geometry of orbifolds heuristically, with many illustrations.

\section{Some open problems}
\label{sec:open:problems}

Our approach raises a range of further questions well beyond the scope of the present paper,
including in particular those concerning numerical and statistical accuracy. We briefly
discuss some aspects of this problem.
\\[.5em]
{\bf Linear representation}.
Suppose we represent a $\group$-invariant continuous function $f$
by evaluating the generalized Fourier basis in \cref{theorem:spectral}
using the spectral algorithm in \cref{sec:spectral:algorithms}.
The algorithm returns numerical approximations ${\hat{e}_1,\hat{e}_2,\ldots}$ of the
basis functions. We may then expand $f$ as
\begin{equation*}
  f\;\approx\;\msum_{i=1}^m c_i\hat{e}_i\;.
\end{equation*}
There are three principal sources of error in this representation:
\begin{enumerate}[itemsep=-.2em]
\item The truncation error, since $m$ is finite.
\item Any error incurred in computation of the coefficients $c_i$.
\item The error incurred by approximating the actual basis functions $\ef_i$ by $\hat{e}_i$.
\end{enumerate}
The truncation error (1) concerns the question how well the vector space ${\text{span}\braces{\ef_1,\ldots,\ef_m}}$
approximates the space $\C_\group$ or $\L_2(\Pi)$. This problem is
studied in approximation theory. Depending on the context, one may choose the first $m$ basis vectors
(a strategy called ``linear approximation'' in approximation theory), or greedily choose those $m$ basis
vectors that minimize some error measure (``nonlinear approximation''), see \citet{DeVore:1998}.
Problem (2) depends on the function $f$, and on how it is represented computationally. If $f$ must
itself be reconstructed from samples, the coefficients are themselves estimators and incur statistical
errors.

The error immediately related to our method is (3), and
for the method of \Cref{sec:linear} depends on how well the graph Laplacian used
in \cref{sec:spectral:algorithms} approximates the Laplacian~$\Delta$.
This problem has been studied in a number of fields, including machine learning in the context of dimensionality
reduction \cite{belkin2003laplacian} and numerical mathematics in the context of
homogenous Helmholtz equations \citep{harari1995accurate}, and is the subject of a rich literature
\cite{lafon2004diffusion, belkin2008towards, singer2006graph, hein2005graphs, dunson2021spectral, garcia2020error}.
Available results show that, as ${\varepsilon\rightarrow 0}$ in the $\varepsilon$-net, the matrix $L$ converges to $\Delta$,
where the approximation can be measures in different notions of convergence, in particular pointwise and spectral convergence.
The cited results all concern the manifold case. We are not aware of similar results for orbifolds.

For the method of \Cref{sec:linear-reps-revisited}, the error depends largely on the choice of basis in~$\mathbb{R}^N$ and the accuracy of the numerical integrals, as well as the orbifold map approximation itself (see below).
Error analysis of the Rayleigh-Ritz method has a long history, see, e.g., \citet{weinberger1952error, wendroff1965bounds, weinberger1974variational}.
\\[.5em]
{\bf Nonlinear representation}.
If we define a $\group$-invariant statistical or machine learning model on $\mathbb{R}^n$ by factoring it through an
orbifold, one may ask approximation questions of a more statistical flavor:
Suppose we define a class~${\mathcal{H}=\braces{h_\theta|\theta\in T}}$ of functions~${h_\theta:\mathbb{R}^N\rightarrow\mathbb{R}}$
on the embedding space $\mathbb{R}^N$, with some parameter space $T$.
We then define a class $\mathcal{F}$ of $\group$-invariant functions on $\mathbb{R}^n$ as
\begin{equation*}
  \mathcal{F}\;:=\;\braces{f_\theta|\theta\in T}
  \quad\text{ where }\quad
  f_\theta\;:=\;h_\theta\circ\rho\;.
\end{equation*}
Depending on the context, we may think of the functions $f_\theta$ e.g., as neural networks or regressors.
The task is then to conduct inference, i.e., to compute a point estimate $\hat{\theta}$ of $\theta$
(say by maximum likelihood estimation or empirical risk minimization), or to compute a posterior
on $T$ in a Bayesian setup.
Since $\mathcal{H}$ and $\mathcal{F}$ share the same parameter space, any such inference task
can be ``pushed forward'' forward to the embedding space, that is,
\begin{equation*}
  \text{inference under }\mathcal{F}\text{ given }x_1,\ldots,x_n
  \;=\;
  \text{inference under }\mathcal{H}\text{ given }\rho(x_1),\ldots,\rho(x_n)
\end{equation*}
The error can again be separated into components:
\begin{enumerate}[itemsep=-.2em]
\item The statistical error associated with fitting $\braces{h_\theta|\theta\in T}$.
\item The ``forward distortion'' introduced by the map ${x_j\mapsto\rho(x_j)}$.
\item The ``backward distortion'' introduced by the map $h_\theta\mapsto h_\theta\circ\rho$.
\end{enumerate}
Problem (1) reduces to the statistical properties of $\mathcal{H}$, and depends on both the
model and the chosen inference method. Problem (2) and (3), however, raise a number of new questions:
The map $\rho$ is, by \cref{theorem:embedding}, bijective (which means it does not introduce
identifiability problems) and continuous. As the proof of \cref{result:kernel:compactness} shows,
it also preserves density properties of certain function spaces, which can be thought of
as a qualitative approximation result. Quantitative results are different matter:
To bound the effect of transformations on
statistical errors typically requires a stronger property than continuity, such as differentiability
or at least a Lipschitz property. In results on manifold learning, the curvature of
$\Omega$ often plays an explicit role. Orbifolds introduce a further challenge,
since smoothness properties fail at the
tips and edges introduced by points with non-trivial stabilizers.
On the other hand, non-differentiabilities of crystallographic orbifolds
have lower-bounded opening angles \citep{Thurston:1997}---note the tip of the cone in
\cref{fig:cone}, for example, is not a cusp---so it may be possible to mitigate these problems.

\subsection*{Acknowledgements}
The authors would like to thank Elif Ertekin and Eric Toberer for valuable discussions.
RPA is supported in part by NSF grants IIS-2007278 and OAC-2118201. PO is supported by
the Gatsby Charitable Foundation.

\bibliography{references}

{\small
\vspace{8em}
{\noindent\sc Department of Computer Science\\[-.1em]
  Princeton University}\\[-.1em]
\url{https://www.cs.princeton.edu/~rpa}
\\[1em]
{\noindent\sc Gatsby Computational Neuroscience Unit\\[-.1em]
  University College London}\\[-.1em]
\url{https://www.gatsby.ucl.ac.uk/~porbanz}
}

\newpage
\appendix

{\centering \bf Appendix \par}
\vspace{1em}
The first three sections of this appendix provide mathematical background on isometries
(\cref{sec:isometries}), function spaces and smoothness (\cref{sec:function:spaces}), orbifolds
(\cref{sec:orbifolds}), and spectral theory (\cref{sec:spectral}). The proof of the Fourier representation is subdivided into three parts:
We first prove of the flux property, \cref{lemma:gradient:field}, in \cref{appendix:proofs:flux}, and
\cref{theorem:laplacian} on self-adjointness of the Laplacian in \cref{appendix:proofs:laplacian}.
Using these results, we then prove the Fourier representation in \cref{appendix:proofs:spectral}.
The proof of the embedding theorem (\cref{theorem:embedding}) follows in \cref{appendix:proofs:embedding}.
\cref{appendix:proofs:kernels} collects all proofs on kernels and Gaussian processes.

\section{Background I: Isometries of Euclidean space}
\label{sec:isometries}

Isometries are invertible functions that preserve distance. To define
an isometry between two sets $V$ and $W$, both must be equipped with metrics,
say $d_V$ and $d_W$.
A map~${\phi:\xspace\rightarrow\yspace}$ is then an \kword{isometry} if it is one-to-one
and satisfies
\begin{equation*}
  d_W(\phi(v_1),\phi(v_2))\;=\;d_V(v_1,v_2)
  \qquad\text{ for all }v_1,v_2\in V\;.
\end{equation*}
Since this implies $\phi$ is Lipschitz, isometries are always continuous.
If ${W=V}$, then $\phi$ is necessarily bijective.
An isometry of $\mathbb{R}^n$ is a bijection ${\phi:\mathbb{R}^n\rightarrow\mathbb{R}^n}$ that satisfies
\begin{equation*}
  d_n(\phi x,\phi y)\;=\;d_n(x,y)\qquad\text{ for all }x,y\in\mathbb{R}^n\;.
\end{equation*}
Identity \eqref{eq:euclidean:isometry} shows that every isometry can be uniquely
represented as an orthogonal transformation followed by a shift.
Loosely speaking, an isometry may shift, rotate, or flip $M$, but cannot change its shape or volume.
Recall that a set $\group$ of functions ${\mathbb{R}^n\rightarrow\mathbb{R}^n}$
is a \kword{group} if
it contains the identity map $\Id$, and if ${\phi,\psi\in\group}$ implies ${\phi\circ\psi\in\group}$
and ${\phi^{-1}\in\group}$.
The set of all isometries of $\mathbb{R}^n$ forms a group, called the
\kword{Euclidean group} of order $n$.

\subsection{More on crystallographic groups}

{\noindent\bf Representation by shifts and orthogonal transformations}.
Since every isometry can be decomposed into an orthogonal transformation
and a shift according to \eqref{eq:euclidean:isometry}, every crystallographic group $\group$ has two natural subgroups:
One is the group
\begin{equation*}
  \group_o\;:=\;\braces{\phi\in\group\,|\,\phi(x)=Ax\text{ for some }A\in\mathbb{O}_n}\;=\;\group\cap\mathbb{O}_n
\end{equation*}
of purely orthogonal transformations. This is an example of a \kword{point group},
since all its elements have a common fixed point (namely the origin).
It is always finite: Fix any $x$ on the unit sphere in $\mathbb{R}^n$. Then ${\phi(x)}$ is also on the sphere for every ${\phi\in\group_o}$, since ${A_\phi}$ is orthogonal. However, discreteness requires there can only be finitely many such points~$\phi(x)$ on the sphere.
The other is the group of pure shifts,
\begin{equation*}
  \group_t\;:=\;\braces{\phi\in\group\,|\,\phi(x)=x+b\text{ for some }b\in\mathbb{R}^n}\;.
\end{equation*}
One can show there are
linearly independent vectors ${b_1,\ldots,b_n}$ such that
\begin{equation*}
  \group_t\;:=\;\braces{x\mapsto x+b\,|\,b=a_1b_1+\ldots+a_nb_n\text{ for }a_1,\ldots,a_n\in\mathbb{Z}}\;.
\end{equation*}
Thus, the generating set for a crystallographic group on $\mathbb{R}^n$ always includes $n$ linearly independent shifts.

\subsection{Equivalence to definitions in the literature}
\label{remark:definition}

Our definition of a crystallographic group in \cref{sec:definitions} differs
from those in the literature---we have chosen it for simplicity, but must verify it is equivalent.
There are two standard definitions of crystallographic groups: Perhaps the most common one, used for example by
\citet{Thurston:1997}, is as a discrete group of isometries for which~${\mathbb{R}^n/\group}$ is compact in the quotient topology.
Another is as a group of isometries of $\mathbb{R}^n$ such that~${\mathbb{R}^n/\group}$ has finite volume when identified with a subset of Euclidean space
\citep{Vinberg:Shvarsman:1993}. These are known to be equivalent \citep[][Corollary of Theorem 1.11]{Vinberg:Shvarsman:1993}.
Our definition is equivalent to both:
\begin{lemma}
  A group $\group$ is crystallographic in the sense of \cref{sec:definitions}
  if and only if it is a discrete group of isometries of $\mathbb{R}^n$ such
  that ${\mathbb{R}^n/\group}$ is compact.
\end{lemma}
\begin{proof}
  If $\group$ is crystallographic in our sense,
  it is discrete (see \cref{sec:definitions}), and $\mathbb{R}^n/\group$ is compact
  by \cref{fact:compact:quotient}, so it satisfies the second definition above.
  Conversely, if $\group$ satisfies Thurston's definition,
  it tiles $\mathbb{R}^n$ with some set $\Pi$. This set can always be chosen as a convex polytope
  \citep[][Theorem 2.5]{Vinberg:Shvarsman:1993}, so $\group$ is crystallographic in our sense.
\end{proof}

We note only en passe that there are tilings that cannot be described
by a group of isometries. That is not at all obvious---the question was one of Hilbert's
problems---but counter-examples of such tilings (with non-convex polytopes) are now
known \citep[see][Chapter 32]{Gruber:2007}.

\section{Background II: Function spaces}
\label{sec:function:spaces}

This section briefly reviews concepts from functional analysis that play a role in the proofs.
Helpful references include \citet{Aliprantis:Border:2006,Brezis:2011} on general functional analysis
and Banach spaces, \citet{Brezis:2011,Adams:Fournier:2003} on Sobolev spaces,
\citet{Steinwart:Christmann:2008} on reproducing kernel Hilbert spaces, and \citet{Aliprantis:Burkinshaw:2006}
on compact operators.

\subsection{Spans and their closures}

Consider a Banach space $V$ and a subset ${\mathcal{F}\subset V}$.
The \kword{span} of $\mathcal{F}$ is the set
\begin{equation*}
  \text{span}(\mathcal{F})
  =
  \braces{\tsum_{i\leq n}c_if_i\,|\,n\in\mathbb{N},c_i\in\mathbb{R},f_i\in\mathcal{F}}
\end{equation*}
of finite linear combinations of elements of $\mathcal{F}$.
Since function spaces are typically infinite-dimensional, we also consider
infinite linear combinations. These are defined with respect to a norm $\|\argdot\|$:
\begin{equation*}
  f\;=\;\tsum_{i\in\mathbb{N}}c_if_i
  \quad\text{ means }\quad
  \|f-\tsum_{i\leq n}c_if_i\|\;\rightarrow\;0\;\text{ as }n\rightarrow\infty\;.
\end{equation*}
In other words, to get from the span to the set of infinite linear combinations,
we take the closure in the relevant norm:
\begin{equation*}
  \braces{\tsum_{i\in\mathbb{N}}c_if_i\,|\,c_i\in\mathbb{R},f_i\in\mathcal{F}}
  \;=\;
  \overline{\text{span}(\mathcal{F})}
\end{equation*}

\subsection{Bases}
A Hilbert space $\mathbb{H}$ is a Banach space whose norm is induced by an inner product~$\sp{\argdot,\argdot}_\mathbb{H}$,
that is,
\begin{equation*}
  \|f\|_{\mathbb{H}}\;=\;\sqrt{\vphantom{\sp{\argdot,\argdot}}\smash{\sp{f,f}_\mathbb{H}}}\;.
\end{equation*}
A sequence ${f_1,f_2,\ldots}$ in a Hilbert space is an
\kword{orthonormal system} if ${\sp{f_i,f_j}=\delta_{ij}}$, where $\delta$ is the Kronecker symbol
(the indicator function of $\braces{i=j}$).
An orthonormal system is \kword{complete} if its span is dense in $\mathbb{H}$, that is, if
\begin{equation*}
  \mathbb{H}\;=\;\overline{\text{span}\braces{f_1,f_2,\ldots}}\;,
\end{equation*}
where the closure is taken in the norm of $\mathbb{H}$. A complete orthonormal system
is also called an \kword{orthonormal basis}. If ${f_1,f_2,\ldots}$ is an orthonormal
basis, $\mathbb{H}$ can be represented as
\begin{equation*}
  \mathbb{H}
  \;=\;
  \bigl\lbrace
  \tsum_{i\in\mathbb{N}}c_if_i
  \text{ (convergence in $\|\argdot\|_{\mathbb{H}}$)}\,|\,
    c_1,c_2,\ldots\in\mathbb{R}\text{ with }\tsum_ic_i^2<\infty
    \bigr\rbrace\;.
\end{equation*}

\subsection{$\L_2$ spaces}

For any $M$ and a
$\sigma$-finite measure $\nu$ on $M$, the $\L_2$-scalar product and pseudonorm are
\begin{equation*}
  \sp{f,g}_{\L_2(\nu)}\;:=\;\mint_{M}f(x)g(x)\nu(dx)
  \quad\text{ and }\quad
  \|f\|_{\L_2(\nu)}:=\sqrt{\sp{f,f}_{\L_2(\nu)}}\;.
\end{equation*}
To make $\|\argdot\|_{\L_2}$ a norm, one defines the equivalence classes
${[f]:=\braces{g\,|\,\|f-g\|_{\L_2}\!=0}}$ of functions identical outside a null set, and
the vector space
\begin{equation*}
  \L_2(\nu)\;:=\;\braces{[f]\,|\,f:M\rightarrow\mathbb{R}\text{ and }\|f\|_{\L_2}<\infty}
\end{equation*}
of such equivalence classes, which is a separable Hilbert space.
Although its elements are not technically functions,
we use the notation ${f\in\L_2}$ rather than ${[f]\in\L_2}$.
We write~${\L_2(\mathbb{R}^n)}$ and~${\L_2(\Pi)}$ respectively if $\nu$ is Euclidean volume
on $\mathbb{R}^n$ or on $\Pi$.
See \citet{Aliprantis:Border:2006} or \citet{Brezis:2011} for background on $\L_2$ spaces.

\subsection{Reproducing kernel Hilbert spaces}
\label{sec:RKHS}

Consider a set ${M\subseteq\mathbb{R}^n}$.
A symmetric positive definite function ${\kappa:M\times M\rightarrow\mathbb{R}}$
is called a \kword{kernel}.
A kernel defines a Hilbert space as follows: The formula
\begin{equation*}
  \sp{\tsum_{i}a_i\kappa(x_i,\argdot),\,\tsum_{j}b_j\kappa(y_j,\argdot)}_{\mathbb{H}}
  \;:=\;
  \tsum_{i,j}a_ib_i\kappa(x_i,y_j)\quad\text{ for }a_i,b_j\in\mathbb{R}\text{ and }x_i,y_j\in M
\end{equation*}
defines a scalar product on ${\text{span}\braces{\kappa(x,\argdot)|x\in M}}$.
The closure
\begin{equation*}
  \mathbb{H}\;:=\;\overline{\text{span}\braces{\kappa(x,\argdot)|x\in M}}
  \quad\text{ with respect to the norm }\quad
  \|f\|_{\kappa}\;:=\;\sqrt{\sp{f,f}_{\mathbb{H}}}
\end{equation*}
is a real, separable Hilbert space with inner product ${\sp{\argdot,\argdot}_{\mathbb{H}}}$, called the
reproducing kernel Hilbert space or \kword{RKHS} of $k$.
Every RKHS satisfies the ``reproducing property''
\begin{equation}
  \label{eq:reproducing}
  f(x)\;=\;\sp{f,\kappa(x,\argdot)}_{\mathbb{H}}\qquad\text{ for all }f\in\mathbb{H}\text{ and all }x\in M\;.
\end{equation}
In particular, ${\kappa(x,y)=\sp{\kappa(x,\argdot),\kappa(y,\argdot)}_{\mathbb{H}}}$.
If ${f_1,f_2,\ldots}$ is an orthonormal basis of $\mathbb{H}$, then
\begin{equation}
  \label{RKHS:ONS}
  \kappa(x,y)\;=\;\tsum_{i\in\mathbb{N}}f_i(x)f_i(y)\qquad\text{ for all }x,y\in M\;.
\end{equation}
If $\mathbb{H}$ is an RKHS, the map ${f\mapsto f(x)}$ is continuous for each ${x\in M}$.
Conversely, if $\mathbb{H}$ is any Hilbert space of real-valued functions on $M$, and if
the maps are continuous on $\mathbb{H}$ for all ${x\in M}$, there
is a unique kernel satisfying \eqref{eq:reproducing} that generates
$\mathbb{H}$ as its RKHS.

\subsection{Spaces of continuous functions}

For any set $M$, the vector space $\C(M)$ of continuous functions equipped with the
norm~${\|\argdot\|_{\sup}}$ is a Banach space. It is separable if $M$ is compact
\citep{Aliprantis:Border:2006}.
In the proof of the spectral theorem, we must also consider the set
\begin{equation*}
  \C_u(M)\;:=\;\braces{f\in\C(M)\,|\,f\text{ uniformly continuous}}\;,
\end{equation*}
and the compactly supported functions
\begin{equation*}
  \C_c(M)\;=\;\braces{f\in\C(M)\,|\,f=0\text{ outside a compact set }K\subset M}\;.
\end{equation*}
We recall some basic facts from analysis that are used in the proofs:
\begin{fact}[\citet{Aliprantis:Border:2006}]
  \label{fact:continuous:extension}
  (i) Every continuous function on a compact set is uniformly continuous.
  (ii) Every uniformly continuous function $f$ on a set ${M\subseteq\mathbb{R}^n}$ has a unique continuous
  extension $\bar{f}$ to the closure $\overline{M}$. Its value at a boundary point ${x\in\partial M}$ is given by
  ${\bar{f}(x)=\lim_j(x_j)}$ for any sequence of points ${x_i\in M}$ with ${x_i\rightarrow x}$.
\end{fact}

\subsection{Smoothness spaces}
\label{sec:smoothness:spaces}

Smoothness spaces quantify the smoothness of functions in terms of a norm.
Two types of such spaces play a role in our results, namely $\C^k$ spaces
and Sobolev spaces. Both define smoothness via derivatives:
We denote partial derivatives as
\begin{equation*}
  \partial^{\alpha}f
  \;:=\;
  \frac{\partial^{|\alpha|}f}
       {\partial x_1^{\alpha_1}\cdots\partial x_n^{\alpha_n}}
       \qquad\text{ where }\alpha=(\alpha_1,\ldots,\alpha_n)\in\mathbb{N}^n\text{ and }|\alpha|:=\alpha_1+\ldots+\alpha_n\;.
\end{equation*}
If we are taking a derivative with respect to the $i$th coordinate, we use a subscript,
\begin{equation*}
\partial_i f\;:=\;\mfrac{\partial f}{\partial x_i}
\end{equation*}
The set $\C^k$ of $k$ times continuously differentiable functions can then be
represented as
\begin{equation*}
  \mathbf{C}^k(M)
  \;:=\;
  \braces{f\in\mathbf{C}(M)\,|\,\partial^{\alpha}f\in\mathbf{C}(M)\text{
      whenever }
    |\alpha|\leq k}
  \qquad\text{ where }k\in\mathbb{N}\cup\braces{0,\infty}\;.
\end{equation*}
Since that means the norm of $\mathbf{C}$ is applicable to $\partial^{\alpha}f$,
we can define
\begin{equation*}
  \|f\|_{\mathbf{C}^k}
  \;:=\;
  \|f\|_{\sup}
  +
  \msum_{|\alpha|\leq r} \|\partial^\alpha f\|_{\sup}\;.
\end{equation*}
It can be shown that this is again a norm, and that it makes $\C^k$ a Banach space
\citep{Brezis:2011}.
$\C^k$ functions are uniformly continuous, and even very smooth functions approximate
elements of $\L_2$ to arbitrary precision:
\begin{fact}
  \label{fact:continuous:extension}
  Let ${M\subseteq\mathbb{R}^n}$ be a set.
  (i) If ${f\in\C^k(M)}$ for ${k\geq 1}$, then $f$ and its first ${k-1}$ derivatives are uniformly continuous.
  (ii) The set ${\C_c(M)\cap\C^{\infty}(M)}$ is dense in $\L_2(M)$.
\end{fact}
The $\C^k$ norms measure smoothness in a worst-case sense. To measure average smoothness instead,
we can replace the sup norm by the $\L_2(M)$-norm: The function
\begin{equation*}
  \|f\|_{\H^k}
  \;:=\;
  \|f\|_{\L_2}
  +
  \msum_{|\alpha|\leq k} \|\partial^\alpha f\|_{\L_2}\;,
\end{equation*}
is a norm, called the \kword{Sobolev norm} of order $k$. It makes the set
\begin{equation*}
  \H^k(M):=\braces{f\in\L_2(M)\,|\,\|f\|_{\H^k}<\infty}\;=\;
  \braces{f\in\L_2(M)\,|\,\partial^\alpha f\in\L_2(M)}
\end{equation*}
a Banach space, and even a Hilbert space, called the \kword{Sobolev space} of order $k$.
We will only work with the spaces $\H^1(M)$. A inner product on $\H^1(M)$ is given by
\begin{equation*}
  \label{eq:sobolev:sp}
  \sp{f,g}_{\H_1}\;:=\;\sp{f,g}_{\L_2}\;+\;\msum_{i\leq n}\sp{\partial_i f,\partial_i g}_{\L_2}\;.
\end{equation*}
The Sobolev norms are stronger than the ${\L_2}$ norm: We have
\begin{equation*}
  \|f\|_{\L_2(M)}\;\leq\;c_M\|f\|_{\H_1(M)}\qquad\text{ for all }f\in\L_2(M)\text{ and some }c_M>0\;.
\end{equation*}
Consequently, the approximation property in \cref{fact:continuous:extension}(ii) does not
necessarily hold in the Sobolev norm. Whether it does depends on whether the geometry of
the domain $M$ is sufficiently regular:
\begin{fact}
  \label{distributions:dense:H1}
  Let $\Gamma$ be a Lipschitz domain (such as a convex polytope). Then
  ${\C_c(\Gamma)\cap\C^\infty(\Gamma)}$ is dense in ${\H^1(\Gamma^{\circ})}$.
\end{fact}
A readable introduction to Sobolev spaces is given by \citet{Brezis:2011}.
The monographs of \citet{Adams:Fournier:2003} and \citet{Mazya:2011} are comprehensive accounts.

\subsection{Inclusion maps}
\label{sec:inclusion}

If ${V\subset W}$ are two sets, the \kword{inclusion map} or \kword{injection map} ${\iota:V\hookrightarrow W}$
is the restriction of the identity on $W$ to $V$. Loosely speaking, $\iota$ maps each point $v$ in $V$ to itself,
but~$v$ is regarded as an element of $V$ and its image $\iota(v)$ as an element of $W$.
This distinction is not consequential if $V$ and $W$ are simply sets without further structure, but if both
are equipped with topologies, the properties of $\iota$ encode relationships
between these topologies.
\\[.5em]
{\noindent\bf Continuous inclusions}. Suppose both $V$ and $W$ are equipped with topologies.
Call these the $V$- and $W$-topology. The restriction of the $W$-topology to $V$, often called the
relative $W$-topology, consists of all sets of the form $A\cap V$, where
${A\subset W}$ is open in $W$. Since ${A\cap V}$ is precisely the preimage $\iota^{-1}A$, and
continuity means that preimages of open sets are open, we have
\begin{equation*}
  \iota \text{ continuous }\quad\Longleftrightarrow\quad
  \text{ the $V$-topology is at least as fine as the restricted $W$-topology.}
\end{equation*}
{\noindent\bf Inclusions between Banach spaces}.
Let ${T:V\rightarrow W}$ be a map from a Banach space~$V$ to another Banach space~$W$.
If such a map is linear, it is called a \kword{linear operator}. It is continuous if and only if it is \kword{bounded},
\begin{equation}
  \label{bounded:linear:operator}
  \sup_{v\in V}\frac{\|T(v)\|_W}{\|v\|_V}\;<\;\infty
  \quad\text{ or equivalently }\quad
  \|T(v)\|_W\;\leq\; c\|v\|_V\text{ for some }c>0\;.
\end{equation}
If $V$ is a vector subspace of $W$, then $\iota$ is automatically linear, so it is continuous iff
\begin{equation*}
  \|v\|_W\;=\;\|\iota(v)\|_W\;\leq\;c\|v\|_V\;.
\end{equation*}
Saying that $\iota$ is continuous is hence another way of saying that ~$\|\argdot\|_V$ is stronger than~$\|\argdot\|_W$. If $V$ and $W$ are smoothness
spaces, continuity of $\iota$ can hence often be interpreted as the elements $V$ being smoother than those of $W$.
A set ${A\subset V}$ is \kword{norm-bounded} if
\begin{equation*}
  \sup\nolimits_{v,v'\in A}\|v-v'\|_V\;<\;\infty\;.
\end{equation*}
A linear operator between Banach spaces is \kword{compact} if the image $T(A)$ of every norm-bounded set ${A\subset V}$
has compact closure in $W$ \citep{Aliprantis:Burkinshaw:2006}. The inclusion is hence compact iff
\begin{equation*}
  A\subset V\text{ is bounded in }\|\argdot\|_V\quad\Longrightarrow\quad\text{ the $\|\argdot\|_W$-closure of }A\text{ in $W$ is compact.}
\end{equation*}
If $V$ and $W$ are smoothness spaces, the inclusion is often compact if
$V$ is in some suitable sense smoother than $W$. The well-known Arzela-Ascoli theorem \citep{Aliprantis:Border:2006}, for example, can be interpreted in
this way. For Sobolev spaces, a family of results known as Rellich-Kondrachov theorems \citep{Adams:Fournier:2003} shows that, under suitable conditions on the domain,
inclusions of the form ${\H^{k+m}\hookrightarrow\H^k}$ and ${\H^{k+m}\hookrightarrow\C^k}$ exist and are compact if
the difference $m$ in smoothness is large enough. The following version is adapted to our purposes:
\begin{lemma}
  \label{lemma:kondrachov}
  Let $\Pi$ be a polytope and ${M\subseteq\intF}$ an open set. Then
  ${\H^{k+1}(M)\subset\C^k(M)}$ for~${k\geq 0}$, and the inclusion map is compact.
\end{lemma}
\begin{proof}
  Since $\Pi$ is a polytope, it has the strong local Lipschitz property in the terminology of
  \citet[][4.9]{Adams:Fournier:2003}.
  By the relevant version of the Rellich-Kondrachov theorem, that implies that the set of restrictions of
  functions in ${\H^{k+1}(\intF)}$ from $\intF$ to $M$ is a compactly embedded subset
  of ${\C^{k}(M)}$  \citep[][6.3 III]{Adams:Fournier:2003}. The image of ${\H^{k+1}(\intF)}$ under the projection ${f\mapsto f|_M}$
  is precisely ${\H^{k+1}(M)}$ \citep[][Chapter 3]{McLean:2000}.
\end{proof}

\section{Background III: Orbifolds}
\label{sec:orbifolds}

In this section, we give a rigorous definition of orbifolds and review those results from the literature required for our proofs.
For more background, see
\citep{Bonahon:2009,Vinberg:Shvarsman:1993,Thurston:1997,Bonahon:Siebenmann:1985,Scott:1983,Cooper:Hodgson:Kerckhoff:2000}.
\citet{Bonahon:2009} provides an accessible introduction to gluing
and quotient spaces. Most results below are adapted from the monograph of
\citet{Ratcliffe:2006}. Ratcliffe's formalism is very general and
can be simplified significantly for our purposes. We state results here in just
enough generality to apply to crystallographic groups.

\subsection{Motivation: Manifolds}

To motivate the somewhat abstract definition of an orbifold, we start with that of a manifold,
and then generalize to orbifolds below.
Recall that a set ${M}$ is a manifold if its topology ``locally looks like $\mathbb{R}^n$''.
This idea can be formalized in a number of ways. We first give a definition using a metric, which is of
the form often encountered in machine learning and statistics. We then generalize the metric definition to
a more abstract one that brings us almost to orbifolds as we see in the following section.
\\[.5em]
{\noindent\bf Metric definition}.
Let $M$ be a set equipped with a metric $d_M$. We then call $M$ a manifold if,
for every~${u\in M}$, we can choose a sufficiently small ${\varepsilon(u)>0}$ such that the~$d_M$-ball around $u$ of radius $\varepsilon$ is isometric to a $d_n$-ball of the same radius in $\mathbb{R}^n$.
There is, in other words, an isometry
\begin{equation*}
  \theta_u:\,B_{\varepsilon(u)}(u,d_M)\;\rightarrow\; B_{\varepsilon(u)}(\theta_u(u),d_n)\;\subset\;\mathbb{R}^n
  \qquad\text{ for each }u\in M\;.
\end{equation*}
For example, the circle, equipped with the geodesic distance, is a manifold in the sense of this definition:
It is not possible to map the entire circle isometrically to a subset of~$\mathbb{R}$.
However, the ball $B_{\varepsilon(u)}(u,d_M)$ around a point $u$ is a semiarc, drawn in black below:
\begin{center}
  \begin{tikzpicture}[>=stealth]
    \draw [thick,domain=0:360,gray!50!white] plot ({cos(\x)}, {sin(\x)});
    \draw [very thick,domain=-30:30,black] plot ({cos(\x)}, {sin(\x)});
    \node[draw,circle,fill,scale=.3] (x) at (1,0) {};
    \draw[thick,gray!50!white] (5,1.5)--(5,-1.5);
    \draw[very thick,black] (5,.6)--(5,-.6);
    \node[draw,circle,fill,scale=.3] (y) at (5,0) {};
    \draw[->] ($(x)+(.5,0)$)--($(y)+(-.5,0)$) node[pos=.5,label={above:{\scriptsize $\theta_u$}}] {};
    \node at ($(x)+(.2,-.2)$) {\scriptsize $u$};
    \node[anchor=west] at ($(y)+(.1,0)$) {\scriptsize $\theta_u(u)$};
    \node[anchor=west,gray] at ($(y)+(0,1.2)$) {\scriptsize $\mathbb{R}$};
    \node[gray] at (-.9,.9) {\scriptsize $M$};
\end{tikzpicture}
  \end{center}
  This semiarc can be mapped isometrically to an open interval in $\mathbb{R}$, and the same is true for the
  ball around any other point.
  \\[.5em]
{\noindent\bf Coherence property}.
Before we generalize this definition, we observe that it implies a coherence property
of the maps $\theta_u$. Suppose the balls around two points $v$ and $w$ in~$M$ overlap, and
$u$ is in both balls. We can then find a sufficiently small ${\varepsilon>0}$ such that
${B_\varepsilon(u,d_M)}$ is completely contained in both balls. Since both maps $\theta_v$ and $\theta_w$
are applicable to the points in this ball, the restrictions
\begin{equation*}
  \theta_v:B_\varepsilon(u,d_M)\rightarrow B_\varepsilon(\theta_v(u),d_n)
  \quad\text{ and }\quad
  \theta_w:B_\varepsilon(u,d_M)\rightarrow B_\varepsilon(\theta_w(u),d_n)
\end{equation*}
are both isometries. The points ${x=\theta_v(u)}$ and ${y=\theta_w(u)}$ are images under different
maps, and the balls around them are not required to overlap. Both are, however, Euclidean
balls of the same radius. If ${\psi=y-x}$ is the (unique) shift of $\mathbb{R}^n$ that maps $x$ to $y$,
we hence have
\begin{equation*}
  B_\varepsilon(\theta_v(u),d_n)=\psi B_\varepsilon(\theta_w(u),d_n)\;.
\end{equation*}
Now observe that ${\psi x=y=\theta_w\theta_v^{-1}(x)}$. There is, in summary, a shift $\psi$ such that
\begin{equation*}
  \psi x=y\quad\text{ and }\quad
  \theta_w\theta_v^{-1}(z)\;=\;\psi z\qquad\text{ for all }z\text{ in the ball }B_\varepsilon(x,d_n)\;.
\end{equation*}
The definition hence implies that the map $\theta_w\theta_v^{-1}$, often called a \kword{coordinate change}
in geometry, behaves like a shift on a sufficiently small neighborhood.
When we drop the metric from the definition below, this property no longer arises
automatically, and we must make it an explicit requirement.
\\[.5em]
  {\noindent\bf Abstract definition}.
  Let $\mathbb{F}$
  be a group of isometries of $\mathbb{R}^n$. The next definition
  generalizes the one above in two ways: It does not use a metric, and instead of requiring
  that coordinate changes look locally like shifts, it requires they look locally like elements of $\mathbb{F}$.
  A Hausdorff space $M$ is a \kword{$\mathbb{F}$-manifold} if:
\begin{enumerate}
\item There is a family $\braces{U_i}_{i\in\mathcal{I}}$ of open connected subsets of $M$ that cover $M$,
  i.e., each point of $M$ is in at least one set $U_i$. The set $\mathcal{I}$ is an arbitrary index set.
\item For each ${i\in\mathcal{I}}$, there is a
  homeomorphism ${\theta_i:U_i\rightarrow V_i}$ of $U_i$ and an open set ${V_i\subset\mathbb{R}^n}$.
\item If two sets $U_i$ and $U_j$ overlap, the maps $\theta_i$ and $\theta_j$ cohere as follows:
  If $x$ and $y$ are points in $\mathbb{R}^n$ that satisfy
  \begin{equation*}
    \theta_j\theta_i^{-1}(x)\;=\;y\;,
  \end{equation*}
  then there is a transformation ${\phi\in\mathbb{F}}$ such that
  \begin{equation*}
    \psi x\;=\;y
    \quad\text{ and }\quad
    \theta_j^{-1}\theta_i(z)\;=\;\psi z
    \qquad\text{ for all }z\text{ in a neighborhood of }x\;.
  \end{equation*}
\end{enumerate}
We recover the metric definition if we make $M$ a metric space (which is always Hausdorff),
set ${\mathcal{I}=M}$, choose $U_i$ as the ball around $i=u$ (which is always connected),
and $\theta_i$ as the isometry $\theta_u$ (isometries are homeomorphisms).

\subsection{Orbifolds}
To capture the properties of the quotient $\mathbb{R}^n/\group$, the definition of a manifold
is in general too restrictive. That follows from the following result:
\begin{fact}[\citet{Bonahon:2009} Theorem 7.8]
  \label{fact:bonahon}
  Let $\group$ be a crystallographic group that tiles $\mathbb{R}^n$ with a convex polytope.
  For every point ${x\in\mathbb{R}^n}$, there exists an ${\varepsilon>0}$ such that the
  open metric ball ${B_{d_\group}(\group(x),\varepsilon)}$ in the quotient space $\mathbb{R}^n/\group$ and the
  quotient ${B_{d_n}(x,\varepsilon)/\text{\rm Stab}(x)}$ of the corresponding open ball in
  $\mathbb{R}^n$ are isometric.
\end{fact}
We note this is precisely the metric definition of a manifold above if ${\Stab(x)=\braces{\Id}}$
for all points in $\mathbb{R}^n/\group$. It follows that, for a crystallographic group $\mathbb{G}$,
\begin{equation*}
  \mathbb{R}^n/\group\text{ is a manifold }
  \quad\Longleftrightarrow\quad
  \text{ no element of $\group$ has a fixed point.}
\end{equation*}
Let $\mathbb{F}$ be a group of isometries of $\mathbb{R}^n$. An \kword{$\mathbb{F}$-orbifold} is a
Hausdorff space $M$ with the following properties:
\begin{enumerate}
\item There is a family $\braces{U_i}_{i\in\mathcal{I}}$ of open connected subsets of $M$ that cover $M$,
  i.e., each point of $M$ is in at least one set $U_i$.
\item For each ${i\in\mathcal{I}}$, there is a discrete group $F_i$ of isometries of $\mathbb{R}^n$
  and a homeomorphism ${\theta_i:U_i\rightarrow\mathbb{R}^n/F_i}$ of $U_i$ and an open subset of the
  quotient space ${\mathbb{R}^n/F_i}$.
\item If two sets $U_i$ and $U_j$ overlap, the maps $\theta_i$ and $\theta_j$ cohere as follows:
  If $x$ and $y$ are points in $\mathbb{R}^n$, and the corresponding points ${F_ix\in\mathbb{R}^n/F_i}$
  and ${F_jy\in\mathbb{R}^n/F_j}$ satisfy
  \begin{equation*}
    \theta_j\theta_i^{-1}(F_i x)\;=\;F_jy\;,
  \end{equation*}
  then there is a transformation ${\phi\in\mathbb{F}}$ such that
  \begin{equation*}
    \psi x\;=\;y
    \quad\text{ and }\quad
    \theta_j^{-1}\theta_i(F_i z)\;=\;F_j(\psi z)
    \qquad\text{ for all }z\text{ in a neighborhood of }x\;.
  \end{equation*}
\end{enumerate}
The family $\braces{\theta_i}_{i\in\mathcal{I}}$ is called an \kword{atlas}.
Clearly, an $\mathbb{F}$-orbifold is an \kword{$\mathbb{F}$-manifold} if and only if
each $F_i$ is the trivial group ${F_i=\braces{\Id}}$.
\begin{lemma}
  If $\group$ is a crystallographic group that tiles $\mathbb{R}^n$ with a convex polytope $\Pi$,
  then $\mathbb{R}^n/\group$ is a $\group$-orbifold. At each point ${i=\group(x)}$, the group $F_i$
  is the stabilizer ${\text{Stab}(x)}$.
\end{lemma}
\noindent This lemma is folklore in geometry---see e.g.,  \citet{Bonahon:2009,Cooper:Hodgson:Kerckhoff:2000,Vinberg:Shvarsman:1993} for results that are phrased differently but amount to the same. We give a proof here only to match our
specific choices of definitions to each other.
\begin{proof}
  Let $\tilde{\Pi}$ be a transversal. We choose ${\mathcal{I}=\mathbb{R}^n/\group}$, so each ${i\in\mathcal{I}}$ is the orbit $\group(x)$
  of some point in $\mathbb{R}^n$, and hence of a unique point ${x\in\tilde{\Pi}}$.
  By \cref{fact:bonahon}, there is hence a map ${\theta_x}$ with ${\theta_x(\group(x))=x}$ that isometrically
  maps a ball ${B_{d_\group}(\group(x),\varepsilon)}$ with suitable radius to
  ${B_{d_n}(x,\varepsilon)/\text{\rm Stab}(x)}$. We hence set ${F_i=\Stab(x)}$, which is a finite subgroup
  of the discrete group $\group$, and hence discrete. What remains to be shown is the coherence property.
  Suppose $x$ and $y$ are points in $\mathbb{R}^n$ with trivial stabilizers. If
  ${\theta_y\theta_x^{-1}(x)=y}$, then $x$ and $y$ are on the same orbit, so there is indeed a map ${\psi\in\group}$
  with ${\psi x=y}$. The coherence property then follows by the same argument as for metric manifolds above.
  If the stabilizers are non-trivial, the same holds if points are substituted by their orbits under stabilizers.
\end{proof}
\begin{example} Consider again the triangle $\Pi$ and rotation $\phi$ in \cref{fig:cone}.
Here, the stabilizer of the center of rotation $x$ is ${\Stab(x)=\braces{\Id,\phi,\phi^2}}$.
The metric ball around the point ${i=\group(x)}$ on the orbifold (the tip of the cone) is a smaller cone:
  \begin{center}
    \begin{tikzpicture}[>=stealth]
      \begin{scope}[xshift=0cm,yshift=-.75cm]
        \fill[
          left color=gray!50!white,
          right color=gray!50!white,
          middle color=white,
          shading=axis,
          opacity=0.25
        ]
        (1,0) -- (0,3) -- (-1,0) arc (180:360:1cm and 0.4cm);
        \draw (-1,0) arc (180:360:1cm and 0.4cm) -- (0,3) -- cycle;
        \fill[
          left color=gray!50!black,
          right color=gray!50!black,
          middle color=gray!10!white,
          shading=axis,
          opacity=0.25
        ]
        (.34,2) -- (0,3) -- (-.34,2) arc (180:360:.34cm and 0.15cm);

        \draw[dotted] (-1,0) arc (180:0:1cm and 0.4cm);
        \node[circle,black,scale=.3,fill,draw] (tip) at (0,3) {};
        \node at ($(tip)+(0,.3)$) {\scriptsize $i=\group(x)$};
      \end{scope}
      \begin{scope}[xshift=5cm]
        \node (origin) at (0,0) {};
        \node (a) at ($(origin)+(90:2)$) {};
        \node (b) at ($(origin)+(210:2)$) {};
        \node (c) at ($(origin)+(330:2)$) {};
        \draw[thick,fill=gray!20!white] (a.center)--(c.center)--(origin.center)--cycle;
        \draw[domain=90:-30,fill,gray!60!white] (origin) -- plot ({.5*cos(\x)}, {.5*sin(\x)}) -- (origin) --cycle;
        \draw[thick] (a.center)--(c.center)--(origin.center)--cycle;
        \node[circle,draw,fill,scale=.3,black] at (origin) {};
        \node[circle,draw,fill,scale=.3,black] at (a) {};
        \node[circle,draw,fill,scale=.3,black] at (c) {};
        \node at ($(origin)+(0,-.3)$) {\scriptsize $x$};
      \end{scope}
      \draw[->] ($(tip)+(.4,0)$) to[bend left] ($(origin)+(-.3,.3)$);
      \node at ($(tip)+(3.2,-.1)$) {\scriptsize $\theta_i$};
      \begin{scope}[xshift=11cm]
        \node (origin) at (0,0) {};
        \node (a) at ($(origin)+(90:2)$) {};
        \node (b) at ($(origin)+(210:2)$) {};
        \node (c) at ($(origin)+(330:2)$) {};
        \draw[thick,fill=gray!20!white] (a.center)--(b.center)--(origin.center)--cycle;
        \draw[thick,fill=gray!20!white] (c.center)--(b.center)--(origin.center)--cycle;
        \draw[thick,fill=gray!20!white] (a.center)--(c.center)--(origin.center)--cycle;
        \draw[domain=0:360,fill,gray!60!white] (origin) -- plot ({.5*cos(\x)}, {.5*sin(\x)}) -- (origin) --cycle;
        \draw[thick] (a.center)--(b.center)--(origin.center)--cycle;
        \draw[thick] (c.center)--(b.center)--(origin.center)--cycle;
        \draw[thick] (a.center)--(c.center)--(origin.center)--cycle;
        \node[circle,draw,fill,scale=.3,black] at (origin) {};
        \node[circle,draw,fill,scale=.3,black] at (a) {};
        \node[circle,draw,fill,scale=.3,black] at (b) {};
        \node[circle,draw,fill,scale=.3,black] at (c) {};
        \node at ($(origin)+(0,-.3)$) {\scriptsize $x$};
      \end{scope}
      \draw[<->] (7,.5)--(9,.5) node[pos=.5,label={above:{\scriptsize $\begin{matrix}\text{isomorphic in}\\
              \mathbb{R}^n/\Stab(x)\end{matrix}$}}] {};

    \end{tikzpicture}
  \end{center}
Its image under $\theta_i$ can be identified with the intersection of $\Pi$ with a Euclidean ball around $x$.
Since $\Pi$ and its image $\Stab(x)\Pi$ under the stabilizer---the equilateral triangle on the right---are
indistinguishable in $\mathbb{R}^n/\Stab(x)$, that corresponds to the quotient of a metric ball in the plane.
\end{example}

\subsection{Path metrics}

An orbifold as defined above is a topological space.
To work with the gluing results stated below, we must know it is also a metric space, and that this space is complete.
\cref{fact:orbifold:completeness} shows that that is true. Before we state the fact, we briefly describe how to
construct the relevant metric, which is the standard metric on orbifolds. Our definition is again adapted from that
of \citet{Ratcliffe:2006}. \citet{Bonahon:2009} offers an accessible introduction to this type of metric.

Intuitively, the
metric generalizes the geodesic on a smooth surface, by measuring the length of the shortest curve between two points.
Formally, a \kword{curve} connecting two points~${\omega_1}$ and~${\omega_2}$ in $M$ is a continuous function
\begin{equation*}
  \gamma:[a,b]\subset\mathbb{R}\rightarrow X
  \qquad\text{ such that }\qquad\gamma(a)=\omega_1\text{ and }\gamma(b)=\omega_2\;.
\end{equation*}
To define the \kword{length} $\|\gamma\|$ of $\gamma$, first suppose $\omega_1$ and $\omega_2$ are in the
same set $U_i$, and define
\begin{equation*}
  \|\gamma\|\;:=\;\sup\braces{\,\tsum_{j\leq k}d_{F_i}(\theta_i\circ\gamma(t_{j-1}),\theta_i\circ\gamma(t_{j})))\,|\,a=t_0<t_1<\ldots<t_k=b\text{ for }k\in\mathbb{N}}\;,
\end{equation*}
that is, the supremum is taken over the sequences ${(t_0,\ldots,t_k)}$.
In words: For each ${t_j\in[a,b]}$, the point $\gamma(t_j)$ lies on the curve $\gamma$ in $M$. By choosing
a sequence ${t_0,\ldots,t_k}$ as above, we approximate the curve by $k$ line segments ${(\gamma(t_{j-1}),\gamma(t_{j}))}$,
and then approximate the length of $\gamma$ by summing the lengths of these segments. Since each line segment lies
in $M$, and we have no tool to measure distance in $M$, we map each point $\gamma(t_j)$ on the curve
to a point $\theta_i(\gamma(t_j))$ in $\mathbb{R}^n/F_i$, where we know how to measure distance using $d_{F_i}$.
We then record the length of the piece-wise approximation as the sum of lengths of the segments.
The length ${\|\gamma\|}$ is the supremum over the lengths of all such approximations.

If there is no set $U_i$ containing both points, one can always subdivide $[a,b]$ into finitely many segments
$[t_{j-1},t_j]$ such that every pair ${\gamma(t_{j-1})}$ and ${\gamma(t_{j})}$ of consecutive points is in
in some set $U_i$ (see \citep[][]{Ratcliffe:2006}). One then defines
\begin{equation*}
  \|\gamma\|\;:=\;\msum_{i\leq k}\|\gamma|_{[t_{i-1},t_{i}]}\|\;,
\end{equation*}
and it can be shown that $\|\gamma\|$ does not depend on the choice of subdivision.
\begin{fact}[\citet{Ratcliffe:2006} Lemma 1 of \S13.2, Theorems 13.2.7 and 13.3.8]
  \label{fact:path:distance}
  \label{fact:orbifold:completeness}
  If $M$ is an $\mathbb{F}$-orbifold, any two points in $M$ can be connected by a curve of finite length.
  The function
  \begin{equation*}
    \dpath(\omega_1,\omega_1)\;:=\;\inf\braces{\|\gamma\|\,|\,\gamma \text{ is a curve connecting $\omega_1$ and $\omega_2$ in  $M$ }}
  \end{equation*}
  is a metric on the set $M$, and metrizes the Hausdorff topology of $M$.
  The metric space so defined is complete.
\end{fact}

\subsection{Orbifolds constructed by abstract gluing}

Let ${S_1,\ldots,S_k}$ be the facets of $\Pi$. A \kword{side pairing} is a finite set
${\mathcal{S}=\braces{\psi_1,\ldots,\psi_k}}$ of isometries of $\mathbb{R}^n$ if, for each
${i\leq k}$, there is a ${j\leq k}$ such that
\begin{equation*}
  \text{(i) }\;
  \psi_i(S_j)=S_i
  \qquad
  \text{(ii) }\;
  \psi_i\;=\;\psi_j^{-1}
  \qquad
  \text{(iii) }\;
  \Pi\cap\psi_i\Pi\;=\;S_i\;.
\end{equation*}
The definition permits ${i=j}$. A crystallographic group is determined by a side pairing:
\begin{fact}[\citet{Bonahon:2009} Theorem 7.11]
  \label{fact:generator}
  If a crystallographic group $\group$ tiles with a convex polytope~$\Pi$, the tiling is exact,
  and $\mathcal{S}$ is a side pairing for $\Pi$ and $\group$, the group
  generated by $\mathcal{S}$ is $\group$.
\end{fact}
The side pairing defines an equivalence relation
$\equiv$ on points ${x,y\in\Pi}$, namely
\begin{equation*}
  x\;\equiv\;y
  \qquad:\Longleftrightarrow\qquad
  \psi_i x=y\quad\text{ for some }i\leq k\;.
\end{equation*}
Let $M$ be the quotient space ${M:=\Pi/\equiv}$, equipped with the quotient topology, that is,
\begin{equation*}
  A\subset M\text{ open }
  \quad:\Leftrightarrow\quad
  \braces{x\in\Pi|\text{ equivalence class of }x\text{ is in }A}\text{ is open set in }\Pi
\end{equation*}
We then refer to $M$ as the quotient \kword{obtained by abstract gluing} from
$\Pi$ and $\mathcal{S}$.
We will be interested in a specific type of side pairing, called a \kword{subproper} side pairing.
The precise definition is somewhat involved, and can be found in \S13.4 of \citet{Ratcliffe:2006}.
We omit it here, since we will see that all side pairings
relevant to our purposes are subproper.\nolinebreak
\begin{fact}[\citet{Ratcliffe:2006} Theorem 13.4.2]
  \label{fact:abstract:gluing}
  Let $\mathbb{F}$ be a group of isometries of $\mathbb{R}^n$ and $\Pi$ a convex polytope.
  Let $M$ be the metric space obtained by abstract gluing from $\Pi$ and a subproper $\mathbb{F}$-side pairing.
  Then $M$ is an $\mathbb{F}$-orbifold. The natural inclusion
  ${\Pi^{\circ}\hookrightarrow M}$, i.e., the map that takes each point ${x\in\Pi^{\circ}}$ to its $\equiv$-equivalence
  class, is continuous.
\end{fact}
For the next result, recall the definition of $d_\group$ from \cref{fact:metric}. We define a metric $d_{\mathbb{S}}$ for a group $\mathbb{S}$
analogously, by substituting $\mathbb{S}$ for $\group$.
\begin{fact}[\citet{Ratcliffe:2006} Theorem 13.5.3]
  \label{fact:poincare}
  Let $M$ be the orbifold in \cref{fact:abstract:gluing}, and $\mathbb{S}$ be the group generated by all maps in the
  side pairing. If $M$ is a complete metric space,
  the natural inclusion map ${\Pi\hookrightarrow M}$ induces an isometry from $M$ to
  ${(\mathbb{R}^n/\mathbb{S},d_\mathbb{S})}$.
\end{fact}
The final result on orbifolds we need gives a precise statement of the idea that the set
of points around which an orbifold does not resemble a manifold
is small. The next definition characterizes those points around which the manifold
property breaks down as having order $>1$:
Consider a point ${z\in M}$. Then we can find some ${x\in\mathbb{R}^n}$ that corresponds to $z$: We know that
${z\in U_i}$ for some $i$, and hence ${\phi_iz= F_i x}$ in the quotient space ${\mathbb{R}^n/F_i}$.
The \kword{order} of ${z\in M}$ is the number of elements of $F_i$ that leave $x$ invariant (formally, the order
of the stabilizer of $x$ in $F_i$). It can be shown that this number does not depend on the choice of $i$, so each
${z\in M}$ has a uniquely defined order.
\begin{fact}[\citet{Ratcliffe:2006} Theorem 13.2.4]
  \label{fact:problematic:points}
  If $M$ is an $\mathbb{F}$-orbifold, the set of points of order $1$ in $M$
  is an open dense subset of $M$. The set of points of order $>1$ is nowhere dense.
\end{fact}

\subsection{Topological dimension}
\label{sec:dimensions}

The notion of dimension we have used throughout is the algebraic dimension $\dim A$ of a set $A$ in a vector space
(see \cref{sec:definitions}). For the proof of the embedding theorem, we also need another notion of dimension that
does not require vector space structure, known
variously as topological dimension, covering dimension, or Lebesgue dimension. The definition is slightly more involved:
Consider a topological space $X$. An \kword{open cover} of $X$ is a collection $\mathcal{A}$ of open sets in $X$ that
cover $X$, that is, each point of $X$ is in at least one of the sets. The \kword{order} of an open cover is
\begin{equation*}
  \text{order}(\mathcal{A})
  \;:=\;
  \sup\braces{\text{ number of elements of $\mathcal{A}$ containing }x\,|\,x\in X}\;.
\end{equation*}
The \kword{topological dimension} $\Dim X$ of $X$ is the smallest value ${m\in\;\mathbb{N}\cup\braces{\infty}}$
such that, for every open covering $\mathcal{B}$ of $X$, there is an open covering $\mathcal{A}$
with ${\text{order}(\mathcal{A})=m+1}$ such that every set of $\mathcal{B}$ contains a set of $\mathcal{A}$.
\begin{fact}[{\citep[][3.2.7]{Pears:1975}}]
  The topological dimension of Euclidean space equals its algebraic dimension, ${\Dim\mathbb{R}^n=\dim\mathbb{R}^n=n}$,
  and any closed metric balls ${B\subset\mathbb{R}^n}$ has ${\Dim B=n}$.
\end{fact}
In general, however, the topological dimension of a set ${A\subset\mathbb{R}^n}$ may differ from its dimension
$\dim A$ as defined in \cref{sec:definitions} (as the algebraic dimension of the linear hull), and even the proof that
${\Dim\mathbb{R}^n=n}$ is not entirely trivial. \citet{Munkres:2000} provides a readable overview.
The reason why topological dimension is of interest in our context is the following classical result.
Recall that, given topological spaces $X$ and $Y$, an \kword{embedding} of $X$ into~$Y$ is an injective
map ${X\rightarrow Y}$ that is a homeomorphism of $X$ and its image.
\begin{fact}[{\citep[][50.5]{Munkres:2000}}]
  Every compact metrizable space $X$ with ${\Dim X<\infty}$ can be embedded into $\mathbb{R}^{2\Dim X+1}$.
\end{fact}
We also collect two additional facts for use in the proofs. Recall that a function is called closed if
the image of every closed set is closed.
\begin{fact}[{\citep[][50.2]{Munkres:2000}} and {\citep[][9.2.10]{Pears:1975}}]
  \label{fact:dimensions}
  (i) If $X$ is a topological space and ${Y_1,\ldots,Y_k}$ are closed and finite-dimensional subspaces, then
  \begin{equation*}
    X\;=\;Y_1\cup\ldots\cup Y_k
    \qquad\text{ implies }\qquad
    \Dim X =\max_i\Dim Y_i\;.
  \end{equation*}
  (ii)
  Let ${f:X\rightarrow Y}$ be a continuous, closed and surjective map between metric spaces.
  If ${|f^{-1}y|\leq m+1}$ for some ${m\in\mathbb{N}\cup\braces{0}}$ and all ${y\in Y}$, then
  \begin{equation*}
    \Dim X
    \;\leq\;
    \Dim Y
    \;\leq\;
    \Dim X + m\;.
  \end{equation*}
\end{fact}

\section{Background IV: Spectral theory}
\label{sec:spectral}

The proof of the Fourier representation draws on the spectral theory of linear operators,
and we now review the relevant facts of this theory.
We are interested in an operator~$A$ (think $-\Delta$) defined on a space
$V$ (think $\H_\group$) which is contained in a space $W$ (think~$\L_2$).
If $V$ approximates $\L_2$ sufficiently well, and if $A$ is self-adjoint on $V$,
a general spectral result guarantees the existence of an orthonormal basis for $\L_2$ consisting of eigenfunctions
(\cref{spectral:lemma}).
To apply the result to the negative Laplacian, we must extend $\Delta$ to an operator on $\H_\group$
(since $\Delta$ is defined on twice differentiable functions, and elements of $\H_\group$ need not
be that smooth). \cref{fact:extended:laplacian} shows that is possible.
Once we have obtained the eigenfunctions, there is a generic way to show they are smooth
(\cref{fact:regularity}).

\subsection{Spectra of self-adjoint operators}
Spectral decompositions of self-adjoint operators have been studied widely,
see \citet{Brezis:2011,McLean:2000} for sample results.
We use the following formulation, adapted from Theorem 2.37 and Corollary 2.38 of \citet{McLean:2000}.
\begin{fact}[Spectral decomposition \citep{McLean:2000}]
  \label{spectral:lemma}
  Let $\cF$ be a polytope, and $V$ a closed subspace of~$\H^1(\intF)$. Require that the inclusion maps
  \begin{equation}
    \label{pivot:triple}
    V\;\hookrightarrow\;\L_2(\intF)\;\hookrightarrow\; V^*
  \end{equation}
  are both continuous and dense,
  and the first inclusion is also compact.
  Let ${A:V\rightarrow V^*}$ be a bounded linear operator that is self-adjoint on $V$ and satisfies
  \begin{equation}
    \label{coercivity:spectral:lemma}
    \sp{A f,f}_{V}\;\geq\;c_V\|f\|^2_{V}-c_L\|f\|^2_{\L_2}
    \qquad\text{ for some }c_V,c_{L}>0\text{ and all }f\in V\;.
  \end{equation}
  Then there is a countable number of scalars
  \begin{equation*}
    \lambda_1\leq\lambda_2\leq\ldots
    \qquad\text{ with }\qquad
    \lambda_i\;\xrightarrow{i\rightarrow\infty}\;\infty
  \end{equation*}
  and functions ${\xi_1,\xi_2,\ldots\in V}$ such that
  \begin{equation*}
    A\xi_i\;=\;\lambda\xi_i \qquad\text{ for all }i\in\mathbb{N}\;.
  \end{equation*}
  The functions $\xi_i$ form an orthonormal basis for $V$. For each ${v\in V}$,
  \begin{equation*}
    \msum_{i\leq m}\lambda_i\sp{v,\xi_i}\xi_i\;\xrightarrow{m\rightarrow\infty}\;A v
   \end{equation*}
  holds in the dual $V^*$. If $A$ is also strictly positive definite,
  then ${\lambda_1>0}$.
\end{fact}
If the inclusions in \eqref{spectral:lemma} are continuous and dense,
$\L_2$ is called a \kword{pivot space} for $V$.
See Remark 3 in Chapter 5 of \citet{Brezis:2011} for a discussion of pivot spaces.

\subsection{Extension of Laplacians to Sobolev spaces}

Recall that the Laplace operator $\Delta$ on a domain $\Gamma$ is defined on twice continuously
differentiable functions. It
can be extended to a continuous linear operator on $\H^1(\Gamma^\circ)$, provided
the geometry of $\Gamma$ is sufficiently regular. That is the case if $\Gamma$
is a \kword{Lipschitz domain}, which loosely speaking means
it is bounded by a finite number of Lipschitz-smooth surfaces. Since a precise definition (which can be found in
\citet{McLean:2000}) is rather technical, we omit details and only note that every polytope is a Lipschitz domain
\citep[][p 90]{McLean:2000}.
\begin{fact}
  \label{fact:extended:laplacian}
  Let $\Gamma$ be a Lipschitz domain, and denote by ${\H^1(\Gamma^{\circ})^*}$ the dual space of
  ${\H^1(\Gamma^{\circ})}$. There is a unique linear operator
  ${\Lambda:\H^1(\Gamma^{\circ})\rightarrow\H^1(\Gamma^{\circ})^*}$ that extends the Laplace
  operator. This operator is bounded on $\H^1(\Gamma^{\circ})$.
\end{fact}

\subsection{Smoothness of eigenfunctions}
\label{sec:eigenfunction:smoothness}

One hallmark of differential operators is that their eigenfunctions tend to be
very smooth. The sines and cosines that make up the standard Fourier basis on
the line are examples. Intuitively, that is due to the fact that
the Laplacian is a second-order differential operator, and ``removes two orders of
smoothness'': If $\Delta f$ is in $\C^k$, then $f$ must be in
$\C^{k+2}$. Since an eigenfunction appears on both sides of the spectral equation
\begin{equation*}
  -\Delta \xi\;=\;\lambda \xi\;,
\end{equation*}
one can iterate the argument: If $\xi$ is in $\C$, it must also be in $\C^2$, hence also in $\C^4$, and so forth.
This argument is not immediately applicable to the functions $\xi$ constructed in
\cref{spectral:lemma} above, since it does not guarantee the functions to be in $\C^2$. It
only shows they are in $V$, which in the context of differential operators
(and specifically in the problems we study) is typically a Sobolev space.
Under suitable conditions on the domain, however, one can show that argument above
generalizes to Sobolev space, at least on certain open subsets. The following version
is again adapted to our problem from a more general result.\nolinebreak
\begin{fact}[{\citep[][4.16]{McLean:2000}}]
  \label{fact:regularity}
  Let $\Pi$ be a polytope and $M$ an open set such that ${\overline{M}\subset\Pi^{\circ}}$. Let
  $\Lambda$ be the extension of the Laplace operator guaranteed by \cref{fact:extended:laplacian}.
  Suppose ${f\in\H^1(M)}$ and~${k\in\mathbb{N}\cup\braces{0}}$.
  Then ${\Lambda f=g}$ on $M$ for ${g\in\H^{k}(M)}$ implies ${f\in\H^{k+2}(M)}$.
\end{fact}

\section{Proofs I: The flux property}
\label{appendix:proofs:flux}

This and the following two sections comprise the proof of \cref{theorem:spectral}, the Fourier representation.
In this section, we prove the flux property of \cref{lemma:gradient:field}.

\subsection{Tools: Subfacets}

We next introduce a simple geometric tool to deal with non-exact tilings:
\cref{theorem:embedding} assumes the tiling is exact, but the the flux property and the Fourier
representation make no such assumption. Although they do not use a gluing construction explicitly, they use the periodic
boundary condition \eqref{pbc}, which matches up points on the boundary
$\partial\Pi$ as gluing does. Absent exactness, that requires dealing
with parts of facets. We call each set of the form
\begin{equation*}
  \sigma\;:=\;(\Pi\cap\phi\Pi)^{\circ}\qquad\text{ for some }\phi\in\group
  \text{ and }\sigma\neq\varnothing
\end{equation*}
a \kword{subfacet} of $\Pi$. Let $\Sigma$ be the (finite) set of subfacets of $\Pi$.
Whereas the division of $\partial\Pi$ into facets is a property
of the polytope that does not depend on $\group$, the subfacets are a property of the tiling.
\begin{example}
Consider an edge of a rectangle ${\Pi\subset\mathbb{R}^2}$.
Suppose $\phi$ is a $180^{\circ}$ rotation around the center $x$ of the edge, as shown on the left:
\begin{center}
  \begin{tikzpicture}[mybraces,>=stealth]
    \begin{scope}
      \node[circle,fill,scale=.4] (a) at (0,0) {};
      \node[circle,fill,scale=.3] (b) at (3,0) {};
      \node[circle,fill,scale=.4] (c) at (6,0) {};
      \draw (a)--(c);
      \draw[-latex,gray] ($(b)+(-1.5,.3)$) arc [start angle=140,end angle=40,x radius=2cm,y radius=2cm];
      \draw[-latex,gray] ($(b)+(1.5,-.3)$) arc [start angle=-40,end angle=-140,x radius=2cm,y radius=2cm];
      \node at ($(b)+(0,-.3)$) {\scriptsize $x$};
    \end{scope}
    \begin{scope}[xshift=7.5cm]
      \node[circle,fill,scale=.4] (a) at (0,0) {};
      \node[circle,fill,scale=.3] (b) at (3,0) {};
      \node[circle,fill,scale=.4] (c) at (6,0) {};
      \draw (a)--(c);
      \draw[brace,gray] ($(a)+(0,.5)$)--($(c)+(0,.5)$);
      \draw[mirrorbrace,gray] ($(a)+(0,-.5)$)--($(b)+(-.05,-.5)$);
      \draw[mirrorbrace,gray] ($(b)+(0,-.5)$)--($(c)+(.05,-.5)$);
      \node at ($(b)+(0,-.3)$) {\scriptsize $x$};
      \node[gray] at ($(b)+(0,.8)$) {\scriptsize facet};
      \node[gray] at ($(b)+(-1.5,-.8)$) {\scriptsize subfacet};
      \node[gray] at ($(b)+(+1.5,-.8)$) {\scriptsize subfacet};
    \end{scope}
  \end{tikzpicture}
\end{center}
Then $\phi$ maps the facet to itself, and maps the point $x$ to itself, but no other point is fixed. In this
case, $x$ divides the interior of the facet into two subfacets (right). If $\phi$ is instead a reflection
about the same edge, each point on the edge is a fixed point, and the entire interior of the edge is a
single subfacet. Another example of a subfacet is the edge segment marked in \cref{fig:dirichlet:pg}/left.
\end{example}

\begin{lemma}
  \label{lemma:subfacets}
  The subfacets are convex $(n-1)$-dimensional open subsets of $\partial\Pi$, and their closures cover $\partial\Pi$.
  In particular,
  \begin{equation*}
    \vol_{n-1}(\sigma)>0\;\text{ for all }\sigma\in\Sigma
    \quad\text{ and }\quad
    \msum_{\sigma\in\Sigma}\vol_{n-1}(\sigma)=\vol_{n-1}(\partial\Pi)\;.
  \end{equation*}
  Each subfacet is mapped by $\group$ to exactly
  one subfacet, possibly itself: For each ${\sigma\in\Sigma}$,
  \begin{equation*}
    \phi_\sigma(\sigma)\;\in\;\Sigma\qquad\text{ for one and only one }\phi_{\sigma}\in\group\setminus\braces{\Id}\;,
  \end{equation*}
  where ${\phi_{\sigma}(\sigma)=\sigma}$ if and only if $\sigma$ contains a fixed point of $\phi_\sigma$.
\end{lemma}

\begin{proof}[Proof of \cref{lemma:subfacets}]
  Each subfacet is by definition of the form ${\sigma=(\Pi\cap\psi\Pi)^{\circ}}$, for some ${\psi\in\group}$.
  Since $\group(\Pi)$ is a tiling, $\Pi$ and ${\psi\Pi}$ are the only tiles intersecting $\sigma$.
  We hence have
  \begin{equation*}
    \sigma\cap\phi_{\sigma}^{-1}\Pi\;\neq\;\varnothing\qquad\text{ for one and only one }\phi_{\sigma}\in\group\setminus\braces{\Id}\;,
  \end{equation*}
  namely for ${\phi_{\sigma}=\psi^{-1}}$.
  Since the set ${\Pi\cap\phi_{\sigma}^{-1}\Pi}$ is the intersection of two facets, and hence of two convex sets, it is convex.
  By the definition of subfacets, its relative interior $\sigma$ is non-empty. That makes $\sigma$ a
  ${(n-1)}$-dimensional, convex, open subset of $\partial\Pi$.
  \\[.5em]
  \step {\em Volumes}.
  Since $\sigma$ is open in $n-1$ dimensions,
  \begin{equation*}
    \vol_{n-1}(\sigma)>0\;.
  \end{equation*}
  The definition of a tiling implies
  each boundary point ${x\in\partial\Pi}$ is on the facet of some adjacent tile $\phi\Pi$. It follows that
  \begin{equation*}
    \partial\Pi
    \;=\;
    \medcup\nolimits_{\phi\in\group\smallsetminus\braces{\Id}}\Pi\cap\phi\Pi
    \;=\;
    \medcup\nolimits_{\sigma\in\Sigma}\overline{\sigma}\;.
  \end{equation*}
  Since distinct subfacets do not intersect, applying volumes on both sides shows
  \begin{equation*}
    \vol_{n-1}(\partial\Pi)\;=\;\msum_{\sigma\in\Sigma}\vol_{n-1}(\sigma)\;.
  \end{equation*}
  \step {\em Each subfacet maps to exactly one subfacet}.
  We have already noted that $\sigma$ intersects only the tiles $\Pi$ and $\phi_{\sigma}^{-1}\Pi$.
  Since ${\phi_{\sigma}^{-1}\Pi}$ is adjacent to $\Pi$, so is ${\phi_{\sigma}\Pi}$.
  That implies ${\phi(\sigma)=(\Pi\cap\phi_\sigma\Pi)^{\circ}}$, and hence
  \begin{equation*}
    \phi_\sigma(\sigma)\in\Sigma
    \qquad\text{ and }\qquad
    \sigma\cap\phi^{-1}\Pi\;=\;\varnothing \quad\text{ if }\phi\neq\Id,\phi_{\sigma}\;.
  \end{equation*}
  Thus, $\sigma$ maps to $\phi_\sigma$ and vice versa, and neither maps to any other subfacet.
  \\[.5em]
  \step {\em Fixed points}.
  We know that $\sigma$ and $\phi_\sigma(\sigma)$ are either identical or disjoint.
  Suppose first that ${\sigma\neq\phi_\sigma(\sigma)}$. Then
  \begin{equation*}
    \phi_{\sigma}^{-1}\Pi\;\neq\;\phi_{\sigma}\Pi
    \quad\text{ and hence }\quad
    \phi_\sigma(\sigma)\cap\sigma=\varnothing\;,
  \end{equation*}
  so $\phi_\sigma$ has no fixed points in $\sigma$.
  On the other hand, suppose ${\phi_{\sigma}(\sigma)=\sigma}$.
  Then the restriction of $\phi_\sigma$ to the closure $\bar{\sigma}$ is a continuous map
  ${\overline{\sigma}\rightarrow\overline{\sigma}}$ from a compact convex set to itself.
  That implies, by Brouwer's theorem \citep[][17.56]{Aliprantis:Border:2006},
  that the closure $\overline{\sigma}$ contains at least one fixed point,
  and we only have to ensure that at least one of these fixed points is in the interior $\sigma$. But if
  the boundary $\partial\sigma$ contains fixed points and $\sigma$ does not, then ${\phi_{\sigma}(\sigma)\neq\sigma}$ since
  $\phi_\sigma$ is an isometry, which contradicts the assumption. In summary, we have shown that
  ${\phi_{\sigma}(\sigma)=\sigma}$ if and only if $\sigma$ contains a fixed point.
\end{proof}

\subsection{Proof of the flux property}

To establish the flux property in \cref{lemma:gradient:field}, we first show how
the normal vector $\n_{\Pi}$ of the boundary of a tile $\Pi$ transforms
under elements of the group $\group$.
\begin{lemma}[Transformation behavior of normal vectors]
  \label{lemma:normal:field}
  If a crystallographic group tiles~$\mathbb{R}^n$ with a convex polytope $\Pi$, then
  \begin{equation}
    \label{alternating:subfaces}
    A_\phi\n_{\cF}(x)\;=\;-\n_{\cF}(\phi x)\qquad\text{ whenever }x,\phi x\in\Pi\;.
  \end{equation}
\end{lemma}
\begin{proof}
  If $\phi\cF$ is a tile adjacent to $\cF$, its normal vector ${\n_{\phi\cF}}$ satisfies
  \begin{equation*}
    \n_{\cF}(y)\;=\;-\n_{\phi\cF}(y)\qquad\text{ if }y\in\cF\cap\phi\cF\;.
  \end{equation*}
  Since ${x\sim\phi x}$ holds, $x$ is on at least one facet $S$ of $\Pi$, and $\phi x$ is hence
  on the facet ${\phi S}$ of $\phi\Pi$. If $\n$ is a normal vector of $S$ (exterior to $\Pi)$,
  then ${A_\phi\n}$ is a normal vector of $\phi S$ (exterior to $\phi\Pi$). That shows
  \begin{equation*}
    \n_{\phi\cF}(\phi x)
    \;=\;
    A_\phi\n_{\cF}(x)
    \qquad\text{ if }x,\phi x\in\Pi\;.
  \end{equation*}
  In summary, we hence have ${A_\phi\n_{\cF}(x)\;=\;-\n_{\cF}(y)}$ whenever $x$ and $\phi x$ are both in $\Pi$.
\end{proof}

\begin{proof}[Proof of \cref{lemma:gradient:field}]
  \step
  Let $\sigma$ be a subfacet. Since $\n_{\cF}$ is constant on $\sigma$, we define the vectors
  \begin{equation*}
    \n_{\cF}(\sigma)\;:=\;\n_{\cF}(x) \text{ for any }x\in\sigma
    \quad\text{ and }\quad
    I(\sigma)\;:=\;\mint_{\sigma}F(x)\vol_{n-1}(dx)\;.
  \end{equation*}
  By \cref{lemma:subfacets}, the subfacets cover $\partial\Pi$ up to a null set.
  We hence have
  \begin{equation*}
    \mint_{\partial\cF}F(x)^{\trans}\n_{\cF}(x)\vol_{n-1}(dx)
    \;=\;
    \msum_{\sigma\in\mathcal{S}}\mint_{\sigma}F(x)^{\trans}\n_{\cF}(x)\vol_{n-1}(dx)
    \;=\;
    \msum_{\sigma\in\mathcal{S}}\n(\sigma)^{\trans}I(\sigma)\;.
  \end{equation*}
  We must show this sum vanishes.
  \\[.2em]
  \step
  If ${\phi\sigma\in\cF}$, the $\phi$-invariance of $\vol_{n-1}$ and condition \eqref{periodic:condition:field} imply
  \begin{align*}
    I(\phi\sigma)
    \;=\;
    \mint_{\phi\sigma}F(x)\vol_{n-1}(dx)
    \;&=\;
    \mint_{\sigma}F(\phi(x))\vol_{n-1}(dx)\\
    \;&=\;
    A_{\phi}\mint_{\sigma}F(x)\vol_{n-1}(dx)
    \;=\;
    A_{\phi}I(\sigma)\;.
    \end{align*}
  \cref{lemma:normal:field} implies ${A_\phi\n_{\cF}(\sigma)=-\n_{\cF}(\phi\sigma)}$ for ${\phi\sigma\subset\cF\cap\phi\cF}$.
  That shows
  \begin{equation*}
    \n_{\cF}(\phi\sigma)^{\trans}I(\phi\sigma)
    \;=\;
    -(A_{\phi}\n_{\cF}(\sigma))^{\trans}A_{\phi}I(\sigma)
    \;=\;
    -\n_{\cF}(\sigma)^{\trans}A_{\phi}^{\trans}A_{\phi}I(\sigma)
    \;=\;
    -\n_{\cF}(\sigma)^{\trans}I(\sigma)\;,
  \end{equation*}
  since ${A_\phi^{\trans}=A_{\phi}^{-1}}$. It follows that
  \begin{equation*}
    \n_{\cF}(\sigma)^{\trans}I(\sigma)
    \,+\,
    \n_{\cF}(\phi\sigma)^{\trans}I(\phi\sigma)
    \,=\,0
    \quad\text{ and even }\quad
    \n_{\cF}(\sigma)^{\trans}I(\sigma)\,=\,0 \;\text{ if }\sigma=\phi\sigma\;.
  \end{equation*}
  \step
  By \cref{lemma:subfacets}, the set $\Sigma$ of subfacets can be sorted into pairs $(\sigma,\phi_\sigma\sigma)$ such that no
  subfacet occurs in more than one pair (though ${\sigma=\phi_\sigma\sigma}$ is possible).
  It follows that
  \begin{equation*}
    \msum_{\sigma\in\Sigma}\n_{\cF}(\sigma)^{\trans} I(\sigma)
    \;=\;
    \mfrac{1}{2}\msum_{\sigma\in\Sigma}\bigl(\n_{\cF}(\sigma)^{\trans} I(\sigma)+\n_{\cF}(\phi_\sigma\sigma)^{\trans} I(\phi_\sigma\sigma)\bigr)
    \;=\;0
  \end{equation*}
  as we set out to show.
\end{proof}

\section{Proofs II: The Laplacian and its properties}
\label{appendix:proofs:laplacian}

The purpose of this section is to prove \cref{theorem:laplacian}.
We use the flux property
to show that the symmetries imposed by a crystallographic group simplify
the Green identity considerably. We can then use this symmetric form of the Green identity
to show $\Lambda$ has the desired properties.

\subsection{Green's identity under crystallographic symmetry}

That the extended Laplace operator is self-adjoint on $\H_\group$ for any crystallographic group
derives from the fact that the symmetry imposed by the group makes the correction
term in Green's identity vanish. That enters in the proof of \cref{theorem:laplacian}
via the two identities in the following lemma. The first one is Green's identity under symmetry;
the second shows that a similar identity holds for the Sobolev inner product.
\begin{lemma}[Symmetric Green identities]
  \label{lemma:symmetric:green:identities}
  If a crystallographic group $\group$ tiles $\mathbb{R}^n$ with a convex polytope $\Pi$, the
  negative Laplace operator satisfies the identities
  \begin{align}
    \label{L2:energy}
    \sp{-\Delta f,h}_{\L_2}
    \;&=\;
    a(f,h)
    \\[.2em]
    \label{H1:energy}
    \sp{-\Delta f,h}_{\H^1}
    \;&=\;
    a(f,h)
    \,+\,
    \tsum_{i\leq n}a(\partial_i f,\partial_i h)
  \end{align}
  for all functions $f$ and $h$ in $\mathcal{H}$.
\end{lemma}
\begin{proof}
  Let $\bar{f}$ and $\bar{h}$ be the unique continuous extensions of $f$ and $h$ to the closure $\Pi$,
  and set ${F:=\bar{f}\nabla\bar{h}}$. Since $\bar{f}$ and $\bar{h}$ satisfy the periodic boundary condition,
  \eqref{eq:equivariant:gradient} shows
  \begin{equation*}
  F(\phi x)
  \;=\;
  \bar{h}(\phi x)\nabla \bar{f}(\phi x)
  \;=\;
  \bar{h}(x)A_\phi\nabla \bar{f}(x)
  \;=\;
  A_\phi F(x)\;.
  \end{equation*}
  By the flux property (\cref{lemma:gradient:field}), we hence have
  \begin{equation*}
    \mint_{\partial\Pi}(\partial_{\n_{\Pi}}\bar{f})\bar{h}
    \;=\;
    \mint_{\partial\Pi}\n_{\Pi}^{\trans}F
    \;=\;
    0\;,
  \end{equation*}
  and substituting into Green's identity (\cref{green:identity}) shows
  \begin{equation}
    \label{eq:aux:green}
    \sp{\Delta f,h}_{\L_2}
    \;=\;
    a(f,h)
    -
    \mint_{\partial\Pi}(\partial_{\n_{\Pi}}f)h
    \;=\;
    a(f,h)\;,
  \end{equation}
  so \eqref{L2:energy} holds. Now consider \eqref{H1:energy}.
  Since $f$ has three continuous derivatives, we have
  \begin{equation*}
    \partial_i^2\partial_j\,f
    \;=\;
    \partial_j
    \partial_i^2\,f
    \quad\text{ and hence }\quad
    \Delta(\partial_jf)
    \;=\;
    \partial_j(\Delta f)\;.
  \end{equation*}
  The $\H^1$-product can then be written as
  \begin{equation*}
    \sp{\Delta f,h}_{\H^1}
    \;=\;
    \sp{\Delta f,h}_{\L_2}+
    \msum_{i}
    \sp{\partial_i(\Delta f),\partial_ih}_{\L_2}
    \;=\;
    \sp{\Delta f,h}_{\L_2}+
    \msum_{i}
    \sp{\Delta(\partial_if),\partial_ih}_{\L_2}\;.
  \end{equation*}
  Substituting the final sum into Green's identity shows
  \begin{equation*}
    \msum_i\sp{\Delta(\partial_if),\partial_ih}_{\L_2}
    \;=\;
    \msum_ia(\partial_if,\partial_ih)
    \,+\,
    \int_{\partial\Pi}\msum_i(\partial_{\n_{\Pi}}\partial_i f)\partial_ih
  \end{equation*}
  Since ${\nabla(\partial_if)}$ is precisely the $i$th row vector of the Hessian ${Hf}$, the integrand is
  \begin{equation*}
    \msum_i(\partial_{\n_{\Pi}}\partial_i f)\partial_ih
    \;=\;
    \msum_i(\n_{\Pi}^{\trans}(\nabla \partial_i f)\partial_ih
    \;=\;
    \n_{\Pi}^{\trans}(Hf)\nabla h\;.
  \end{equation*}
  Consider the vector field ${F(x):=(Hf)\nabla h}$. By \cref{lemma:differentials}, $F$ transforms as
  \begin{equation*}
    F(\phi x)
    \;=\;
    Hf(\phi x)\nabla h(\phi x)
    \;=\;
    A_\phi\cdot Hf(x)\cdot A_\phi^{\trans}A_\phi\nabla h(x)
    \;=\;
    A_{\phi}\cdot F(x)\;,
  \end{equation*}
  and hence satisfies \eqref{periodic:condition:field}.
  Another application of the flux property then shows
  \begin{equation*}
    \msum_i\sp{\Delta(\partial_if),\partial_ih}_{\L_2}
    \;=\;
    \msum_ia(\partial_if,\partial_ih)\;.
  \end{equation*}
  Substituting this identity and \eqref{L2:energy} into the $\H^1$-product above yields \eqref{H1:energy}.
\end{proof}

\subsection{Approximation properties of the space $\H_\group$}
\label{sec:pivot:lemma}

That we can use the space $\H_\group$ to prove results about continuous and $\L_2$-functions
relies on the fact that such functions are sufficiently well approximated by elements of~$\H_\group$, and that~$\H_\group$ can in turn be approximated by useful dense subsets.
We collect these technical facts in the following lemma.
Consider the space of functions
\begin{equation*}
  \C_{\text{\rm pbc}}(\Pi^\circ)\,=\,\braces{f|_{\Pi^\circ}\,|\,f\in\C_\group}
\end{equation*}
which we equip with the supremum norm. These are precisely those
uniformly continuous functions on the interior $\Pi^\circ$ whose unique
continuous extension to $\Pi$ satisfies the periodic boundary conditions.
Note that we can then express the definition of $\mathcal{H}$ in \eqref{definition:space:H}
as
\begin{equation*}
\mathcal{H}\;=\;\C_{\text{\rm pbc}}(\Pi^\circ)\cap\C^{\infty}(\Pi^{\circ})\;.
\end{equation*}

\begin{lemma}
  \label{lemma:pivot}
  If $\group$ is crystallographic and tiles with $\Pi$, the inclusions
  \begin{equation*}
  \mathcal{H}
    \;\xhookrightarrow{\iota_1}\;
    \H_{\group}
    \;\xhookrightarrow{\iota_2}\;
    \L_2(\Pi^{\circ})
    \;\xhookrightarrow{\iota_3}\;
    \H_{\group}^*\;,
  \end{equation*}
  are all dense, $\iota_2$ and $\iota_3$ are continuous, and $\iota_2$ is compact. Moreover,
  if ${\mathcal{F}\subset\H_\group\cap\C_{\text{\rm pbc}}(\Pi^\circ)}$ is dense in~$\H_\group$,
  it is also dense in ${\C_{\text{\rm pbc}}(\Pi^\circ)}$ in the supremum norm.
\end{lemma}
When we take closures in the proof, we write
${\overline{\mathcal{F}}^{\,\sup}}$ and ${\overline{\mathcal{F}}^{\,\H^1}}$ to indicate the norm
used to take the closure of a set $\mathcal{F}$.

\begin{proof}
That $\mathcal{H}$ is dense in $\H_\group$ holds by definition, see \eqref{definition:space:H}.
\\[.5em]
\step {\em Inclusions $\iota_2$ and $\iota_3$ are dense and continuous}.
  Denote by ${\C_c^\infty:=\C_c(\intF)\cap\C^\infty(\intF)}$ the set of compactly supported and infinitely
  differentiable functions on $\intF$. Denote by
  \begin{equation*}
    \H_0^1\;:=\;\smash{\overline{\C_c^\infty}^{\H^1}}
  \end{equation*}
  its $\H^1$-closure. This is, loosely speaking, the Sobolev space of functions that vanish on the boundary
  \citep{Brezis:2011,McLean:2000}, and it is a standard result that
  \begin{equation*}
    \H_0^1
    \;\subset\;
    \L_2(\Pi^{\circ})
    \;\subset\;
    (\H_0^1)^*\;,
  \end{equation*}
  where both inclusion maps are dense and bounded \cite[][Chapter 9.5]{Brezis:2011}.
  Consider any~${f\in\C_c^\infty}$. Since $f$ is uniformly continuous, it has a unique continuous
  extension $\bar{f}$ to~$\Pi$. This extension satisfies $\bar{f}=0$ on the boundary $\partial\Pi$.
  (This fact is well known \citep[e.g.,][]{Brezis:2011}, but also easy to verify: Since the support of $f$
  is a closed subset of the open set $\Pi^{\circ}$, each point $x$ on the boundary is the center
  of some open ball $B$ that does not intersect the support, so $\bar{f}=0$ on $\intF\cap B$.)
  It therefore trivially satisfies the
  periodic boundary condition \eqref{pbc}, which shows ${\C_c^\infty\subset\mathcal{H}}$.
  Taking $\H_1$-closures shows ${\H^1_0\subset\H_\group}$. We hence have
  \begin{equation*}
    \H_0^1(\Pi^\circ)
    \;\subset\;
    \H_{\group}
    \;\subset\;
    \H^1(\Pi^{\circ})
    \;\subset\;
    \L_2(\Pi^{\circ})
    \;\subset\;
    \H^1(\Pi^{\circ})^*
    \;\subset\;
    \H_{\group}^*
    \;\subset\;
    \H_0^1(\Pi^\circ)^*\;.
  \end{equation*}
  Since ${\H_0^1\hookrightarrow\L_2}$ and ${\L_2\hookrightarrow(\H_0^1)^*}$ are both
  dense and bounded, ${\H_\group\hookrightarrow\L_2}$ and
  ${\L_2\hookrightarrow\H_\group^*}$ are dense and bounded (and hence continuous), and ${\H_\group\hookrightarrow\H^1}$
  is bounded (and hence continuous).\\[.5em]
  \step {\em Inclusion $\iota_2$ is compact}. We can decompose $\iota_2$ as
  \begin{equation*}
    \H_\group
    \;\hookrightarrow\;
    \H^1
    \;\hookrightarrow\;
    \L_2\;.
  \end{equation*}
  It is known that ${\H^1\hookrightarrow\L_2}$ is compact \citep{Adams:Fournier:2003}. If one of two inclusions is compact, their composition is compact (see \citep[][]{Adams:Fournier:2003}, or simply note that
  any bounded sequence in $\H_\group$ is also bounded in $\H^1$). That shows ${\H_\group\hookrightarrow\L_2}$ is compact.
  \\[.5em]
  \step {\em $\mathcal{F}$ is dense in ${\C_{\text{\rm pbc}}}$}.
  We know from \cref{lemma:kondrachov} that ${\H^1(\Pi^\circ)\subset\C(\Pi^\circ)}$,
  and hence ${\|h\|_{\H^1}\geq\|h\|_{\sup}}$ for all ${h\in\C(\Pi^\circ)}$.
  In other words, the sup-closure of the $\H^1$-closure is the sup-closure, so
  \begin{equation*}
    \overline{\mathcal{F}}^{\,\sup}
    \;=\;
    \overline{(\overline{\mathcal{F}}^{\H^1})}^{\,\sup}
    \;=\;
    \overline{\H_\group}^{\,\sup}
    \;=\;
    \overline{(\overline{\mathcal{H}}^{\H^1})}^{\,\sup}
    \;=\;
    \overline{\mathcal{H}}^{\,\sup}\;.
  \end{equation*}
  It hence suffices to show $\mathcal{H}$ is dense in $\C_{\text{\rm pbc}}$.
  To this end, we use a standard fact: If we consider the closed set $\Pi$ instead of the interior,
  ${\C^\infty(\Pi)}$ is dense in ${\C(\Pi)}$, since $\Pi$ is compact. (One way to see this is that $\C^\infty$ contains all
  polynomials, which are dense in $\C(\Pi)$ by the Stone-Weierstrass theorem \citep{Adams:Fournier:2003}.)
  Since ${\C_{\text{\rm pbc}}(\Pi)}$ is a closed linear subspace of $\C(\Pi)$, it follows that
  \begin{equation*}
    \C_{\text{\rm pbc}}(\Pi)\cap\C^\infty(\Pi)
    \quad\text{ is dense in }\quad
    \C_{\text{\rm pbc}}(\Pi)\cap\C(\Pi)\;=\;\C_{\text{\rm pbc}}(\Pi)\;.
  \end{equation*}
  Consider a function ${f\in\C_{\text{\rm pbc}}(\Pi^{\circ})}$. Then $f$ has a unique continuous extension
  ${\bar{f}}$ to $\Pi$, which satisfies the periodic boundary condition. That shows that
  \begin{equation}
    \label{eq:hbar:isometry}
    f\mapsto\bar{f}
    \quad\text{ is an isometric isomorphism }\quad
    \C_{\text{\rm pbc}}(\Pi^\circ)\rightarrow\C_{\text{\rm pbc}}(\Pi)\;,
  \end{equation}
  since the extension is unique and does not change the supremum norm. If $f$ is also infinitely differentiable
  (and hence in $\mathcal{H}$), then $\bar{f}$ is infinitely differentiable, so the same map is also an isometric
  isomorphism
  \begin{equation*}
    \mathcal{H}\;=\;\C_{\text{\rm pbc}}(\Pi^\circ)\cap\C^\infty(\Pi^\circ)
    \;\rightarrow\;
    \C_{\text{\rm pbc}}(\Pi)\cap\C^\infty(\Pi)\;.
  \end{equation*}
  In summary, we hence have
  \begin{equation*}
    \mathcal{H}
    \;\xrightarrow{\;\text{isomorphic}\;}\;
    \C_{\text{\rm pbc}}(\Pi)\cap\C^\infty(\Pi)
    \;\xhookrightarrow{\;\text{dense}\;}\;
    \C_{\text{\rm pbc}}(\Pi)
    \;\xrightarrow{\;\text{isomorphic}\;}\;
    \C_{\text{\rm pbc}}(\Pi^\circ)\;,
  \end{equation*}
  and since isomorphisms preserve dense subsets, $\mathcal{H}$ is indeed dense in ${\C_{\text{\rm pbc}}(\Pi^\circ)}$.
\end{proof}

\subsection{Existence and properties of the Laplacian}

\begin{proof}[Proof of \cref{theorem:laplacian}]
  Since $\Pi$ is a convex polytope, it is a Lipschitz domain, and ${\Delta}$ hence extends to
  a bounded linear operator $\Lambda$ on $\H^1(\Pi^{\circ})$, by \cref{fact:extended:laplacian}.
  The restriction of~$\Lambda$ to the closed linear subspace of $\H_\group$
  is again a bounded linear operator that extends~${\Delta}$.
  It remains to verify self-adjointness and \eqref{eq:theorem:laplacian} on $\H_\group$.
  Since $\Lambda$ is bounded and hence continuous,
  it suffices to do so on the dense subset $\mathcal{H}$.
  For (\ref{eq:theorem:laplacian}i), that has already been established in \cref{lemma:symmetric:green:identities}.
  To show (\ref{eq:theorem:laplacian}ii), we note \eqref{eq:H1:L2:a} implies
  \begin{equation*}
    \|f\|^2_{\H^1}
    \;=\;
    \sp{f,f}_{\H^1}
    \;=\;
    \sp{f,f}_{\L_2}\,+\,a(f,f)\qquad\text{ for }f\in\mathcal{H}
  \end{equation*}
  and hence
  \begin{equation*}
    a(f,f)\;=\;    \|f\|^2_{\H^1}\,-\,\|f\|^2_{\L_2}\;.
  \end{equation*}
  Since ${f\in\mathcal{H}}$ and hence ${\Lambda f=\Delta f}$, we can substitute into \eqref{H1:energy}, which shows
  \begin{equation*}
    \sp{-\Delta f,f}_{\H^1}
    \;=\;
    a(f,f)\,+\,\msum_{i\leq n}a(\partial_if,\partial_if)
    \;\geq\;
    \|f\|^2_{\H^1}\,-\,\|f\|^2_{\L_2}
  \end{equation*}
  where the last step uses the fact that $a$ is positive semi-definite by \eqref{energy:positive:semidefinite}.
  That proves coercivity. Since the bilinear form $a$ is symmetric, \eqref{H1:energy} also shows
  \begin{equation*}
    \sp{-\Delta f,h}_{\H^1}
    \;=\;
    a(f,h)
    \,+\,
    \msum_{i}
    a(\partial_i f,\partial_i h)
    \;=\;
    a(h,f)
    \,+\,
    \msum_{i}
    a(\partial_i h,\partial_i f)
    \;=\;
    \sp{-\Delta h,f}_{\H^1}
  \end{equation*}
  on $\mathcal{H}$, so $\Lambda$ is self-adjoint on $\H_\group$.
\end{proof}

\section{Proofs III: Fourier representation}
\label{appendix:proofs:spectral}

We now prove the Fourier representation. We first restrict all function to a single
tile~$\Pi$. By \cref{lemma:pivot},
we can then choose the space $V$ in the spectral theorem (\cref{spectral:lemma}) as $\H_\group$.
Since we also know the Laplacian is self-adjoint on $\H_{\group}$, we can use the
spectral theorem to obtain an eigenbasis. We then deduce \cref{theorem:spectral} by extending the representation
from~$\Pi$ to the entire space $\mathbb{R}^n$.

\subsection{Proof of the Fourier representation on a single tile}

The eigenvalue problem \eqref{sturm:liouville} in \cref{theorem:spectral} is defined on the
unbounded domain $\mathbb{R}^n$. We first restrict the problem to the compact domain $\Pi$, that is, we consider
\begin{equation}
  \label{sturm:liouville:restricted}
  \begin{aligned}
    -\Delta h&\;=\;\lambda h\quad&&\text{ on }\Pi^{\circ}\\
    \text{subject to }\qquad h(x)&\;=\;h(y)&&\text{ whenever }x\sim y\text{ on }\partial\Pi\;.
  \end{aligned}
\end{equation}
That allows us to apply \cref{spectral:lemma} and \ref{fact:regularity} above, which hold on compact domains.
(The deeper relevance of compact domains is that function spaces on such domains
tend to have better approximation properties than on unbounded domains.)
The restricted version of \cref{theorem:spectral} we prove first is as follows.
\begin{lemma}
  \label{result:spectral:restricted}
  Let $\group$ be a crystallographic group that tiles $\mathbb{R}^n$ with
  a convex polytope $\Pi$. Then \eqref{sturm:liouville:restricted}
  has solutions for countably many distinct values of $\lambda$, which  satisfy
  \begin{equation*}
    0\,=\,\lambda_1\,<\,\lambda_2\,<\,\lambda_3<\ldots
    \qquad\text{ and }\qquad
    \lambda_i\;\xrightarrow{i\rightarrow\infty}\;\infty\;.
  \end{equation*}
  Each solution $h$ is infinitely often differentiable on $\Pi^{\circ}$.
  There exists a sequence of solutions ${h_1,h_2,\ldots}$ that is an orthonormal
  basis of $\L_2(\Pi)$, and satisfies
  \begin{equation*}
    \big|\braces{\,j\in\mathbb{N}\,|\,h_j\text{ solves \eqref{sturm:liouville:restricted} for }\lambda_i}\big|\;=\;k(\lambda_i)\;.
  \end{equation*}
\end{lemma}
\noindent In the proof, we again use the notation $\smash{\overline{M}^{\L_2}}$ and $\smash{\overline{M}^{\H^1}}$ to indicate
the norm used to take the closure of a set $M$.
\begin{proof}[Proof of \cref{result:spectral:restricted}]
We apply the spectral decomposition result (\cref{spectral:lemma}), with ${A=\Lambda}$ and ${V=\H_\group}$.
We have already established its conditions are satisfied (except for
the optional assumption of strict positive definiteness):
By \cref{theorem:laplacian}, $\Lambda$ exists, is a bounded and
self-adjoint linear operator on $\H_\group$, and satisfies \eqref{coercivity:spectral:lemma}. By \cref{lemma:pivot},
$\H_\group$ approximates~$\L_2(\intF)$ in the sense of \eqref{pivot:triple}.
\cref{spectral:lemma} hence shows that there is an orthonormal basis of eigenfunctions for $\H_\group$,
i.e., functions ${\xi_1,\xi_2,\ldots}$ that satisfy
\begin{equation}
  \label{eq:proto:ef}
  \text{(i) }\;
  \Lambda\xi_i\;=\;\lambda_i\xi_i
  \qquad\text{(ii) }\;
  \sp{\xi_i,\xi_j}_{\H^1}\;=\;\delta_{ij}
  \qquad\text{(iii) }\;
  \smash{\overline{\text{span}\braces{\xi_1,\xi_2,\ldots}}^{\H^1}}\;=\;\H_\group\;.
\end{equation}
What remains to be shown are the properties of the eigenvalues and eigenfunctions, and
that the ONB of $\H_\group$ can be translated into an ONB of $\L_2$.
\\[.5em]
{\bf Non-negativity of eigenvalues}.
The operator $\Lambda$ is positive semi-definite, but not strictly positive definite, on $V$.
To show this, it again suffices to consider $-\Delta$ on $\mathcal{H}$.
By \eqref{H1:energy}, we have
\begin{align}
  \label{positive:definite}
  \sp{\Lambda f,f}_{\H^1}
  \;&=\;
  a(f,f)
  \,+\,
  \msum_i
  a(\partial_i f,\partial_i f)\\
  \;&=\;
  \mint\|\nabla f(x)\|^2_{\mathbb{R}^n}\vol(dx)
  \,+\,
  \msum_i
  \mint\|\nabla\partial_i f(x)\|^2_{\mathbb{R}^n}\vol(dx)
  \;\geq\;0\;.\nonumber
\end{align}
That shows $\Lambda$ is positive semi-definite. Now consider, for any ${\varepsilon>0}$, the operator
\begin{equation*}
  \Lambda_\varepsilon:
  \H_\group\rightarrow\H_\group
  \quad\text{ defined by }\quad
  \Lambda_\varepsilon f\;:=\;\Lambda f+\varepsilon f\;.
\end{equation*}
This is operator is still bounded, coercive and self-adjoint, so
\cref{spectral:lemma} is applicable. Clearly,~$\Lambda$ has the same eigenfunctions
as $\Lambda$, with eigenvalues ${\lambda_i+\varepsilon}$. It is also strictly
positive definite, since
\begin{equation*}
  \sp{\Lambda_{\varepsilon} f,f}_{\H^1}
  \;=\;
  \sp{\Lambda f,f}_{\H^1}
  +
  \sp{\varepsilon f,f}_{\H^1}
  \;\geq\;
  \varepsilon\|f\|_{\H^1}\;.
\end{equation*}
It hence follows from \cref{spectral:lemma} that the smallest eigenvalue satisfies ${\lambda_1+\varepsilon>0}$.
Since that holds for every ${\varepsilon>0}$, we have ${\lambda_1\geq 0}$.
\\[.5em]
  {\bf The smallest eigenvalue and its eigenspace}.
  If a function $f$ is constant on $\intF$, then
  \begin{equation*}
    f\in\H_\group
    \qquad\text{ and }\qquad
    \Lambda f\;=\;-\Delta f\;=\;0\;.
  \end{equation*}
  That shows the smallest eigenvalue is ${\lambda_1=0}$, and its eigenspace
  $\mathcal{H}(0)$ contains all constant functions. To show that it contains no other functions,
  note that
  \begin{equation*}
    \sp{\Lambda f,f}\;=\;0
    \qquad
    \text{ and by \eqref{positive:definite} hence }
    \qquad
    \|\nabla f(x)\|^2=0\text{ for almost all }x\in\intF\;.
  \end{equation*}
  That implies $f$ is piece-wise constant. Since the only piece-wise constant function
  contained in $\H^1$ are those that are strictly constant (see \citet{Adams:Fournier:2003}), $\mathcal{H}(0)$ is the set of
  constant functions, and ${\dim\mathcal{H}(0)=1}$.
\\[.5em]
  {\bf Regularity of eigenfunctions.}
  We now use the strategy outlined in \cref{sec:eigenfunction:smoothness}.
Let $\xi$ be an eigenfunction. We have shown that implies
${\xi\in\H_\group}$, and hence ${\xi\in\H^1(\intF)}$. Consider any ${x\in\intF}$. Since the interior
is open, we can find ${\varepsilon>0}$ such that the open ball~${B=B_{\varepsilon}(x)}$ of radius $\varepsilon$
centered at $x$ satisfies ${\overline{B}\subset\intF}$.
The restriction $\xi|_B$ of $\xi$ to $B$ then satisfies
\begin{equation*}
  f|_{B_\varepsilon(x)}\in\H^1(B)\quad\text{ and }\quad
  \Lambda f|_B\;=\;\lambda f|_B\;.
\end{equation*}
Since $\xi|_B$ appears on both sides of the equation,
\cref{fact:regularity} implies that $f|_B$ is also in~$\H^{1+2}(B)$, hence also in ${\H^{1+4}(B)}$, and so forth, so ${\xi|_B\in\H^{k}(B)}$ for all ${k\in\mathbb{N}}$.
\cref{lemma:kondrachov} then shows that ${\xi|_B}$ is even in ${\C^k(B)}$ for each ${k\in\mathbb{N}}$,
and hence in ${\C^{\infty}(B)}$. We have thus shown that $\xi$ has infinitely many derivatives on
a neighborhood of each ${x\in\intF}$, and hence that
${\xi\in\C^{\infty}(\intF)}$.
\\[.5em]
{\bf Turning the Sobolev basis into an $\L_2$ basis}.
The functions ${\xi_i}$ form an orthonormal basis of $\H_{\group}$, by \eqref{eq:proto:ef}. To obtain an orthonormal basis
for $\L_2(\intF)$, we substitute \eqref{eq:H1:L2:a} into (\ref{eq:proto:ef}ii), and obtain
\begin{align*}
  \delta_{ij}
  \;=\;
  \sp{\xi_i,\xi_j}_{\H^1}
  \;=\;
  \sp{\xi_i,\xi_j}_{\L_2}
  \;+\;
  a(\xi_i,\xi_j)
  \;=\;
  \sp{\xi_i,\xi_j}_{\L_2}
  \;+\;
  \sp{\Lambda\xi_i,\xi_j}_{\L_2}\;.
\end{align*}
Since $\xi_i$ is an eigenfunction, it follows that
\begin{align*}
  \delta_{ij}
  \;=\;
  \sp{\xi_i,\xi_j}_{\L_2}
  +\lambda_i\sp{\xi_i,\xi_j}_{\L_2}
  \quad\text{ and hence }\quad
  \mfrac{1}{1+\lambda_i}
  \sp{\xi_i,\xi_j}_{\L_2}
  \;=\;
  \delta_{ij}\;.
\end{align*}
The functions ${h_i:=\xi_i/\sqrt{1+\lambda_i}}$ then satisfy
\begin{equation*}
  -\Delta h_i\,=\,\lambda_i h_i\;\text{ on }\Pi^{\circ}
  \qquad\text{ and }\qquad
  \sp{h_i,h_j}_{\L_2}\,=\,\delta_{ij}\;.
\end{equation*}
Since we have merely scaled the functions $\xi_i$, we also have
\begin{equation*}
  \text{span}\braces{h_1,h_2,\ldots}
  \;=\;
  \text{span}\braces{\xi_1,\xi_2,\ldots}\;.
\end{equation*}
That implies
\begin{align*}
  \L_2\text{-closure of }\text{span}\braces{h_1,h_2,\ldots}
  \;&=\;
  \L_2\text{-closure of }\H^1\text{-closure of }\text{span}\braces{h_1,h_2,\ldots}\\
  \;&=\;
  \L_2\text{-closure of }\H_\group\;,
\end{align*}
and since the inclusion ${\H_\group\hookrightarrow\L_2(\intF)}$ is dense by \cref{lemma:pivot}, we have
\begin{equation*}
  \overline{\text{span}\braces{h_1,h_2,\ldots}}^{\L_2}\;=\;\L_2(\intF)\;.
\end{equation*}
In summary, we have shown that $\braces{h_1,h_2,\ldots}$ is an orthonormal basis of $\L_2(\intF)$ consisting
of eigenfunctions of ${-\Delta}$.
\\[.5em]
{\bf Extending the basis on $\Pi^{\circ}$ to a basis on $\Pi$}.
Each $h_i$ is in ${\C^{\infty}(\intF)}$,
and hence has a unique continuous extension $\bar{h}_i$ to $\intF$.
Since ${\vol_n(\partial\Pi)=0}$, we can isometrically identify
$\L_2(\Pi^\circ)$ with $\L_2(\Pi)$: Under this identification, each function $h_i$ on
the interior $\Pi^{\circ}$ is equivalent to any measurable extension of $h_i$ to $\Pi$,
so
\begin{equation*}
  \overline{\text{span}\braces{h_1,h_2,\ldots}}^{\L_2}\;=\;\L_2(\Pi)\;.
\end{equation*}
The extended functions also satisfy
\begin{equation*}
  -\Delta \bar{h}_i\;=\;\lambda_i \bar{h}_i\;\text{ on }\Pi
  \qquad\text{ and }\qquad
  \sp{\bar{h}_i,\bar{h}_j}_{\L_2(\Pi)}\;=\;\delta_{ij}\;,
\end{equation*}
where the first identity extends from $\Pi^{\circ}$ to $\Pi$ by $\C^{\infty}$-continuity, and the
second holds since the boundary does not affect the integral. The functions $\bar{h}_i$ are hence eigenfunctions
of $-\Delta$ on $\Pi$, and form and orthonormal basis of $\L_2(\Pi)$.
\end{proof}

\subsection{Proof of the Fourier representation on $\mathbb{R}^n$}

\begin{proof}[Proof of \cref{theorem:spectral}]
  To deduce the theorem from \cref{result:spectral:restricted}, we must (1) extend the basis constructed on $\Pi$ above to a basis
  on $\mathbb{R}^n$, and (2) show that every continuous invariant function can be represented in this basis.
  \\[.2em]
  \step Consider the function $\bar{h}_i$ in the proof of \cref{result:spectral:restricted}. Recall each $\bar{h}_i$
  is infinitely smooth on $\Pi$ and satisfies
  the periodic boundary condition. It follows by \eqref{eq:projector:composition} that
  \begin{equation*}
    e_i\;:=\;\bar{h}_i\circ p
  \end{equation*}
  is in $\C_\group$. Let $\Delta^k$ denote the $k$-fold application of $\Delta$.
  By \cref{lemma:differentials}, the fact that $\bar{h}_i$ satisfies the
  periodic boundary condition \eqref{pbc} implies that the continuous extension $\smash{\overline{\Delta h_i}}$
  also satisfies \eqref{pbc}. Iterating the argument shows that the same holds for the continuous
  extension of $\Delta^k h_i$ for any ${k\in\mathbb{N}}$. We hence have
  \begin{equation*}
    \Delta^k e_i\;=\;\Delta^k(\bar{h}_i\circ p)\;=\;(\overline{\Delta^k h_i})\circ p
    \qquad\text{ and }\qquad
    (\overline{\Delta^k h_i})\circ p\in\C_\group\quad\text{ for all }k\in\mathbb{N}\;,
  \end{equation*}
  so $e_i$ has infinitely many continuous derivatives on $\mathbb{R}^n$. Since it is also $\group$-invariant,
  it solves the constrained eigenvalue problem \eqref{sturm:liouville} on $\mathbb{R}^n$.
  That extends \cref{result:spectral:restricted} to $\mathbb{R}^n$.
  \\[.2em]
  \step
It remains to be shown that a function $f$ on $\mathbb{R}^n$ is in ${\C_{\group}}$ if and only if
${f=\sum c_i e_i}$ for some sequence ${(c_i)}$, where the series converges in the supremum norm.
Combining \cref{corollary:C:isometries} and \eqref{eq:hbar:isometry} shows that
\begin{equation*}
  h\mapsto\bar{h}\circ p
  \quad\text{ is an isometry }\quad
  \C_{\text{\rm pbc}}(\Pi^\circ)\rightarrow\C_\group\;.
\end{equation*}
For any ${f:\mathbb{R}^n\rightarrow\mathbb{R}}$, we hence have
\begin{equation*}
  f=\msum c_i e_i
  \qquad\Longleftrightarrow\qquad
  f|_{\Pi^{\circ}}=\msum c_i e_i|_{\Pi^{\circ}}\;.
\end{equation*}
In other words, we have to show that
\begin{equation*}
  h\in\C_{\text{\rm pbc}}(\Pi^\circ)
  \;\Leftrightarrow\;
  h=\msum c_ie_i|_{\Pi^\circ}
  \quad\text{and hence that}\quad
  \C_{\text{\rm pbc}}(\Pi^\circ)=\overline{\text{span}\braces{e_i|_{\Pi^\circ}\,|\,i\in\mathbb{N}}}^{\,\sup}\;.
\end{equation*}
Since the proof of \cref{result:spectral:restricted} shows ${\braces{e_i|_{\Pi^\circ}\,|\,i\in\mathbb{N}}}$ is a rescaled orthonormal basis of $\H_\group$, and hence a subset of ${\H_\group\cap\C_{\text{\rm pbc}}(\Pi^\circ)}$ that is dense in $\H_\group$, that holds by \cref{lemma:pivot}.
\end{proof}

\section{Proofs IV: Embeddings}
\label{appendix:proofs:embedding}

To prove \cref{theorem:embedding}, we first establish two auxiliary results on topological dimensions
of quotient spaces.
Recall from \cref{fact:bonahon} that $\mathbb{R}^n/\group$ is locally isometric to quotients of metric
balls. The first lemma considers the effect of taking a quotient on the dimension of a ball. The second lemma
combines this result with \cref{fact:bonahon} to bound the dimension of $\mathbb{R}^n/\group$.
\begin{lemma}[Quotients of metric balls]
  \label{lemma:ball:quotient:dimension}
  Let $B$ be an open metric ball in $\mathbb{R}^n$, and $G$ a finite group of isometries of $\mathbb{R}^n$.
  Then the quotient $B/G$ has topological dimension
  \begin{equation*}
    n
    \;\leq\;
    \Dim (B/G)
    \;<\;
    n+|G|\;.
  \end{equation*}
\end{lemma}
\begin{proof}
  The quotient map ${q:B\rightarrow B/G}$ is, by definition, continuous and
  surjective. Recall that preimages of points under $q$ are orbits: If ${\omega\in\mathbb{R}^n/G}$
  is the orbit $G(x)$ of some ${x\in\mathbb{R}^n}$, then ${q^{-1}\omega=G(x)}$.
  We show $q$ is also closed: Let ${A\subset B}$ be a subset.
  First observe that
  \begin{equation*}
    qA \text{ closed }
    \;\Leftrightarrow\;
    B/G\setminus qA \text{ open }
    \;\Leftrightarrow\;
    q^{-1}(B/G\setminus qA) \text{ open,}
  \end{equation*}
  by continuity of $q$.
  This set can be expressed as
  \begin{equation*}
    q^{-1}(B/G\setminus qA)
    \;=\;
    B\setminus q^{-1}qA
    \;=\;
    B\setminus\Bigl(\medcup\nolimits_{\phi\in G}\phi A\Bigr)
    \;=\;
    \medcap\nolimits_{\phi\in G}\phi(B\setminus A)\;,
  \end{equation*}
  and is therefore open whenever $A$ is closed, since each $\phi$ is an isometry and $G$ is finite.
  Consider any element ${\omega\in B/G}$. Then there is some ${x\in B}$ with ${\omega=q(x)}$, and
  \begin{equation*}
    q^{-1}\omega
    \;=\;
    \braces{\phi x|\phi\in G\text{ and }\phi x\in B}
    \quad\text{ which shows that }\quad
    |q^{-1}\omega|\;\leq\;|G|\;.
  \end{equation*}
  \cref{fact:dimensions}(ii) is now applicable, and shows
  \begin{equation*}
    \Dim B
    \;\leq\;
    \Dim (B/G)
    \;<\;
    \Dim B+|G|\;,
  \end{equation*}
  and by \cref{fact:dimensions}(i), ${\Dim B=n}$.
\end{proof}
\begin{lemma}[Topological dimension of the quotient space]
  \label{lemma:quotient:space:dimension}
  Let $\group$ be a crystallographic group that tiles $\mathbb{R}^n$ with a convex polytope $\Pi$.
  Then ${\mathbb{R}^n/\group}$ is a $\group$-orbifold, of topological dimension
  \begin{equation*}
    n
    \;\leq\;
    \Dim \mathbb{R}^n/\group
    \;<\;
    n
    \;+\;
    \max_{x\in\Pi}|\Stab(x)|
  \end{equation*}
\end{lemma}
\begin{proof}
  Choose the index set in the orbifold definition as ${\mathcal{I}=\Pi}$. By
  \cref{fact:bonahon}, we may then choose ${U_x=B_{d_{\group}}(q(x),\varepsilon)}$,
  the group $H_x$ as $\Stab(x)$, and the map
  \begin{equation*}
    \theta_x:B_{d_{\group}}(q(x),\varepsilon)\rightarrow B_{d_n}(x,\varepsilon)
  \end{equation*}
  as the isometry guaranteed by \cref{fact:bonahon}. That makes $\mathbb{R}^n/\group$
  an orbifold. Isometry of the open balls also implies for the corresponding
  closed balls of radius ${\delta=\varepsilon/2}$ that
  \begin{equation*}
    \overline{B}_{d_\group}(q(x),\delta)
    \quad\text{ is homeomorphic to }\quad
    \overline{B}_{d_\group}(x,\delta)/\text{Stab}(x)
    \quad\text{ for each }x\in\Pi\;.
  \end{equation*}
  Since homeomorphic spaces have the same topological dimension, \cref{lemma:ball:quotient:dimension}
  shows
  \begin{equation*}
    \overline{B}_{d_\group}(q(x),\delta)
    \;=\;
    \overline{B}_{d_\group}(x,\delta)/\text{Stab}(x)
    \;<\;
    n+|\text{Stab}(x)|\;.
  \end{equation*}
  Since $\group$ is crystallographic, the quotient space is compact, and
  we can hence cover it with a finite number of the closed balls above.
  Applying \cref{fact:dimensions}(i) then shows the result.
\end{proof}

\begin{proof}[Proof of \cref{theorem:embedding}]
  Let $\mathcal{S}$ be the side pairing defined by $\group$ for $\Pi$.
  Since $\group$ is by definition a discrete group of isometries, $\mathcal{S}$ is subproper
  (see \citep{Ratcliffe:2006}, 13.4, problem 2).
  The gluing construction hence constructs a set $M$ that is a $\group$-orbifold, according to
  \cref{fact:abstract:gluing}. By definition of $M$ as a quotient, the gluing construction also defines a quotient map
  \begin{equation*}
    q_M:\Pi\rightarrow M\;,
  \end{equation*}
  which is continuous and surjective.
  By \cref{fact:path:distance}, the quotient topology is metrized by~$\dpath$. By \cref{fact:orbifold:completeness},
  the metric space $(M,\dpath)$ is complete. It hence follows by \cref{fact:poincare} that there exists
  a isometry
  \begin{equation*}
    \gamma_M:(M,\dpath)\rightarrow(\mathbb{R}^n/\mathbb{S},d_\mathbb{S})\;,
  \end{equation*}
  where $\mathbb{S}$ is the group generated by $\mathcal{S}$. In our case,
  ${\mathbb{S}=\mathcal{S}_\Pi}$, and by \cref{fact:generator}, the generated group is ${\mathbb{S}=\mathbb{G}}$.
  That shows $\gamma_M$ in fact an isometry
  \begin{equation*}
    \gamma_M:(M,\dpath)\rightarrow(\mathbb{R}^n/\group,d_\group)\;.
  \end{equation*}
  Since isometric spaces have the same topological dimension, \cref{lemma:quotient:space:dimension}
  shows
  \begin{equation*}
    \Dim M\;<\;n+\max_{x\in\Pi}|\Stab(x)|\;.
  \end{equation*}
  By \cref{fact:dimensions}(ii) there is an embedding ${e:M\rightarrow\mathbb{R}^N}$ with ${N\leq 2(n+\max|\Stab(x)|)-1}$.
  Since $\group$ is crystallographic, and $\mathbb{R}^n/\group$ hence compact, $M$ and ${\Omega:=e(M)}$ are compact.
  Using the restriction ${q:\Pi\rightarrow\mathbb{R}^n/\group}$ of the quotient map to $\Pi$, we
  can define
  \begin{equation*}
    \rho_\Pi:\Pi\rightarrow\Omega
    \qquad\text{ as }\qquad
    \rho_\Pi\;:=\;e\circ\gamma^{-1}\circ q\;.
  \end{equation*}
  By the properties of the constituent maps, $\rho_\Pi$ is continuous and satisfies the periodic boundary condition \eqref{pbc}.
  That makes ${\rho:=\rho_\Pi\circ p}$ continuous and $\group$-invariant.
  If ${h:\mathbb{R}^N\rightarrow Y}$ is a continuous function, the composition ${f=h\circ\rho}$ is hence continuous and $\group$-invariant
  on $\mathbb{R}^n$. Conversely, suppose ${f:\mathbb{R}^n\rightarrow Y}$ is continuous and $\group$-invariant.
  For each ${z\in\Omega}$, the preimage ${\rho^{-1}\braces{z}}$ is precisely the orbit $\group(x)$ of some ${x\in\Pi}$.
  Since $\group$-invariant function are constant on orbits, the assignment
  \begin{equation*}
    \hat{h}(z)\;:=\;\text{ the unique value of }f\text{ on the orbit }\rho^{-1}\braces{z}
  \end{equation*}
  is a well-defined and continuous function ${\hat{h}:\Omega\rightarrow Y}$. Since $\Omega$ is compact,
  $\hat{h}$ has a (non-unique) continuous extension to a function ${h:\mathbb{R}^N\rightarrow Y}$, which
  satisfies ${f=h\circ\rho}$.
\end{proof}

\section{Proofs V: Kernels and Gaussian processes}

\label{appendix:proofs:kernels}

\subsection{Kernels}

\begin{proof}[Proof of \cref{result:RKHS:invariant:elements}]
  Suppose $\kappa$ is invariant. For any ${f\in\mathbb{H}}$, \eqref{eq:reproducing}
  implies
  \begin{equation*}
    f(\phi x)
    \;=\;
    \sp{f,\kappa(\phi x,\argdot)}_{\mathbb{H}}
    \;=\;
    \sp{f,\kappa(x,\argdot)}_{\mathbb{H}}
    \;=\;
    f(x)\;,
  \end{equation*}
  so $f$ is $\group$-invariant. Conversely, suppose all ${f\in\mathbb{H}}$ are $\group$-invariant.
  Let ${f_1,f_2,\ldots}$ be a complete orthonormal system. Then all $f_i$ are $\group$-invariant,
  so \eqref{RKHS:ONS} shows
  \begin{equation*}
    \kappa(\phi x,\psi y)
    \;=\;
    \tsum_{i\in\mathbb{N}}f_i(\phi x)f_i(\psi y)
    \;=\;
    \tsum_{i\in\mathbb{N}}f_i(x)f_i(y)
    \;=\;
    \kappa(x,y)
  \end{equation*}
  and $\kappa$ is invariant. Suppose $\kappa$ is also continuous.
  If $\kappa$ is invariant, its infimum and supremum on $\mathbb{R}^n$ equal its infimum and supremum
  on the compact set $\Pi$, and since $\kappa$ is continuous, that implies it is bounded.
  That shows all functions in $\mathbb{H}$ are continuous \citep[][4.28]{Steinwart:Christmann:2008}.
\end{proof}
The main ingredient in the proof of \cref{result:kernel:compactness} is the following lemma,
which shows that the RKHS of $\kappa$ is isometric to that of $\hat{\kappa}$,
and that an explicit isometric isomorphism between them is given by composition with the embedding map $\rho$.
\begin{lemma}
  \label{lemma:RKHS:isometry}
  Let ${\hat{\kappa}}$ be a continuous kernel on $\Omega$ with RKHS $\hat{\mathbb{H}}$. Set
  \begin{equation*}
    \kappa:=\hat{\kappa}\circ(\rho\otimes\rho)
    \qquad\text{ and }\qquad
    \mathbb{H}:=\text{ RKHS defined by }\kappa\;.
  \end{equation*}
  Then $\kappa$ is a continuous kernel on $\mathbb{R}^n$, is $\group$-invariant in both arguments,
  and ${\mathbb{H}\subset\C_{\group}}$. The map
  \begin{equation*}
    I:\hat{\mathbb{H}}\rightarrow\mathbb{H}
    \qquad\text{ defined by }\qquad
    \hat{f}\mapsto\hat{f}\circ\rho
  \end{equation*}
  is a linear isometric isomorphism, and two functions $\hat{f}$ and $\hat{g}$ in $\hat{\mathbb{H}}$ are
  orthogonal if and only if $\hat{f}\circ\rho$ and $\hat{g}\circ\rho$ are orthogonal in $\mathbb{H}$.
\end{lemma}
\begin{proof}
  \step The kernel $\kappa$ is clearly continuous, since $\hat{\kappa}$ and $\rho$ are.
  Since $\Omega$ is compact, $\hat{\kappa}$ is bounded,
  and since ${\|\kappa\|_{\sup}=\|\hat{\kappa}\|_{\sup}}$, it follows that $\kappa$ is bounded. Bounded continuity of~$\kappa$ implies
  all elements of $\mathbb{H}$ are continuous \citep[][Lemma 4.28]{Steinwart:Christmann:2008}.
  That shows ${\mathbb{H}\subset\C_{\group}}$.
  \\[.2em]
  \step
  Next, consider the map $I$. Linearity of $I$ is obvious. To show it is bijective, write
  \begin{equation*}
    S:=\text{span}\braces{\kappa(x,\argdot)\,|\,x\in\mathbb{R}^n}
    \quad\text{ and }\quad
    \hat{S}:=\text{span}\braces{\hat{\kappa}(\omega,\argdot)\,|\,\omega\in\Omega}\;.
  \end{equation*}
  Note that makes $\mathbb{H}$ the norm closure of $S$, and $\hat{\mathbb{H}}$ the norm closure of $\hat{S}$
  (see \cref{sec:RKHS}).
  \\[.2em]
  \step
  Consider any ${\hat{f}\in\hat{S}}$. Then ${\hat{f}=\sum a_i\hat{\kappa}(\omega_i,\argdot)}$
  for some scalars $a_i$ and points $\omega_i$ in $\Omega$. Since $\rho$ is surjective by
  \cref{theorem:embedding}, we can find points $x_i$ in
  $\mathbb{R}^n$ such that ${\omega_i=\rho(x_i)}$. It follows that
  \begin{equation*}
    f\;=\;
    \hat{f}\circ\rho
    \;=\;
    \bigl(\tsum a_i\hat{\kappa}(\rho(x_i),\argdot)\bigr)\circ\rho
    \;=\;
    \tsum a_i\kappa(x_i,\argdot)\;\in\;S
    \quad\text{ and hence }\quad
    I(\hat{S})\subset S\;.
  \end{equation*}
  Reversing the argument shows ${I^{-1}(S)\subset\hat{S}}$.
  Thus, $I$ is a linear bijection of $\hat{S}$ and $S$.
  \\[.2em]
  \step Substituting ${\hat{f}\in\hat{S}}$ as above into the definition
  of the scalar product shows
  \begin{equation*}
    \sp{\hat{f},\hat{f}}_{\hat{\mathbb{H}}}
    \;=\;
    \tsum a_ia_j\hat{\kappa}(\rho(x_i),\rho(x_j))
    \;=\;
    \tsum a_ia_j\kappa(x_i,x_j)
    \;=\;
    \sp{f,f}_{\mathbb{H}}
  \end{equation*}
  and hence ${\|f\|_{\mathbb{H}}=\|\hat{f}\|_{\hat{\mathbb{H}}}}$ for all ${\hat{f}\in\hat{S}}$.
  In summary, we have shown that the restriction of $I$ to $\hat{S}$ is a bijective linear isometry
  ${\hat{S}\rightarrow S}$.
  \\[.2em]
  \step Since $I$ is an isometry on a dense subset, it has a unique uniformly continuous extension
  to the norm closure $\hat{\mathbb{H}}$, which takes the norm closure $\hat{\mathbb{H}}$ to
  the norm closure $\mathbb{H}$ of the image and is again an isometry
  \citep[][3.11]{Aliprantis:Border:2006}.
\end{proof}

\begin{proof}[Proof of \cref{result:kernel:compactness}]
  \step
  By \cref{theorem:embedding}, there is a
  unique continuous function
  \begin{equation*}
    \hat{\kappa}:\Omega\times\Omega\rightarrow\mathbb{R}
    \qquad\text{ that satisfies }\qquad
    \kappa=\hat{\kappa}\circ(\rho\otimes\rho)\;.
  \end{equation*}
  \cref{lemma:RKHS:isometry} then implies all ${f\in\mathbb{H}}$ are $\group$-invariant
  and continuous.
  \\[.2em]
  \step
  We next show the inclusion is compact. Consider first the map ${I:\hat{f}\mapsto\hat{f}\circ\rho}$ as
  in \cref{lemma:RKHS:isometry}, but now defined on the larger space ${\C(\Omega)}$. We know from
  \cref{theorem:embedding} that $I$
  is an isometric isomorphism ${\C(\Omega)\rightarrow\C_\group}$ (with respect to the
  supremum norm). By \cref{lemma:RKHS:isometry} its restriction to a map
  ${\hat{\mathbb{H}}\rightarrow\mathbb{H}}$ is also an isometric isomorphism (with respect to
  the RKHS norms).
  It follows that the inclusion maps
  \begin{equation*}
    \iota:\mathbb{H}\rightarrow\C_{\group}
    \quad\text{ and }\quad
    \hat{\iota}:\hat{\mathbb{H}}\rightarrow\C(\Omega)
    \quad\text{ satisfy }\quad
    \iota=I\circ\hat{\iota}\circ I^{-1}\;.
  \end{equation*}
  Since $\hat{\kappa}$ is a continuous kernel by step 1,
  and its domain $\Omega$ is compact by \cref{theorem:embedding},
  the inclusion $\hat{\iota}$ is compact \citep[][4.31]{Steinwart:Christmann:2008}.
    The composition of a compact linear operator with any continuous linear
    operator is again compact \citep[][Theorem 5.1]{Aliprantis:Burkinshaw:2006}.
    Since $I$ and its inverse are linear and continuous, that indeed makes $\iota$ compact.
    \\[.2em]
    \step
    Since $\hat{\kappa}$ is a continuous kernel on a compact domain, Mercer's theorem
    \citep[][4.49]{Steinwart:Christmann:2008} holds for $\hat{\kappa}$. It shows there
    are functions ${\hat{f}_1,\hat{f}_2,\ldots}$ and scalars ${c_1\geq c_2\geq\ldots>0}$ such
    that
    \begin{equation*}
      (\sqrt{c_i}\hat{f}_i)_{i\in\mathbb{N}}\text{ is an ONB for }\hat{\mathbb{H}}
      \quad\text{ and }\quad
      \hat{\kappa}(\omega,\omega')=\msum_ic_i\hat{f}_i(\omega)\hat{f}_i(\omega')
      \quad\text{ for all }\omega,\omega'\in\Omega\;.
    \end{equation*}
    The functions ${f_i:=\hat{f}_i\circ\rho}$ then satisfy
    \begin{equation*}
      \kappa(x,y)\;=\;\hat{\kappa}(\rho(x),\rho(y))\;=\;\tsum_ic_i\hat{f}_i(\rho(x))\hat{f}_i(\rho(y))
      \;=\;\tsum_ic_if_i(x)f_i(y)\;.
    \end{equation*}
    Since the map ${\hat{f}\mapsto\hat{f}\circ\rho}$ preserves the scalar product by
    \cref{lemma:RKHS:isometry}, the sequence ${(\sqrt{c_i}f_i)}$ is an ONB for $\mathbb{H}$.
    \\
    \step
    It remains to verify the representation
    \begin{align*}
      \mathbb{H}\;&=\;
      \braces{\;f\!=\!\tsum_{i\in\mathbb{N}}a_i\sqrt{c_i}f_i\,|\,
        a_1,a_2,\ldots\in\mathbb{R}\text{ with }\tsum_{i}|a_i|^2<\infty}\;.\\
      \intertext{Since Mercer's theorem applies to $\hat{\kappa}$, the analogous representation}
      \hat{\mathbb{H}}\;&=\;
      \braces{\;\hat{f}\!=\!\tsum_{i\in\mathbb{N}}a_i\sqrt{c_i}\hat{f}_i\,|\,
        a_1,a_2,\ldots\in\mathbb{R}\text{ with }\tsum_{i}|a_i|^2<\infty}
    \end{align*}
    holds on $\Omega$, by \citet[][4.51]{Steinwart:Christmann:2008}.
    As ${\hat{f}\mapsto\hat{f}\circ\rho}$ is an isometric isomorphism by \cref{lemma:RKHS:isometry},
    that yields the representation for $\mathbb{H}$ above.
\end{proof}

\subsection{Gaussian processes}

\begin{proof}[Proof of \cref{result:invariant:GP}]
  That $F$ is continuous and $\group$-invariant almost surely follows immediately from
  \cref{theorem:embedding}. Let $\tilde{\Pi}$ be a transversal. Our task is to show that the restriction
  $F|_{\tilde{\Pi}}$ is a continuous Gaussian process on $\tilde{\Pi}$. To this end,
  suppose $h$ is a continuous function on $\mathbb{R}^N$.
  Then ${h\circ\rho}$ is continuous by \cref{theorem:embedding},
  and the restriction is again continuous. That means
  \begin{equation*}
    \tau:h\mapsto (h\circ\rho)|_{\tilde{\Pi}}
    \qquad\text{ is a map }\qquad
    \C(\Omega)\rightarrow\C(\tilde{\Pi})\;.
  \end{equation*}
  Since both composition with a fixed function and restriction to a subset
  are linear as operations on functions, $\tau$ is linear, and since neither composition nor restriction can increase
  the sup norm, it is bounded. The restriction
  \begin{equation*}
    F|_{\tilde{\Pi}}\;=\;\tau(H)
  \end{equation*}
  is hence the image of a Gaussian process with values in the separable Banach space
  ${\C(\Omega)}$ under a bounded linear map into the Banach space ${\C(\tilde{\Pi})}$. That implies
  it is a Gaussian process with values in ${\C(\tilde{\Pi})}$, and that
  $\kappa$ and $\mu$ transform accordingly
  \citep[][Lemma 7.1]{vanDerVaart:vanZanten:2008}.
\end{proof}

\end{document}